\documentclass{article}

\usepackage[main, final]{neurips_2026}
\usepackage[utf8]{inputenc} %
\usepackage[T1]{fontenc}    %
\usepackage{hyperref}       %
\hypersetup{
   colorlinks,
   linkcolor={red},
   citecolor={blue},
   urlcolor={blue}
}
\usepackage{url}            %
\usepackage{booktabs}       %
\usepackage{longtable}      %
\usepackage{amsfonts}       %
\usepackage{nicefrac}       %
\usepackage{microtype}      %
\usepackage{xcolor}         %

\usepackage{amsmath,amssymb, amsthm}
\usepackage{enumitem}
\usepackage{MnSymbol} 
\usepackage{empheq}
\usepackage{rotating}
\usepackage{titletoc}
\usepackage{framed}
\usepackage{multirow,bigdelim}
\usepackage{bbm}
\usepackage{algorithm}
\usepackage{algorithmic}
\usepackage{pgf}
\usepackage{pgfplots}
\usepackage{tikz}
\pgfplotsset{compat=1.17}
\usepackage{shadethm}
\usepackage{dsfont}
\usepackage{thmtools}
\usepackage{mdframed}
\usepackage{lipsum}
\usepackage{fixmath}
\usepackage{multicol}
\usepackage{footmisc}
\usepackage{macros}

\usepackage{caption}
\usepackage{subcaption}

\usepackage{booktabs}
\usepackage{cleveref}

\let\originalparagraph\paragraph
\ExplSyntaxOn
\cs_new:Npn \neurips_paragraph_title:n #1
  {
    \tl_set:Nn \l_tmpa_tl {#1}
    \tl_trim_spaces:N \l_tmpa_tl
    \tl_if_blank:VTF \l_tmpa_tl
      { \tl_use:N \l_tmpa_tl }
      {
        \str_if_eq:eeTF { \tl_to_str:n { \tl_item:Nn \l_tmpa_tl { -1 } } } { . }
          { \tl_use:N \l_tmpa_tl }
          { \tl_use:N \l_tmpa_tl . }
      }
  }
\RenewDocumentCommand{\paragraph}{s o m}
  {
    \IfBooleanTF{#1}
      {
        \IfNoValueTF{#2}
          { \originalparagraph*{\neurips_paragraph_title:n {#3}} }
          { \originalparagraph*[#2]{\neurips_paragraph_title:n {#3}} }
      }
      {
        \IfNoValueTF{#2}
          { \originalparagraph{\neurips_paragraph_title:n {#3}} }
          { \originalparagraph[#2]{\neurips_paragraph_title:n {#3}} }
      }
  }
\ExplSyntaxOff

\theoremstyle{plain}
\newtheorem{theorem}{Theorem}[section]

\newtheorem{proposition}[theorem]{Proposition}
\newtheorem{lemma}[theorem]{Lemma}

\theoremstyle{definition}

\theoremstyle{remark}
\newtheorem{remark}[theorem]{Remark}

\newcommand{\ICNN}{\operatorname{ICNN}}
\newcommand{\HyCNN}{\operatorname{HyCNN}}
\newcommand{\ReLU}{\operatorname{ReLU}}

\newcommand{\reb}{\color{black}}

\DeclareMathOperator*{\argmax}{argmax}
\DeclareMathOperator*{\argmin}{argmin}

\newcommand{\parabold}[1]{\textbf{#1.\;}}

\title{Hyper Input Convex Neural Networks for \\Shape Constrained Learning and Optimal Transport}

\author{%
   Shayan Hundrieser$^{1}$\thanks{Equal contribution} \And
  Insung Kong$^{2}$\footnotemark[1] \And Johannes Schmidt-Hieber$^{3}$ \\[0.5em]
  $^{1}$Department of Statistics and Data Science, Yale University, New Haven, CT, USA\\
  $^{2}$Department of Mathematical Sciences, UNIST, Ulsan, South Korea\\ 
  $^{3}$Humboldt-University of Berlin and  Weierstrass Institute, Berlin, Germany\\
  \texttt{shayan.hundrieser@yale.edu, kong@unist.ac.kr, schmidt-hieber@wias-berlin.de}
}

\begin{document}

\maketitle

\begin{abstract}
   We introduce Hyper Input Convex Neural Networks (HyCNNs), a novel neural network architecture designed for learning convex functions. HyCNNs combine the principles of Maxout networks with input convex neural networks (ICNNs) to create a neural network   that is always convex in the input, theoretically capable of leveraging depth, and performs reliably when trained at scale compared to ICNNs. Concretely, we prove that HyCNNs require exponentially fewer parameters than ICNNs to approximate quadratic functions up to a given precision. Throughout a series of synthetic experiments, we demonstrate that HyCNNs outperform existing ICNNs and MLPs in terms of predictive performance for convex regression and interpolation tasks. We further apply HyCNNs to learn high-dimensional optimal transport maps for synthetic examples and for single-cell data, where they oftentimes outperform ICNN-based neural optimal transport methods and other baselines across a wide range of settings.
\end{abstract}

\section{Introduction}\label{sec:introduction}

This paper introduces a novel neural network architecture, called Hyper Input Convex Neural Network (HyCNN)%
, which $(i)$ is always convex in its inputs, $(ii)$ is explicitly designed to harness depth for approximation, and $(iii)$ outperforms alternative procedures in terms of predictive power when trained at depth for large data sets, see \Cref{fig:HyCNN}. 

\parabold{Shape constrained learning}
Among shape constraints, \emph{convexity in the input} is favorable for multiple reasons.
First, convexity goes beyond being a convenient inductive bias as it is intrinsic to several core problems.
In optimal transport (OT) with squared Euclidean costs, Brenier's theorem asserts that the OT map between absolutely continuous probability measures with finite second moment is given by the gradient of a convex potential \citep{brenier1991polar,villani2008optimal,peyre2019computational}. Hence, learning an OT map is naturally linked to learning a convex function \citep{taghvaei20192,makkuva2020optimal,bunne2023learning}. In invertible generative modeling, gradients of (strongly) convex potentials yield monotone diffeomorphisms that serve to define normalizing flows \citep{rezende2015variational,huang2020convex}. 
Convex models also arise in optimal control, where convexity is used to ensure planning computational feasibility while using expressive learned dynamics or cost surrogates \citep{chen2018optimal}.

Second, for a real-valued function $f(\cdot)$ that is convex and coercive, operations such as $$\arg\min_{\bx\in \RR^d} f(\bx)$$ become tractable via convex optimization, and gradients/subgradients are globally well-behaved. As such, network architectures that are convex in the input are natural candidates for learning energy functions, which are ubiquitous in machine learning and statistics, e.g., to devise better optimization heuristics and parameterizations for data-driven mirror descent \citep{tan2023data}, to design convolutional networks which are more robust to noise \citep{sivaprasad2021curious}, and to enforce related shape-specific properties on level sets  \citep{nesterov2022learning}. 

\begin{figure}[t!]
  \begin{center}
    \centerline{\includegraphics[width=\textwidth]{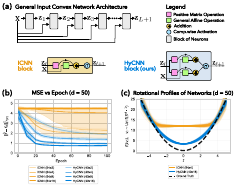}}
    \caption{{Input convex neural networks (ICNNs) versus Hyper Input Convex Neural Networks (HyCNNs).} \textbf{(a)} Depiction of generic input convex network architecture. ICNN blocks comprise one lane, whereas Hyper ICNN (HyCNN) blocks involve two lanes of matrix operations. \textbf{(b)} Averaged Empirical MSE over 10 repetitions with $(10\%-90\%)$ confidence bands for learning the function $f_0(x) = \|\bx\|_2^2$ on $\RR^d$ with $d = 50$ based on $n = 5{,}000$ samples of the form $(\bX_i, Y_i)_{i\in [n]}$ with $\bX_i \sim \mathrm{Unif}[-1,1]^d$ and $Y_i = f_0(\bX_i) + \epsilon_i$ with $\epsilon_i \sim \mathcal{N}(0,1)$. Two network architectures are used, ICNNs \citep{amos2017input} and HyCNNs (ours)  where the label $(m \times L)$ indicates width $m$ and depth $L$ of the network.  \textbf{(c)} Rotational profiles $t \mapsto f(t \bv_i)$ of ground truth (dashed) and learned networks with best MSE performance within network class, for $100$ randomly drawn directions $\bv_i \sim \mathrm{Unif}(\SS^{d-1})$. Choosing width $64$ for ICNN and $40$ for HyCNN yields near equal parameter-counts per depth.}
    \label{fig:HyCNN}
    \vspace{-0.8cm}
 \end{center}
\end{figure}

Third, convexity lies at the heart of classical shape-restricted statistics problems such as convex regression \citep{seijo2011convex,hannah2013multivariate,guntuboyina2015global,balazs2015near} and log-concave density estimation \citep{samworth2018recent}. Convex energy parameterizations have recently also been used to scale such problems with deep learning infrastructure \citep{lin2023falcon} and interact naturally with score-based learning objectives \citep{hyvarinen2005estimation,song2021score}, see \cite{feng2024optimal} for a recent example of shape-constrained score-based regression. 

Despite these motivations, the performance of existing convex neural architectures in small to moderate dimensions is often only comparable to that of unconstrained (multi-layer perceptron, MLP) networks, even when convexity is known to hold, see \Cref{app:interpolation} of the supplementary material, in particular for deeper architectures. 
A potential reason is that most input-convex constructions are not easy to train at depth, and thus fail to leverage the approximation benefits of deeper architectures. Moreover, as we show in \Cref{sec:approximation}, ICNNs are provably inefficient at approximating quadratic functions, which are fundamental building blocks for convex functions. Our architecture aims to resolve both issues. 

\parabold{Neural networks which are convex in the input}\label{sec:related_work_icnns} 
\emph{Input convex neural networks} (ICNNs)  \citep{amos2017input} are the 
dominant architecture to guarantee convexity with respect to the inputs. A principled initialization strategy improves the practical performance of ICNNs \citep{hoedt2023principled}. 
Modifications to ICNNs have been proposed to improve their expressiveness, e.g., by using bilinear positive definite skip connections from the input \citep{bunne2023learning}.

Based on the representation of convex piecewise-affine functions as maxima of affine forms, alternative input convex network constructions focus on maxout units \citep{goodfellow2013maxout}. This includes
groupwise maximum networks \citep{warin2023groupmax} for convex approximation and input convex Kolmogorov--Arnold networks (KANs) \citep{warin2024p1,deschatre2025input,thakolkaran2025can}. Despite the conceptual appeal of these architectures, they have not been able to consistently outperform ICNNs in practice, and their quantitative approximation properties are not well understood.  

\parabold{Contributions} 
We introduce \emph{Hyper Input Convex Neural Networks (HyCNNs)}, a novel shape-constrained architecture designed to approximate convex functions while leveraging depth and remaining trainable at scale.
At a high level, a HyCNN combines two ideas:
$(i)$ parameter constraints as in ICNNs to ensure convexity, and $(ii)$ using maxout units to allow each layer to refine the piecewise-affine structure of the approximation. %
With appropriate non-negativity constraints on hidden-to-hidden weights, HyCNNs are convex by construction (Proposition~\ref{prop:HyCNN_convexity}), include standard ReLU ICNNs as a special case, and inherit known universal approximation guarantees for convex functions \citep{amos2017input,chen2018optimal}. We also provide a principled initialization strategy for HyCNNs that ensures stable signal propagation across layers at initialization, which is crucial for training deep networks. 

Our core technical insight is that the modified convex unit fundamentally changes the \emph{depth--expressiveness trade-off}.
Specifically, compared to a classical ICNN layer, each HyCNN layer can \emph{double} the number of affine pieces in the induced piecewise-affine approximation, leading to an exponential-in-depth refinement mechanism. We formalize this by proving that constant-width HyCNNs approximate the quadratic function on $[0,1]$ with uniform error decaying as $O(2^{-2L})$ in the number of layers $L$ (Theorem~\ref{thm:quadratic_approximation}).
In contrast, ReLU-based ICNNs in one dimension remain convex and piecewise-affine, but their number of linear pieces grows only linearly with depth (\Cref{prop:piecewise_linear_icnn}), which yields a polynomial lower bound on the best achievable uniform error for $x\mapsto x^2$ (Theorem~\ref{thm:lowerBoundPiecewiseLinear}).
This gap provides a concrete explanation for why ``making ICNNs deeper'' often fails to deliver commensurate gains in practice.

We then apply HyCNNs to learn high-dimensional OT maps, where we see that they often outperform
multiple variants of ICNNs and  the Monge gap MLP estimator in terms of predictive performance. In addition, we demonstrate the versatility of HyCNNs by applying them to various synthetic regression problems, demonstrating superior performance over existing convex neural architectures. 

\parabold{Paper organization} 
Section~\ref{sec:HyCNNs} introduces the HyCNN architecture, establishes convexity, and discusses initialization and training rules. 
Section~\ref{sec:approximation} compares the approximation properties of HyCNNs and ICNNs. 
Section \ref{sec:applications} presents applications of HyCNNs to optimal transport map estimation. Section \ref{sec:experiments} provides empirical results. The supplementary material  contains additional results to our approximation theory, omitted proofs, and additional experiments and data analysis.

\section{HyCNNs: Hyper Input Convex Neural Networks}\label{sec:HyCNNs}

\parabold{Architecture}
Our architecture of the HyCNN takes inspiration from maxout networks \citep{goodfellow2013maxout} and input convex neural networks \citep{amos2017input}. To introduce HyCNNs, we first recall the construction of ICNNs. Formally, an ICNN 
takes a $d$-dimensional input $\bx \in \RR^d$ and outputs a real number; it exhibits $L\in \NN$ hidden layers and the number of hidden neurons per layer is characterized by $(d_1, \dots, d_L)\in \NN^L$. The output of each layer and the function are given by the following equations for $\ell = 0, \dots, L-1$ (see \Cref{fig:HyCNN}(a)):
\begin{align}\label{eq:ICNN}
    \   \operatorname{ICNN}(\bx) := \bz_{L+1} \coloneqq V_{L}\bz_L + W_L \bx + \bb_L, \qquad \bz_{\ell+1} &:= \sigma(V_\ell \bz_{\ell} + W_{\ell} \bx + \bb_{\ell}),
\end{align}
where $\bz_0 \coloneqq \mathbf{0}_d$ and $\bz_\ell$ is given by the layerwise activation function $\sigma\colon \RR \to \RR$ applied componentwise, with parameters 
consisting of weight matrices and biases. 
To ensure that $\operatorname{ICNN}(\bx)$ is convex in $\bx$,
the activation function $\sigma$ is convex and non-decreasing (e.g., ReLU),
the weight-matrices $\{V_\ell\}_{\ell=0}^L$ are constrained to be non-negative, while $\{W_\ell\}_{\ell=0}^L$ and $\{\bb_\ell\}_{\ell=0}^L$ are unconstrained \citep[Proposition 1]{amos2017input}.

For the HyCNN architecture, we modify the layerwise construction of ICNNs by introducing an additional lane of matrix operations. Concretely, a HyCNN 
with $L$ layers and $(d_1, \dots, d_L)$ hidden neurons per layer is given by the following equations for $\ell = 0, \dots, L-1$:
\begin{align}\label{eq:HyCNN}
\operatorname{HyCNN}(\bx) := \bz_{L+1} = V_{L}\bz_L + W_L \bx + \bb_L, \qquad     \bz_{\ell+1} &:= \overline\sigma\left((V_\ell^{k} \bz_{\ell} + W_{\ell}^{k} \bx + \bb_{\ell}^{k})_{k\in \{1,2\}}\right), 
\end{align}
where $\bz_0\coloneqq \mathbf{0}_d$ and $\overline \sigma$ represents the \emph{(gating) activation function} $\overline \sigma\colon \RR^2 \to \RR$ that is convex and componentwise non-decreasing, and the notation $\overline\sigma\left((a_k)_{k\in \{1,2\}}\right)$ denotes the vector with entries $\overline\sigma(a_1, a_2)$, applied componentwise.
The parameters consist of weight matrices and biases 
with appropriate dimensions to ensure compatibility. 

Based on this construction, we have the following convexity guarantee for HyCNNs.

\begin{proposition}\label{prop:HyCNN_convexity}
	Let $\overline \sigma \colon \RR^2 \to \RR$ be convex and componentwise non-decreasing, and let the matrices $\{V_{\ell}^1, V_{\ell}^2\}_{\ell =0}^{L-1}$ and $V_{L}$ be componentwise non-negative. Then, $\operatorname{HyCNN} \colon \RR^d\to \RR$ in \eqref{eq:HyCNN} is convex. 
\end{proposition}

\begin{proof}
Note that $\bz_0=\mathbf{0}_d$ is a convex function of $\bx$. For the induction step, assume that each component of $\bz_\ell$ is a convex function. 
Then, since every component of $V_\ell^{k}$ is non-negative, for each $k \in \{1,2\}$
each component of $V_\ell^{k} \bz_{\ell} + W_{\ell}^{k} \bx + \bb_{\ell}^{k}$ is convex.  
Since $\overline \sigma$ is convex and componentwise non-decreasing, $\bz_{\ell+1}$ is convex as well (see Lemma~\ref{lemma_comp_convexnondecre}).
Finally, the output layer is given by a non-negative affine combination of convex functions plus an affine function, which is again a convex function. Thus, by induction, $\operatorname{HyCNN}$ is convex.
\end{proof}

{\reb By choosing $\overline \sigma(a_1, a_2) = \sigma(a_1)$, where $\sigma\colon \RR \to \RR$ is a convex and non-decreasing activation function such as the ReLU, we recover the standard ICNN architecture. Hence, HyCNNs can be interpreted as a generalization of ICNNs, and enjoy similar universal approximation guarantees for convex functions as classical ICNNs  \citep{chen2018optimal}.
}

{\reb In this paper, we consider activation functions of the form
\begin{align}\label{eq:logsumexp}
    \overline \sigma_\tau(a_1, a_2) = \begin{cases}
    \tau \log\left(e^{a_1/\tau} + e^{a_2/\tau}\right), \quad &\text{if} \ \tau>0, \\
    \max(a_1, a_2), &\text{if} \ \tau=0.
    \end{cases}.
\end{align}
The componentwise maximum $\overline \sigma_0$ is also known as a \emph{maxout unit} \citep{goodfellow2013maxout} and yields piecewise-affine HyCNNs. For $\tau>0,$ the log-sum-exp function $\overline \sigma_\tau$ is a smooth approximation of the componentwise maximum. Lemma \ref{lem.necessary_non_neg} and the surrounding discussion show that for activation function $\overline \sigma_\tau$ with $\tau\geq 0$, the entrywise non-negativity of the hidden-to-hidden weight matrices $\{V_{\ell}^1, V_{\ell}^2\}_{\ell =0}^{L-1}$ is necessary to ensure that all HyCNNs are input convex functions. }

\medskip

\begin{remark}[Comparison with Maxout and GroupMax networks]
	HyCNNs exhibit some similarities and differences with Maxout networks \citep{goodfellow2013maxout} and GroupMax networks \citep{warin2023groupmax}. Unlike HyCNNs, in Maxout networks the weights for hidden-to-hidden neuron connections are not constrained to be non-negative. Imposing this constraint results in the GroupMax network architecture. A key distinction of HyCNNs to both Maxout and GroupMax networks is the presence of (distinct) skip connections entering into both inputs of the maxout unit. This is a central technical device in our approximation theory of the quadratic function, and seems to stabilize the gradient norm of parameters across different layers. Our empirical results in \Cref{app:regression_experiments} also suggest that this architectural change is crucial for training deeper networks.
\end{remark}

\parabold{Initialization of HyCNNs}\label{sec:trainingHyCNNs} 
To train HyCNNs, in particular, deep architectures, the gradient norms of parameters across different layers need to be stable at initialization to mitigate the vanishing/exploding gradient problem and to ensure that all layers are effectively trained. In the following, we introduce a suitable initialization scheme, see \Cref{fig:HyCNN_profiles_init} for numerical evidence, which ensures that the signal propagation and gradient profiles are well-behaved across layers at initialization, and lead to stable training dynamics of most weights across the network, with a tendency for gradients of deeper layers to be slightly smaller than those of shallower layers. {\reb{After 100 epochs of training, the network parameters still exhibit stable gradient norms, while the signal norm across layers appears to grow.
}}%

\begin{figure}[t!]
    \begin{center}
    \centerline{\includegraphics[width=\textwidth, trim=0 0 0 46, clip]{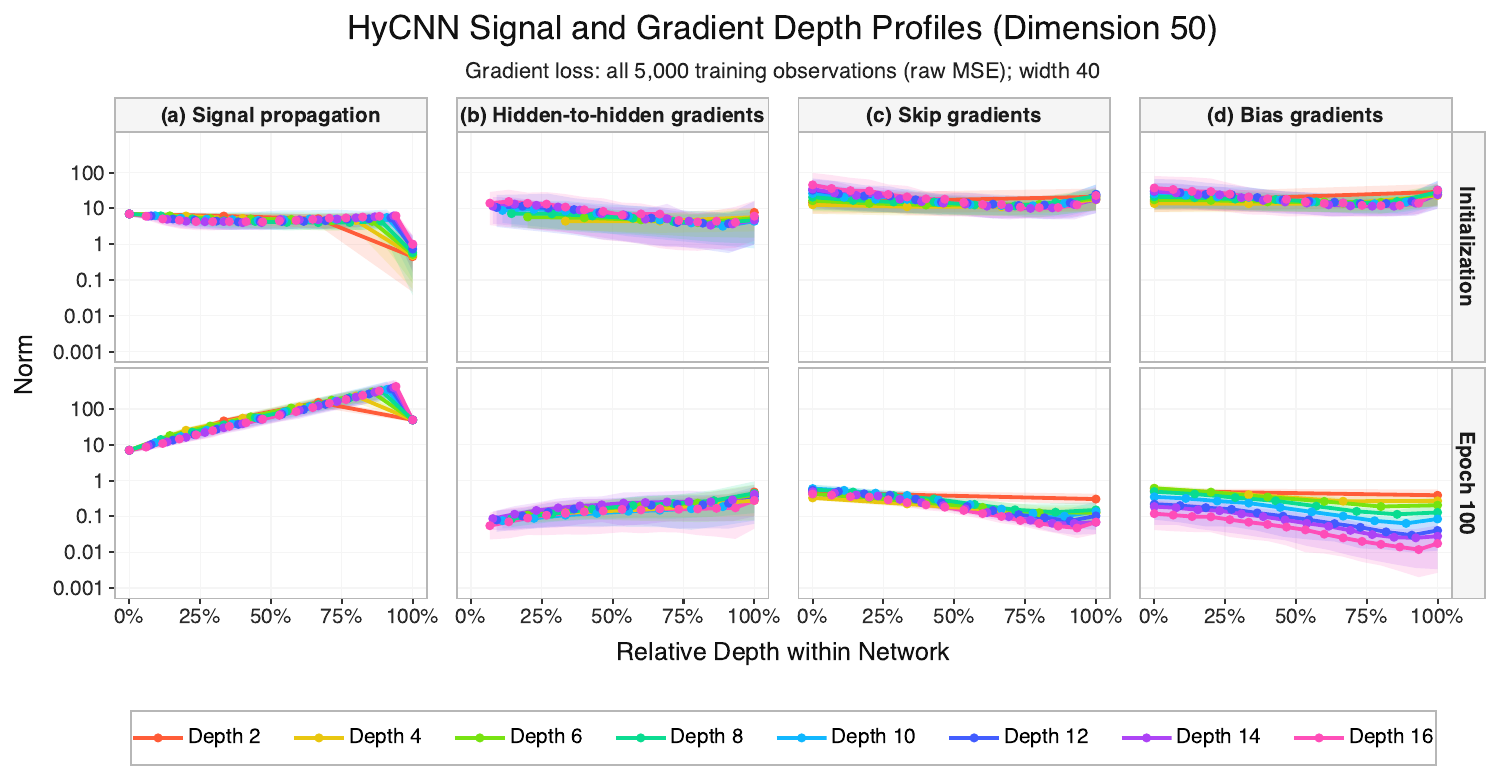}}
        \caption{Signal propagation and gradient profiles across relative depth in HyCNN {\reb at initialization (top row) and after 100 epochs of training (bottom row) }for width $W=40$, depths $L \in \{2,4,\ldots,16\}$ (colors), and input dimension $d = 50$.  {\reb The relative depth within the network refers to the respective layer index normalized by network depth, with $0\%$ referring to the input, and $100\%$ to the scalar output of the network. }
        \textbf{(a)} Signals are computed from a single forward pass at initialization with standard Gaussian input $\bX \sim \calN(0,I_d)$. Each curve is the mean hidden-state $\ell_2$-norm of the neurons over 100 random seeds. The shaded bands show the $10$th$-$$90$th percentile range. $0\%$ corresponds to the input norm $\|\bX\|_2$ (which is close to $\sqrt{d} \approx 7.07$).  
    \textbf{(b-d)} Gradients are computed from a single backward pass at initialization for least-squares regression on synthetic data of size $n = 5{,}000$, with $\bX_i \sim \calN(0, I_d)$ and $Y_i = \|\bX_i\|_2^2 + \epsilon_i$, where $\epsilon_i \sim \calN(0,1)$. Batch size for training equals $1{,}000$. Each  curve is the mean of the gradient norm for (b) hidden-to-hidden weight,\protect\footnotemark \, (c) skip weights, and (d) bias weights. The shaded bands show the $10$th$-$$90$th percentile range.}
         \label{fig:HyCNN_profiles_init}
         \vspace{-0.5cm}
         \end{center}
\end{figure}
\footnotetext{We reparameterize the hidden-to-hidden weights as $V_\ell = \mathrm{softplus}(R_\ell)$ for unconstrained matrices $R_\ell$ and softplus applied element-wise. The gradient in the figure depicts the gradient norm of $R_\ell$ across different layers. }

Due to the non-negativity constraints of the hidden-to-hidden weights, standard initialization strategies based on centered weight distributions, which are commonly adopted in unconstrained neural networks \citep{lecun2002efficient, he2015delving}, are not directly applicable to HyCNNs. Inspired by the approach of \citet{hoedt2023principled} for ICNNs, we develop in the following an initialization strategy for HyCNNs. 
The key objective of our approach is to ensure that pre-activations exhibit comparable statistical properties across all layers at initialization.
To this end, we examine how the statistics of the pre-activations evolve from one layer to the next and derive fixed-point equations that keep their mean, variance, and covariance stable.

 More specifically, for $\ell \in [L-1]$ let
\begin{align*}
    \bs_{\ell+1} = (s_{\ell+1,i})_{i \in [d_{\ell+1}]} := V_\ell^{1} \bz_{\ell} + W_{\ell}^{1} \bx + \bb_{\ell}^{1}, \qquad \bt_{\ell+1} = (t_{\ell+1,i})_{i \in [d_{\ell+1}]} := V_\ell^{2} \bz_{\ell} + W_{\ell}^{2} \bx + \bb_{\ell}^{2}.
\end{align*}
Notice that $\bz_{\ell} = \max(\bs_{\ell}, \bt_{\ell})$.
At initialization, the entries of $V_\ell^1$ and $V_\ell^2$ are drawn i.i.d.\ from a distribution supported on $\RR_+$.
The entries of $W_\ell^1$ and $W_\ell^2$ are drawn i.i.d.\ from a centered Gaussian distribution, while the entries of $\bb_\ell^1$ and $\bb_\ell^2$ are set to a fixed value.
Since the variables 
$s_{\ell,1}, t_{\ell,1}, \ldots, s_{\ell,d_\ell}, t_{\ell,d_\ell}$
are exchangeable as a collection of $2d_\ell$ variables,\footnote{A collection of random variables $x_1,\ldots,x_n$ is exchangeable if its joint distribution is invariant under permutations.} %
it is sufficient to impose the following
fixed-point equations
\begin{align*}
    \EE(s_{\ell+1,1}) = \EE(s_{\ell,1}),\qquad 
    \EE(s^2_{\ell+1,1}) = \EE(s_{\ell,1}^2),\qquad 
    \EE(s_{\ell+1,1} t_{\ell+1,1}) = \EE(s_{\ell,1} t_{\ell,1})
\end{align*}
to ensure that the mean, variance, and covariance of all neuron pre-activations are stable across layers at initialization and thus induce a well-conditioned signal propagation. 
To obtain explicit solutions to these equations, we assume that $\|\bx\|^2_2 \approx d$ and use a Gaussian approximation of the joint distribution of $s_{\ell,1},\ldots,s_{\ell,d_{\ell}},t_{\ell,1},\ldots,t_{\ell,d_{\ell}}.$ 

\textit{Initialization scheme for HyCNNs:}
\begin{itemize}[leftmargin=1.5em, labelsep=0.5em]
    \item The entries of $V_\ell^1$ and $V_\ell^2$ for $\ell \in [L-1]$ are drawn i.i.d. from a distribution supported on $\RR_+$ with mean and variance, respectively, given by
    $$\mu := \left(d_{\ell}^2 + (1-\tfrac{1}{\pi})d_{\ell}\right)^{-1/2}, \quad \sigma^2 := \frac{1}{4d_{\ell}}.$$ 
    We take the log-normal distribution as it is easy to sample from. %
    The above constraints are then realized by drawing
    $$X\coloneqq \exp(\tilde X), \quad  \tilde{X} \sim \mathcal{N}\left( \ln\Big(\frac{\mu^2}{\sqrt{\tau}}\Big), \ln\Big(\frac{\tau}{\mu^2}\Big)\right) \quad \text{ with }\, \tau\coloneqq \mu^2 + \sigma^2,$$
    where $\calN(\mu, \sigma^2)$ denotes the normal distribution with mean $\mu$ and variance $\sigma^2.$
    \item The entries of $W_\ell^1$ and $W_\ell^2$  for $\ell \in [L-1]$ are drawn i.i.d. from $\calN(0, 1/(4d))$.
    \item The entries of $\bb_\ell^1$ and $\bb_\ell^2$  for $\ell \in [L-1]$ are set to $-\sqrt{\frac{d_{\ell}}{2 \pi d_{\ell} + 2 \pi -2}}$. 
    \item The parameters $V_L, W_L, \bb_L$ in the last layer use the same initialization scheme.
    \item The parameters $W_0^{1},W_0^{2}$ and $\bb_0^{1}, \bb_0^{2}$ in the first layer adopt the initialization of \citet{lecun2002efficient}  (i.e., entries are i.i.d.\ samples from $N(0,1/d)$).
\end{itemize}
Derivations are provided in Appendix~\ref{appendix_init}. We note that the initialization scheme for HyCNNs is motivated by that of ICNNs proposed by \citet{hoedt2023principled}, but tailored to the architectural differences of HyCNNs. 
Based on this initialization scheme, we empirically verify that the signals across different layers are stable at initialization (see Figure \ref{fig:HyCNN_profiles_init}(a)), which is crucial for training deep networks. It is surprising that there is such a close agreement across the relative depth profiles. It suggests that a neural ODE limit \citep{pmlr-v139-cohen21b} might exist for HyCNNs. For a visualization of HyCNNs at initialization across different depths see \Cref{fig:HyCNN_init_vis}. 
\section{Approximation of Quadratic Functions}\label{sec:approximation}

The quadratic function $x\mapsto x^2$ is fundamental for approximating more complex convex functions, and thus serves as a natural testbed for understanding the approximation capabilities of convex neural architectures. 
In the following we analyze the capabilities of ICNNs and HyCNNs in approximating the quadratic function on the unit interval $[0,1]$.  While the focus of this section is on the univariate setting for simplicity, the results can be easily extended to the multivariate setting, and we discuss this in Appendix~\ref{app:QuantitativeApprox}. All proofs of this section are deferred to Appendix~\ref{app:proofs}.

\parabold{Limited Expressivity of ICNNs} 
A preliminary observation is that (standard or leaky) ReLU ICNNs are piecewise-affine, and the number of linear pieces grows only linearly with the number of layers. Recall that a leaky ReLU is defined as $\sigma(x) = \max(\alpha x, x)$ for some $\alpha \in (0,1)$.

\begin{proposition} \label{prop:piecewise_linear_icnn}
Consider an  ICNN $f \colon \RR\to \RR$ with $L$ hidden layers and $d_\ell$ neurons on layer $\ell$, whose activation function is given by a (standard or leaky) ReLU. Then, $f$ is piecewise-affine with at most $d_1 + 2d_2 + \dots + 2d_L$ pieces.
\end{proposition}

\begin{theorem}\label{thm:lowerBoundPiecewiseLinear}
    Let $\calG_k$ be the set of real-valued, piecewise-affine, convex functions on $[0,1]$ with at most $k$ pieces.
    Then,     \begin{align}\label{eq:lowerBoundApproxPiecewiseLinear}
       \inf_{g\in \calG_k} \, \int_{0}^1 \big|g(x)-x^2 \big| \, dx \geq \frac{1}{16 k^2}, \quad \text{ and } \quad
       \inf_{g\in \calG_k} \, \sup_{x\in [0,1]} \, |g(x)-x^2| \geq \frac{1}{8 k^2}.
    \end{align}
    In particular, any (standard or leaky) ReLU ICNN $f\colon \RR\to \RR$ with $L$ hidden layers and $d_\ell$ neurons in layer $\ell$ satisfies
    \begin{align*}
    \int_{0}^1 \big|f(x)-x^2 \big| \, dx \geq \frac{1}{16(d_1 + 2 \sum_{\ell = 2}^L d_\ell)^2} \quad \text{and} \quad 
    \sup_{x\in [0,1]}|f(x)-x^2| \geq \frac{1}{8 (d_1 + 2 \sum_{\ell = 2}^L d_\ell)^2}.
    \end{align*}
\end{theorem}

\parabold{Expressivity of HyCNNs} 
Next, we focus on the approximation capabilities of HyCNNs.

\begin{theorem}[Quadratic Function Approximation]\label{thm:quadratic_approximation}
	For any $L\in \NN$, even numbers $d_1, \ldots, d_L$ {\reb and $\tau \geq 0$}, there exists a HyCNN $h \colon \RR\to \RR$ with {\reb $\overline \sigma(a_1,a_2) = \overline \sigma_\tau(a_1, a_2)$},  $L$ hidden layers, and $(d_1, \dots, d_L)$ neurons per layer such that 
     \begin{align*} 
        \sup_{x\in [0,1]}|h(x) - x^2| \leq \frac{1}{7.99\prod_{\ell=1}^L d_{\ell}^2}.
     \end{align*}
\end{theorem}

    A first key takeaway from \Cref{thm:quadratic_approximation} is that wider HyCNNs lead to a \emph{polynomial} increase of the approximation accuracy, while deeper networks lead to an \emph{exponential} increase of the approximation accuracy. In particular, HyCNNs are able to \emph{leverage depth} for approximation, which is in stark contrast to ICNNs. Our empirical findings show that deeper HyCNNs can also be effectively trained. %

\begin{remark}[Interpretation of Theorem \ref{thm:quadratic_approximation}]
    The proof is detailed in Appendix \ref{app:proofs}. A similar result is well-known for feedforward ReLU networks. However, the known construction crucially relies on highly oscillating sawtooth functions, which are not input convex \citep{telgarsky2016, yarotski2017, schmidthieber2017}. It is therefore surprising that each HyCNN with at least two neurons per layer can at least \emph{double} the number of linear pieces in the piecewise-affine approximation of $x^2$, leading to an exponential decay of the approximation error with respect to the number of layers. To explain why it works, we derive a second proof for the case $d_1=\ldots=d_L=2$ that reveals how the oscillatory functions can be cached in input convex functions. %
\end{remark}

Based on our construction for the quadratic function, we also obtain an efficient approximation of higher order monomials $x\mapsto x^n$ on $[0,1]$ for $n=2,3,\ldots$ by suitably concatenating HyCNN blocks. %

\begin{theorem}[Monomial Function Approximation]\label{thm_monomial}
    For any integer $n \geq 2$, $L\in \NN$ and odd number $m \geq 3$ {\reb and $\tau \geq 0$}, there exists a HyCNN $h \colon \RR\to \RR$ with {\reb $\overline \sigma(a_1,a_2) = \overline \sigma_\tau(a_1, a_2)$}, $ \lceil \log_2(n)\rceil L$ layers and $m$ neurons per layer such that 
    \begin{align*}
        \sup_{x\in [0,1]} \big|h(x) - x^n\big| \leq \frac{n}{1.99} (m-1)^{-2L}.
    \end{align*}
\end{theorem}

The underlying HyCNN construction is not a straightforward extension of the network construction for the quadratic function (\Cref{thm:quadratic_approximation}), and requires a more involved proof strategy. This is in contrast to feedforward ReLU networks where the approximation of, say, $x\mapsto x^3$ first constructs a network approximately computing $x\mapsto (x,x^2)$ and then concatenates the outcome with a multiplication network that approximates the map $(u,v)\mapsto uv,$ therefore computing approximately $x\mapsto x x^2 =x^3.$ However,  $(u,v)\mapsto uv$ is not a convex function and therefore not representable by a HyCNN. 

{\reb 
 The exponential approximation rates for monomials can also be used to derive approximation bounds for analytic functions whose Taylor coefficients of order two and higher are non-negative. 
 In addition, the approximation advantage of HyCNNs over standard ICNNs extends to gradient approximation when the target function is convex and smooth, which is useful for applications in optimal transport. %
See \Cref{prop:positive_coeff} and \Cref{theorem_approx_grad} in Appendix~\ref{app:QuantitativeApprox}  for formal statements.}

We close this section with further implications of the above results. 
Denote by $\ICNN_d(L,m)$, $\HyCNN_d(L,m)$ and $\ReLU_d(L,m)$ the respective classes of all (ReLU) ICNNs, (maxout) HyCNNs and ReLU networks with $d$ input neurons, $L$ hidden layers and $m$ neurons per layer, respectively. As a direct consequence of Theorems \ref{thm:lowerBoundPiecewiseLinear} and \ref{thm:quadratic_approximation} we have that whenever $m'$ is even and $(m')^{L'} \geq 2mL,$
    \begin{align}
        \HyCNN_d(L',m') \subsetneq \ICNN_d(L,m) \subseteq \HyCNN_d(L,m)
        \subseteq \ReLU_d(L,3m+2d).
        \label{eq.84gr}
    \end{align}
    This shows that a much larger number of ICNN parameters is needed to represent all HyCNN functions. The $\subsetneq$ part also reveals that ICNNs are limited in their ability to represent the maximum of two functions as otherwise they could emulate HyCNNs contradicting \eqref{eq.84gr}.   
    The last inclusion implies that any lower bound on the approximation error of ReLU networks directly translates to those of HyCNNs. Moreover, known bounds on the VC dimension and covering number of ReLU networks imply corresponding bounds for HyCNNs.
    A formal proof of \eqref{eq.84gr} is given in Appendix~\ref{appendix_proof_equation_inclusion}. 

\section{Application: Neural Optimal Transport Map Learning}\label{sec:applications}

\begin{figure}[t!]
    \begin{center}
    \centerline{\includegraphics[width=\textwidth, trim=0 16 0 5, clip]{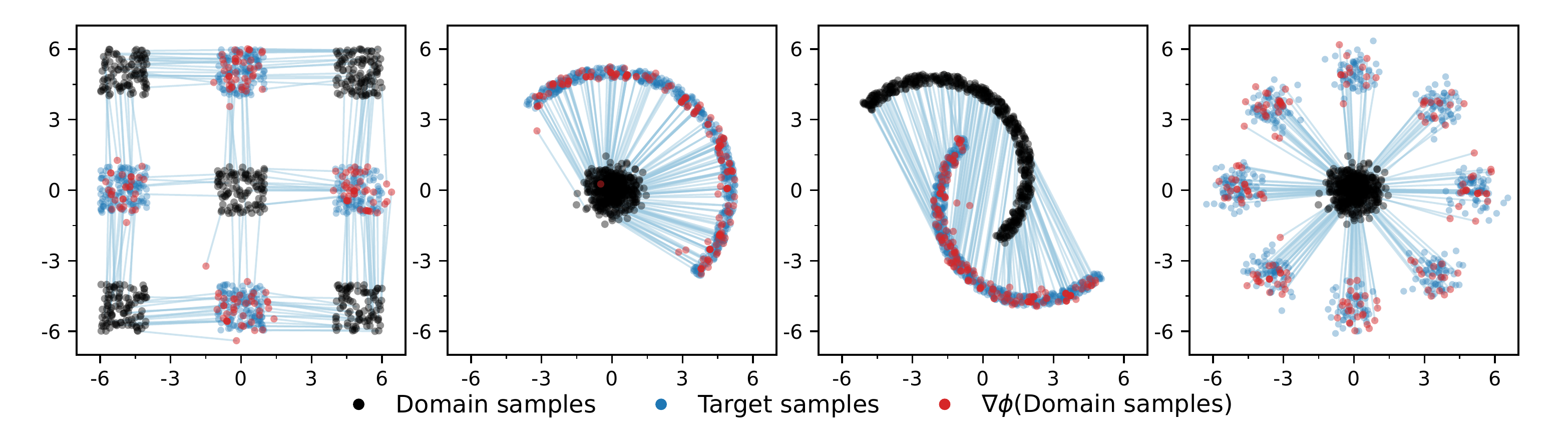}}
        \caption{Pushforwards under learned HyCNN OT map (width $40$, depth $4$). The HyCNN is trained for $n = m = M = 2{,}000$ data points for $2500$ outer iterations for learning $\phi$, with $S_{\textup{in}} = 10$ steps for the inner loop, via Adam and decaying learning rates and smoothing parameter $\tau$. 
        }
         \label{fig:HyCNN_OT_shapes_2d}
         \end{center}
\end{figure}
\label{sec:neural_ot}

Learning an optimal transport (OT) map between probability measures is a central problem with broad impact in machine learning \citep{courty2016optimal,gordaliza2019obtaining}, statistics \citep{hallin2021distribution,de2024transport}, and natural sciences \citep{wang2010optimal,schiebinger2019optimal,tameling2021colocalization,bunne2023learning}.  
To formalize the problem, consider probability measures $P$ and $Q$ on $\RR^d$ with finite second moments. For squared Euclidean costs, the \emph{OT map} is defined as%
\begin{equation*}
  \overline T \in \argmin_{T:\,T_\# P = Q}\; \int_{\RR^d}\|\bx-T(\bx)\|^2\,dP(\bx), \quad \text{ where } T_\# P \coloneqq P\circ T^{-1}, 
\end{equation*} assuming it exists. If $P$ and $Q$ are absolutely continuous, Brenier's theorem \citep[Theorem 1.22]{santambrogio2015optimal} guarantees the existence of a $P$-a.s.\ unique convex potential $\phi:\RR^d\to\RR$, the \emph{OT potential}, such that $\overline T = \nabla \phi$. This links OT map estimation with learning the OT potential.

Recent work has produced a variety of estimators for $\overline T$, including plug-in methods \citep{manole2024plugin}, entropic approaches \citep{seguy2017large,pooladian2021entropic}, semidual formulations \citep{hutter2021minimax,divol2025optimal}, and neural estimators \citep{taghvaei20192,makkuva2020optimal,bunne2023learning,uscidda2023monge}. It is worth noting that empirical plug-in and entropic estimators, though adaptive to low-dimensional structures \citep{weed2019sharp,hundrieser2024empirical,groppe2024lower,stromme2024minimum,divol2025optimal}, are oftentimes outperformed by neural estimators in genuinely high-dimensional settings \citep{makkuva2020optimal,uscidda2023monge}, which motivates the use of neural OT in this regime.

We use the variational formulation of \citet{makkuva2020optimal}, based on the semidual framework by \cite{hutter2021minimax}, which characterizes the OT potential $\phi$ via a saddle-point problem, 
\begin{align*}
  \phi \in \underset{\text{convex } f\in L^1(P),f^*\in L^1(Q)}{\argmax}
  \;\; \underset{\vphantom{f^*}g\in L^1(Q)\vphantom{f^*}}{\vphantom{g}\inf\vphantom{g}}\; -\EE_{\bX\sim P}[f(\bX)]
  -\EE_{\bY\sim Q}\!\big[\langle \bY,\nabla g(\bY)\rangle - f(\nabla g(\bY))\big]. 
\end{align*}
Given $f$, the optimal $g$ corresponds to the convex conjugate of $f$, i.e., $g = f^*$. In practice, we replace the expectations by empirical averages (using mini-batches), and the function classes are restricted to neural network families. %
Denoting the training data $(\bx_i)_{i=1}^n\sim P$ and $(\by_j)_{j=1}^m\sim Q$, we follow the mini-batch training protocol of \citet{makkuva2020optimal} but replace the ICNNs by HyCNNs with activation function $\overline\sigma_\tau(a_1,a_2) = \tau \log(e^{a_1/\tau} + e^{a_2/\tau})$. This leads to a saddle-point problem for  mini-batches $(\bX_i)_{i=1}^M$ and $(\bY_j)_{j=1}^M$ from $(\bx_i)_{i=1}^n$ and $(\by_j)_{j=1}^m$, respectively, of size $M\in \NN$, 
\begin{align}\label{eq:HyCNN_OT_potential_estimation}
  \hat\phi \in &
  \argmax_{f\in \texttt{HyCNN}_\tau(\RR^d)}
  \;\underset{{g\in \texttt{HyCNN}_\tau(\RR^d)}}{\vphantom{g}\min\vphantom{g}}\;
  \textstyle -\frac{1}{M}\left(\sum_{i = 1}^{M} f(\bX_i)+\langle \bY_i, \nabla g(\bY_i)\rangle - f(\nabla g(\bY_i)) \right).
\end{align}
The \textit{HyCNN-based OT map estimator} is then defined by $\hat T \coloneqq \nabla \hat \phi$, see \Cref{fig:HyCNN_OT_shapes_2d} for an illustration. The training method is detailed in \Cref{alg:HyCNN_ot}.
In \Cref{app:OT_implementation}, it uses Adam \citep{kingma2014adam} and involves an outer loop (with $S_{\textup{out}}$ iterations) and an inner loop (with $S_{\textup{in}}$ iterations per outer iteration). During training we parameterize hidden-to-hidden weights via a softplus function to ensure non-negativity, deviating from the approach by \citet{makkuva2020optimal}, who use a regularization term for a soft-enforcement of non-negativity. 

{\reb 
The simulations in Appendix \ref{app:OT_experiments}, make two key observations on the choice of $\tau$. In small-dimensional settings, a small $\tau>0$ better captures the discontinuities in the OT map. A cosine annealing schedule for $\tau$ often leads to improved performance compared to using a fixed $\tau$. In moderate to high-dimensional settings, a positive fixed $\tau$ during training and chosen according to a Sinkhorn divergence evaluation on a validation set tends to work well. Given that HyCNNs empirically appear to already often perform well for relatively short training, i.e., $S_{\textup{out}}$  in the range of $1{,}000$ to $5{,}000$ iterations and $S_{\textup{in}}$ in the range of $5$ to $10$ steps per outer iteration, which can be realized in a matter of minutes, it is feasible to tune $\tau$ by some validation criterion without incurring significant computational overhead. } 
\section{Experiments}\label{sec:experiments}

\parabold{Convex regression} 
We draw covariates $(\bX_i)_{i=1}^n$ from $\textup{Unif}[-1,1]^d$ with responses according to the model $Y_i = f(\bX_i) + \epsilon_i$, where $\epsilon_i\sim N(0,\sigma^2)$ are independent Gaussian errors with standard deviation $\sigma\geq 0$, and $f\colon \RR^d \to \RR$ is one of the following functions, 
\begin{align*}
   &\textstyle \text{(a) } f_1(\bx) = \|\bx\|^2_2, &&\text{(b) } f_2(\bx) = \|\bx\|_4^4, \;&&\text{(c) } f_3(\bx) = \|\bx\|^2_2 + 0.25\sin\left(20\|\bx\|_2\right)%
\end{align*}
Note that $f_3$ is not convex, but it can be fairly well approximated by a convex function. 
We consider dimension $d = 50$, training sample sizes $n \in \{100, 200, 500, 1{,}000, 2{,}000, 5{,}000, 1{,}0000\}$, and Gaussian noise with variance $\sigma^2 = 1$ for $f_1, f_2$ and $\sigma = 0$ for $f_3$.  
For each regression setting, we train feedforward ReLU networks (multilayer perceptrons, MLPs), ICNNs \citep{amos2017input}, GroupMax networks (GMNs, \citealt{warin2023groupmax}, i.e., HyCNNs without skip connections from the input), and HyCNNs. {\reb{We use widths $m=84$ for MLPs, $m=64$ for ICNNs, $m=60$ for GMNs, and $m=40$ for HyCNNs, yielding similar parameter counts in dimension $d=50$, see \Cref{tab:model_visualization_dim_param_count}. 
Further we consider depth $L \in \{2, 4, 8, 16,32\}$ for all architectures.}} \Cref{app:exp_setup_regression} contains additional details.

In \Cref{fig:HyCNN_training_dim50}, we show the prediction MSE on an independent test set of size $1{,}000$. For clarity, we only display results for HyCNNs with depths $L \in \{2, 4, 8, 16\}$ and the method which performs best for training sample size $n = 10{,}000$.
HyCNNs with depth at least four consistently outperform all other methods across all regression settings, with deeper HyCNNs performing better than shallower ones. {\reb 
The Appendices \ref{app:univariate_regression}--\ref{app:interpolation}, cover extensive additional experimental results showing that HyCNNs consistently outperform the baseline methods. We refer to Appendix \ref{app:tau_depth} for an analysis of the effect of the smoothing parameter on signal propagation and regression performance, and to Appendix \ref{app:branch_count} for an empirical evaluation of the effect of the number of activation branches. Overall, deep HyCNNs with small $\tau$ and two branches achieve the best performance.
\begin{figure}[h!]
  \begin{center}
    \centerline{\includegraphics[width=\textwidth, trim = 0 0 0 0, clip]{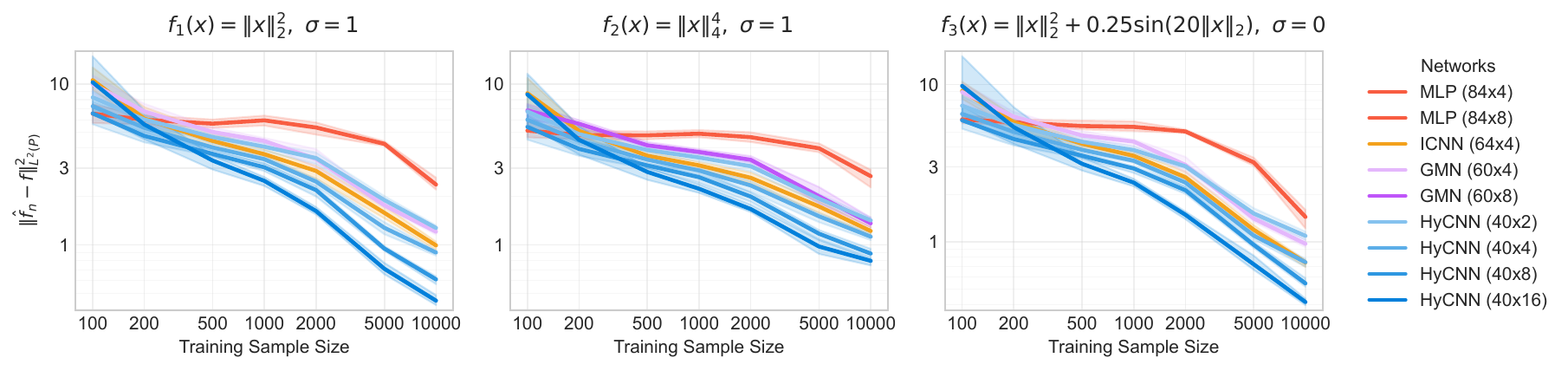}}
    \caption{Prediction MSE for different training sample sizes $n$ for the $50$-dimensional regression tasks with 10th-90th percentile bands across $10$ runs. Each panel shows HyCNNs with depths $2,4,8,16$ and the best performing MLP, ICNN, GMN at training sample size $n = 10^4$ among depths $\in \{2,4,8,16\}$. %
    }
    \label{fig:HyCNN_training_dim50}
  \end{center}
\end{figure}

\parabold{Neural OT map estimation}
To assess the performance of the HyCNN-based OT map estimator, we consider the following OT potentials with corresponding OT maps given by their gradients,
\begin{align*}
    &\textstyle \text{(a) } \textstyle\phi_1(\bx) = \frac{\|\bx\|^2_2}{2}, \quad T_1(\bx) =  \bx, &&\text{(b) } \textstyle\phi_2(\bx) = \frac{1}{2}\sum_{i = 1}^{d} \big(1+ \frac{\sin(i)}{2}\big)x_i^2, \quad T_2(\bx) = \big((1+\frac{\sin(i)}{2})x_i\big)_{i=1}^d, %
\end{align*}
and draw $n=m = 5{,}000$ independent samples from $P = N(0, I_d)$ and $Q = (\nabla \phi_\ell)_\# P$ for $\ell \in \{1,2\}$ with $d = 50$. In \Cref{fig:HyCNN_training_OT}\textbf{(a)} we display the (dimension-normalized) prediction MSE for an independent testing set of size $1{,}000$ for the two displayed OT map estimation tasks for $d = 50$ across training iterations. %
We compare HyCNNs of width $44$, depth $4$, and smoothness $\tau=10$ with ICNN variants \citep{makkuva2020optimal} of width $64$ and baseline depths $L\in\{2,4,6\}$. The Monge gap MLP estimator \citep{uscidda2023monge} uses width $104$ and the same depths, see \Cref{tab:ot_total_parameter_counts} for a summary of the total parameter counts. HyCNNs outperform all other methods across all tasks, with ICNNs
performing better than the Monge gap MLP estimator. Appendices \ref{app:OT_dim2} and \ref{app:OT_dimensions} contain additional information on training rules and numerical evidence for two-dimensional and higher-dimensional OT map estimation tasks, respectively. Appendix \ref{app:OT_ablations} contains ablation studies for the OT map estimation experiments and reveals that convexity, the two-branch architecture, smoothness, and the initialization all contribute jointly to the performance of the HyCNN-based OT map estimator. }

 Moreover, for the 4i single-cell dataset, preprocessed by \cite{bunne2023learning}, we train HyCNNs and ICNNs to map the unperturbed control distribution onto each drug-perturbed distribution and find that HyCNNs attain a lower Sinkhorn divergence  on a $20\%$ held-out test partition than the best ICNN baseline across nearly all considered perturbations, see \Cref{fig:HyCNN_training_OT}\textbf{(b)} or \Cref{fig:HyCNN_training_OT_large} for an enlarged view. Details on training rules and further aspects on the data analysis are given in  Appendix \ref{app:single_cell}. To make the comparison robust to variability across checkpoints, each reported value averages the test-set Sinkhorn divergence over the $K = 10$ checkpoints with smallest validation loss; see Appendix \ref{app:single_cell} for details.
\Cref{fig:HyCNN_training_OT}\textbf{(b)} contrasts a short and a long training schedule per architecture. Longer training improves performance, but architecture changes outweigh the training budget: with $100$ times less compute, the short-trained HyCNN is competitive with the long-trained ICNN, and is consistently outperformed by the long-trained HyCNN.

\begin{figure}[t]
\centering
\includegraphics[width=\textwidth]{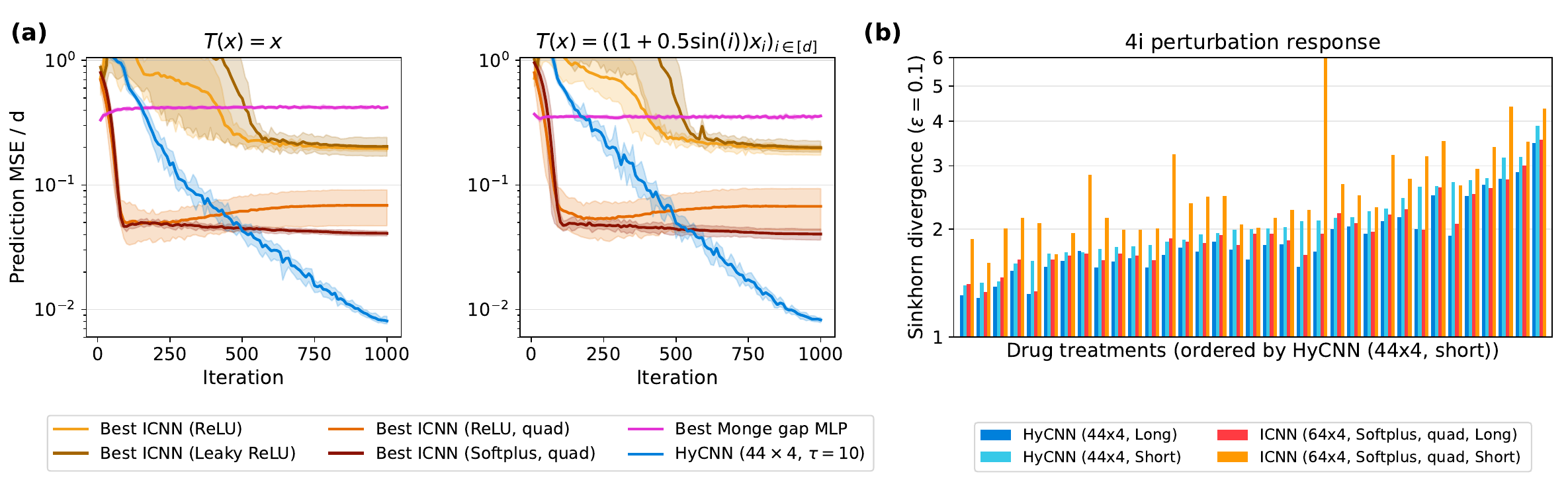}
\caption{Performance of OT map estimation for synthetic and 4i single cell data. \textbf{(a)} Prediction MSE (normalized by dimension $d=50$) for $n = 5{,}000$ training samples for OT map estimation  tasks with 10th-90th percentile bands across $10$ runs. 
    Both panels show the same HyCNN (width $44$) and the best ICNN (width $64$) with ReLU or leaky ReLU, or with a quadratic term in the first layer with ReLU or Softplus activation, and the Monge gap MLP (width $104$) estimator, each within a network class of depths $L \in \{2, 4,6\}$ and $\tau \in \{0.1, 1, 10\}$ whenever softplus or log-sum-exp is used. \textbf{(b)} Sinkhorn divergence between predicted and ground-truth drug-perturbed distributions on a $20\%$ held-out test partition of the 4i dataset for each drug based on ICNN ($64\times4$, Softplus, $\tau = 1$, with a quadratic term in the first layer) and HyCNN ($44\times 4$, log-sum-exp, $\tau = 1$) trained for $S_{\textup{out}} = 1{,}000$ outer iterations and $S_{\textup{in}} = 5$ inner iterations (short schedule) as well as $S_{\textup{out}} = 50{,}000$ iterations and $S_{\textup{in}} = 10$ inner iterations (long schedule).} 
\label{fig:HyCNN_training_OT}
\end{figure}

\section{Discussion and Limitations}\label{sec:limitations}
HyCNNs appear as a promising architecture for learning convex functions in high dimensions: for regression, deeper models perform better than shallower ones, and for OT map estimation, HyCNNs with large $\tau$ perform better than those with small $\tau$. 
{\reb We also note that the HyCNN two-lane architecture could also be applied in combination with other frameworks. For example, 
 Convolutional HyCNNs can be defined by constraining the hidden-to-hidden convolution kernels to be non-negative, leaving the input-skip convolutions unconstrained, and applying the two-lane maxout or log-sum-exp gate channelwise. 
 Input-convex convolutional neural networks \citep{korotin2021wasserstein, amos2023amortizing} are recovered as the single-lane ($\overline \sigma(a_1, a_2) = \sigma(a_1)$) special case.}

We note two limitations of our work. By construction, HyCNNs exhibit an inductive bias toward smooth, coercive convex functions, which benefits many applications but may constrain training in settings that deviate from these assumptions. Further, our empirical study is bounded by computational constraints: though we conducted extensive simulations, our training tasks were limited, and a more large-scale study remains a future step.
\begin{ack}
S.\ Hundrieser is funded from the German National Academy of Sciences Leopoldina under grant number LPDS 2024-11. I.\ Kong and J.\ Schmidt-Hieber are supported
by ERC grant A2B (grant agreement number 101124751). S.\ Hundrieser further acknowledges helpful discussions with Paul Hagemann in an early phase of this project.  Claude Code Opus 4.6  and ChatGPT Codex 5.4 and later 6.1 Sol were used for coding, in particular for implementing the neural network architectures.
\end{ack}

\relax%

\relax%
\space%
\relax{}

\newpage
\appendix
\onecolumn

\section*{Supplementary Material}
The supplementary material is organized as follows. \Cref{app:QuantitativeApprox} provides
additional theoretical results on the approximation capabilities of HyCNNs.
\Cref{app:proofs} collects all omitted proofs of the results stated in the main text and in \Cref{app:QuantitativeApprox}. \Cref{appendix_init} details the derivation of the initialization scheme for HyCNNs introduced in \Cref{sec:HyCNNs}. \Cref{app:additional_experiments} contains additional experimental results and implementation details (complementing \Cref{sec:experiments}), including visualizations of randomly initialized HyCNNs, further regression experiments, and additional results on neural OT map estimation. \Cref{app:single_cell} presents further details to neural OT based single-cell perturbation response prediction for the preprocessed 4i single-cell dataset \citep{bunne2023learning}.

\section{Additional Theoretical Results}\label{app:QuantitativeApprox}

In this section, we extend the approximation results to multivariate quadratic functions and a class of univariate analytic functions. The proofs are deferred to Appendix \ref{appendix_proof_multivariate}.

The following theorems state the multivariate extensions of the approximation lower and upper bounds for approximating the quadratic function $\bx \mapsto \|\bx\|_2^2$ on $[0,1]^d$ by ReLU ICNNs and HyCNNs, respectively. 

\begin{theorem}\label{prop:lowerBoundMultivariateICNN}
    Any (standard or leaky) ReLU ICNN $f_\theta\colon \RR^d \to \RR$ with $L$ hidden layers and $d_\ell$ neurons in layer $\ell$ satisfies the lower bound 
    \begin{align*}
        \int_{[0,1]^d} \big|f_\theta(\bx)-\|\bx\|_2^2 \big| \, d\bx &\geq \frac{1}{16(d_1 + 2 \sum_{\ell = 2}^L d_\ell)^2} \quad \text{and}\\
        \sup_{\bx\in [0,1]^d}|f_\theta(\bx)-\|\bx\|_2^2| &\geq \frac{d}{8 (d_1 + 2 \sum_{\ell = 2}^L d_\ell)^2}.
    \end{align*}
\end{theorem}

\begin{theorem}\label{thm:quadratic_approximation_multi}
	For any $d, L, m\in \NN$ with $m\geq 3$ {\reb and $\tau \geq 0$}, there exists a HyCNN $f\colon \RR^d \to \RR$ with {\reb $\overline \sigma(a_1,a_2) = \overline \sigma_\tau(a_1, a_2)$}, $L$ hidden layers and $m$ neurons per layer, %
    such that 
     \begin{align*}
        \sup_{\bx\in [0,1]^d} \big|f(\bx) - \|\bx\|_2^2\big| \leq \frac{d}{7.99} \inf_{p,q \in \NN: pq \geq d}  \left\lfloor \frac{m-1}{q}-1 \right\rfloor^{-2\lfloor L/p \rfloor}.
     \end{align*}
\end{theorem}

The key conceptual insights are that also in high-dimensional settings, HyCNNs can approximate the quadratic function $\bx \mapsto \|\bx\|_2^2$ with an error that decays exponentially in the depth and polynomially in the width, whereas (standard or leaky) ReLU ICNNs can only achieve a polynomial decay in the width and no decay in the depth.
Indeed, combining Theorem 11 of \citet{liang2017why} with the last inclusion in \eqref{eq.84gr} shows that the exponential decay with respect to the depth is optimal, for both ReLU networks and HyCNNs.

{\reb 
We next extend the approximation result beyond quadratic functions to analytic functions whose Taylor coefficients of order two and higher are non-negative. The following theorem provides an approximation error bound that decays exponentially in $L$.

\begin{theorem}\label{prop:positive_coeff}
    Assume that $r>1$ and that $f_0 : \RR \to \RR$ is real analytic on an open interval containing $[0,r]$, with $f_0^{(k)}(0) \geq 0$
for all $k \geq 2$. 
    For any $L\in \mathbb{N}$, odd number $m \geq 3$ and $\tau \geq 0$, there exists a HyCNN $h \colon \mathbb{R} \to \mathbb{R}$ with $\overline \sigma(a,b) = \overline \sigma_\tau(a,b)$, $ \lceil \log_2(2L)\rceil L$ layers and $2Lm$ neurons per layer such that 
    \begin{align*}
        \sup_{x\in [0,1]} \big|h(x) - f_0(x)\big| \lesssim r^{-2L} + (m-1)^{-2L},
    \end{align*}
    where the implicit constant depends only on $f_0$. 
\end{theorem}

The target function class in \Cref{prop:positive_coeff} is much larger than the class of monomials.
In fact, it contains functions which are analytic on an open interval containing $[0,r]$ for some $r>1$ and absolutely monotonic on $[0,1]$, meaning that their derivatives of all orders are non-negative on $[0,1]$.
Examples include $\exp(x)$, $\operatorname{sinh}(x)$, $\operatorname{cosh}(x)$, as well as the functions  $\sin^{-1}(\gamma x)$,  $-\log(1-\gamma x)$ and $(1-\gamma x)^{-\alpha}$ with $\gamma \in (0,1)$ and $\alpha>0$.

The following proposition shows that uniform approximation of a smooth convex target function $g$ by a differentiable convex function $f$ yields a bound on the gradient approximation error.
We write $K:=[0,1]^d\;\subset\;K':=(-1,\,2)^d.$

\begin{proposition}\label{theorem_approx_grad}
Let $f,g\colon K'\to\RR$ be differentiable and convex, and let $g$ be $M$-smooth on $K'$ with $M\geq 0$, i.e.,
\begin{align*}
  g(\by)\;\leq\; g(\bx)+\langle\nabla g(\bx),\,\by-\bx\rangle+\tfrac{M}{2}\|\by-\bx\|^2
  \qquad\text{for all }\bx, \by\in K'.
\end{align*}
Then, for any $p\in [1, \infty)$, 
\begin{align*}
  \|\nabla f-\nabla g\|_{L^p(K)} \leq \|\nabla f-\nabla g\|_{L^\infty(K)}
  \leq 
     2\sqrt{M\varepsilon} +4\varepsilon, \qquad \varepsilon:=\|f-g\|_{L^\infty(K')}.
\end{align*}
\end{proposition}

Hence, exponential approximation rates for smooth convex target functions yield exponential approximation rates for their gradients.

}

\section{Omitted Proofs}\label{app:proofs}

\subsection{Technical lemmas}

\begin{lemma} \label{lemma_comp_convexnondecre}
    Let $\overline \sigma \colon \RR^2 \to \RR$ be convex and componentwise non-decreasing. 
    If $f: \mathbb{R}^d \to \mathbb{R}$ and $g: \mathbb{R}^d \to \mathbb{R}$ are convex functions, then
    $h : \bx \mapsto \overline \sigma(f(\bx), g(\bx))$ is convex.
\end{lemma}
\begin{proof}
    For any $\bx_1, \bx_2 \in \RR^d$ and $t \in [0,1]$, we have
    \[
        f\big( t\bx_1 + (1-t)\bx_2\big) \leq t f(\bx_1) + (1-t) f(\bx_2), \quad 
        g\big( t\bx_1 + (1-t)\bx_2\big) \leq t g(\bx_1) + (1-t) g(\bx_2).
    \]
    Hence, by the assumptions imposed on $\overline \sigma$ we get
    \begin{align*}
        h\big( t\bx_1 + (1-t)\bx_2  \big) 
        &= \overline \sigma\Big(f\big( t\bx_1 + (1-t)\bx_2  \big), \, g\big( t\bx_1 + (1-t)\bx_2  \big)\Big)\\
        & \leq \overline \sigma\big(t f(\bx_1) + (1-t) f(\bx_2), \, t g(\bx_1) + (1-t) g(\bx_2) \big)\\
        & = \overline \sigma\Big(t \big(f(\bx_1), g(\bx_1)\big)^{\top} + (1-t) \big(f(\bx_2), g(\bx_2)\big)^{\top} \Big)\\
        & \leq  t \overline \sigma\big(f(\bx_1), g(\bx_1)\big) +
        (1-t) \overline \sigma\big(f(\bx_2), g(\bx_2)\big)\\
        & = t h(\bx_1) +
        (1-t) h(\bx_2).\qedhere
    \end{align*}
\end{proof}

{\reb

The following lemma shows that any maxout HyCNN can be approximated to arbitrary accuracy by a HyCNN with a  smooth activation $\overline{\sigma}_\tau$ and the same architecture. 

\begin{lemma}\label{thm_tau_approx}
    For any HyCNN $h \colon \mathbb{R}^{d} \to \mathbb{R}$ with $\overline \sigma(a_1,a_2) = \max(a_1,a_2)$ and any $\tau, \epsilon > 0$, there exists a HyCNN $g \colon \mathbb{R}^{d} \to \mathbb{R}$ with $\overline \sigma(a_1,a_2) = \overline \sigma_\tau(a_1, a_2)$ and the same architecture as $h$ 
    such that 
    \[
    \sup_{\bx \in \mathbb{R}^{d}} |g(\bx) - h(\bx)| \leq \epsilon.
    \]     
\end{lemma}

\begin{proof}
    For $m \coloneq \max(a_1, a_2)$,  
    \begin{align*}
            \overline \sigma_\tau(a_1, a_2) = 
    \tau \log\left(e^{a_1/\tau} + e^{a_2/\tau}\right) &= m + \tau \log\left(e^{(a_1-m)/\tau} + e^{(a_2-m)/\tau}\right).
    \end{align*}
For any fixed $\tau>0$ and any $a_1,a_2\in\mathbb{R}$, we have
\begin{align*}
    \max(a_1, a_2) \leq \overline \sigma_\tau(a_1, a_2) \leq \max(a_1, a_2) + \tau \log 2 \leq \max(a_1, a_2) + \tau.
\end{align*}
Hence, for any $s>0$, it holds that 
\begin{align}
    \max(a_1, a_2) \leq \overline \sigma_{s \tau}(a_1, a_2) =  s  \overline \sigma_{\tau}(a_1/s, a_2/s) \leq \max(a_1, a_2) + s \tau.
    \label{eq.max_approx}
\end{align}
The idea of the proof is to show that, for any fixed $\tau>0$, a HyCNN with activation function $\overline{\sigma}_{\tau}$ can approximate the maxout HyCNN $h$ to arbitrary precision by appropriately rescaling the network parameters.

Formally, let $L$ denote the number of hidden layers of $h$, and let $d_{\ell}$ denote the number of neurons in the $\ell$-th hidden layer of $h$. 
The output of $h$ and the hidden layers are defined by
\begin{align*}
    h(\bx) \coloneq V_{L}\bz_{L} + W_{L} \bx + \bb_{L}, \qquad \bz_{\ell+1} &\coloneq \max\left(V_\ell^{1} \bz_{\ell} + W_{\ell}^{1} \bx + \bb_{\ell}^{1}, V_\ell^{2} \bz_{\ell} + W_{\ell}^{2} \bx + \bb_{\ell}^{2}\right),
\end{align*}
for $\ell=0,\ldots,L-1$ with $\bz_0 \coloneqq \mathbf{0}_d$. 
We let 
\[
B \coloneqq
\max\Big(1, \,
\|V_L\|_{\max},\,
\|W_L\|_{\max},\,
\|\bb_L\|_{\max},\,
\max_{\substack{0\le \ell\le L-1\\ k\in\{1,2\}}}
\big(
\|V_\ell^k\|_{\max},\,
\|W_\ell^k\|_{\max},\,
\|\bb_\ell^k\|_{\max}
\big)
\Big)
\]
denote the maximum absolute value among these parameters.
Here, $\|\cdot\|_{\max}$ denotes the maximum absolute entry of a matrix or vector.

We construct a HyCNN $g \colon \mathbb{R}^{d} \to \mathbb{R}$ defined by
\begin{align*} g(\bx) \coloneq M_{L}\by_{L} + N_{L} \bx + \bc_{L}, \qquad \by_{\ell+1} &\coloneq \overline \sigma_\tau\left(M_\ell^{1} \by_{\ell} + N_{\ell}^{1} \bx + \bc_{\ell}^{1}, M_\ell^{2} \by_{\ell} + N_{\ell}^{2} \bx + \bc_{\ell}^{2}\right), \end{align*} for $\ell=0,\ldots,L-1$ with $\by_0 \coloneqq \mathbf{0}_d$, where the parameters will be chosen as follows.
Fix arbitrary $s>0$. 
Define 
\[
    M_{0}^k \coloneq V_{0}^k/s, \quad N_{0}^k \coloneq W_{0}^k/s, \quad \bc_0^k \coloneq \bb_{0}^k/s, \qquad k=1,2.
\]
Then, from 
\begin{align*}
        \by_{1} &= \overline \sigma_\tau\left(M_0^{1} \by_{0} + N_{0}^{1} \bx + \bc_{0}^{1}, M_0^{2} \by_{0} + N_{0}^{2} \bx + \bc_{0}^{2}\right) =
    \overline \sigma_\tau\left( \frac{W_{0}^{1} \bx + \bb_{0}^{1}}{s}, \frac{ W_{0}^{2} \bx + \bb_{0}^{2}}{s}\right),
\end{align*}
and \eqref{eq.max_approx},
we get 
$\|s \by_{1} - \bz_{1}\|_{\infty} \leq s \tau.$
Assume that, for some $\ell \in [L-1]$, there exist parameters
$(M_i^k, N_i^k, \bc_i^k)_{i \in \{0,\ldots,\ell-1\},\; k \in \{1,2\}}$ of $g$
such that
\[
    \|s \by_{\ell} - \bz_{\ell}\|_{\infty}
    \leq s \tau \prod_{i=1}^{\ell-1} (2B d_i)
\]
with the convention that $\prod_{\emptyset} = 1$.
Define
\[
    M_{\ell}^k \coloneq V_{\ell}^k, \quad N_{\ell}^k \coloneq W_{\ell}^k/s, \quad \bc_\ell^k \coloneq \bb_{\ell}^k/s, \qquad k=1,2.
\]
Then, from 
\begin{align*}
        \by_{\ell+1} &= \overline \sigma_\tau\left(M_\ell^{1} \by_{\ell} + N_{\ell}^{1} \bx + \bc_{\ell}^{1}, M_\ell^{2} \by_{\ell} + N_{\ell}^{2} \bx + \bc_{\ell}^{2}\right)\\
    &=
    \overline \sigma_\tau\left( \frac{V_\ell^{1} (s\by_{\ell}) + W_{\ell}^{1} \bx + \bb_{\ell}^{1}}{s}, \frac{V_\ell^{2} (s \by_{\ell}) + W_{\ell}^{2} \bx + \bb_{\ell}^{2}}{s}\right),
\end{align*}
\eqref{eq.max_approx} and the inequality $|\max(a,b) - \max(a',b')| \leq \max(|a-a'|,|b-b'|)$, we get
\begin{align*}
    \|s \by_{\ell+1} - \bz_{\ell+1}\|_{\infty} &\leq 
    \bigg\| s \overline \sigma_\tau\left( \frac{V_\ell^{1} (s\by_{\ell}) + W_{\ell}^{1} \bx + \bb_{\ell}^{1}}{s}, \frac{V_\ell^{2} (s \by_{\ell}) + W_{\ell}^{2} \bx + \bb_{\ell}^{2}}{s}\right) \\
    & \qquad - \max \Big( V_\ell^{1} (s\by_{\ell}) + W_{\ell}^{1} \bx + \bb_{\ell}^{1}, V_\ell^{2} (s \by_{\ell}) + W_{\ell}^{2} \bx + \bb_{\ell}^{2}\Big) \bigg\|_{\infty}\\
    & \quad +
        \bigg\| \max \Big( V_\ell^{1} (s\by_{\ell}) + W_{\ell}^{1} \bx + \bb_{\ell}^{1}, V_\ell^{2} (s \by_{\ell}) + W_{\ell}^{2} \bx + \bb_{\ell}^{2}\Big) \\
    & \qquad \quad - \max \Big( V_\ell^{1} \bz_{\ell} + W_{\ell}^{1} \bx + \bb_{\ell}^{1}, V_\ell^{2}  \bz_{\ell} + W_{\ell}^{2} \bx + \bb_{\ell}^{2}\Big) \bigg\|_{\infty}\\
    &\leq s \tau + B d_\ell s \tau \prod_{i=1}^{\ell-1} (2 B d_i)\\
    &\leq s \tau \prod_{i=1}^{\ell} (2 B d_i).
\end{align*}
Hence, by induction, we obtain 
\[
    \|s \by_{L} - \bz_{L}\|_{\infty}
    \leq s \tau \prod_{i=1}^{L-1} (2Bd_i).
\]
By defining 
    $M_{L} \coloneq s V_{L}$, $N_{L} \coloneq W_{L}$ and $\bc_L \coloneq \bb_{L}$, we get
    \[
    |g(\bx) - h(\bx)| \leq B d_L \|s \by_{L} - \bz_{L}\|_{\infty} \leq s \tau 2^{L-1} B^{L} \prod_{i=1}^{L} d_i.
    \]
    Recall that $s>0$ can be chosen arbitrarily. By choosing $0 < s \leq \epsilon/(\tau 2^{L-1} B^{L} \prod_{i=1}^{L} d_i)$, we obtain the assertion.
\end{proof}
}

{\reb The entrywise non-negativity constraint imposed on the hidden-to-hidden weight matrices $V_1, V_2,\ldots,V_{L-1}$ in ICNNs is necessary to ensure the input convexity. Because ICNNs form a subclass of HyCNNs, the result also shows that the input convexity of HyCNNs requires entrywise non-negativity of the matrices $V_1^1, V_1^2, V_2^1, V_2^2, \ldots,V_{L-1}^1, V_{L-1}^2.$
        \begin{lemma}
        \label{lem.necessary_non_neg}
            Consider ReLU ICNNs with $L\geq 2$ layers and assume that for an arbitrary $\ell\in \{1,\ldots,L-1\}$ and an arbitrary entry of $V_\ell,$ the non-negativity constraint does not hold. For any $i=1,\ldots,d,$ the resulting function class will then also contain the non-convex functions $f_{i}:[0,1]^d \to \mathbb{R},$ $f_{i}(x_1,\ldots,x_d)=\min(1/2, 1-x_i).$
        \end{lemma}
        
\begin{proof}
        Let the selected $\ell, i$ be denoted by $\ell^*$ and $i^*.$ For arbitrary $m\in \{1, \ldots, d_{\ell^*+1}\}, n \in \{1,\ldots,d_{\ell^*}\},$ assume that the $(m,n)$-th entry of $V_{\ell^*}$ is allowed to be negative. We show that the right-hand side of the identity $\min(1/2, 1-x_i)=(1/2-(x_i-1/2)_+)$ with $x_i\in [0,1]$ can be rewritten as an ICNN. Choose all the matrices $W_1,\ldots,W_L$ as zero matrices. For an $r\in \{1,\ldots, d_1\},$ setting $\mathbf{b}_0:=-\tfrac 12 \mathbf{e}_r$ with $\mathbf{e}_r$ the $r$-th standard basis vector and $W_0:=(\mathbf{1}(a=r, b=i))_{a,b},$ we have that $\mathbf{z}_1=(x_i-1/2)_+\mathbf{e}_r$.  
        For $\ell\geq 1$, $p\in \{1,\ldots,d_{\ell+1}\}$, and $q\in \{1,\ldots,d_{\ell}\},$ if $V_\ell=(\mathbf{1}(a=p, b=q))_{a,b}$ and $\bb_\ell=\mathbf{0}_{d_{\ell+1}}$, then $\mathbf{z}_{\ell}=\gamma \mathbf{e}_{q}$ with $\gamma>0$ implies $\mathbf{z}_{\ell+1}=\gamma \mathbf{e}_{p}.$ Thus, one can achieve that $\mathbf{z}_{\ell^*}=(x_i-1/2)_+\mathbf{e}_n.$ Choosing $V_{\ell^*}=(-\mathbf{1}(a=m,b=n))_{a,b}$ and $\mathbf{b}_{\ell^*}=\tfrac 12 \mathbf{e}_m,$ we then have that $\mathbf{z}_{\ell^*+1}=(1/2-(x_i-1/2)_+)_+\mathbf{e}_m.$ By arguing for $\ell>\ell^*+1$ as for the layers $\ell=1,\ldots, \ell^*,$ one can then show that the ICNN output is $(1/2-(x_i-1/2)_+)=\min(1/2, 1-x_i)$ for all $\bx \in [0,1]^d.$ 
        \end{proof}
        The extension to the log-sum-exp activation function follows from the sandwich formula \eqref{eq.max_approx}.}

The composition of two convex functions is not necessarily again a convex function. The same holds for the composition of two HyCNNs. The next result reveals that specific compositions of HyCNNs can still be represented by deeper HyCNNs. 

\begin{lemma}[Composition of two HyCNNs]
\label{lemma_hycnn_comp}
    Let $\tau \geq 0$. Consider a HyCNN $g \colon \RR^{d} \to \RR^{p} $ with $L_1$ hidden layers and $(d_1, \dots, d_{L_1})$ neurons per layer,  specified for $\ell \in \{0, \dots, L_1-1\}$ by
\begin{align*}
    g(\bx) \coloneqq V_{L_1}\bz_{L_1} + W_{L_1} \bx + \bb_{L_1}, \qquad \bz_{\ell+1} &\coloneqq \bz_{\ell+1}(\bx) = \overline \sigma_\tau \left((V_\ell^{k} \bz_{\ell} + W_{\ell}^{k} \bx + \bb_{\ell}^{k})_{k\in \{1,2\}}\right)
\end{align*}
with $\bz_0 \coloneqq \mathbf{0}_d$.
    In addition, consider another HyCNN $h \colon \RR^{p+d} \to \RR^{q} $ with $L_2$ hidden layers and $(p_{1}, \dots, p_{L_2})$ neurons per layer,  specified for $\ell \in \{0, \dots, L_2-1\}$ by 
\begin{align*}
    h(\by,\bx) &\coloneqq H_{L_2}\bu_{L_2} + M_{L_2} \by + N_{L_2} \bx + \bc_{L_2}, \\ \bu_{\ell+1} &\coloneqq \bu_{\ell+1}(\by,\bx) = \overline \sigma_\tau \left((H_\ell^{k} \bu_{\ell} + M_\ell^{k} \by + N_\ell^{k} \bx + \bc_{\ell}^{k})_{k\in \{1,2\}}\right)
\end{align*}
with $\bu_0 \coloneqq \mathbf{0}_p$.
If the matrices $\{M_{\ell}^{1}, M_{\ell}^{2}\}_{\ell =0}^{L_2 - 1}$ and
$M_{L_2}$ are all componentwise non-negative, 
then there exists a HyCNN $f \colon \RR^{d} \to \RR^{q} $ with {\reb $\overline \sigma(a_1,a_2) = \overline \sigma_\tau(a_1, a_2)$}, $L_1+L_2$ hidden layers and $(d_1, \dots, d_{L_1}, p_{1}+p, \ldots, p_{L_2}+p)$ neurons per layer such that $f(\bx) = h(g(\bx),\bx)$ for all $\bx\in \RR^d$.
\end{lemma}
\begin{proof}
    We construct the HyCNN $f$ with weight matrices $\{R_\ell^k\}$, skip connection matrices $\{S_\ell^k\}$ and bias vectors $\{\bd_\ell^k\}$ as follows. The first $L_1$ hidden layers exactly match those of $g$, i.e., for $\ell \in \{0, \dots, L_1-1\}$ and $k \in \{1,2\}$, we set $R_\ell^k = V_\ell^k$, $S_\ell^k = W_\ell^k$ and $\bd_\ell^k = \bb_\ell^k$. Choosing the matrices and bias vectors identical  ensures that the output of the $L_1$-th hidden layer of $f$ is exactly $\bz_{L_1}$.

    For the $(L_1+\ell)$-th hidden layer with $\ell\in \{1, \dots, L_2\}$, we consider $p_{\ell}+p$ neurons, where the first $p_\ell$ neurons are used to compute the $\ell$-th hidden layer of $h$ and the remaining $p$ neurons are used to propagate the output of $g$ to the $\ell$-th hidden layer of $h$. Concretely, for
    $k \in \{1,2\}$,
    we set 
    \[R_{L_1}^k = \begin{bmatrix}
         M_0^{k} V_{L_1}\\
        V_{L_1}
        \end{bmatrix}, \qquad
    S_{L_1}^k = \begin{bmatrix}
        M_0^{k} W_{L_1} + N_0^{k}   \\ W_{L_1}
         \end{bmatrix} , \qquad 
    \bd_{L_1}^k = 
    \begin{bmatrix}
        M_0^{k} \bb_{L_1} + \bc_0^{k} \\
        \bb_{L_1} - (\tau \log 2) \,  \mathbf{1}_{p}
    \end{bmatrix},\]
    which yields
    \begin{align*}
        R_{L_1}^k \bz_{L_1}
        + S_{L_1}^k \bx
        + \bd_{L_1}^k
    &= \begin{bmatrix}
    M_0^{k} V_{L_1} \bz_{L_1} + M_0^{k} W_{L_1} \bx + N_0^{k} \bx + M_0^{k} \bb_{L_1} + \bc_0^{k} \\
    V_{L_1} \bz_{L_1} + W_{L_1} \bx +  \bb_{L_1} - (\tau \log 2) \,  \mathbf{1}_{p}
    \end{bmatrix}\\
    &=
    \begin{bmatrix}
       M_0^{k} g(\bx) + N_0^{k} \bx + \bc_0^{k}\\
       g(\bx) - (\tau \log 2) \,  \mathbf{1}_{p}
    \end{bmatrix}.
    \end{align*}
    Using $\overline \sigma_\tau(a-\tau \log 2,a-\tau \log 2) = a$, we have 
\[
    \overline{\sigma}_\tau\left(
        \left(
            R_{L_1}^k \bz_{L_1}
            + S_{L_1}^k \bx
            + \bd_{L_1}^k
        \right)_{k \in \{1,2\}}
    \right)
    =
    \begin{bmatrix}
        \overline{\sigma}_\tau\left(
            \left(
                M_0^k g(\bx) + N_0^k \bx + \bc_0^k
            \right)_{k \in \{1,2\}}
        \right) \\
        g(\bx)
    \end{bmatrix}
    =
    \begin{bmatrix}
        \bu_1\big(g(\bx),\bx\big) \\
        g(\bx)
    \end{bmatrix}.
\]   
    Then, for $\ell \in \{1, \dots, L_2-1\}$ and $k \in \{1,2\}$, we set weight matrices, skip connection matrices and bias vectors as
    \begin{align*}
        R_{L_1+\ell}^k = \begin{bmatrix}
        H_\ell^k & M_\ell^{k}\\
        \mathbf{0}_{p \times p_\ell} & I_p
        \end{bmatrix}, \qquad S_{L_1+\ell}^k =  \begin{bmatrix}
        N_\ell^{k} \\ \mathbf{0}_{p \times d}
         \end{bmatrix}
        , \qquad \bd_{L_1+\ell}^k = \begin{bmatrix}
        \bc_\ell^k\\
        - (\tau \log 2) \,  \mathbf{1}_{p}
         \end{bmatrix},
    \end{align*}
    which yields 
    \begin{align*}
    &\overline{\sigma}_\tau\left(
        \left(
            R_{L_1+\ell}^k
            \begin{bmatrix}
                \bu_\ell\big(g(\bx),\bx\big) \\
                g(\bx)
            \end{bmatrix}
            + S_{L_1+\ell}^k \bx
            + \bd_{L_1+\ell}^k
        \right)_{k \in \{1,2\}}
    \right) \\
    & =
    \overline{\sigma}_\tau\left(
        \left(
            \begin{bmatrix}
                H_\ell^k \bu_\ell\big(g(\bx),\bx\big)
                + M_\ell^k g(\bx) + N_\ell^k \bx + \bc_\ell^k \\
                g(\bx) - (\tau \log 2)\mathbf{1}_p
            \end{bmatrix}
        \right)_{k \in \{1,2\}}
    \right)\\
    & =
    \begin{bmatrix}
        \bu_{\ell+1}\big(g(\bx),\bx\big) \\
        g(\bx)
    \end{bmatrix}.
\end{align*}
    Finally, for $\ell = L_1 + L_2$, we set
    \begin{align*}
        R_{L_1+L_2} = \begin{bmatrix}
        H_{L_2} & M_{L_2}
        \end{bmatrix}, \qquad S_{L_1+L_2} = 
            N_{L_2}, \qquad \bd_{L_1+L_2} = 
        \bc_{L_2}
         .
    \end{align*}
By construction, the last hidden layer of $f$ outputs
\[
    \begin{bmatrix}
        \bu_{L_2}\big(g(\bx),\bx\big) \\
        g(\bx)
    \end{bmatrix}.
\]
Hence,
\[
    f(\bx)
    = H_{L_2}\bu_{L_2}\big(g(\bx),\bx\big)
      + M_{L_2}g(\bx) + N_{L_2}\bx + \bc_{L_2}
    = h\big(g(\bx),\bx\big).
\]
    In addition, since the matrices $\{R_{\ell}^{1}, R_{\ell}^{2}\}_{\ell =0}^{L_1 + L_2 - 1}$ and
$R_{L_1 + L_2}$ are all componentwise non-negative, we obtain the assertion.
\end{proof}

\subsection{Proofs for limited expressivity of ICNNs}

\begin{proof}[Proof of \Cref{prop:piecewise_linear_icnn}]
Consider an ICNN architecture with $L$ hidden layers, $d_\ell$ neurons in layer $\ell$, and activation function given either by the ReLU $\sigma(x)= \max(x,0)$ or the leaky ReLU  $\sigma(x) = \max(x,\alpha x)$ with $\alpha \in (0,1)$. For $\ell\in \{0, \dots, L\}$ and $j\in [d_\ell],$ we denote  by $W^{j}_\ell$, $V^j_\ell$, and $b^j_\ell$ the $j$-th row of the respective weight matrices $W_\ell$, $V_\ell$, and the bias $\bb_\ell$.

For $\ell = 1$, each hidden neuron $z_{1,j} = \sigma(W_0^{j} x + b_0^{j})$ for $j \in [d_1]$ can be viewed as a piecewise-affine function in $x$ with at most one kink at the location where $x\mapsto W_0^{j} x + b_0^{j}$ crosses zero.  Denote by $\calK_1$ the set of kink locations of the first layer. 

For $\ell = 2$, each function $x\mapsto z_{2,j} = \sigma(V_1^{j} \bz_1 + W_1^{j} x + b_1^{j})$ for $j \in [d_2]$ is also a piecewise-affine function of the input $x$ with kink locations being a subset of the kink locations of the previous layer and the locations of the roots of $x\mapsto V_1^{h} \bz_1 + W_1^{h} x + b_1^{h}.$ Since all elements of $V_1^{h}$ are non-negative and $\bz_{1}$ is convex as a function of the input $x$, it follows that $x\mapsto V_1^{h} \bz_1 + W_1^{h} x + b_1^{h}$ is again a  convex function and therefore has at most two roots. Thus, the set of kink locations of the second layer $\calK_2$ satisfies $|\calK_2| \leq |\calK_1| + 2 d_2 \leq d_1 + 2 d_2$. 

Likewise, by induction, it follows that the set of kink locations of the $\ell$-th hidden layer $\calK_\ell$ satisfies $|\calK_\ell| \leq |\calK_{\ell-1}| + 2 d_\ell \leq d_1 + 2 d_2 + \dots + 2 d_\ell$. Finally, for the output layer, note that $f(x) = V_{L} \bz_{L} + W_{L} x + b_{L}$ is piecewise-affine with kink locations included in the kink locations of the last hidden layer. Thus, the set of kink locations of $f$ is $\subseteq \calK_{L}$ and by the induction argument, $|\calK_L| \leq d_1 + 2 d_2 + \dots + 2 d_L$.
\end{proof}

The proof of \Cref{thm:lowerBoundPiecewiseLinear} is based on the following bound on approximation of quadratic functions by affine functions.

\begin{lemma}\label{lem:bestline_lp}
    For any real numbers $a<b$, 
\begin{align*}
\inf_{\ell\ \text{\em affine}}\ \int_a^b \big|x^2-\ell(x)\big| \, dx &= \frac{(b-a)^{3}}{16},\\
\inf_{\ell\ \text{\em affine}}\ \sup_{x\in[a,b]}|x^2-\ell(x)| &= \frac{(b-a)^2}{8}.
\end{align*}
\end{lemma}

The second claim of the previous result has been previously shown in Lemma 4.5 of \citet{eckle2019comparison}. For the sake of completeness, we include a proof below. 

\begin{proof}
Substitute $t:=(x-m)/h \in [-1,1]$ with $m=(a+b)/2$ and $h=(b-a)/2.$ There exists a one-to-one relationship between lines $\ell(x)=\alpha x+\beta$ and real parameters $c,d$ such that $x^2-\ell(x)=h^2\bigl(t^2-ct-d\bigr).$ Hence,
\begin{align}\label{eq:reduce}
\inf_{\ell\ \text{\em affine}}\ \int_a^b \big|x^2-\ell(x)\big| \, dx =
h^{3} \inf_{c,d}\ \int_{-1}^1 |t^2-ct-d| \, dt.
\end{align}
For any real numbers $c,d,$ and any $t\in[-1,1]$, it holds 
\begin{align*}
\bigl|t^2-d\bigr|
=\left|\frac{(t^2-ct-d)+(t^2+ct-d)}{2}\right|
\le \frac{|t^2-ct-d|+|t^2+ct-d|}{2}.
\end{align*}
Consequently, it follows
\begin{align*}\int_{-1}^1 |t^2-ct-d| \, dt  &= \int_{0}^1 \left(|t^2-ct-d| + |t^2+ct-d| \right) \, dt\\
&\geq 2 \int_{0}^1 |t^2-d| \, dt \\
&= \int_{-1}^1 |t^2-d| \, dt.
\end{align*}
Thus, the linear term $ct$ cannot improve the $L^1$ error, the infimum in \eqref{eq:reduce} is attained for $c=0$, and
\begin{align*}
\inf_{c,d}\ \int_{-1}^1 |t^2-ct-d| \, dt
=\inf_{d}\ \int_{-1}^1 |t^2-d| \, dt
=2\inf_{d}\ \int_{0}^1 |t^2-d| \, dt = \frac{1}{2},
\end{align*}
where the equality is attained for $d=1/4.$
Combined with \eqref{eq:reduce}, we obtain 
\[
\inf_{\ell\ \text{\em affine}}\ \int_a^b \big|x^2-\ell(x)\big| \, dx =
\Big( \frac{b-a}{2} \Big)^3 \inf_{c,d}\ \int_{-1}^1 |t^2-ct-d| \, dt = \frac{(b-a)^3}{16}.
\]

To prove the second claim of the lemma, \eqref{eq:reduce} can be replaced by
\begin{equation}\label{eq:reduce2}
\inf_{\ell\ \text{\em affine}}\ \sup_{x\in[a,b]}|x^2-\ell(x)|=h^2 \inf_{c,d}\ \sup_{t\in[-1,1]}\bigl|t^2-ct-d\bigr|.
\end{equation}
Using that for any $t\in[-1,1]$,
\begin{align*}
\bigl|t^2-d\bigr|
&=\left|\frac{(t^2-ct-d)+(t^2+ct-d)}{2}\right|\\
&\le \frac{|t^2-ct-d|+|t^2+ct-d|}{2}\\
&\le \max\{|t^2-ct-d|,|t^2+ct-d|\}
\end{align*}
yields $\sup_{t\in[-1,1]}|t^2-d|
\le \sup_{t\in[-1,1]}|t^2-ct-d|.$ Thus, the linear term $ct$ cannot improve the uniform error, and the infimum in \eqref{eq:reduce2} is attained for $c=0.$ Since $t^2$ takes all values in $[0,1]$ on $[-1,1]$, we have $\sup_{t\in[-1,1]}|t^2-d|=\max\{|0-d|,|1-d|\}$, which is minimized at $d=\tfrac12$ with value $\tfrac12$. Therefore,
\begin{align*}
\inf_{\ell\ \text{\em affine}}\ \sup_{x\in[a,b]}|x^2-\ell(x)|
= h^2\cdot\frac12
= \frac{(b-a)^2}{8}.&\qedhere
\end{align*}
\end{proof}

\begin{proof}[Proof of Theorem \ref{thm:lowerBoundPiecewiseLinear}]
Let $g\in\calG_k$. Then there exists a partition $0=t_0<t_1<\cdots<t_m=1$ with $m\le k$ such that $g$ is affine on each $[t_{j-1},t_j]$. Let $L_j:=t_j-t_{j-1}$ and $L_{\max}:=\max_j L_j$. Since $\sum_j L_j=1$, we have $L_{\max}\ge 1/m\ge 1/k$.
On the interval $I=[t_{j-1},t_j]$ with $L_j=L_{\max}$ the function $g$ is affine, hence by Lemma~\ref{lem:bestline_lp}
\begin{align*}
\sup_{x\in [0,1]}|g(x)-x^2| \geq \sup_{x\in I}|g(x)-x^2|
\;\ge\; \inf_{\ell\ \text{affine}}\ \sup_{x\in I}|\ell(x)-x^2|
= \frac{L_{\max}^2}{8}
\;\ge\; \frac{1}{8k^2}.
\end{align*}

On the other hand, 
Lemma~\ref{lem:bestline_lp} and Jensen's inequality applied to $z \mapsto z^{3}$ yield
\begin{align*}
    \int_0^1 \big|g(x)-x^2\big | \, dx \geq 
    \sum_{k=1}^m
    \left(\inf_{\ell\ \text{\em affine}}\ \int_{t_{k-1}}^{t_k} \big|x^2-\ell(x)\big| \, dx \right) = \frac{1}{16} \,   \sum_{j=1}^m L_j^{3}
    \geq \frac{1}{16 m^2} \geq \frac{1}{16 k^2}.
\end{align*}
 Since these bounds hold for every $g\in \calG_k$, the assertion follows by taking the infimum over $g\in\calG_k$.
 By combining these results with Proposition~\ref{prop:piecewise_linear_icnn}, we obtain the remaining assertions. 
\end{proof}

\begin{remark}[Lower bounds for higher order monomials]
 For higher order monomials $x\in [0,1]\mapsto x^n$ with $n=3,\ldots,$ similar lower bounds as for the quadratic function hold, though the proof is more involved. E.g.,  \citet[Equation (6.17)]{mcclure1975nonlinear} asserts for $n\geq 2$ that $$\liminf_{k\to \infty} k^2\inf_{g\in \calG_k}\smash{\int_{0}^{1}}|g(x) - x^n| \, dx\geq C_n>0$$ where $\calG_k$ is the set of real-valued piecewise-affine, convex functions and $C_n$ only depends on $n$. %
\end{remark}

\subsection{Proof for approximation of quadratic functions with HyCNNs}

In this section, we provide the proof of Theorem \ref{thm:quadratic_approximation}.
{\reb We first establish the following lemma, which will be used to prove the main result.
The proof of Theorem \ref{thm:quadratic_approximation} then follows by applying Lemma~\ref{thm_tau_approx}.}

{\reb
\begin{lemma}[Quadratic Function Approximation with maxout unit]\label{thm:quadratic_approximation_max}
	For any $L\in \NN$ and even numbers $d_1, \ldots, d_L$, there exists a HyCNN $h \colon \RR\to \RR$ with {\reb $\overline \sigma(a_1,a_2) = \max(a_1,a_2)$},  $L$ hidden layers, and $(d_1, \dots, d_L)$ neurons per layer
    such that 
     \begin{align*} 
        \sup_{x\in [0,1]}|h(x) - x^2| \leq \frac{1}{8\prod_{\ell=1}^L d_{\ell}^2}.
     \end{align*}
\end{lemma}

To prove  \Cref{thm:quadratic_approximation_max}, we first establish the following lemma, writing $x_+ \coloneq \max(x,0)$.
}
 
\begin{lemma}\label{lemma_sq_approx_pc}
    For any $m=1,2,\ldots,$
    $$\sup_{x \in [0,1]} \left|\left(- \frac{1}{8 m^2} + \frac{x}{m} + \frac{2}{m}\sum_{k=1}^{m-1} \Big(x - \frac{k}{m}\Big)_+ \right) - x^2  \right|\leq \frac{1}{8m^2}.$$
\end{lemma}
\begin{proof}
    We fix $j \in [m]$. For any $x \in [\tfrac{j-1}{m}, \tfrac{j}{m}]$,
    observe that
\begin{align*}
- \frac{1}{8 m^2} + \frac{x}{m} + \frac{2}{m}\sum_{k=1}^{m-1} \Big(x - \frac{k}{m}\Big)_+
= 
    - \frac{1}{8m^2} + \frac{x}{m} + \frac{2}{m} \sum_{i=1}^{j-1} \left(x-\frac{i}{m}\right)_+
    = \frac{2j-1}{m}x - \frac{j(j-1)}{m^2} - \frac{1}{8m^2}.
\end{align*}
Hence,
$$\sup_{x \in [\frac{j-1}{m}, \frac{j}{m}]} \left|\frac{2j-1}{m}x - \frac{j(j-1)}{m^2} - \frac{1}{8m^2} - x^2 \right| = \sup_{x \in [\frac{j-1}{m}, \frac{j}{m}]}  \left|-\left(x - \frac{j-\tfrac{1}{2}}{m} \right)^2 + \frac{1}{8m^2}\right| \leq \frac{1}{8m^2}.$$
Since the inequality holds for arbitrary $j \in [m]$, we obtain the assertion.
\end{proof}

\begin{proof}[Proof of  \Cref{thm:quadratic_approximation_max}]
See Figure \ref{fig_qudratic_proof_illu} for an illustration of the proof. 
We define   
$D_{\ell} := \prod_{i=1}^{\ell} d_{i}$ for $\ell \in [L]$.
The construction of the network is done by induction with respect to the layer $\ell$ via the following three steps:
\begin{enumerate}
    \item[(i)] Let $\bz_{\ell}(x) = (z_{\ell, j}(x))_{j=0}^{d_{\ell}-1}$ be the values of the $d_{\ell}$ hidden neurons in the $\ell$-th hidden layer.
    \item[(ii)] We define three functions $u_\ell(x), \phi_\ell(x)$ and $\psi_\ell(x)$,
    which are non-negative affine combinations of $\bz_{\ell}(x)$.
    \item[(iii)] We define $\bz_{\ell+1}(x) = (z_{\ell+1, j}(x))_{j=0}^{d_{\ell+1}-1}$ via \eqref{tmp_14} and \eqref{tmp_13}. Since $u_\ell(x), \phi_\ell(x)$ and $\psi_\ell(x)$ are non-negative affine combinations of $\bz_{\ell}(x)$, $\bz_{\ell+1}(x)$ is of the form
    $\max (V_\ell^{1} \bz_{\ell}(x) + \bw_{\ell}^{1} x + \bb_{\ell}^{1}, V_\ell^{2} \bz_{\ell}(x) + \bw_{\ell}^{2} x + \bb_{\ell}^{2})$ with element-wise non-negative $V_\ell^{1}$ and $V_\ell^{2}$.
\end{enumerate}

We begin the induction with the first hidden layer $\ell=1$.
In the first hidden layer, we construct $d_1-1$ functions  $\bz_{1}(x) = (z_{1,j}(x))_{j=1}^{d_{1}-1}$ of the form
\begin{align}
z_{1,j}(x) := \left(x - \frac{j}{d_1}\right)_+  \text{ for } j \in [d_1-1]. \label{tmp_15}
\end{align}
Only $d_1 - 1$ neurons are used here.
We consider three non-negative affine combinations of $\bz_{1}(x)$:
\begin{align*}
    u_1(x) &:=  \sum_{k=1}^{d_1-1} \frac{2}{d_1} \left(x - \frac{k}{d_1}\right)_+,\\
    \phi_1(x) &:= \sum_{k=1}^{d_1-1} \big(a_1\mathds{1}(k\text{ odd}) + b\mathds{1}(k\text{ even})\big) \left(x - \frac{k}{d_1}\right)_+,\\
    \psi_1(x) &:= \sum_{k=1}^{d_1-1} \big(b\mathds{1}(k\text{ odd}) + a_1\mathds{1}(k\text{ even})\big) \left(x - \frac{k}{d_1}\right)_+
\end{align*}
with $a_1 := 4/d_1$ and $b := 4/(L D_{L})$. This proves step (ii) for $\ell=1.$

Let
$$a_{\ell+1}:= \frac{a_\ell - (d_{\ell+1}-1) b}{d_{\ell+1}}.$$
This implies that $a_\ell \geq (L+1-\ell) (\prod_{i=\ell+1}^L d_i) b\geq b$ for every $\ell \in [L]$. Indeed, this is true for $\ell=1$ as $a_1 = L (\prod_{i=2}^L d_i) b$. Moreover, if $a_\ell \geq (L+1-\ell) (\prod_{i=\ell+1}^L d_i) b$, then
$$a_{\ell+1} \geq \frac{a_\ell - d_{\ell+1} b}{d_{\ell+1}} \geq \frac{(L-\ell) (\prod_{i=\ell+1}^L d_i) b}{d_{\ell+1}} = (L-\ell) \Big(\prod_{i=\ell+2}^L d_i\Big) b.$$

We now describe the induction step $\ell \to \ell+1,$ assuming that for a given $\ell \in [L-1],$ 
\begin{align}
    u_\ell(x) &:=  \sum_{k=1}^{D_{\ell}-1} \frac{2}{D_{\ell}} \left(x - \frac{k}{D_{\ell}}\right)_+, \label{def_u_ell_0}\\
    \phi_\ell(x) &:= \sum_{k=1}^{D_{\ell}-1} \big(a_\ell\mathds{1}(k\text{ odd}) + b\mathds{1}(k\text{ even})\big) \left(x - \frac{k}{D_{\ell}}\right)_+,
    \label{def_u_ell_1} \\
    \psi_\ell(x) &:= \sum_{k=1}^{D_{\ell}-1} \big(b\mathds{1}(k\text{ odd}) + a_\ell\mathds{1}(k\text{ even})\big) \left(x - \frac{k}{D_{\ell}}\right)_+,
    \label{def_u_ell_2}
\end{align}
are non-negative affine combinations of the values of the neurons in the $\ell$-th hidden layer $\bz_{\ell}(x) = (z_{\ell, j}(x))_{j=0}^{d_{\ell}-1}.$

\begin{figure}[h!] 
  \centering  \includegraphics[width=.97\textwidth]{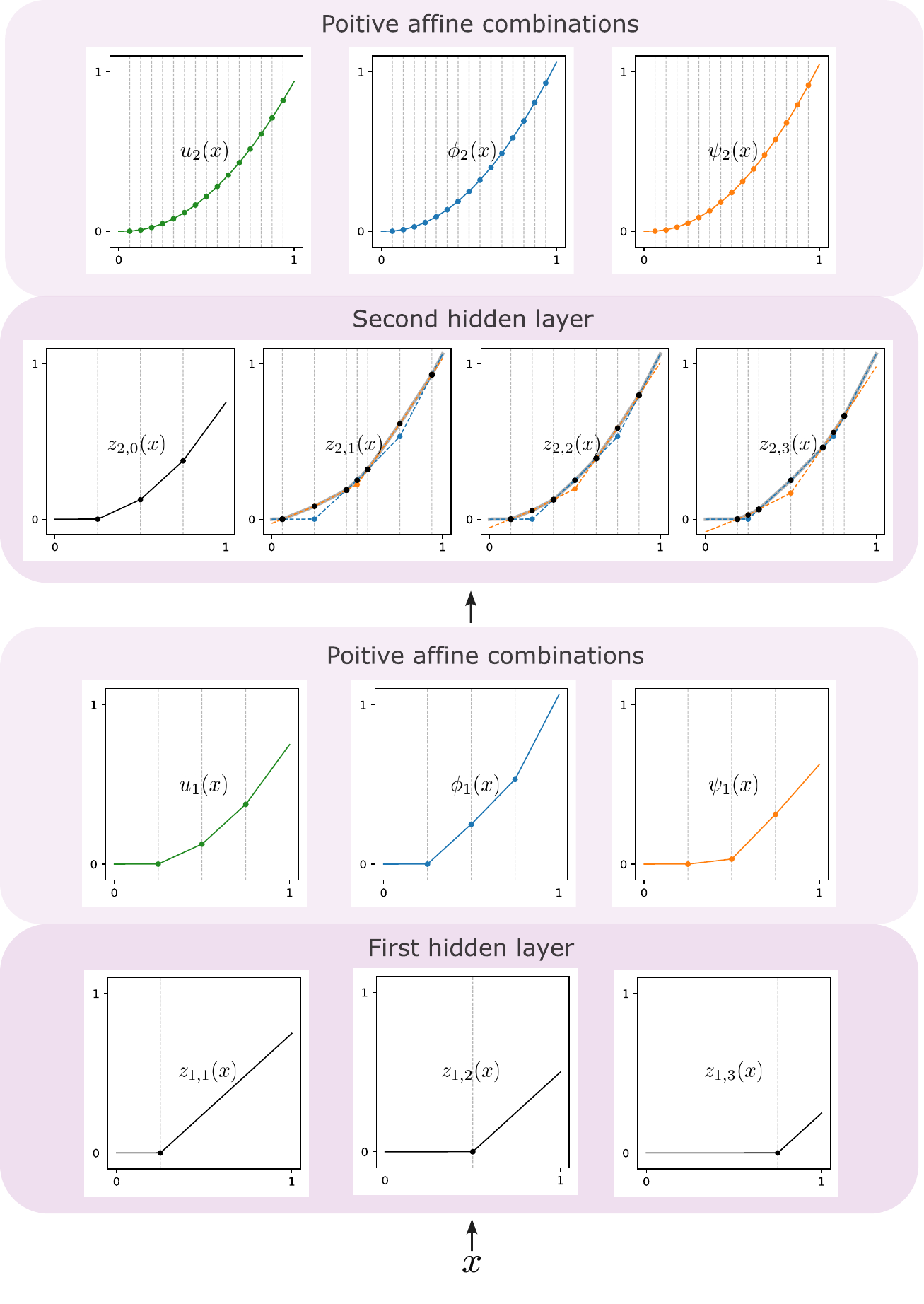}
  \caption{\textbf{Illustration of the proof of  \Cref{thm:quadratic_approximation_max}} for $L=2$ and $(d_1, d_2)=(4,4)$. 
  $u_{\ell}(x)$, $\phi_{\ell}(x)$ and $\psi_{\ell}(x)$ are defined by \eqref{def_u_ell_0}, \eqref{def_u_ell_1} and \eqref{def_u_ell_2}, respectively.  
  From \eqref{tmp_14} and \eqref{tmp_13}, 
  $z_{2,0}(x) = u_{1}(x)$ and
    $z_{2,j}(x)=\max(\phi_1(x), \psi_1(x) + \frac{7}{16} x - \frac{7j}{256})$ for $j \in \{1,2,3\}.$ Gray vertical lines are drawn at the kink locations of the functions.}
    \label{fig_qudratic_proof_illu}
\end{figure}

In the $(\ell+1)$-st hidden layer, we can construct $d_{\ell+1}$ neurons with 
\begin{align}
    z_{\ell+1, 0}(x) \coloneq \max\big( u_\ell(x), 0 \big) 
    =u_\ell(x)=  \sum_{k=1}^{D_{\ell}-1} \frac{2}{D_{\ell}} \left(x - \frac{k}{D_{\ell}}\right)_+
    \label{tmp_14}
\end{align}
and 
\begin{align}
    z_{\ell+1, j}(x)
    &\coloneq \max\left(\phi_\ell(x), \psi_\ell(x) + \frac{a_\ell - b}{2} x - \frac{(a_\ell - b)j}{2D_{\ell+1}} \right). \label{tmp_13}
\end{align}
for $j \in [d_{\ell+1}-1]$. The following two lemmas complete the induction step. Their proofs are given below.
\begin{lemma} \label{lemma_z_equal} For each $j \in \{1,\ldots, d_{\ell+1}-1\}$, $z_{\ell+1, j}(x)$ defined by \eqref{tmp_13} equals
\begin{align*} \sum_{k=1}^{D_{\ell}-1} b \,  \Big(x - \frac{k}{D_{\ell}}\Big)_++  \sum_{k=0}^{D_{\ell}-1}  \frac{a_\ell - b}{2} \left( \mathds{1}(k\text{ even}) \Big(x - \frac{k}{D_{\ell}} - \frac{j}{D_{\ell+1}}\Big)_+
+ \mathds{1}(k\text{ odd}) \Big(x - \frac{k}{D_{\ell}} - \frac{d_{\ell+1} - j}{D_{\ell+1}}\Big)_+ \right).
\end{align*}
\end{lemma}

\begin{lemma}\label{lemma_u_affine}
    If $d_{\ell+1}$ is even, then, the functions $u_{\ell+1}$, $\phi_{\ell+1}$, and $\psi_{\ell+1}$, defined by
    \begin{align*}
    u_{\ell+1}(x) &:=  \sum_{k=1}^{D_{\ell+1}-1} \frac{2}{D_{\ell+1}} \left(x - \frac{k}{D_{\ell+1}}\right)_+,\\
    \phi_{\ell+1}(x) &:= \sum_{k=1}^{D_{\ell+1}-1} \big(a_{\ell+1}\mathds{1}(k\text{ odd}) + b\mathds{1}(k\text{ even})\big) \left(x - \frac{k}{D_{\ell+1}}\right)_+,\\
    \psi_{\ell+1}(x) &:= \sum_{k=1}^{D_{\ell+1}-1} \big(b\mathds{1}(k\text{ odd}) + a_{\ell+1}\mathds{1}(k\text{ even})\big) \left(x - \frac{k}{D_{\ell+1}}\right)_+,
\end{align*}
    are non-negative affine combinations of $(z_{\ell+1, j}(x))_{j=0}^{d_{\ell+1}-1}$, defined by \eqref{tmp_14} and \eqref{tmp_13}.
\end{lemma}
Hence, by Lemma \ref{lemma_u_affine} and the induction argument, the function 
$$u_{L}(x) :=  \sum_{k=1}^{D_{L}-1} \frac{2}{D_{L}} \left(x - \frac{k}{D_{L}}\right)_+$$
is a non-negative affine combination of the last hidden layer $\bz_{L}(x)$.
The final output of the network is then given by
\begin{align}
    f(x) \coloneqq - \frac{1}{8 D_{L}^2} + \frac{x}{D_{L}} + u_{L}(x) = - \frac{1}{8 D_{L}^2} + \frac{x}{D_{L}} + \frac{2}{D_{L}}\sum_{k=1}^{D_{L}-1} \left(x - \frac{k}{D_{L}}\right)_+.
    \label{tmp_16}
\end{align}
Finally, by applying Lemma \ref{lemma_sq_approx_pc}, we have $|x^2-f(x)| \leq 1/(8 D_{L}^2)$ and obtain the assertion.
\end{proof}

\begin{remark}
\label{rem.74f}
    From \eqref{tmp_15}, \eqref{tmp_14} and \eqref{tmp_13},  we observe that not only elements of $\{V_\ell^{1}\}_{\ell=0}^{L-1}$ and $V_L$, but also
    $\{\bw_\ell^{1}\}_{\ell=0}^{L-1}$ and $w_L$ in the network are non-negative.
    This property is exploited in the approximation of monomial functions in Appendix~\ref{appendix_monomial}. 
\end{remark}

\begin{proof}[Proof of Lemma \ref{lemma_z_equal}]
    We fix $\ell \in \NN$, $j \in [d_{\ell+1}-1]$ and recall that $D_{\ell} = \prod_{i=1}^{\ell} d_{i}$ and
\begin{align*}
    \phi_\ell(x) &= \sum_{k=1}^{D_{\ell}-1} \big(a_\ell\mathds{1}(k\text{ odd}) + b\mathds{1}(k\text{ even})\big) \left(x - \frac{k}{D_{\ell}}\right)_+,\\
    \psi_\ell(x) &= \sum_{k=1}^{D_{\ell}-1} \big(b\mathds{1}(k\text{ odd}) + a_\ell\mathds{1}(k\text{ even})\big) \left(x - \frac{k}{D_{\ell}}\right)_+.
\end{align*}    
    Further, we write $\Delta \coloneq \Delta_{\ell} = a_\ell - b \geq 0$ and  
\begin{align*}
\mathcal{L}(x):= - \frac{\Delta j}{2 D_{\ell} d_{\ell+1}} + \frac{\Delta}{2}\,x , \qquad 
F(x):=\max\Bigl(\phi_\ell(x),\,\mathcal{L}(x)+\psi_\ell(x)\Bigr).
\end{align*}
We prove that on $[0,1]$
\begin{align*}
F(x)&=\sum_{k=1}^{D_{\ell}-1} b  \Big(x - \frac{k}{D_{\ell}}\Big)_+\\
&\quad + \sum_{k=0}^{D_{\ell}-1} \frac{\Delta}{2} \left( \mathds{1}(k\text{ even}) \Big(x - \frac{k}{D_{\ell}} - \frac{j}{D_{\ell+1}}\Big)_+
+ \mathds{1}(k\text{ odd}) \Big(x - \frac{k}{D_{\ell}} - \frac{d_{\ell+1} - j}{ D_{\ell+1}}\Big)_+ \right).
\end{align*}
The proof is divided into multiple steps.

\emph{Step 1: Slopes and switching locations.}
Define
\begin{align*}
    t(x) &\coloneqq \mathcal{L}(x) + \psi_\ell(x) - \phi_\ell(x) \\
    &= - \frac{\Delta j}{2 D_{\ell} d_{\ell+1}} + \frac{\Delta}{2}\,x + \sum_{k=1}^{D_{\ell}-1} \big( - \Delta \mathds{1}(k\text{ odd}) + \Delta \mathds{1}(k\text{ even})\big) \left(x - \frac{k}{D_{\ell}}\right)_+.
\end{align*}
Then, $t$ is a continuous and piecewise affine function. Moreover, on each open interval  
$I_k \coloneqq (\frac{k}{D_{\ell}},\frac{k+1}{D_{\ell}})$ for $k =0,\dots,D_{\ell}-1$,
\begin{align} \label{eq:slope-d_2}
    t'(x) = \frac{\Delta}{2} + \sum_{i=1}^k (-1)^i \Delta = (-1)^k \frac{\Delta}{2}.
\end{align}
We have
$t(0)=-\Delta j (2 D_{\ell} d_{\ell+1})^{-1}<0$. Using \eqref{eq:slope-d_2} and induction over $k = 0,1,\ldots, D_{\ell}-1$, it follows that 
\begin{align*}
t\!\left(\frac{k}{D_{\ell}}\right)
=\begin{cases}
-\frac{j}{d_{\ell+1}} \, \frac{\Delta}{2 D_{\ell} }, & k \text{ even},\\
(1-\frac{j}{d_{\ell+1}}) \frac{\Delta}{2 D_{\ell} }, & k \text{ odd}.
\end{cases}
\end{align*}
Thus on every $I_k$ the continuous piecewise affine function $t$ changes sign exactly once at
\begin{align*}
x_k:= 
\begin{cases}
\frac{k}{D_{\ell}} + \frac{j}{d_{\ell+1}} \, \frac{1}{D_{\ell}}, & k \text{ even},\\
\frac{k}{D_{\ell}} + \frac{d_{\ell+1} - j}{d_{\ell+1}} \, \frac{1}{D_{\ell}} , & k \text{ odd}.
\end{cases}
\end{align*}
Consequently, on $I_k$ we have 
\begin{align} \label{tmp_31}
    F= \begin{cases}
        \phi_\ell &\text{ on }(0,x_0)\cup (x_{1}, x_2) \cup (x_{3}, x_4) \cup \ldots \cup (x_{D_{\ell}-1}, 1)\\
        \mathcal{L}+ \psi_\ell &\text{ on } (x_0, x_1) \cup (x_2,x_3) \cup \ldots \cup (x_{D_{\ell}-2}, x_{D_{\ell}-1}).
    \end{cases}
\end{align}
In particular,
$F$ is piecewise-affine on $[0,1]$ with kinks at
\begin{align}
     \Bigl\{\tfrac{k}{D_{\ell}}: k=1,\dots,D_{\ell}-1\Bigr\} 
     \cup
    \Bigl\{x_k: k=0,\dots,D_{\ell}-1\Bigr\}.
\label{tmp_6}
\end{align}
In addition, $F(0)=0$ and $F'(0)=0.$

\emph{Step 2: Jump heights of the derivative.}
For a real-valued function $f:\RR \to \RR$ which has right and left derivatives, we define the jump height of $f'$ as 
$$[\![ f ]\!](x) := \lim_{h\downarrow 0}\frac{f(x+h)-f(x)}{h} - \lim_{h\uparrow 0}\frac{f(x+h)-f(x)}{h}.$$
At each coarse grid point $\frac{k}{D_{\ell}}$ for $k \in [D_{\ell}-1]$, we pass from 
$I_{k-1}$ to $I_k$. 
By \eqref{tmp_31}  the jump height of $F'$ at $\frac{k}{D_{\ell}} \in (x_{k-1},x_k)$ is 
\begin{align}
    [\![F]\!]\Big(\frac{k}{D_{\ell}}\Big) &= \begin{cases}
    [\![\phi_\ell]\!](\frac{k}{D_{\ell}}) = b, & k \text{ even},\\
    [\![\mathcal{L} + \psi_\ell]\!](\frac{k}{D_{\ell}}) = b, & k \text{ odd}.
    \end{cases}\label{tmp_7}
\end{align}
At each kink $x_k$ the active branch switches, and by \eqref{eq:slope-d_2}
the jump height of $F'$ equals
\begin{align}
[\![F]\!](x_k)&=(-1)^k \bigl(( \mathcal{L} +\psi_\ell)'-(\phi_\ell)'\bigr)\big|_{x_k}=\frac{\Delta}{2}. \label{tmp_8}
\end{align}
Since every jump height of $F'$ on $[0,1]$ is non-negative, $F$ is convex.

\emph{Step 3: Hinge representation.}
Let $G$ be convex, piecewise-affine on $[0,1]$ with $G(0)=G'(0)=0$, and let
$\{\delta_m\}$ be its kink locations with jump height $\alpha_m > 0$ of $G'$ at $\delta_m$.
Then $G'(x)=\sum_m \alpha_m \,\1(x>\delta_m)$ on $x \in [0,1] \setminus \{\delta_m\}$, hence
\begin{align*}
G(x)=\int_0^x G'(u)\,du=\sum_m \alpha_m\,\pp{x-\delta_m}.
\end{align*}
Applying this to $G=F$ with kink locations obtained in \eqref{tmp_6} and jump heights
obtained in \eqref{tmp_7} and \eqref{tmp_8} finally yields
\begin{align*}
F(x)=\sum_{k=1}^{D_{\ell}-1} b  \Big(x - \frac{k}{D_{\ell}}\Big)_+
+ \sum_{k=0}^{D_{\ell}-1} \frac{\Delta}{2} \left( \mathds{1}(k\text{ even}) \Big(x - \frac{k}{D_{\ell}} - \frac{j}{ D_{\ell+1}}\Big)
+ \mathds{1}(k\text{ odd}) \Big(x - \frac{k}{D_{\ell}} - \frac{d_{\ell+1} - j}{ D_{\ell+1}}\Big) \right).
\end{align*}
\end{proof}

\begin{proof}[Proof of Lemma \ref{lemma_u_affine}]
Lemma \ref{lemma_z_equal} states that
\begin{align*}
    z_{\ell+1, j}(x) =& \sum_{k=1}^{D_{\ell}-1} b  \Big(x - \frac{k}{D_{\ell}}\Big)_+ \\
&+ \sum_{k=0}^{D_{\ell}-1} \frac{\Delta_{\ell}}{2} \left( \mathds{1}(k\text{ even}) \Big(x - \frac{k}{D_{\ell}} - \frac{j}{D_{\ell+1}}\Big)_+
+ \mathds{1}(k\text{ odd}) \Big(x - \frac{k}{D_{\ell}} - \frac{d_{\ell+1} - j}{D_{\ell+1}}\Big)_+ \right).
\end{align*}
We define constants $r_1, \ldots, r_{d_{\ell+1}-1}$ as 
\begin{align}
r_j := \frac{2a_{\ell+1}}{\Delta_{\ell}} \mathds{1}(j\text{ odd}) + \frac{2b}{\Delta_{\ell}} \mathds{1}(j\text{ even}).
\label{tmp_10}
\end{align}
By assumption $d_{\ell+1}$ is even, $a_{\ell+1}= (a_\ell - (d_{\ell+1}-1) b)/d_{\ell+1},$ and therefore
\begin{align}
    \sum_{j=1}^{d_{\ell+1}-1} r_j
&= \frac{d_{\ell+1}}{2} \, \frac{2a_{\ell+1}}{\Delta_{\ell}} + \frac{d_{\ell+1}-2}{2} \, \frac{2b}{\Delta_{\ell}}= \frac{a_\ell - (d_{\ell+1}-1) b + (d_{\ell+1}-2)b}{\Delta_{\ell}} =1. \label{tmp_9}
\end{align}
It follows
\begin{align*}
    \sum_{j=1}^{d_{\ell+1}-1} r_j \, z_{\ell+1}^{(j)}(x) 
    &= 
    \sum_{k=1}^{D_{\ell}-1} b  \Big(x - \frac{k}{D_{\ell}}\Big)_+\\
    & \quad  + \sum_{k=0}^{D_{\ell}-1} \sum_{j=1}^{d_{\ell+1}-1}
    \frac{r_j \Delta_{\ell}}{2} 
    \left( \mathds{1}(k\text{ even}) \Big(x - \frac{k}{D_{\ell}} - \frac{j}{D_{\ell+1}}\Big)_+
+ \mathds{1}(k\text{ odd}) \Big(x - \frac{k}{D_{\ell}} - \frac{d_{\ell+1} - j}{D_{\ell+1}}\Big)_+ \right)\\
&= \sum_{k=1}^{D_{\ell}-1} b  \Big(x - \frac{k}{D_{\ell}}\Big)_+
+ \sum_{k=0}^{D_{\ell}-1} \sum_{j=1}^{d_{\ell+1}-1}
    \frac{r_j \Delta_{\ell}}{2} 
    \Big(x - \frac{k}{D_{\ell}} - \frac{j}{D_{\ell+1}}\Big)_+\\
&= \sum_{k=1}^{D_{\ell}-1} b  \Big(x - \frac{k}{D_{\ell}}\Big)_+
+ \sum_{k=0}^{D_{\ell}-1} \sum_{j=1}^{d_{\ell+1}-1}
    \Big( a_{\ell+1} \mathds{1}(j\text{ odd}) + b \mathds{1}(j\text{ even}) \Big) 
    \Big(x - \frac{k}{D_{\ell}} - \frac{j}{D_{\ell+1}}\Big)_+\\
&= \sum_{k=1}^{D_{\ell}-1} b  \Big(x - \frac{k d_{\ell+1}}{D_{\ell+1}}\Big)_+
+ \sum_{k=0}^{D_{\ell}-1} \sum_{j=1}^{d_{\ell+1}-1}
    \Big( a_{\ell+1} \mathds{1}(j\text{ odd}) + b \mathds{1}(j\text{ even}) \Big) 
    \Big(x - \frac{k d_{\ell+1} + j}{D_{\ell+1}}\Big)_+\\
    & = \sum^{D_{\ell+1}-1}_{q=1} \Big( a_{\ell+1} \mathds{1}(q\text{ odd}) + b \mathds{1}(q\text{ even}) \Big) \Big(x - \frac{q}{D_{\ell+1}}\Big)_+ \\
    &=:u^{(1)}_{\ell+1}(x).
\end{align*}
We used \eqref{tmp_9} and \eqref{tmp_10} for the first and third equalities, the second equality uses that $d_{\ell+1}$ is even, and the fourth equality expands $q=kd_{\ell+1}+j.$ Since $d_{\ell+1}$ is even, $q$ is even if and only if $j$ is even. Hence $u^{(1)}_{\ell+1}(x)$ is a non-negative affine combination of $(z^{(j)}_{\ell+1}(x))_{j=0}^{d_{\ell+1}-1}$.

Similarly, define $$s_0 := \frac{D_{\ell}}{2} \,\Big(a_{\ell+1} - b \sum_{j=1}^{d_{\ell+1}-1} s_j \Big), \quad \text{and} \quad s_j := \frac{2b}{\Delta_{\ell}}  \mathds{1}(j\text{ odd}) + \frac{2a_{\ell+1}}{\Delta_{\ell}}  \mathds{1}(j\text{ even}),\quad j=1,\ldots,d_{\ell+1}-1.$$
From  $\sum_{j=1}^{d_{\ell+1}-1} s_j \leq \sum_{j=1}^{d_{\ell+1}-1} r_j =1$, we get $s_0 \geq 0$.
Moreover,
\begin{align}
    \frac{2s_0}{D_{\ell}} + b \sum_{j=1}^{d_{\ell+1}-1} s_j
&= a_{\ell+1}, \label{tmp_11}
\end{align}
and for $j \in \{1,\ldots, d_{\ell+1}-1\}$,
\begin{align}
\frac{s_j \Delta_{\ell}}{2}  &=   
b \mathds{1}(j\text{ odd}) +  a_{\ell+1}  \mathds{1}(j\text{ even}). \label{tmp_12}
\end{align}
Using arguments as for the rewriting of $\sum_{j=1}^{d_{\ell+1}-1} r_j \, z_{\ell+1}^{(j)}(x)$ above, 
\begin{align*}
    \sum_{j=0}^{d_{\ell+1}-1} s_j \, z_{\ell+1}^{(j)}(x) 
    &= \frac{2s_0}{D_{\ell}} \sum_{k=1}^{D_{\ell}-1}\left(x - \frac{k}{D_{\ell}}\right)_+ + \Big( \sum_{j=1}^{d_{\ell+1}-1} s_j \Big) 
    \sum_{k=1}^{D_{\ell}-1} b  \Big(x - \frac{k}{D_{\ell}}\Big)_+\\
    & \quad  + \sum_{k=0}^{D_{\ell}-1} \sum_{j=1}^{d_{\ell+1}-1}
    \frac{s_j \Delta_{\ell}}{2}
    \left( \mathds{1}(k\text{ even}) \Big(x - \frac{k}{D_{\ell}} - \frac{j}{D_{\ell+1}}\Big)_+
+ \mathds{1}(k\text{ odd}) \Big(x - \frac{k}{D_{\ell}} - \frac{d_{\ell+1} - j}{D_{\ell+1}}\Big)_+ \right)\\
&= \sum_{k=1}^{D_{\ell}-1} a_{\ell+1}  \Big(x - \frac{k}{D_{\ell}}\Big)_+
+ \sum_{k=0}^{D_{\ell}-1} \sum_{j=1}^{d_{\ell+1}-1}
    \frac{s_j \Delta_{\ell}}{2}
    \Big(x - \frac{k}{D_{\ell}} - \frac{j}{D_{\ell+1}}\Big)_+\\
&= \sum_{k=1}^{D_{\ell}-1} a_{\ell+1}  \Big(x - \frac{k}{D_{\ell}}\Big)_+
+ \sum_{k=0}^{D_{\ell}-1} \sum_{j=1}^{d_{\ell+1}-1}
    \Big( b \mathds{1}(j\text{ odd}) + a_{\ell+1} \mathds{1}(j\text{ even}) \Big) 
    \Big(x - \frac{k}{D_{\ell}} - \frac{j}{D_{\ell+1}}\Big)_+\\
    & = \sum^{D_{\ell+1}-1}_{q=1} \Big( b \mathds{1}(q\text{ odd}) + a_{\ell+1} \mathds{1}(q\text{ even}) \Big) \Big(x - \frac{q}{D_{\ell+1}}\Big)_+ \\
    =: u^{(2)}_{\ell+1}(x)
\end{align*}
where we used \eqref{tmp_11} and \eqref{tmp_12} for the second and third equalities, respectively. Hence, $u^{(2)}_{\ell+1}(x)$ is also a non-negative affine combination of $(z^{(j)}_{\ell+1}(x))_{j=0}^{d_{\ell+1}-1}$.
Finally, we get the remaining assertion by noticing that 
$$u^{(0)}_{\ell+1}(x) = \sum_{k=1}^{D_{\ell+1}-1} \frac{2}{D_{\ell+1}} \left(x - \frac{k}{D_{\ell+1}}\right)_+ = \frac{2}{D_{\ell+1}(a_{\ell+1} + b)} \left(u^{(1)}_{\ell+1}(x) + u^{(2)}_{\ell+1}(x)\right)$$
can also be represented as a non-negative affine combination of $(z^{(j)}_{\ell+1}(x))_{j=0}^{d_{\ell+1}-1}$.
\end{proof}

We also present an alternative theorem and proof below, which is more accessible.
The network construction below is similar to that of the original proof for ReLU networks (e.g., \citet{yarotski2017}),
and highlights the key idea behind the exponential approximation rate.

\begin{lemma} \label{lemma_comp_multiple_bump}
Let $\psi:[0,1]\to [0,1]$ be the function $\psi(u)=|2u-1|$ and define $\psi^{\circ k}=\psi \circ \ldots \circ \psi$ as the $k$-fold composition of $\psi.$ Then, 
\begin{itemize}
    \item[(i)] $\psi^{\circ k}+(1-\psi)^{\circ k}=1,$
    \item[(ii)] for a real number $u$ denote by $\{u\} \coloneqq u - \lfloor u \rfloor \in [0,1)$ the fractional part. For any $y\in [0,1],$ and any $k=1,2,\ldots,$ $$\psi^{\circ k}(y)=\psi\big(\{2^{k-1}y\}\big),$$ 
    \item[(iii)] for any $L=1,2,\ldots,$ the function $$x\mapsto \sum_{k=1}^L 4^{-k} \big(1-\psi^{\circ k}(x)\big)$$ interpolates $x\mapsto x(1-x)$ at the point $\ell 2^{-L}$ for any $\ell=0,\ldots, 2^L,$
    and is affine on the interval $[\ell 2^{-L}, (\ell +1 )2^{-L}]$ for any $\ell = 0,\ldots, 2^L - 1$,
    \item[(iv)] for any $L=1,2,\ldots$
\begin{align*}
\sup_{x\in [0,1]} \, \Big|x+\sum_{k=1}^L 4^{-k}\big(\psi^{\circ k}(x)-1\big) - 2^{-2L-3} 
-x^2\Big|= 2^{-2L-3}. 
\end{align*}
\end{itemize}
\end{lemma}

Before we give the proof of the lemma, we state the main result. 

\begin{theorem}[Special case of  \Cref{thm:quadratic_approximation_max}]
    For any $L\in \NN$, there exists a HyCNN $h \colon \RR\to \RR$ with $L$ hidden layers and $2$ neurons per layer,  
    such that 
     \begin{align*}
       \int_{0}^1|h(x)- x^2| \, dx\leq  \sup_{x\in [0,1]}|h(x) - x^2| = 2^{-2L-3}.
     \end{align*}
\end{theorem}
\begin{proof}
For real numbers $u,v,$ we have $2\max(u,v)=u+v+|u-v|.$ Therefore, a layer of the form
\[\bz\mapsto  %
    \max\big(V^{1} \bz + W^{1} \bx + \bb^{1}, V^{2} \bz + W^{2} \bx + \bb^{2}\big)\]
with $V^1, V^2$ entrywise non-negative, is equivalent to the existence of matrices $A, B, D, E$ and vectors $\ba, \bb$ such that the absolute value of every entry in $D$ is upper bounded by the corresponding entry of $A$ (implying that $A-D$ and $A+D$ are entrywise non-negative) and 
\begin{align}
    \bz\mapsto 
A\bz +B\bx +\ba + \big|D\bz+E\bx+\bb\big|,
\label{eq.layer_reform}
\end{align}
where $|\cdot|$ is the componentwise absolute value. Let $\psi:[0,1]\to [0,1]$ be the function $\psi(u)=|2u-1|$ and define $\psi^{\circ k}=\psi \circ \ldots \circ \psi$ as the $k$-fold composition of $\psi.$ For any $m=1,2,\ldots$ we now show that there exists a network with $L$ hidden layers, one input neuron, one output neuron, and two neurons in each of the hidden layers, 
such that the $m$-th hidden layer computes 
\[
\left(
\begin{array}{c}
2^{-2m}\psi^{\circ m}(x)+\sum_{k=1}^m 2^{-2k}\psi^{\circ k}(x)    \\
-2^{-2m}\psi^{\circ m}(x)+\sum_{k=1}^m 2^{-2k}\psi^{\circ k}(x)
\end{array}\right)
=
\left(
\begin{array}{c}
2^{-2m}\psi^{\circ m}(x)+\sum_{k=1}^m 2^{-2k}\psi^{\circ k}(x)    \\
\sum_{k=1}^{m-1} 2^{-2k}\psi^{\circ k}(x)
\end{array}\right).
\]
Then, we can average the two activations in the output layer and add the term $x-\sum_{k=1}^L 2^{-2k}  - 2^{-2L-3}$ to get a network computing 
$$x+\sum_{k=1}^L 2^{-2k}\big(\psi^{\circ k}(x)-1\big) - 2^{-2L-3},$$
and get the assertion by applying Lemma~\ref{lemma_comp_multiple_bump} (iv).

The case $m=1$ follows directly from the form \eqref{eq.layer_reform} with $A=B=\ba=D=(0,0)^\top,$ $E=(1/2,0)^\top,$ $\bb=(-1/4,0)^\top.$ For the induction step from $m$ to $m+1$, take $B=E=\ba=(0,0)^\top,$ 
\[
A= \left(
\begin{array}{ccc}
  1/2   & 1/2  \\
 1/2   & 1/2 
\end{array}\right), \quad  
D= 
\left(
\begin{array}{cc}
1/2   & -1/2  \\
 0  &  0
\end{array}\right)
, 
\quad  \bb = \left(\begin{array}{c}
-2^{-2m-1} \\
     0 
\end{array}\right).
\]
The choice satisfies the constraint that the entries in $D$ are in absolute value dominated by the corresponding entries in $A.$ Observe that
$$
2^{-2m-1}\psi^{\circ (m+1)}(x)
= 2^{-2m-1} |2\psi^{\circ m}(x) -1|
= |2^{-2m} \psi^{\circ m}(x) -2^{-2m-1} |. 
$$
Let $\bz_m$ denote the activations of the previous layer. Using the induction hypothesis, $D\bz_m=(2^{-2m}\psi^{\circ m}(x),0)^\top$ and $$A\bz_m+|D\bz_m+\bb|= \sum_{k=1}^m 2^{-2k}\psi^{\circ k}(x)\left(
\begin{array}{c}
     1  \\
     1 
\end{array}
\right)+  2^{-2m-2}\psi^{\circ (m+1)}(x)
\left(
\begin{array}{c}
     2  \\
     0 
\end{array}
\right)$$ proving the induction hypothesis.

\end{proof}

\begin{proof}[Proof of Lemma~\ref{lemma_comp_multiple_bump}] \textit{(i):}
We prove $\psi^{\circ k}+(1-\psi)^{\circ k}=1$ using induction on $k.$ The formula holds for $k=1.$ To go from $k$ to $k+1,$ observe that by $\psi^{\circ (k+1)}=\psi^{\circ k}\circ \psi$, $(1-\psi)^{\circ (k+1)}=(1-\psi)^{\circ k}\circ (1-\psi)$, the induction hypothesis, and the fact that $\psi$ is symmetric around $1/2$ in the sense that $\psi(u)=\psi(1-u),$ thus, $\psi\circ \psi=\psi\circ (1-\psi),$ we obtain $\psi^{\circ (k+1)}+(1-\psi)^{\circ (k+1)}= \psi^{\circ k}\circ \psi + 1-  \psi^{\circ k}\circ (1-\psi)=1,$ proving the induction step. 

\textit{(ii):} The formula holds trivially for $k=1.$ 

We now prove it for $k=2:$ Note that $y=\{2y\}/2$ if $0\leq y< 1/2$ and $y=1/2+\{2y\}/2$ if $1/2\leq y< 1.$ Therefore, $\psi(y)= |2 y - 1 | = |\{2y\}-1| = 1 - \{2y\} $ if $0\leq y<1/2$ and $\psi(y)= |2 y - 1 | = \{2y\}$ if $1/2\leq y <1.$
Since $\psi(u)=\psi(1-u)$ for all real numbers $u,$ we find $\psi\circ \psi(y)=\psi(\{2y\})$ for all $0\leq y<1.$ For $y=1$, the formula holds since $\psi(0)=\psi(1)=1.$ 

To prove $\psi^{\circ k}(y)=\psi(\{2^{k-1}y\})$ for general $k$, it is enough to show the induction step $k\to k+1.$ Using the induction hypothesis and the case $k=2$ proved above, $\psi^{\circ (k+1)}(y)=\psi(\psi^{\circ k}(y))=\psi\circ \psi(\{2^{k-1}y\})=\psi(\{2\{2^{k-1}y\}\})=\psi(\{2^ky\}).$

\textit{(iii):} By (ii), $\psi^{\circ k}(y)=\psi(\{2^{k-1}y\}).$ Since $\psi(x)=|2x-1|$ is affine on $[0,1/2]$ and $[1/2,1],$ it follows that  $\psi^{\circ k}$ is piecewise-affine on the intervals $[\ell/2^k,(\ell+1)/2^k]$ with endpoints $\psi^{\circ k}(\ell/2^k) = \psi(1/2) =0$ if $\ell$ is odd and $\psi^{\circ k}(\ell/2^k) = \psi(0)= 1$ if $\ell$ is even. 

For convenience, write $g(x) = x(1-x).$ In the next step, we show that for any $L \geq 1$ and any $\ell =0,1, \ldots, 2^L,$
\begin{align*}
	g(\ell 2^{-L}) = \sum_{k=1}^L 4^{-k} \big(1-\psi^{\circ k}(\ell 2^{-L})\big).
\end{align*}
We prove this by induction over $L.$ For $L=1$ the result can be checked directly. For the inductive step, suppose that the claim holds for $L.$ If $\ell$ is even we use that $\psi^{\circ (L+1)}(\ell 2^{-L-1})=1$ to obtain that $g(\ell 2^{-L-1}) = \sum_{k=1}^L 4^{-k}(1-\psi^{\circ k}(\ell 2^{-L-1})) = \sum_{k=1}^{L+1} 4^{-k}(1-\psi^{\circ k}(\ell 2^{-L-1})).$ It thus remains to consider $\ell$ odd. Recall that $x\mapsto \sum_{k=1}^L 4^{-k} (1-\psi^{\circ k}(x))$ is affine on $[(\ell-1)2^{-L-1}, (\ell+1)2^{-L-1}]$ and observe that for any real $t,$
\begin{align*}
	g(x) = t^2 + \frac{g(x-t)+g(x+t)}{2}.
\end{align*}
For odd $\ell,$ choose $x =\ell 2^{-L-1}$ and $t=2^{-L-1}.$ Applying the identity and $\psi^{\circ (L+1)}(\ell 2^{-L-1}) =0$ yields
\begin{align*}
	g(\ell 2^{-L-1})  
&= 2^{-2L-2} + \frac{g\big((\ell-1)2^{-L-1}\big) + g\big((\ell+1)2^{-L-1}\big)}{2}  \\
&= 4^{-L-1}+\frac{1}{2}\bigg(
\sum_{k=1}^L 4^{-k}\Big(1-\psi^{\circ k}\big((\ell-1) 2^{-L-1}\big)\Big)
+
\sum_{k=1}^L 4^{-k}\Big(1-\psi^{\circ k}\big((\ell+1) 2^{-L-1}\big)\Big)
\bigg) \\
&= 4^{-L-1}+\sum_{k=1}^L 4^{-k}\big(1-\psi^{\circ k}(\ell 2^{-L-1})\big) \\
    &= \sum_{k=1}^{L+1} 4^{-k} \big(1-\psi^{\circ k}(\ell 2^{-L-1})\big).
\end{align*}
This completes the inductive step and proves assertion \textit{(iii)}.

\textit{(iv):} For a quadratic function $h$ with $h(a)=h(b)=0$ and $K=h''(a)=h''(b),$ one has that $h(x)=\tfrac K2 (x-a)(x-b).$ Since $g''(x)=-2,$ for any $\ell =0,\ldots, 2^L-1,$ and any $x\in [\ell 2^{-L}, (\ell+1) 2^{-L}],$ 
\begin{align}
	g(x) - \sum_{k=1}^L 4^{-k} \big(1-\psi^{\circ k}(x)\big) 
	&= -\big(x - (\ell+1)2^{-L}\big) \big(x-\ell2^{-L}\big) = \big((\ell+1)2^{-L} -x\big) \big(x-\ell2^{-L}\big).
\end{align}
For any $x\in [\ell 2^{-L}, (\ell+1) 2^{-L}],$
\begin{align*}
   0 \leq \Big(\frac{\ell+1}{2^{L}} -x\Big) \Big(x-\frac{\ell}{2^{L}}\Big) \leq
   \Big(\frac{\ell+1}{2^{L}} -\frac{\ell+1/2}{2^{L}}\Big) \Big(\frac{\ell+1/2}{2^{L}}-\frac{\ell}{2^{L}}\Big)= 2^{-2L-2}.
\end{align*}
The second inequality becomes an equality for the interval midpoint $x=(\ell+1/2)2^{-L}.$ This shows that 
\begin{align*}
    \sup_{x\in [0,1]} \, \Big|g(x) - \sum_{k=1}^L 4^{-k} \big(1-\psi^{\circ k}(x)\big) -2^{-2L-3}\Big|=2^{-2L-3}.
\end{align*}
The assertion follows using that $g(x)=x(1-x)=x-x^2.$

\end{proof}

{\reb
\begin{proof}[Proof of Theorem \ref{thm:quadratic_approximation}]
    By  \Cref{thm:quadratic_approximation_max}, there exists a HyCNN $h \colon \RR\to \RR$ with  $\overline \sigma(a_1,a_2) = \max(a_1,a_2)$,  $L$ hidden layers, and $(d_1, \dots, d_L)$ neurons per layer 
    such that 
     \begin{align*} 
        \sup_{x\in [0,1]}|h(x) - x^2| \leq \frac{1}{8\prod_{\ell=1}^L d_{\ell}^2}.
     \end{align*} 
 By Lemma~\ref{thm_tau_approx}, there exists a HyCNN $g \colon \mathbb{R} \to \mathbb{R}$ with $\overline \sigma(a_1,a_2) = \overline \sigma_\tau(a_1, a_2)$ and the same architecture as $h$ 
    such that 
    \[
    \sup_{x \in [0,1]} |g(x) - h(x)| \leq \left(\frac{1}{7.99} - \frac{1}{8} \right) \frac{1}{\prod_{\ell=1}^L d_{\ell}^2}.
    \]     
    The assertion follows from the triangle inequality.
\end{proof}
}

\subsection{Proofs for monomial functions}
\label{appendix_monomial}

In this section, we provide the proof of Theorem \ref{thm_monomial}.
{\reb We first establish the following lemma, which will be used to prove the main result.
The proof of Theorem \ref{thm_monomial} then follows by applying Lemma~\ref{thm_tau_approx}.}

\begin{lemma}[Monomial Function Approximation  with maxout unit]\label{thm_monomial_maxout}
    For any integer $n \geq 2$, $L\in \NN$ and odd number $m \geq 3$, there exists a HyCNN $h \colon \RR\to \RR$ with {\reb $\overline \sigma(a_1,a_2) = \max(a_1, a_2)$}, $ \lceil \log_2(n)\rceil L$ layers and $m$ neurons per layer such that 
    \begin{align*}
        \sup_{x\in [0,1]} \big|h(x) - x^n\big| \leq \frac{n}{2} (m-1)^{-2L}.
    \end{align*}
\end{lemma}

The following lemma states that how HyCNNs can approximate the functional operators $\calD_0$ and $\calD_1$, defined by
$$\calD_0 [f](x) := f^2(x), \qquad \calD_1 [f](x) := f^2(x)/x.$$ 
These operators are crucially used in the proof of \Cref{thm_monomial_maxout}.
Approximating $\calD_0$ is straightforward from \Cref{thm:quadratic_approximation_max}, whereas approximating $\calD_1$ is a non-trivial problem. 

\begin{lemma}\label{lem_approx_operator_0}
Consider a HyCNN $g \colon \RR\to \RR$ with {\reb $\overline \sigma(a_1,a_2) = \max(a_1, a_2)$}, $L'$ hidden layers and $m'$ neurons per hidden layer.
Assume that $0 \leq g(x) \leq x$ for every $x \in [0,1]$.
Then, for any $L \in \NN$ and $m \in 2 \NN$, 
\begin{enumerate}
    \item [(i)] there exists a HyCNN $\tilde g_0 \colon \RR\to \RR$ with {\reb $\overline \sigma(a_1,a_2) = \max(a_1, a_2)$}, $L' + L$ hidden layers and $\max(m', m+1)$ neurons per hidden layer such that $0 \leq \tilde g_0(x) \leq x$ for every $x \in [0,1]$ and
$$\sup_{x \in [0,1]} \big|\tilde g_0(x) - \calD_0 [g](x) \big| \leq \frac{1}{4 m^{2L}}.$$
    \item[(ii)] there exists a HyCNN $\tilde g_1 \colon \RR\to \RR$ with {\reb $\overline \sigma(a_1,a_2) = \max(a_1, a_2)$}, $L' + L$ hidden layers and $\max(m', m+1)$ neurons per hidden layer such that $0 \leq \tilde g_1(x) \leq x$ for every $x \in [0,1]$ and 
$$\sup_{x \in [0,1]} \big|\tilde g_1(x) - \calD_1 [g](x) \big| \leq \frac{1}{4 m^{2L}}.$$
\end{enumerate}

\end{lemma}

To prove this lemma, we need the following two key results.
First, we modify our network to approximate the quadratic function, in order to ensure that the \textit{skip-connection weight matrices}, i.e., the matrices
$\{W_{\ell}^1, W_{\ell}^2\}_{\ell =0}^{L-1}$ and $W_{L}$ in \eqref{eq:HyCNN}, are element-wise non-negative. This is relevant for proving \Cref{lem_approx_operator_0}$(i)$, concretely in order to ensure composability of HyCNNs according to \Cref{lemma_hycnn_comp}. 

\begin{lemma}\label{lemma_positive_qudratic}
    For any $L \in \NN$ and $m \in 2 \NN$, there exists a HyCNN $f : \RR \to \RR$ with {\reb $\overline \sigma(a_1,a_2) = \max(a_1, a_2)$}, $L$ hidden layers and $m$ neurons per hidden layer such that
$$\sup_{x \in [0,1]} \big|f(x) - x^2 \big| \leq \frac{1}{4 m^{2L}}$$
and $0 \leq f(x) \leq x$ for every $x \in [0,1]$.
Moreover, all the skip-connection weight matrices of $f$ are element-wise non-negative. 
\end{lemma}
\begin{proof}
    We define $M \coloneqq m^{L}$.
    By adding $(8M^2)^{-1}$ to \eqref{tmp_16} in the proof of \Cref{thm:quadratic_approximation_max}, we can construct a HyCNN $f: \RR \to \RR$ with $L$ hidden layers and $m$ neurons per hidden layer such that
    \[
    f(x) = \frac{x}{M} + \frac{2}{M}\sum_{k=1}^{M-1} \left(x - \frac{k}{M}\right)_+,
    \]
    and all the skip-connection weight matrices of $f$ are element-wise non-negative (see \Cref{rem.74f}). 
    By Lemma~\ref{lemma_sq_approx_pc}, we obtain
    $$\sup_{x \in [0,1]} |f(x)-x^2| \leq (8M^2)^{-1} + \sup_{x \in [0,1]} \left|-(8M^2)^{-1}+f(x)-x^2\right| \leq (4M^2)^{-1} = \frac{1}{4m^{2L}}.$$
    Notice that $f(0)=0$ and 
    $$f(1) = \frac{1}{M} + \frac{2(M-1)}{M} - \frac{M(M-1)}{M^2} = 1.$$ 
    Since $f$ is non-decreasing and convex, we obtain the remaining assertion by
    \[
    0 \leq f(x) \leq (1-x)f(0) + x f(1) = x.
    \]
\end{proof}

The next lemma serves as a crucial tool to prove Lemma \ref{lem_approx_operator_0}(ii). 
\begin{lemma}\label{lem:uy_y_yhu}
Consider a HyCNN $h \colon \RR^d \to \RR$ with {\reb $\overline \sigma(a_1,a_2) = \max(a_1, a_2)$}, $L$ hidden layers and $m$ neurons per hidden layer,  specified for $\ell \in \{0, \dots, L-1\}$ by
\begin{align*}
    h(\bx) = \bv_L^{\top} \bz_L(\bx) + \bw_L^{\top} \bx + b_L, \quad
\bz_{\ell+1}(\bx) = \max_{k\in \{1,2\}}\left(V_\ell^k \bz_{\ell}(\bx) + W_{\ell}^k \bx  + \bb_{\ell}^k\right), \quad
\bz_{0}(\bx)=0.
\end{align*}
For $\bx_1 \in \RR^d$ and $x_2 \in \RR$,
define $\tilde h (\bx_1, x_2)$ by 
\begin{equation}\label{tmp_17}
\begin{aligned}
\tilde h(\bx_1, x_2) &\coloneqq \bv_L^{\top} \tilde \bz_L(\bx_1, x_2) + \bw_L^{\top} \bx_1 + b_L x_2,\\
\tilde \bz_{\ell+1}(\bx_1, x_2) &\coloneqq 
\max_{k\in \{1,2\}}\left(V_\ell^k \tilde \bz_{\ell}(\bx_1, x_2) + W_{\ell}^k \bx_1 + \bb_{\ell}^k x_2 \right).
\end{aligned}
\end{equation}    
with $\bz_{0}(\bx_1, x_2)=\mathbf{0}_{d+1}$.
Then, it follows that $\tilde h(\bu y, y) = y h (\bu)$ for any $\bu \in \RR^d$ and $y \geq 0$.
\end{lemma}
\begin{proof}
The claim follows from induction on the depth $\ell$. For $\ell = 0$, we have 
\begin{align*}
   \tilde  \bz_1(\bu y, y) &= \max_{k\in \{1,2\}}\left(W_{0}^k \bu y  + \bb_{0}^k y \right) \\
    &= y \max_{k\in \{1,2\}}  \left(W_{0}^k \bu  + \bb_{0}^k \right)  \\
    &= y \bz_1(\bu).
\end{align*}
Assuming that $\tilde \bz_\ell(\bu y, y) = y \bz_\ell(\bu)$ for some $\ell\in \{0, \dots, L-1\}$, we get 
\begin{align*}
    \tilde  \bz_{\ell+1}(\bu y, y) &= \max_{k\in \{1,2\}}\left(V_\ell^k \tilde \bz_{\ell}(\bu y, y) +W_{\ell}^k \bu y  + \bb_{\ell}^k y \right)  \\
    &= \max_{k\in \{1,2\}}  \left(y V_\ell^k \bz_{\ell}(\bu) + W_{\ell}^k \bu y  + \bb_{\ell}^k y\right)  \\
    &= y \max_{k\in \{1,2\}}  \left(V_\ell^k \bz_{\ell}(\bu) + W_{\ell}^k \bu  + \bb_{\ell}^k \right)  \\
    &= y \bz_{\ell+1}(\bu),
\end{align*}
which proves the induction step. Finally,
\begin{align*}
    \tilde h(\bu y, y) &= \bv_L^{\top} \tilde \bz_L(uy, y)  + \bw_L^{\top} \bu y + b_L y \\
    &= y \bv_L^{\top} \bz_L(\bu) + \bw_L^{\top} \bu y + b_L y \\
    &= y h(\bu).\qedhere
\end{align*}
\end{proof}

\begin{proof}[Proof of Lemma~\ref{lem_approx_operator_0}]
{(i).}
    By Lemma~\ref{lemma_positive_qudratic}, we can construct a HyCNN $h: \RR \to \RR$ with {\reb $\overline \sigma(a_1,a_2) = \max(a_1, a_2)$}, $L$ hidden layers and $m$ neurons per hidden layer such that
    $$\sup_{x \in [0,1]} |h(x)-x^2| \leq \frac{1}{4 m^{2L}},$$
    $0 \leq h(x) \leq x$ for every $x \in [0,1]$, and 
    all the skip-connection weight matrices of $h$ are element-wise non-negative.
    By Lemma~\ref{lemma_hycnn_comp},
there exists a HyCNN $\tilde g_0: \RR \to \RR$ with {\reb $\overline \sigma(a_1,a_2) = \max(a_1, a_2)$}, $L' + L$ hidden layers and $\max(m', m+1)$ neurons per hidden layer such that $\tilde g_0(x) = h \circ g(x)$.
We obtain the upper bound of the approximation error by
\begin{align*}
    \sup_{x \in [0,1]} \big|\tilde g_0(x) - g^2(x) \big| 
    =
    \sup_{x \in [0,1]} \big| h \circ g(x) - g^2(x) \big|    
    &\leq 
    \sup_{z \in [0,1]} \big| h(z) -  z^2 \big|\\    
    &\leq 
    \frac{1}{4m^{2L}},
\end{align*}
provided that $0 \leq g(x) \leq 1$.
Finally, from 
$0 \leq h \circ g(x) \leq x$ for every $x \in [0,1]$, we get
$0 \leq \tilde g_0(x) \leq x$.

\medskip 

{(ii).}
    By Lemma~\ref{lemma_positive_qudratic}, we can construct a HyCNN $h: \RR \to \RR$ with {\reb $\overline \sigma(a_1,a_2) = \max(a_1, a_2)$}, $L$ hidden layers and $m$ neurons per hidden layer such that
    $$\sup_{x \in [0,1]} |h(x)-x^2| \leq \frac{1}{4 m^{2L}},$$
    $0 \leq h(x) \leq x$ for every $x \in [0,1]$, and 
    all the skip-connection weight matrices of $h$ are element-wise non-negative.
 
By applying Lemma~\ref{lem:uy_y_yhu} to the HyCNN $h$, 
we define a function $\tilde h \colon \RR^2 \to \RR$ by
\begin{align*}
\tilde h(x_1, x_2) &\coloneqq \bv_L^{\top} \tilde \bz_L(x_1, x_2) + w_L x_1 + b_L x_2,\\
\tilde \bz_{\ell+1}(x_1, x_2) &\coloneqq 
\max_{k\in \{1,2\}}\left(V_\ell^k \tilde \bz_{\ell}(x_1, x_2) + \bw_{\ell}^k x_1 + \bb_{\ell}^k x_2 \right)
\end{align*}
with $\tilde z_{0}(x_1, x_2)=0$, where the 
parameters
$\bv_{L}, w_{L}, \bb_{L}$, $\{V_\ell^1, \bw_{\ell}^1, \bb_{\ell}^1\}_{\ell=0}^{L-1}$ and $\{ V_\ell^2, \bw_{\ell}^2, \bb_{\ell}^2\}_{\ell=0}^{L-1}$ coincide with those of the HyCNN $h$.
Then, $\tilde h$ 
satisfies $\tilde h(uy, y) = y h (u)$ for any $u \in \RR$ and $y \geq 0$.     
By plugging in $u=g(x)/x$ and $y=x$, we get
$$\tilde h\big(g(x), x\big) = x h\big(g(x)/x\big), \qquad \forall x \in [0,1].
$$
All the $\bv_{L}, w_{L}$ and $\{V_\ell^1, \bw_{\ell}^1, V_\ell^2, \bw_{\ell}^2 \}_{\ell=0}^{L-1}$ of $h$ (and of $\tilde h$) are element-wise non-negative.
Therefore, we can apply
Lemma~\ref{lemma_hycnn_comp}
to obtain a HyCNN $\tilde g: \RR \to \RR$ with {\reb $\overline \sigma(a_1,a_2) = \max(a_1, a_2)$}, $L' + L$ hidden layers and $\max(m', m+1)$ neurons per hidden layer such that $\tilde g(x) = \tilde h(g(x), x).$
We obtain the upper bound of the approximation error by
\begin{align*}
    \sup_{x \in [0,1]} \big|\tilde g(x) - g^2(x)/x \big| 
    =
    \sup_{x \in [0,1]} \big|x h\big(g(x)/x\big) - x (g(x)/x)^2 \big|    
    &\leq 
    \sup_{x \in [0,1]} \big| h\big(g(x)/x\big) -  (g(x)/x)^2 \big|\\    
    &\leq 
    \frac{1}{4m^{2L}},
\end{align*}
provided that $0 \leq g(x)/x \leq 1$.
Finally, from 
$0 \leq h(g(x)/x) \leq 1$ for every $x \in [0,1]$, we get
$0 \leq \tilde g(x) \leq x$.
\end{proof}

\begin{proof}[Proof of \Cref{thm_monomial_maxout}]
    For functions $f\colon \RR\to \RR$, define the operators
    $$\calD_0 [f](x) := f^2(x), \qquad \calD_1 [f](x) := f^2(x)/x.$$ 
    Further, let $\phi_m(x) = x^m$ for $m\in \NN$. 
    We define $k \coloneqq \lceil \log_2(n)\rceil$ and 
    represent $2^k - n$ through its binary digits.
    Since $0 \leq 2^k - n \leq 2^{k-1}-1$, there exist $b_1, \ldots, b_{k-1}$ with $b_i \in \{0,1\}$ for $i=1,\ldots,k-1$ such that
    $$2^k - n = \sum_{i=1}^{k-1} b_i 2^{i-1}.$$  
    Set $b_k \coloneqq 0$. 
    Since 
    $\calD_{b}[\phi_m](x) = \phi_{2m-b}(x)$ for any $b \in \{0,1\}$ and $m \in \NN$, 
    we get
 \begin{align}
   \calD_{b_1} \circ \calD_{b_2} \circ \ldots \circ \calD_{b_{k-2}} \circ \calD_{b_{k-1}} \circ \calD_{b_k}[\phi_1](x) &=
   \calD_{b_1} \circ \calD_{b_2} \circ \ldots \circ \calD_{b_{k-2}} \circ \calD_{b_{k-1}}[\phi_2](x) \notag \\
   &= \phi_n(x)
   \label{tmp_18}
\end{align}
since
\begin{align*}
    \overbrace{2 \Big( \ldots 2\big(2(2}^{k-1 \text{ times}} \cdot \, 2 - b_{k-1}) - b_{k-2}\big)- b_{k-3} \ldots \Big) - b_1 =&
2^{k} - 2^{k-2} b_{k-1} - 2^{k-3} b_{k-2} - \ldots - b_1\\
=& 2^k - \sum_{i=1}^{k-1} b_i 2^{i-1} = n.
\end{align*}

We approximate the  composition with HyCNNs and track the approximation error using backward induction, starting with $k$ and then going down to $1.$ By Lemma~\ref{lemma_positive_qudratic}, there exists a HyCNN $h_1: \RR \to \RR$ with {\reb $\overline \sigma(a_1,a_2) = \max(a_1, a_2)$}, $L$ hidden layers and $m$ neurons per hidden layer such that 
$$\sup_{x \in [0,1]} \big|h_1(x) - \calD_{b_k}[\phi_1](x) \big| 
=
\sup_{x \in [0,1]} \big|h_1(x) - x^2 \big| 
\leq \frac{1}{4 m^{2L}}$$
and $0 \leq h_1(x) \leq x$ for every $x \in [0,1]$.
Assume that for $i \in \{1,2,\ldots,k-1\}$, 
there exists a HyCNN $h_i: \RR \to \RR$ with {\reb $\overline \sigma(a_1,a_2) = \max(a_1, a_2)$}, $Li$ hidden layers and $m+1$ neurons per hidden layer such that 
$$\sup_{x \in [0,1]} \big|h_i(x) - \underbrace{\calD_{b_{k+1-i}} \circ \calD_{b_{k+2-i}} \circ \ldots \circ \calD_{b_k}}_{i \text{ times}}[\phi_1](x) \big|  
\leq \frac{2^i - 1}{4m^{2L}}$$
and $0 \leq h_i(x) \leq x$ for every $x \in [0,1]$.
For $x \in [0,1]$, we also have
$$0 \leq \calD_{b_{k+1-i}} \circ \ldots \circ \calD_{b_k}[\phi_1](x) \leq x.$$
By applying Lemma \ref{lem_approx_operator_0}-(i) for the case $b_{k-i}=0$ and Lemma \ref{lem_approx_operator_0}-(ii) 
for the case $b_{k-i}=1$, there exists a HyCNN 
$h_{i+1}: \RR \to \RR$ with {\reb $\overline \sigma(a_1,a_2) = \max(a_1, a_2)$}, $L(i+1)$ hidden layers and $m+1$ neurons per hidden layer such that 
\begin{align*}
    \sup_{x \in [0,1]} \big|h_{i+1}(x) - \calD_{b_{k-i}}[h_{i}](x) \big|  
\leq \frac{1}{4m^{2L}}
\end{align*}
and $0 \leq h_{i+1}(x) \leq x$ for every $x \in [0,1]$.
Using the triangle inequality and the induction hypothesis,
\begin{align*}
    &\sup_{x \in [0,1]} \big|h_{i+1}(x) - \underbrace{\calD_{b_{k-i}} \circ \calD_{b_{k+1-i}} \circ \ldots \circ \calD_{b_k}}_{(i+1)\text{ times}}[\phi_1](x) \big|\\
    &\leq 
    \sup_{x \in [0,1]} \big|h_{i+1}(x) - \calD_{b_{k-i}}[h_{i}](x) \big| 
    +
    \sup_{x \in [0,1]} \Big|\calD_{b_{k-i}}[h_{i}](x) - \calD_{b_{k-i}} \circ \calD_{b_{k+1-i}} \circ \ldots \circ \calD_{b_k}[\phi_1](x) \Big|\\   
    & = \sup_{x \in [0,1]} \big|h_{i+1}(x) - \calD_{b_{k-i}}[h_{i}](x) \big|
    +
    \sup_{x \in [0,1]} \left|\frac{h_{i}^2(x)}{x^{b_{k-i}}} - 
    \frac{\big(\calD_{b_{k+1-i}} \circ \ldots \circ \calD_{b_k}[\phi_1](x)\big)^2}{x^{b_{k-i}}}
    \right|\\
    & \leq \frac{1}{4m^{2L}} + 
    \sup_{x \in [0,1]} \frac{\left|\big(h_{i}(x) - \calD_{b_{k+1-i}} \circ \ldots \circ \calD_{b_k}[\phi_1](x)\big) 
    \big(h_{i}(x) + \calD_{b_{k+1-i}} \circ \ldots \circ \calD_{b_k}[\phi_1](x)\big)
    \right|}{x^{b_{k-i}}}\\
    & \leq \frac{1}{4m^{2L}} + 
    2 \cdot \frac{2^i - 1}{4m^{2L}} = \frac{2^{i+1} - 1}{4m^{2L}}.
\end{align*}
Hence, by backward induction on $i$, 
there exists a HyCNN 
$h_{k}: \RR \to \RR$ with {\reb $\overline \sigma(a_1,a_2) = \max(a_1, a_2)$}, $Lk$ hidden layers and $m+1$ neurons per hidden layer such that 
$$\sup_{x \in [0,1]} \big|h_k(x) - \calD_{b_{1}} \circ \calD_{b_{2}} \circ \ldots \circ \calD_{b_k}[\phi_1](x) \big|  
\leq \frac{2^k - 1}{4m^{2L}} < \frac{2^{k-2}}{m^{2L}} \leq \frac{n}{2 m^{2L}}.$$
The assertion then follows by recalling \eqref{tmp_18} and replacing $m$ by $m-1$. 
\end{proof}

{\reb
\begin{proof}[Proof of Theorem \ref{thm_monomial}]
    By  \Cref{thm_monomial_maxout}, there exists a HyCNN $h \colon \RR\to \RR$ with  $\overline \sigma(a_1,a_2) = \max(a_1,a_2)$, $ \lceil \log_2(n)\rceil L$ layers and $m$ neurons per layer such that 
     \begin{align*} 
        \sup_{x\in [0,1]}|h(x) - x^2| \leq \frac{n}{2} (m-1)^{-2L}.
     \end{align*} 
 By Lemma~\ref{thm_tau_approx}, there exists a HyCNN $g \colon \mathbb{R} \to \mathbb{R}$ with $\overline \sigma(a_1,a_2) = \overline \sigma_\tau(a_1, a_2)$ and the same architecture as $h$ 
    such that 
    \[
    \sup_{x \in [0,1]} |g(x) - h(x)| \leq \left(\frac{1}{1.99} - \frac{1}{2} \right) (m-1)^{-2L}.
    \]     
    The assertion follows from the triangle inequality.
\end{proof}
}

\subsection{Proof of Inclusions \eqref{eq.84gr}} \label{appendix_proof_equation_inclusion}
\begin{proof}
    We write $\bx = (x_1,\ldots,x_d)^{\top}$.
    Theorems \ref{thm:lowerBoundPiecewiseLinear} and \ref{thm:quadratic_approximation} yield
    \[
    \inf_{f \in \ICNN_d(L,m)} \,  \sup_{\bx \in [0,1]^d}|f(\bx)-x_1^2| \geq \frac{1}{8 (m + 2(L-1)m)^2} > \frac{1}{8(2mL)^2}
    \]
    and
    \[
    \inf_{h \in \HyCNN_d(L',m')} \sup_{\bx \in [0,1]^d}|h(\bx)-x_1^2| \leq \frac{1}{8 (m')^{2 L'}} \leq \frac{1}{8(2mL)^2}.
    \]
    respectively. Hence, we get $\HyCNN_d(L',m') \subsetneq \ICNN_d(L,m)$. 

    The second assertion follows by setting
    $V_\ell^{2} = W_\ell^{2} = 0$ and $\bb_\ell^{2} = 0$ for each $\ell = 0, \ldots, L-1$ in \eqref{eq:HyCNN}.

    To establish the last inclusion, write $(x_1,\ldots, x_d)_+ \coloneqq (x_1 \vee 0, \ldots, x_d \vee 0)$.
    For a given HyCNN $h \in \HyCNN_d(L,m)$,
    we construct a ReLU network $f \in \ReLU_d(L,3m+2d)$ that represents $h$.
    In each hidden layer of $f$, we store $\bx_+$ and $(-\bx)_+$, which represent the positive and negative parts of the input vector, respectively.
    This requires $2d$ neurons per layer, and these neurons can be used to represent skip connections from the input via 
    \[
        \ba^{\top} \bx = \ba^{\top}\big(\bx_+ - (-\bx)_+\big) = \ba^{\top} \bx_+ - \ba^{\top} (-\bx)_+.
    \]
    On the other hand, from 
    $x \vee y = x + (y-x)_+$ and $x = x_+ - (-x)_+$,
    we have
    \begin{align*}
            (\ba_1^{\top} \bz + b_1) 
    \vee (\ba_2^{\top} \bz + b_2) 
    &= (\ba_1^{\top} \bz + b_1) + \big( (\ba_2 - \ba_1)^{\top} \bz + b_2 - b_1 \big)_+\\  
    &= (\ba_1^{\top} \bz + b_1)_+ - (-\ba_1^{\top} \bz - b_1)_+  + \big( (\ba_2 - \ba_1)^{\top} \bz + b_2 - b_1 \big)_+.    
    \end{align*}
    Hence, instead of $m$ neurons in each hidden layer of $h$ computing the left-hand side of the previous identity, $3m$ neurons in the ReLU network $f$ can compute the right-hand side.
    This proves the assertion.
\end{proof}

\subsection{Proofs for additional theoretical results} \label{appendix_proof_multivariate}

\begin{proof}[Proof of \Cref{prop:lowerBoundMultivariateICNN}]
    Let $f \colon \RR^d \to \RR$ be a multivariate (standard or leaky) ReLU ICNN as in the claim. Then, by continuity of $f$ and $\|\cdot\|_2^2$, there exists $\bx' \in [0,1]^{d-1}$ such that
    \begin{align*}
         &\int_{[0,1]^d} \big|f(\bx)-\|\bx\|_2^2 \big| \, d\bx\geq  \int_0^1 \big|f(x, \bx') - x^2 - \|\bx'\|_2^2\big| \, dx
    \end{align*}
    In particular, the function $x\mapsto f(x, \bx')$ is a univariate (standard or leaky) ReLU ICNN with $L$ hidden layers and $d_\ell$ neurons in layer $\ell$. Hence, by \Cref{thm:lowerBoundPiecewiseLinear}, we have 
    \begin{align*}
        \int_0^1 \big|f(x, \bx') - x^2 - \|\bx'\|_2^2\big| \, dx &\geq \frac{1}{16(d_1 + 2 \sum_{\ell = 2}^L d_\ell)^2}.
    \end{align*}
    For the supremum norm, we infer upon denoting $\mathbf{1}_d = (1, \dots, 1)^{\top} \in \RR^d$ that 
    \begin{align*}
        \sup_{\bx\in [0,1]^d}\big|f(\bx)-\|\bx\|_2^2\big| \geq \sup_{t\in [0,1]}\big|f(t\mathbf{1}_d)-t^2 d\big| = d\sup_{t\in [0,1]}|d^{-1}f(t \mathbf{1}_d)-t^2|.
    \end{align*}
    Since $t\mapsto d^{-1}f(t\mathbf{1})$ is a univariate (standard or leaky) ReLU ICNN on $[0,1]$ with $L$ hidden layers and $d_\ell$ neurons in layer $\ell$, the asserted lower bound also follows from \Cref{thm:lowerBoundPiecewiseLinear}. 
\end{proof}

\begin{proof}
[Proof of \Cref{thm:quadratic_approximation_multi}]

\begin{figure}[t]
  \centering  \includegraphics[width=.6\textwidth]{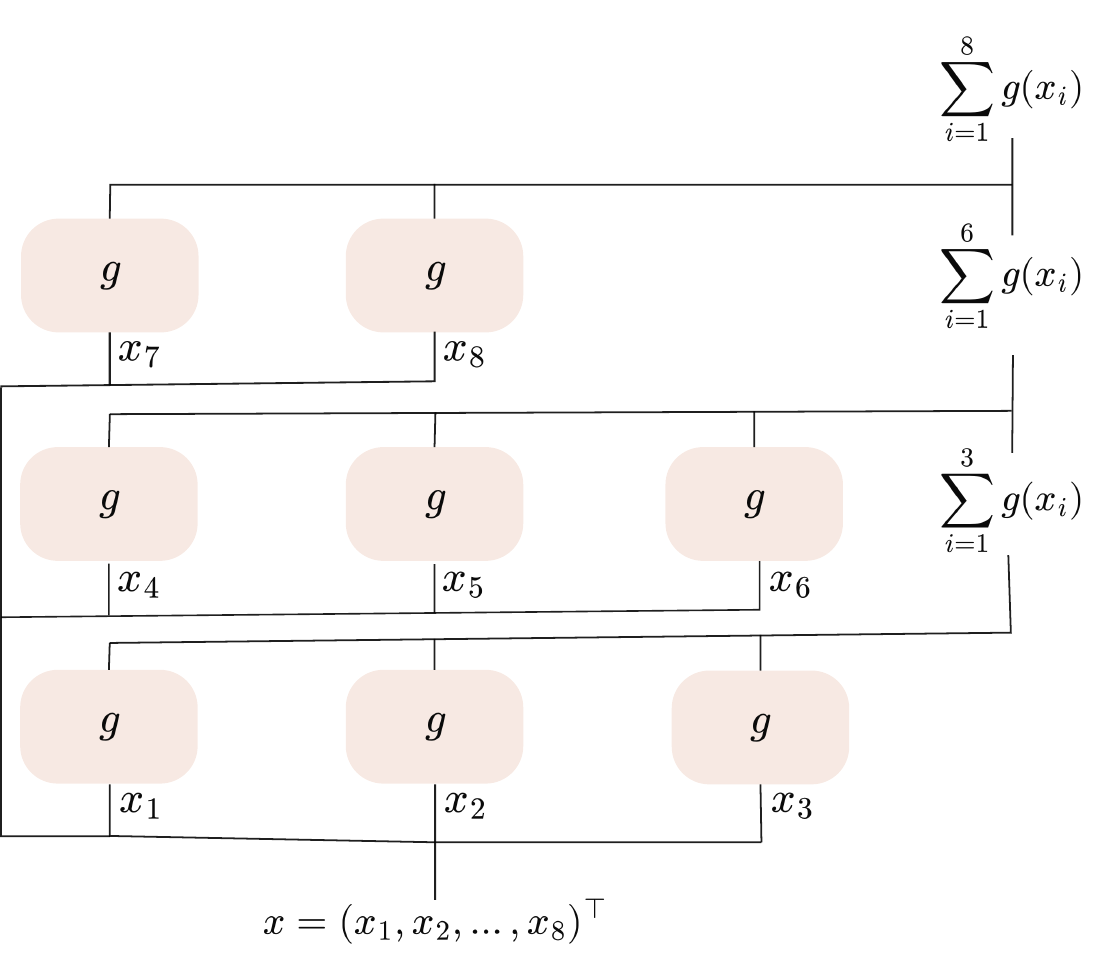}
  \caption{{Illustration of the proof of \Cref{thm:quadratic_approximation_multi}} for $d=8$ and $p=q=3$. 
    } \label{fig_multi_qudratic}
\end{figure}

 We write $\bx = (x_1, \ldots, x_d)^{\top}$.
 We define $\kappa : \NN \to 2\NN_0$ as 
 \[
\kappa(n) \coloneq   \begin{cases}
     n &n \text{ is even},\\
     n-1 &n \text{ is odd}.\\
 \end{cases}
 \]
 Consider arbitrary $p,q \in \NN$ with $pq \geq d$.
 By Theorem~\ref{thm:quadratic_approximation}, there exists a HyCNN $g \colon \RR\to \RR$ with $\lfloor L/p \rfloor$ hidden layers, $\kappa(\lfloor (m-1)/q \rfloor)$ neurons per layer and {\reb $\overline \sigma(a_1,a_2) = \overline \sigma_\tau(a_1, a_2)$}, such that 
     \begin{align}
        \sup_{x\in [0,1]}|g(x) - x^2| \leq \frac{1}{7.99} \kappa\left(\left\lfloor \frac{m-1}{q} \right\rfloor\right)^{-2 \lfloor L/p \rfloor} \leq \frac{1}{7.99} \left\lfloor \frac{m-1}{q} -1 \right\rfloor^{-2 \lfloor L/p \rfloor}. 
        \label{eq.8gfew}
     \end{align}
     We construct a HyCNN $f\colon \RR^d \to \RR$ with $L$ hidden layers and $m$ neurons per layer as follows. 
    For each $j \in [p]$,
    we approximate $x_{(j-1)q+1}^2,\ldots, x_{j q}^2$ by applying $q$ parallel copies of the HyCNN $g$ in the hidden layers from $((j-1)\lfloor L/p \rfloor+1)$ to $j \lfloor L/p \rfloor$ of $f$. 
    This is possible as the network input is accessible in each hidden layer.
    If $pq > d$, we apply the construction only up to index $d$. 
    This construction requires $p \lfloor L/p \rfloor \leq L$ layers and $q \kappa(\lfloor (m-1)/q \rfloor) \leq m-1$ neurons per layer.

    It remains to sum the outputs $g(x_1), \ldots, g(x_d)$.
    This can be done by adding one neuron per hidden layer that stores the sum of the approximated square functions using  
    $${\reb \overline \sigma_\tau\bigg( \sum_{i=1}^k g(x_i) - \tau \log 2, \, \sum_{i=1}^k g(x_i) - \tau \log 2 \bigg) = \sum_{i=1}^k g(x_i), \qquad \forall k \in [d].}$$
    See Figure~\ref{fig_multi_qudratic} for an illustration.
    The HyCNN output $f(\bx)$ is then $\sum_{i=1}^d g(x_i)$. By the triangle inequality and \eqref{eq.8gfew}, the approximation error is upper bounded by 
    \[
    \sup_{\bx\in [0,1]^d} \big|f(\bx) - \|\bx\|_2^2\big| \leq \sum_{i=1}^d \sup_{x_i \in [0,1]} \big|g(x_i) - x_i^2 \big| \leq \frac{d}{7.99} \inf_{p,q \in \NN: pq \geq d} \,  \left\lfloor \frac{m-1}{q}-1 \right\rfloor^{-2\lfloor L/p \rfloor}.
    \]   
\end{proof}

{\reb

\begin{proof}[Proof of \Cref{prop:positive_coeff}]
Let $c_k \coloneq f_0^{(k)}(0)/k!$, so that $c_k \geq 0$ for all $k\geq2$.
   By the Vivanti--Pringsheim theorem, the Taylor series of $f_0$ at the origin has radius of convergence strictly greater than $r$. Hence,
    \[    f_0(x) = \sum_{k=0}^{\infty} c_k x^k,
    \qquad x \in [0,r].
\]
For an integer $q \geq 2$ which will be specified later, define
\[    P_q(x) \coloneq  \sum_{k=0}^{q-1} c_k x^k, 
\qquad
R_q(x) \coloneq \sum_{k=q}^{\infty} c_k x^k.
\]
It follows
\[
\sup_{x \in [0,1]}|R_q(x)| = \sum_{k=q}^{\infty} c_k
\leq r^{-q} \sum_{k=q}^{\infty} c_k r^k
\leq r^{-q}\big( f_0(r) - f_0(0) - r f'_0(0) \big).
\]

Using Theorem \ref{thm_monomial}, we construct $q-2$ HyCNNs $(h_k)_{k=2,\ldots,q-1}$, where $h_k \colon \RR\to \RR$ uses $\overline \sigma(a_1,a_2) = \overline \sigma_\tau(a_1, a_2)$, has 
$\lceil \log_2(k)\rceil L$ layers and $m$ neurons per layer and satisfies
\begin{align*}
        \sup_{x\in [0,1]} \big|h_k(x) - x^k\big| \leq \frac{k}{1.99} (m-1)^{-2L}.
\end{align*}
After padding these $q-2$ HyCNNs with identity layers (using $\overline \sigma_\tau(a-\tau \log 2,a-\tau \log 2) = a$) to equalize their depths, we combine them in parallel.
Then, we can construct a HyCNN $h \colon \mathbb{R} \to \mathbb{R}$ with $\overline \sigma(a,b) = \overline \sigma_\tau(a,b)$, $ \lceil \log_2(q-1)\rceil L$ layers and $mq$ neurons per layer such that $h(x) = c_0 + c_1 x + \sum_{k=2}^{q-1} c_k h_k(x)$. 
By choosing $q = 2L$, we get the assertion by
    \begin{align*}
        \sup_{x\in [0,1]} \big|h(x) - f_0(x)\big|
        &\leq \sup_{x\in [0,1]} \big|R_q(x)\big| + \sup_{x\in [0,1]} \big|h(x) - P_q(x)\big|
        \\
        &\lesssim r^{-2L} + \sum_{k=2}^{q-1} c_k \sup_{x\in [0,1]} |h_k(x) - x^k| \\
        &\lesssim r^{-2L} + (m-1)^{-2L} \sum_{k=2}^{\infty} k c_k\\
        &\lesssim r^{-2L} + (m-1)^{-2L},
    \end{align*}
    where we use $\sum_{k=2}^{\infty} k c_k = f'_0(1) - f'_0(0)$ for the last inequality.
\end{proof}

\begin{proof}[Proof of \Cref{theorem_approx_grad}]
We write $h:=f-g$ and $\varepsilon:=\|h\|_{L^\infty(K')}$.
Fix $\bx\in K$, a unit vector $\bu\in\RR^d$, and $r \in (0,1)$, and let $\by:=\bx+r \bu\in K'$.
By convexity of $f$ and $M$-smoothness of $g$, we have
    \begin{align*}
        f(\by) \geq f(\bx)+r\langle\nabla f(\bx), \bu \rangle, \qquad
        g(\by) \leq g(\bx)+r\langle\nabla g(\bx), \bu \rangle
        +\frac{M}{2}r^2.
    \end{align*}
Subtracting the second inequality from the first inequality yields
\begin{align*}
  h(\by)-h(\bx)\geq r\langle\nabla h(\bx), \bu\rangle
  -\frac{M}{2}r^2.
\end{align*}
 Since $h(\by)-h(\bx)\leq 2\varepsilon$, we obtain 
\[
    \langle\nabla h(\bx),\bu\rangle
    \leq \frac{2\varepsilon}{r}+\frac{Mr}{2}.
\]
If $\nabla h(\bx)\neq0$, choosing $\bu=\nabla h(\bx)/\|\nabla h(\bx)\|_2$ gives     
\[
    \|\nabla h(\bx)\|_2
    \leq \frac{2\varepsilon}{r}+\frac{Mr}{2},
    \qquad 0 <r < 1.
\]
This bound also holds trivially when $\nabla h(\bx)=0$.

If $\varepsilon=0$, letting $r \downarrow 0$ gives
$\nabla h(\bx)=\mathbf{0}_d$.
If $0<\varepsilon\leq M/4$, choosing
$r=2\sqrt{\varepsilon/M}\leq1$ gives
\[
    \|\nabla h(\bx)\|_2\leq 2\sqrt{M\varepsilon}.
\]
Otherwise, if $\varepsilon > M/4$, letting $r \uparrow 1$ gives
\[
    \|\nabla h(\bx)\|_2
    \leq 2\varepsilon+\frac{M}{2}
    \leq 4\varepsilon.
\]
Taking the supremum over $\bx \in K$, we conclude that
    \[
        \|\nabla f-\nabla g\|_{L^\infty(K)}
        =\|\nabla h\|_{L^\infty(K)}
        \leq 2\sqrt{M\varepsilon}+4\varepsilon.
\]
Finally, since $K$ has unit Lebesgue measure,
\[
    \|\nabla f-\nabla g\|_{L^p(K)}
    \leq \|\nabla f-\nabla g\|_{L^\infty(K)}
\]
for every $p\in[1,\infty)$.
\end{proof}

}

\section{Derivation of the Initialization Scheme for HyCNNs} \label{appendix_init}
In this section, we derive the initialization scheme introduced in Section \ref{sec:HyCNNs}, closely following the steps for the ICNN initialization proposed by \cite{hoedt2023principled} but appropriately adapted for HyCNNs.

For $\ell \in [L-1]$, 
we rewrite the HyCNN layer \eqref{eq:HyCNN} as
\begin{align*}
    \bs_{\ell+1} &:= V_{\ell}^{1} \bz_{\ell} + W_{\ell}^{1} \bx + \bb_{\ell}^{1}, \\
    \bt_{\ell+1} &:= V_{\ell}^{2} \bz_{\ell} + W_{\ell}^{2} \bx + \bb_{\ell}^{2}, \\
    \bz_{\ell+1} &:= \max(\bs_{\ell+1}, \bt_{\ell+1}).
\end{align*}
We write $\bx = (x_i)_{i \in [d]}$ and for $k \in \{1,2\}$,
\begin{align*}
    \bs_{\ell} = (s_{\ell, i})_{i \in [d_{\ell}]}, \quad \bt_{\ell} &= (t_{\ell, i})_{i \in [d_{\ell}]}, \quad \bz_{\ell} = (z_{\ell, i})_{i \in [d_{\ell}]}, \quad \text{for all} \ \ell \in [L], \\
    \quad V_{\ell}^k = (v^{(k)}_{\ell,i,j})_{i \in [d_{\ell+1}], j \in [d_{\ell}]}, \quad
    W_{\ell}^k &= (w^{(k)}_{\ell,i,j})_{i \in [d_{\ell+1}], j \in [d]}, \quad \bb^k_{\ell} = (b^{(k)}_{\ell,i})_{i \in [d_{\ell+1}]} \quad \text{for all} \ \ell\in [L-1].
\end{align*}
At initialization, the weight parameters are drawn independently with distributions ($\ell \in [L-1]$),
\begin{align*}
    v_{\ell, i,j}^{(k)} &\sim \calD_{\ell}, && k \in \{1,2\}, \, i \in [d_{\ell+1}], \, j \in [d_{\ell}],\\
    w_{\ell, i,j}^{(k)} &\sim N(0, \sigma_{w, \ell}^2), && k \in \{1,2\}, \, i \in [d_{\ell+1}], \, j \in [d],\\
    b_{\ell, i}^{(k)} &\sim N(\mu_{b, \ell}, \sigma_{b, \ell}^2), && k \in \{1,2\}, \, i \in [d_{\ell+1}],    
\end{align*}
where $\calD_{\ell}$ is a distribution in a pre-specified class with support on $\RR_+$. 
Denote by $\mu_{v,\ell}$ and $\sigma_{v,\ell}^2$ the expectation and the variance of $\calD_{\ell}$, respectively.  
For each $\ell \in [L]$, the $2d_{\ell}$
random variables 
$s_{\ell,1}, t_{\ell,1}, \ldots, s_{\ell,d_\ell}, t_{\ell,d_\ell}$
are exchangeable.
It follows that $\EE(s_{\ell,1}) = \EE(s_{\ell,i}) = \EE(t_{\ell,i})$, $\EE(s_{\ell,1}^2) = \EE(s_{\ell,i}^2) = \EE(t_{\ell,i}^2)$ and 
$\EE(s_{\ell,1} t_{\ell,1}) =  \EE(s_{\ell,i} s_{\ell,i'}) = \EE(t_{\ell,i} t_{\ell,i'}) = 
\EE(s_{\ell,i} t_{\ell,i}) = \EE(s_{\ell,i} t_{\ell,i'})$ for any $i, i' \in [d_{\ell}]$ with $i \neq i'$.
Moreover,
$\EE(z_{\ell,1}) = \EE(z_{\ell,i})$, $\EE(z_{\ell,1}^2) = \EE(z_{\ell,i}^2),$ and $\EE(z_{\ell,1} z_{\ell,2}) = \EE(z_{\ell,i} z_{\ell,i'}).$

We now fix a layer $\ell \in [L-1]$.
By definition, for $i \in [d_{\ell+1}]$ we have 
\begin{align}
    s_{\ell+1,i} &= \sum_{j=1}^{d_{\ell}} v_{\ell, i, j}^{(1)} z_{\ell, j} + \sum_{j=1}^d w_{\ell, i, j}^{(1)} x_j + b_{\ell, i}^{(1)}, \label{tmp_20} \\
    t_{\ell+1,i} &= \sum_{j=1}^{d_{\ell}} v_{\ell, i, j}^{(2)} z_{\ell, j} + \sum_{j=1}^d w_{\ell, i, j}^{(2)} x_j + b_{\ell, i}^{(2)}
    \label{tmp_21}
\end{align} 
with
$$z_{\ell,j} = s_{\ell,j} \vee t_{\ell,j}.$$

The key objective is to ensure
that pre-activations exhibit comparable statistical properties across all layers at initialization.
To this end, we derive fixed-point equations for the mean, variance, and covariance, 
\begin{align}
    \EE(s_{\ell+1,1}) = \EE(s_{\ell,1}), \quad
    \EE(s^2_{\ell+1,1}) = \EE(s_{\ell,1}^2), \quad 
    \EE(s_{\ell+1,1} t_{\ell+1,1}) = \EE(s_{\ell,1} t_{\ell,1}),  \label{sp_cond}
\end{align} 
and identify suitable hyperparameter configurations
$(\mu_{v,\ell}, \sigma_{v,\ell}^2, \sigma_{w,\ell}^2, \mu_{b,\ell}, \sigma_{b,\ell}^2)$.
For simplicity, we drop the subscript $\ell$ from the hyperparameters below.  
From \eqref{tmp_20} and \eqref{tmp_21}, 
\begin{align*}
    &\EE(s_{\ell+1,1}) = d_{\ell} \mu_v \EE(z_{\ell,1}) + \mu_b,\\
    &\EE(s^2_{\ell+1,1}) = \sigma_b^2 + \mu_b^2 + \sigma_w^2 \|\bx\|_2^2 + d_{\ell}(\mu_v^2 + \sigma_v^2) \EE(z_{\ell,1}^2) + d_{\ell}(d_{\ell}-1) \mu_v^2 \EE(z_{\ell,1} z_{\ell,2}) + 2d_{\ell} \mu_b \mu_v \EE(z_{\ell,1}),\\
    &\EE(s_{\ell+1,1} t_{\ell+1,1}) = \mu_b^2 + d_{\ell} \mu_v^2 \EE(z_{\ell,1}^2) + d_{\ell}(d_{\ell}-1) \mu_v^2 \EE(z_{\ell,1} z_{\ell,2}) + 2d_{\ell} \mu_b \mu_v \EE(z_{\ell,1}).
\end{align*}
Assume we already picked hyperparameters in the previous layer so that $\EE(s_{\ell,1}) =0$. 
Choosing $\mu_b = -d_{\ell} \mu_v \EE(z_{\ell,1})$ ensures then $\EE(s_{\ell+1,1}) = 0,$
$\mu_b^2 + 2d_{\ell} \mu_v \mu_b \EE(z_{\ell,1}) = -d_{\ell}^2 \mu_v^2 \EE(z_{\ell,1})^2$, and
\begin{align}
    &\EE(s^2_{\ell+1,1}) = \sigma_b^2 + \sigma_w^2 \|\bx\|_2^2 + d_{\ell}(\mu_v^2 + \sigma_v^2) \EE(z_{\ell,1}^2)  + d_{\ell}(d_{\ell}-1) \mu_v^2 \EE(z_{\ell,1} z_{\ell,2}) -d_{\ell}^2 \mu_v^2 \EE(z_{\ell,1})^2, \label{eq_sm_1}\\
    &\EE(s_{\ell+1,1} t_{\ell+1,1}) =  d_{\ell} \mu_v^2 \EE(z_{\ell,1}^2)  + d_{\ell}(d_{\ell}-1) \mu_v^2 \EE(z_{\ell,1} z_{\ell,2}) -d_{\ell}^2 \mu_v^2 \EE(z_{\ell,1})^2.\label{eq_sm_2}
\end{align}
We approximate the joint distribution of the random variables $s_{\ell,1},\ldots,s_{\ell,d_{\ell}},t_{\ell,1},\ldots,t_{\ell,d_{\ell}}$ with a Gaussian distribution with the same mean and covariance matrix. Define $A := \EE(s_{\ell,1}^2)$ and $B := \EE(s_{\ell,1} t_{\ell,1})$. Applying Lemma \ref{lemma_mul_momentum} with $(X,Y,Z,W)=(s_{\ell,1}, t_{\ell,1}, s_{\ell,2}, t_{\ell,2})$, $\sigma^2=A$, and $\rho=B/A,$ we have
\begin{align}
    \EE(z_{\ell,1}) &= \EE(s_{\ell,1} \vee t_{\ell,1}) = \sqrt{\frac{A-B}{\pi}}, \label{eq_zm_1}
    \\
    \EE(z_{\ell,1}^2) &= \EE\big((s_{\ell,1} \vee t_{\ell,1})^2\big) = A, \label{eq_zm_2} 
    \\
    \EE(z_{\ell,1} z_{\ell,2}) &= \EE\big((s_{\ell,1} \vee t_{\ell,1}) (s_{\ell,2} \vee t_{\ell,2})\big) = B + \frac{A-B}{\pi}. \label{eq_zm_3} 
\end{align}
Now, we
solve the equations \eqref{sp_cond}.
By plugging \eqref{eq_zm_1}, \eqref{eq_zm_2} and \eqref{eq_zm_3} into \eqref{eq_sm_1}, we obtain
\begin{align}
    A = \EE(s_{\ell+1,1}^2) 
    &= \sigma_b^2 + \sigma_w^2 \|\bx\|_2^2 + d_{\ell}(\mu_v^2 + \sigma_v^2) A + d_{\ell}(d_{\ell}-1) \mu_v^2 \left(B + \frac{A-B}{\pi} \right) - d_{\ell}^2 \mu_v^2 \frac{A-B}{\pi} \nonumber\\
    &= \sigma_b^2 + \sigma_w^2 \|\bx\|_2^2 + d_{\ell}(\mu_v^2 + \sigma_v^2) A + d_{\ell}(d_{\ell}-1) \mu_v^2 B -d_{\ell} \mu_v^2 \frac{A-B}{\pi}.
    \label{tmp_2}
\end{align}
Similarly, by plugging \eqref{eq_zm_1}, \eqref{eq_zm_2} and \eqref{eq_zm_3} into \eqref{eq_sm_2}, we obtain
\begin{align}
    B = \EE(s_{\ell+1, 1} t_{\ell+1, 1}) 
    &= d_{\ell} \mu_v^2 A + d_{\ell}(d_{\ell}-1) \mu_v^2 B -d_{\ell} \mu_v^2 \frac{A-B}{\pi}.
    \label{tmp_3}
\end{align}
Using $\rho = B/A,$ \eqref{tmp_3} yields
$$\mu_v^2 = \frac{B}{d_{\ell}A + d_{\ell}(d_{\ell}-1)B - \tfrac{d_{\ell}(A-B)}{\pi}} = \frac{\rho}{d_{\ell} + d_{\ell}(d_{\ell}-1)\rho - \tfrac{d_{\ell}(1-\rho)}{\pi}}.$$
Also, by subtracting \eqref{tmp_3} from \eqref{tmp_2} we get
$$A(1 - \rho) = \sigma_b^2 + \sigma_w^2 \|\bx\|_2^2 + d_{\ell}\sigma_v^2 A.$$
Since the extreme cases $\rho = 0 = \mu_v^2$ and $\rho = 1$ are undesirable, we choose $\rho = \tfrac12$.
It follows
$$\mu_v = \sqrt{\frac{1}{d_{\ell}^2 + (1-\tfrac{1}{\pi})d_{\ell}}}$$
and
$$\frac{1}{2}A = \sigma_b^2 + \sigma_w^2 \|\bx\|_2^2 + d_{\ell}\sigma_v^2 A.$$
Provided that $\|\bx\|_2^2 \approx d$, a reasonable choice for other hyperparameters would be $\sigma_b^2 = 0$, $\sigma_w^2 = 1/(4d)$ and $\sigma_v^2 = 1/(4d_{\ell})$.
Then we get $A = 1$ and $B=\tfrac12$.
Finally, we obtain
$$\mu_b = -d_{\ell} \mu_v \EE(z_{\ell, 1}) = -d_{\ell}  \sqrt{\frac{1}{d_{\ell}^2 + (1-\tfrac{1}{\pi})d_{\ell}}}
\sqrt{\frac{1}{2 \pi}} =  -  \sqrt{\frac{d_{\ell}}{2 \pi d_{\ell} + 2 \pi -2}}.$$

\begin{lemma} \label{lemma_mul_momentum}
If $\rho \in (-1,1)$ and
    $$
\begin{pmatrix}
X \\ Y \\ Z \\ W
\end{pmatrix}
\sim \mathcal{N}\!\left(
\begin{pmatrix}
0 \\ 0 \\ 0 \\ 0
\end{pmatrix},
\,
\sigma^2
\begin{pmatrix}
1      & \rho & \rho & \rho \\
\rho   & 1    & \rho & \rho \\
\rho   & \rho & 1    & \rho \\
\rho   & \rho & \rho & 1
\end{pmatrix}
\right),
$$
then, 
\begin{align*}
    \EE[X \vee Y] =  \sigma \sqrt{\frac{1-\rho}{\pi}}, \quad
\EE[(X \vee Y)^2] = \sigma^2, \quad
\EE[(X \vee Y)(Z \vee W)] = \left(\rho + \frac{1-\rho}{\pi} \right) \sigma^2.
\end{align*}
\end{lemma}
\begin{proof}
    We define 
    $$S_1 := X+Y, \quad S_2 := Z+W, \quad D_1 := X-Y, \quad D_2 := Z-W.$$
    Since $(S_1, S_2, D_1, D_2)^{\top}$ is jointly Gaussian and $\operatorname{Var}(X)=\operatorname{Var}(Y)=\operatorname{Var}(Z)=\operatorname{Var}(W)=\sigma^2$, we have $S_1 \perp D_1, S_1 \perp D_2, S_2 \perp D_1, S_2 \perp D_2$, $D_1 \perp D_2$ and
    $$X \vee Y = \frac{S_1 + |D_1|}{2}, \quad Z \vee W = \frac{S_2 + |D_2|}{2}.$$
    The absolute value of a standard Gaussian random variable has mean $\sqrt{2/\pi}$.
    From $S_1, S_2 \sim N(0, 2\sigma^2(1+\rho))$ and $D_1, D_2 \sim N(0, 2\sigma^2(1-\rho))$, we get
    \begin{align*}
        \EE[X \vee Y] = \EE\left[\frac{S_1 + |D_1|}{2}\right] = \frac{1}{2}\sqrt{\frac{4\sigma^2(1-\rho)}{\pi}} = \sigma \sqrt{\frac{1-\rho}{\pi}}
    \end{align*}
    and
    \begin{align*}
        \EE\big((X \vee Y)^2\big) 
        = \EE\left(\frac{S_1^2 + D_1^2 + 2S_1 |D_1|}{4}\right)
        = \frac{2\sigma^2(1+\rho) + 2\sigma^2(1-\rho)}{4} = \sigma^2.
    \end{align*}    
    We derive the last statement from
    \begin{align*}
        \EE\big((X \vee Y)(Z \vee W)\big) &= \EE\left(\frac{S_1 S_2 + S_1 |D_2| + S_2 |D_1| + |D_1||D_2|}{4}\right)\\
        &= \EE\left(\frac{(X+Y)(Z+W)}{4}\right) + \frac{1}{4}\EE(|D_1|)(|D_2|)\\
        &= \rho \sigma^2 + \frac{1-\rho}{\pi}\sigma^2.
    \end{align*}
\end{proof}

\section{Additional Experiments}\label{app:additional_experiments}

In this section, we provide additional experimental results and implementation details complementing the experiments described in \Cref{sec:experiments}. The section is organized as follows. In \Cref{app:initialization_visualization}, we visualize randomly initialized HyCNNs. In \Cref{app:regression_experiments}, we present additional regression experiments, including algorithmic details, univariate regression fits, results across different sample sizes and dimensions, and detailed performance tables. In \Cref{app:OT_experiments}, we present the HyCNN-based OT map estimation procedure, simulation settings, and detailed performance tables. 
All experiments were run on a MacBook Pro 16\textquotedbl{} with Apple M4 Pro with 24\,GB of unified memory and 512\,GB of storage using the integrated GPU. ChatGPT Codex 5.4 and later 6.1 Sol as well as Claude Code Opus 4.6 were used to assist in implementation and debugging, and generated code was carefully reviewed and tested by the authors. The code to some experiments is provided on GitHub at \url{https://www.github.com/hundrieser/HyCNN}.

Beyond the experiments reported in the paper, the full research project required additional compute for preliminary and exploratory work that is not included in the final manuscript. This includes the development and validation of the HyCNN initialization scheme, hyperparameter exploration (e.g., learning rate, batch size, number of epochs, width, and depth), ablations on skip-connection variants, as well as failed runs and debugging. All such runs were executed on the same hardware described above. We estimate that the total compute consumed during the project exceeded that required to reproduce the reported experiments by approximately a factor of three to five. %

\subsection{Visualization of HyCNN initialization}\label{app:initialization_visualization}

Following our initialization scheme from \Cref{sec:HyCNNs}, we display the surfaces of randomly initialized HyCNNs across different depths in \Cref{fig:HyCNN_init_vis} for input dimension $d=2$ (i.e. $\bx\in [-10,10]^2$) and width $m = 32$. Upon close inspection, we see that shallow networks behave more polyhedral, while deeper networks exhibit more smoothness. Moreover, the surfaces also do not explode at initialization as the depth increases, and are roughly centered around the origin.

\begin{figure}[h!]
    \begin{center}
    \centerline{\includegraphics[width=\textwidth, trim=0 0 0 0, clip]{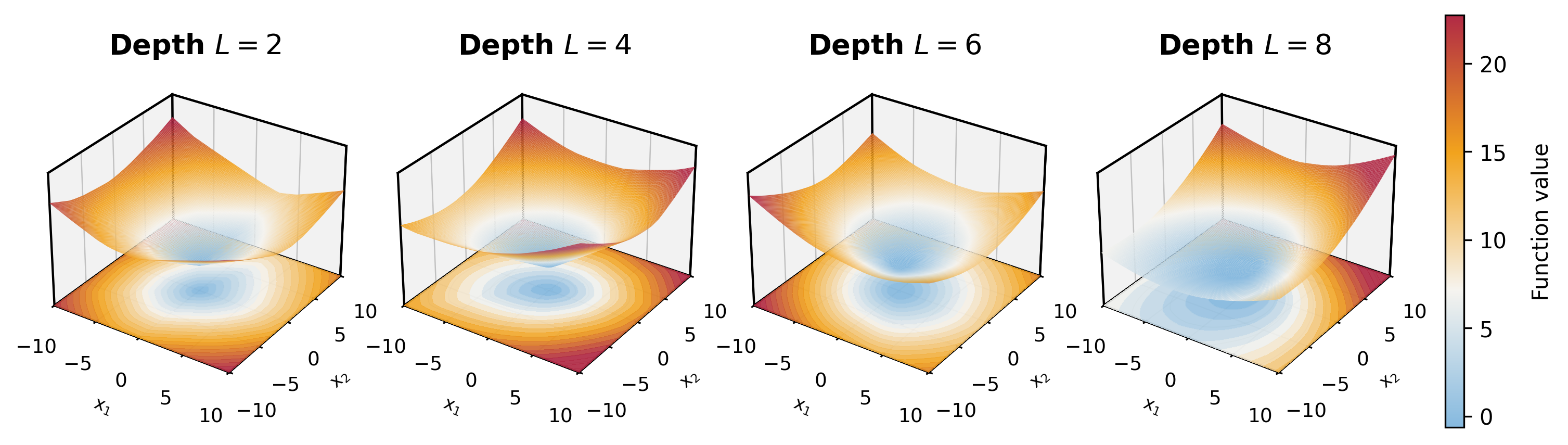}}%
        \caption{HyCNNs at initialization across different depths. Each panel shows a single randomly initialized HyCNN according to the initialization scheme from \Cref{sec:trainingHyCNNs} with maxout activation, input dimension $d=2$, domain $\bx\in [-10,10]^2$, width $m = 32$, and depths $L\in \{2,4,6,8\}$. The surface plot is shown together with its contour projection on the base plane. The visualization demonstrates that the initialized functions remain bounded and roughly centered around the origin as depth increases, and that deeper networks produce smoother surfaces compared to the more polyhedral shallow networks.}
         \label{fig:HyCNN_init_vis}
         \end{center}
\end{figure}

\subsection{Regression experiments}\label{app:regression_experiments}

In this section, we provide details and additional simulations to the regression experiments described in \Cref{sec:experiments}. We first describe the experimental setup (\Cref{app:exp_setup_regression}), then present univariate regression fits (\Cref{app:univariate_regression}), $50$-dimensional regression across different sample sizes (\Cref{app:regression_dim50}), regression across different dimensions (\Cref{app:regression_dimensions}), and interpolation across different dimensions (\Cref{app:interpolation}).
Appendices \ref{app:tau_depth} and \ref{app:branch_count} present ablation studies investigating the effects of the smoothing parameter on propagation and regression, and assess the effect of a multi-lane architecture compared to the two-lane HyCNN framework.

\subsubsection{Experimental setup}
\label{app:exp_setup_regression}
In total, we consider the following ground-truth functions:
\begin{align*}
&\textstyle \text{(a) } f_1(\bx) = \|\bx\|^2_2, &&\text{(b) } f_2(\bx) = \|\bx\|_4^4, \;&&\text{(c) } f_3(\bx) = \|\bx\|^2_2 + 0.25\sin\left(20\|\bx\|_2\right),\\
       &\textstyle\text{(d) } f_4(\bx) = \|\bx\|_1, &&\textstyle \text{(e) } f_5(\bx) =\max\left(\|\bx-\boldsymbol{\mu}\|_2^2, \|\bx+\boldsymbol{\mu}\|_2^2\right),&&\textstyle \text{(f) } f_6(\bx) = \exp(\|\bx\|_1/\sqrt{d}),
\end{align*}
where for every generated sample, we first draw the entries of the vector $\boldsymbol{\mu} \in \RR^d$ independently from $N(0,1/d)$.  %

\parabold{Data generation}
We draw training samples $(\bx_i, f(\bx_i) + \epsilon_i)_{i=1}^n$ with $\bx_i$ drawn independently from $\text{Unif}([-1,1]^d)$ and the noise $\epsilon_i$ drawn independently from $N(0,\sigma^2)$ for $\sigma^2 \geq 0$.

\parabold{Network architectures}
We compare MLPs with ReLU activation, ICNNs with ReLU activation, GroupMax networks (GMNs, \citealt{warin2023groupmax}, which correspond to maxout HyCNNs without skip connections) and HyCNNs. For each architecture, we consider networks of depth $L = 2,4,8, 16, 32$. Across the following settings, for small and large dimensions, we adjust the network widths to ensure a comparable number of trainable parameters across architectures. 

\parabold{Training}
For each architecture $\hat f_\theta$ we first normalize the data per
coordinate by subtracting the empirical mean and dividing by the empirical standard deviation along each input dimension. Each network is then trained on the standardized data, and finally the trained models are destandardized to compute the reported MSE in the original units. For training, we perform empirical risk minimization with respect to the mean-squared-error loss, i.e., we minimize
\begin{align*}
\theta \;\mapsto\; \frac{1}{n}\sum_{i=1}^{n}\big(\hat f_\theta(\bx_i) - y_i\big)^2,
\end{align*}
over the full training set $(\bx_i, y_i)_{i=1}^n$, where $y_i = f(\bx_i) + \epsilon_i$. Optimization is carried out with the Adam optimizer \citep{kingma2014adam} with learning rate $\lambda = 10^{-2}$, $\beta_1 = 0.9$, $\beta_2 = 0.999$, batch size $\min(n, 1{,}000)$, and $100$ epochs, applied jointly to all trainable parameters. For the ICNN and HyCNN architectures, the hidden-to-hidden weight matrices $V_\ell$ must remain elementwise non-negative throughout training, as this is required for the resulting network to be input-convex in $\bx$. We impose this constraint via a smooth reparametrization of the weights: we store an unconstrained parameter tensor $R_\ell \in \RR^{n_{\ell+1}\times n_\ell}$ and use, on every forward pass, the effective weight $V_\ell = \operatorname{softplus}(R_\ell) = \log(1 + \exp(R_\ell))$, applied entrywise. Because $\operatorname{softplus}\colon \RR \to (0,\infty)$ is smooth and strictly positive, the entries of $V_\ell$ are automatically non-negative for every value of $R_\ell$, and gradients flow back through the softplus so that Adam can be applied to the unconstrained $R_\ell$ directly. All other parameters (input-to-hidden weights, biases, quadratic-skip matrices where applicable) are unconstrained real-valued parameters and are updated without any reparametrization.

For initialization, we use for the HyCNNs the scheme described in \Cref{sec:HyCNNs} (the initialization acts directly on $V_\ell$ and not on $R_{\ell}$),  for the ICNN we use the initialization scheme of \citet{hoedt2023principled}, for GMNs we use the same initialization scheme as for HyCNNs while setting the skip connection weights to zero, and for MLPs we use PyTorch's default initialization scheme for \texttt{nn.Linear} layers, which draws each weight entry independently from $\mathrm{Uniform}(-1/\sqrt{d_{\mathrm{in}}},\, 1/\sqrt{d_{\mathrm{in}}})$, where $d_{\mathrm{in}}$ denotes the number of input features to the respective layer; the bias vector is sampled from the same distribution. Internally, PyTorch implements this via \texttt{kaiming\_uniform\_(a=math.sqrt(5))} \citep{he2015delving}, which with $a=\sqrt{5}$ reduces the Kaiming bound $\sqrt{6/((1+a^2)\cdot d_{\mathrm{in}})}$ to $1/\sqrt{d_{\mathrm{in}}}$.

\parabold{Evaluation}
Given a trained network $\hat f$, we evaluate the prediction MSE on an independent test set $(\tilde \bx_j)_{j=1}^{N_{\mathrm{test}}}$ of size $N_{\mathrm{test}} = 1{,}000$ drawn from $\textup{Unif}([-1,1]^d)$ (i.e., the same distribution as the training covariates but without noise), defined as
\begin{equation}\label{eq:prediction_MSE}
    \mathrm{MSE}_{\mathrm{pred}} \coloneqq \frac{1}{N_{\mathrm{test}}} \sum_{j=1}^{N_{\mathrm{test}}} \big( f(\tilde \bx_j) - \hat f(\tilde \bx_j) \big)^2,
\end{equation}
where $f$ denotes the ground-truth function. Each simulation run, except for the univariate regression which is only conducted once for visualization, is repeated $10$ times with different random seeds for the training sample generation and the network initialization.

\subsubsection{Univariate regression}\label{app:univariate_regression}

For univariate regression $(d = 1)$, we visualize the fits of the different architectures across different depths for the ground-truth functions (a--f) and different noise levels. \Cref{fig:Univariate_regression} shows the best MLP, ICNN and GMN across the different depths, and the outcome of different HyCNN architectures. The widths of the networks are chosen as $64$ for the MLP, $64$ for the ICNN, $46$ for the GMN, and $44$ for the HyCNNs to ensure similar parameter counts across the different architectures.

We see that deeper HyCNNs often exhibit a better fit than more shallow HyCNNs. Except for the function $f_3$ which is non-convex, the best MLP, ICNN and GMN across the different depths often exhibit a worse fit than the best HyCNN across the different depths. Moreover, for the non-convex function $f_3$, the best convex shape-constrained network is given by a HyCNN.

\begin{figure}[h!]
    \begin{center}
    \centerline{\includegraphics[width=\textwidth, trim=0 0 0 0, clip]{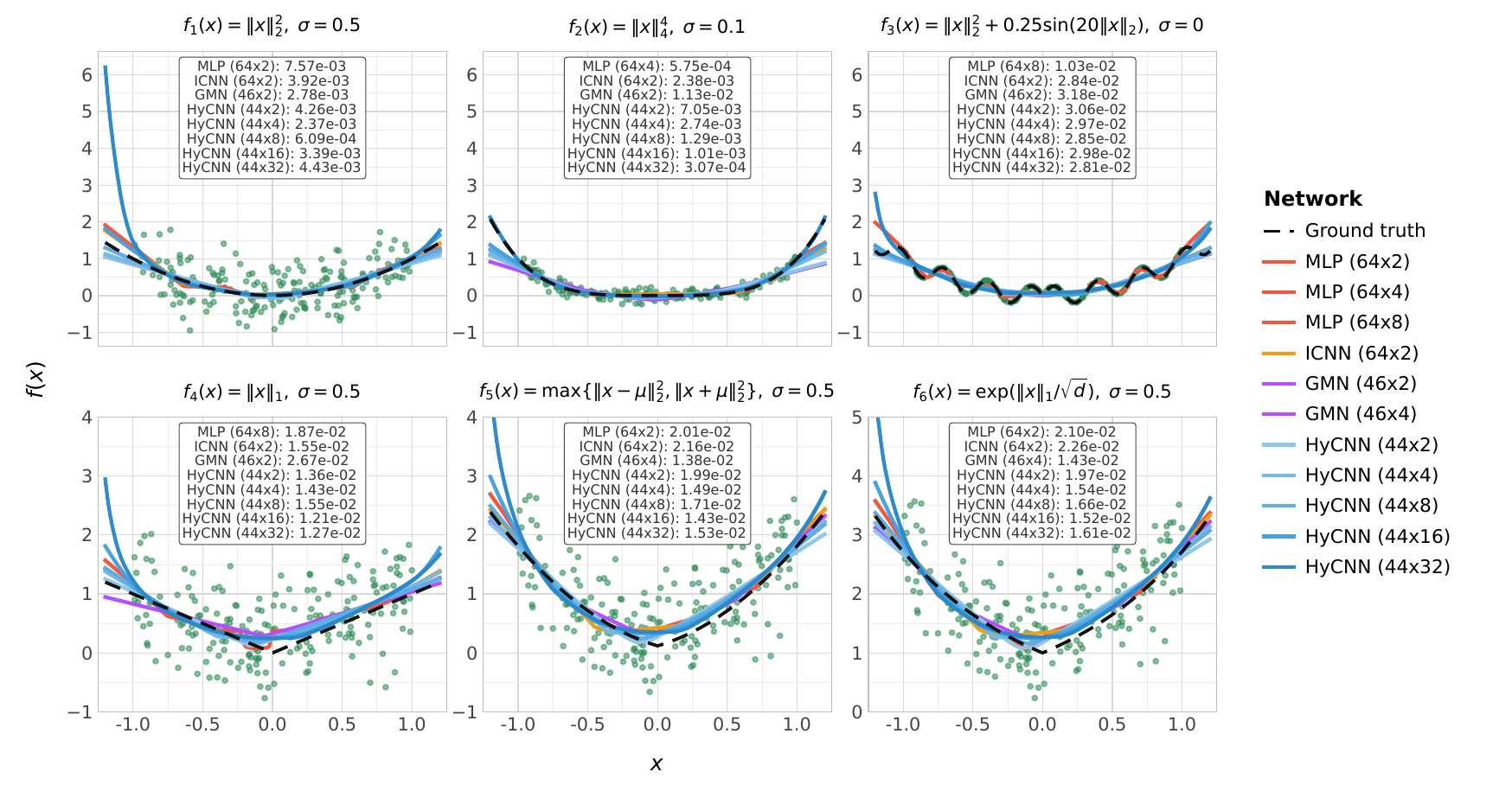}}%
        \caption{Univariate regression results ($d=1$) for the six ground-truth functions (a--f). Each panel shows HyCNNs of width $44$ and depths $L\in\{2,4,8,16,32\}$, together with the best MLP (width $64$), ICNN (width $64$), and GMN (width $46$) among the same depths. Green dots show training samples and black dashed curves show the ground-truth functions. Each panel uses one run with $n=200$ samples and the original sample and method seeds. The noise standard deviations are $\sigma=0.5$ for $f_1,f_4,f_5,f_6$, $\sigma=0.1$ for $f_2$, and $\sigma=0$ for $f_3$. The insets report the prediction MSE on $1{,}000$ independent noiseless targets.}
         \label{fig:Univariate_regression}
         \end{center}
\end{figure}

\subsubsection{Regression in dimension $d = 50$ across different sample sizes}\label{app:regression_dim50}

For the same network classes, i.e., MLPs, ICNNs, GMNs, and HyCNNs, with candidate depths $L \in \{2,4,8,16\}$, we now consider for dimension $50$ the same functions (a--f) and solve regression problems from the Gaussian noise model across different sample sizes $n \in \{100, 200, 500, 1{,}000, 2{,}000, 5{,}000, 10{,}000\}$. The widths of the MLPs are chosen as $84$, the ICNNs as $64$, the GMNs as $60$, and the HyCNNs as $40$, to ensure similar parameter counts across the different architectures, see \Cref{tab:model_visualization_dim_param_count}.
The resulting prediction MSEs are visualized in Figures \ref{fig:HyCNN_training_dim50} and \ref{fig:HyCNN_training_dim50_part2}.  

We see that deeper HyCNNs exhibit a better fit than more shallow HyCNNs across larger sample sizes $n\geq 500$. Moreover, the best MLP, ICNN and GMN across the different depths often exhibit a worse fit than most HyCNNs across depths $L\geq 4$, especially for larger sample sizes. In particular, even for the non-convex function $f_3$, the best network is nonetheless a HyCNN, which is likely due to the fact that $f_3$ is well approximated by a convex function. For sample sizes several orders of magnitude larger, it is likely that the MLP will perform best, however for the sample sizes considered here, it seems that the HyCNNs are able to more effectively leverage the inductive bias of convexity to achieve better generalization performance than the MLPs.

\begin{figure}[t!]
  \begin{center}
    \centerline{\includegraphics[width=\textwidth, trim = 0 0 0 0, clip]{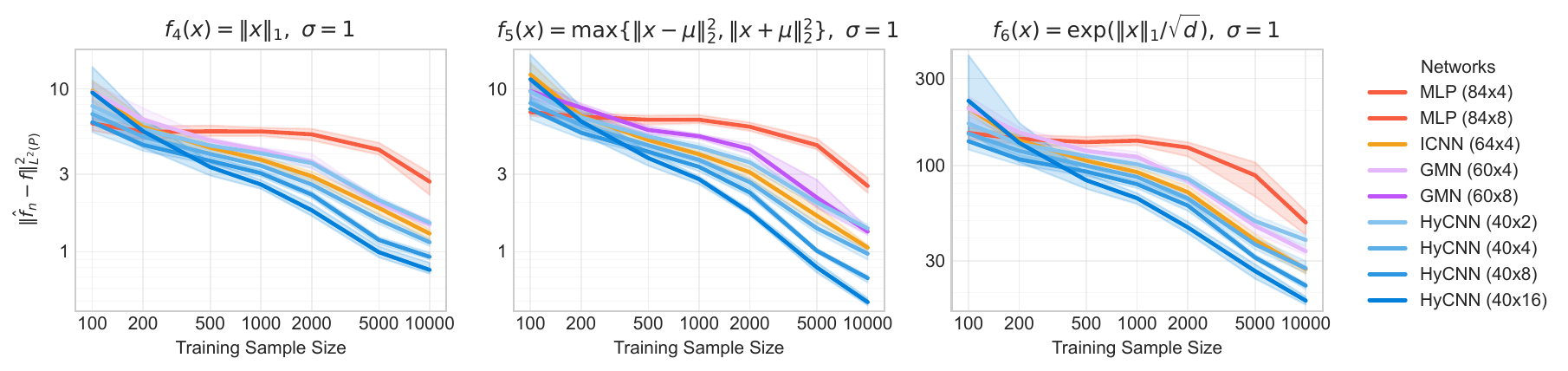}}
    \caption{Prediction MSE for different training sample sizes $n$ for the $50$-dimensional regression tasks with 10th-90th percentile bands across $10$ runs. Each panel shows HyCNNs with depths $2,4,8,16$ and the best MLP, ICNN, GMN at $n = 10^4$ among similar depths.  %
    }
    \label{fig:HyCNN_training_dim50_part2}
  \end{center}
\end{figure}

\subsubsection{Regression across different dimensions for sample size $n = 5{,}000$}\label{app:regression_dimensions}

We also computed the prediction MSE for the same network classes with depths $L \in \{2, 4,8,16,32\}$ and widths as in Appendix \ref{app:regression_dim50} for the regression tasks under the same functions (a--f) across different dimensions $d \in \{1,2,5,10,20,50\}$ for a fixed sample size of $n = 5{,}000$ in the noisy setting with additive Gaussian noise ($\sigma = 1$ for $f_1, f_2,  f_4, f_5, f_6$ and $\sigma = 0$ for $f_3$). The results are reported in \Cref{tab:model_l2} ($f_1$), \Cref{tab:model_l4} ($f_2$), \Cref{tab:model_quadratic_sinusoidal} ($f_3$), \Cref{tab:model_l1} ($f_4$), \Cref{tab:model_piecewise_quadratic} ($f_5$), and \Cref{tab:model_exp} ($f_6$).

Across all settings, there are some consistent observations. HyCNNs with depth at least four typically achieve the lowest prediction MSE, particularly in higher dimensions ($d \geq 10$), and in most examples they also consistently outperform both ICNNs and MLPs. GMNs are occasionally competitive with HyCNNs at moderate depths, which may be attributed to their implicit regularization. Deeper HyCNNs tend to improve performance as the dimension increases, with depth $L = 16$ often providing the best results for $d = 50$. Very deep architectures ($L = 32$) can exhibit instabilities due to occasional training failures, as reflected in the larger standard errors. In contrast, ICNNs and GMNs degrade substantially beyond depth $8$, specifically ICNNs at depth $32$ often exhibit numerical instabilities (MSE $> 10^{11}$), and GMNs at depth $32$ consistently diverge. Since GMNs are HyCNNs without skip connections from the input, this suggests that the skip connections play a critical role in stabilizing training of deeper architectures. MLPs, while trainable at all depths, underperform the convex architectures for most functions and dimensions, reflecting their lack of inductive bias towards convexity.
For the non-convex function $f_3$, the picture is more nuanced: in low dimensions ($d \leq 5$), the MLP performs considerably better than the convex architectures, as it is not constrained to produce convex fits. However, for $d \geq 10$ the HyCNN achieves the best prediction MSE, suggesting that the convexity bias of HyCNNs acts as a beneficial regularizer in higher dimensions, where $f_3$ is well approximated by a convex function.

\begin{table}[h!]
\caption{Number of trainable parameters by architecture, width $m$, depth $L$, and input dimension $d$. Network size is denoted by $(m\times L)$, where $L$ is the number of hidden layers. }
\label{tab:model_visualization_dim_param_count}
\scriptsize
\begin{center}
\begin{tabular}{llcccccc}
\toprule
\multicolumn{2}{c}{Network $\backslash$ $d$} & 1 & 2 & 5 & 10 & 20 & 50 \\
\midrule
\multirow{5}{*}{\rotatebox[origin=c]{90}{MLP}} & $(84\times 2)$ & 7,393 & 7,477 & 7,729 & 8,149 & 8,989 & 11,509 \\
 & $(84\times 4)$ & 21,673 & 21,757 & 22,009 & 22,429 & 23,269 & 25,789 \\
 & $(84\times 8)$ & 50,233 & 50,317 & 50,569 & 50,989 & 51,829 & 54,349 \\
 & $(84\times 16)$ & 107,353 & 107,437 & 107,689 & 108,109 & 108,949 & 111,469 \\
 & $(84\times 32)$ & 221,593 & 221,677 & 221,929 & 222,349 & 223,189 & 225,709 \\
\midrule
\multirow{5}{*}{\rotatebox[origin=c]{90}{ICNN}} & $(64\times 2)$ & 4,418 & 4,547 & 4,934 & 5,579 & 6,869 & 10,739 \\
 & $(64\times 4)$ & 12,866 & 13,123 & 13,894 & 15,179 & 17,749 & 25,459 \\
 & $(64\times 8)$ & 29,762 & 30,275 & 31,814 & 34,379 & 39,509 & 54,899 \\
 & $(64\times 16)$ & 63,554 & 64,579 & 67,654 & 72,779 & 83,029 & 113,779 \\
 & $(64\times 32)$ & 131,138 & 133,187 & 139,334 & 149,579 & 170,069 & 231,539 \\
\midrule
\multirow{5}{*}{\rotatebox[origin=c]{90}{GMN}} & $(60\times 2)$ & 7,500 & 7,620 & 7,980 & 8,580 & 9,780 & 13,380 \\
 & $(60\times 4)$ & 21,900 & 22,020 & 22,380 & 22,980 & 24,180 & 27,780 \\
 & $(60\times 8)$ & 50,700 & 50,820 & 51,180 & 51,780 & 52,980 & 56,580 \\
 & $(60\times 16)$ & 108,300 & 108,420 & 108,780 & 109,380 & 110,580 & 114,180 \\
 & $(60\times 32)$ & 223,500 & 223,620 & 223,980 & 224,580 & 225,780 & 229,380 \\
\midrule
\multirow{5}{*}{\rotatebox[origin=c]{90}{HyCNN}} & $(40\times 2)$ & 3,562 & 3,723 & 4,206 & 5,011 & 6,621 & 11,451 \\
 & $(40\times 4)$ & 10,282 & 10,603 & 11,566 & 13,171 & 16,381 & 26,011 \\
 & $(40\times 8)$ & 23,722 & 24,363 & 26,286 & 29,491 & 35,901 & 55,131 \\
 & $(40\times 16)$ & 50,602 & 51,883 & 55,726 & 62,131 & 74,941 & 113,371 \\
 & $(40\times 32)$ & 104,362 & 106,923 & 114,606 & 127,411 & 153,021 & 229,851 \\
\bottomrule
\end{tabular}%
\end{center}
\end{table}

\begin{table}[h!]
\caption{Prediction MSE and standard error (SE) for learning $f_1(\bx)=\|\bx\|_2^2$ with $n=5{,}000$ training samples, $\bX_i\sim\mathrm{Unif}([-1,1]^d)$, and Gaussian noise of standard deviation $\sigma=1$. \textbf{All entries for MSE and SE are multiplied by $10^2$ for depiction purposes}. Means and SEs are computed across ten independent training runs for independent data; networks are trained for 100 epochs. Widths are $84$ for MLP, $64$ for ICNN, $60$ for GMN, and $40$ for HyCNN, with network depths $L \in \{2, 4,8,16,32\}$ to ensure similar parameter counts across different depths at dimension $d = 50$. The smallest mean within each network family and dimension is bold and colored; the \underline{three smallest mean prediction errors} across all networks are underlined.}
\label{tab:model_l2}
\scriptsize
\makebox[\textwidth][c]{%
\resizebox{\textwidth}{!}{%
\begin{tabular}{llcccccc}
\toprule
\multicolumn{2}{c}{Network $\backslash$ $d$} & 1 & 2 & 5 & 10 & 20 & 50 \\
\midrule
\multirow{5}{*}{\rotatebox[origin=c]{90}{MLP}} & $(84\times 2)$ & \textcolor[HTML]{f85d42}{\textbf{0.23}} (0.03) & \textcolor[HTML]{f85d42}{\textbf{0.97}} (0.10) & \textcolor[HTML]{f85d42}{\textbf{10.24}} (0.53) & \textcolor[HTML]{f85d42}{\textbf{36.54}} (1.91) & \textcolor[HTML]{f85d42}{\textbf{83.12}} (2.58) & 463.32 (11.16) \\
 & $(84\times 4)$ & 0.35 (0.06) & 1.79 (0.15) & 38.19 (1.90) & 78.06 (2.85) & 107.56 (2.53) & 432.46 (11.96) \\
 & $(84\times 8)$ & 0.56 (0.10) & 1.67 (0.14) & 26.07 (2.32) & 67.91 (2.20) & 117.32 (6.76) & \textcolor[HTML]{f85d42}{\textbf{419.92}} (4.74) \\
 & $(84\times 16)$ & 2.32 (1.05) & 9.70 (2.64) & 26.38 (4.65) & 70.49 (6.26) & 137.62 (15.72) & 424.10 (9.37) \\
 & $(84\times 32)$ & 8.90 (0.07) & 17.73 (0.23) & 44.41 (0.41) & 88.12 (1.18) & 174.48 (1.96) & 440.19 (4.04) \\
\midrule
\multirow{5}{*}{\rotatebox[origin=c]{90}{ICNN}} & $(64\times 2)$ & \textcolor[HTML]{F3A11C}{\textbf{0.16}} (0.04) & \textcolor[HTML]{F3A11C}{\textbf{0.41}} (0.03) & \textcolor[HTML]{F3A11C}{\textbf{2.31}} (0.07) & \textcolor[HTML]{F3A11C}{\textbf{8.53}} (0.31) & 32.51 (0.90) & 221.73 (4.82) \\
 & $(64\times 4)$ & 1.90 (1.09) & 3.81 (2.14) & 3.43 (0.92) & 16.61 (7.96) & \textcolor[HTML]{F3A11C}{\textbf{29.71}} (1.10) & \textcolor[HTML]{F3A11C}{\textbf{189.36}} (27.34) \\
 & $(64\times 8)$ & 8.90 (0.06) & 16.18 (1.64) & 40.19 (4.00) & 88.49 (1.28) & 134.74 (18.55) & 400.70 (30.23) \\
 & $(64\times 16)$ & 8.97 (0.07) & 17.78 (0.24) & 44.52 (0.39) & 88.43 (1.26) & 175.74 (2.09) & 445.51 (3.82) \\
 & $(64\times 32)$ & 8.90 (0.06) & 17.75 (0.24) & 44.52 (0.40) & 88.45 (1.27) & 175.75 (2.10) & $>10^{10}$ ($\sim 10^{10}$) \\
\midrule
\multirow{5}{*}{\rotatebox[origin=c]{90}{GMN}} & $(60\times 2)$ & 0.16 (0.03) & 0.48 (0.04) & 2.22 (0.07) & 7.25 (0.24) & 26.88 (0.41) & 206.77 (4.18) \\
 & $(60\times 4)$ & 0.12 (0.04) & 0.37 (0.02) & \textcolor[HTML]{b842f8}{\textbf{2.14}} (0.04) & \textcolor[HTML]{b842f8}{\textbf{7.02}} (0.29) & \textcolor[HTML]{b842f8}{\textbf{25.75}} (0.50) & \textcolor[HTML]{b842f8}{\textbf{173.75}} (1.60) \\
 & $(60\times 8)$ & \textcolor[HTML]{b842f8}{\underline{\textbf{0.10}}} (0.03) & \textcolor[HTML]{b842f8}{\underline{\textbf{0.36}}} (0.02) & 2.25 (0.25) & 11.45 (1.75) & 30.73 (3.14) & 213.01 (12.95) \\
 & $(60\times 16)$ & 3.17 (1.45) & 5.17 (1.04) & 14.41 (0.57) & 36.74 (0.99) & 91.64 (1.67) & 329.94 (2.64) \\
 & $(60\times 32)$ & 1,792 (473) & 7,176 (2,797) & 3,169 (949) & 2,172 (741) & 4,615 (2,000) & 3,836 (995) \\
\midrule
\multirow{5}{*}{\rotatebox[origin=c]{90}{HyCNN}} & $(40\times 2)$ & 0.14 (0.02) & 0.51 (0.04) & 2.50 (0.09) & 7.91 (0.19) & 26.17 (0.48) & 184.94 (2.09) \\
 & $(40\times 4)$ & 0.13 (0.02) & 0.44 (0.05) & \underline{2.05} (0.05) & \textcolor[HTML]{0080DB}{\underline{\textbf{6.69}}} (0.24) & \underline{23.90} (0.64) & \underline{133.67} (3.29) \\
 & $(40\times 8)$ & 0.12 (0.03) & 0.37 (0.03) & 2.13 (0.05) & \underline{6.84} (0.29) & \textcolor[HTML]{0080DB}{\underline{\textbf{20.64}}} (0.67) & \underline{91.16} (1.53) \\
 & $(40\times 16)$ & \textcolor[HTML]{0080DB}{\underline{\textbf{0.10}}} (0.03) & \underline{0.33} (0.03) & \textcolor[HTML]{0080DB}{\underline{\textbf{1.91}}} (0.09) & \underline{6.74} (0.45) & \underline{21.69} (0.43) & \textcolor[HTML]{0080DB}{\underline{\textbf{71.65}}} (1.75) \\
 & $(40\times 32)$ & \underline{0.10} (0.03) & \textcolor[HTML]{0080DB}{\underline{\textbf{0.27}}} (0.02) & \underline{2.01} (0.25) & 9.92 (3.10) & 75.66 (41.94) & 1,017 (855) \\
\bottomrule
\end{tabular}%
}%
}
\end{table}

\begin{table}[h!]
\caption{Prediction MSE (standard error) for different models and dimensions (\textbf{all entries multiplied by $10^2$}), as in Table \ref{tab:model_l2}, but with true function $f_2(\bx) = \|\bx\|_4^4$.}
\label{tab:model_l4}
\scriptsize
\makebox[\textwidth][c]{%
\resizebox{\textwidth}{!}{%
\begin{tabular}{llcccccc}
\toprule
\multicolumn{2}{c}{Network $\backslash$ $d$} & 1 & 2 & 5 & 10 & 20 & 50 \\
\midrule
\multirow{5}{*}{\rotatebox[origin=c]{90}{MLP}} & $(84\times 2)$ & \textcolor[HTML]{f85d42}{\textbf{0.26}} (0.05) & \textcolor[HTML]{f85d42}{\textbf{0.91}} (0.08) & \textcolor[HTML]{f85d42}{\textbf{13.16}} (0.67) & \textcolor[HTML]{f85d42}{\textbf{45.95}} (1.59) & \textcolor[HTML]{f85d42}{\textbf{107.10}} (2.21) & 445.50 (11.23) \\
 & $(84\times 4)$ & 0.32 (0.05) & 1.80 (0.11) & 38.28 (2.12) & 91.17 (2.71) & 131.48 (2.21) & 399.68 (9.83) \\
 & $(84\times 8)$ & 0.61 (0.10) & 1.92 (0.26) & 31.60 (2.64) & 76.21 (4.00) & 128.36 (5.18) & 402.67 (8.70) \\
 & $(84\times 16)$ & 4.91 (0.92) & 6.88 (2.08) & 22.88 (3.83) & 57.85 (6.62) & 140.48 (1.97) & 357.13 (5.30) \\
 & $(84\times 32)$ & 7.16 (0.08) & 14.37 (0.25) & 35.80 (0.41) & 70.66 (0.92) & 140.48 (1.96) & \textcolor[HTML]{f85d42}{\textbf{351.70}} (3.31) \\
\midrule
\multirow{5}{*}{\rotatebox[origin=c]{90}{ICNN}} & $(64\times 2)$ & \textcolor[HTML]{F3A11C}{\textbf{0.19}} (0.04) & \textcolor[HTML]{F3A11C}{\underline{\textbf{0.55}}} (0.04) & \textcolor[HTML]{F3A11C}{\textbf{4.27}} (0.08) & \textcolor[HTML]{F3A11C}{\textbf{13.52}} (0.38) & \textcolor[HTML]{F3A11C}{\textbf{39.89}} (1.10) & 225.95 (3.72) \\
 & $(64\times 4)$ & 1.72 (0.84) & 3.28 (1.72) & 7.69 (3.12) & 19.54 (5.91) & 40.27 (1.52) & \textcolor[HTML]{F3A11C}{\textbf{190.77}} (17.96) \\
 & $(64\times 8)$ & 7.17 (0.08) & 13.21 (1.29) & 32.72 (3.02) & 71.09 (1.00) & 114.03 (12.66) & 328.51 (18.12) \\
 & $(64\times 16)$ & 7.23 (0.09) & 14.41 (0.26) & 35.94 (0.40) & 71.04 (0.98) & 141.49 (2.04) & 355.71 (3.44) \\
 & $(64\times 32)$ & 7.17 (0.08) & 14.38 (0.25) & 35.94 (0.39) & 71.04 (0.99) & 141.51 (2.04) & $>10^{10}$ ($\sim 10^{10}$) \\
\midrule
\multirow{5}{*}{\rotatebox[origin=c]{90}{GMN}} & $(60\times 2)$ & 0.32 (0.06) & 0.81 (0.07) & 4.62 (0.05) & 13.10 (0.26) & 36.79 (0.67) & 206.79 (4.24) \\
 & $(60\times 4)$ & 0.21 (0.04) & \textcolor[HTML]{b842f8}{\underline{\textbf{0.54}}} (0.04) & \textcolor[HTML]{b842f8}{\textbf{4.37}} (0.05) & \textcolor[HTML]{b842f8}{\textbf{12.64}} (0.31) & \textcolor[HTML]{b842f8}{\textbf{35.56}} (0.79) & \textcolor[HTML]{b842f8}{\textbf{184.60}} (2.15) \\
 & $(60\times 8)$ & \textcolor[HTML]{b842f8}{\underline{\textbf{0.17}}} (0.02) & 0.59 (0.03) & 4.54 (0.22) & 16.47 (1.13) & 39.35 (2.42) & 205.34 (8.66) \\
 & $(60\times 16)$ & 4.03 (1.50) & 5.80 (1.05) & 13.05 (0.46) & 32.36 (0.68) & 81.50 (1.96) & 276.86 (3.60) \\
 & $(60\times 32)$ & 1,780 (468) & 7,004 (2,735) & 2,983 (889) & 1,967 (671) & 4,037 (1,751) & 3,209 (835) \\
\midrule
\multirow{5}{*}{\rotatebox[origin=c]{90}{HyCNN}} & $(40\times 2)$ & 0.24 (0.03) & 0.98 (0.08) & 5.56 (0.16) & 13.51 (0.31) & 35.06 (0.55) & 193.32 (4.28) \\
 & $(40\times 4)$ & 0.25 (0.04) & 0.61 (0.05) & 4.23 (0.08) & \underline{12.15} (0.28) & \underline{33.36} (0.79) & \underline{150.90} (2.61) \\
 & $(40\times 8)$ & 0.21 (0.05) & 0.61 (0.05) & \underline{4.13} (0.06) & \textcolor[HTML]{0080DB}{\underline{\textbf{11.82}}} (0.42) & \textcolor[HTML]{0080DB}{\underline{\textbf{32.65}}} (0.80) & \underline{115.02} (2.24) \\
 & $(40\times 16)$ & \textcolor[HTML]{0080DB}{\underline{\textbf{0.13}}} (0.03) & \textcolor[HTML]{0080DB}{\underline{\textbf{0.58}}} (0.04) & \textcolor[HTML]{0080DB}{\underline{\textbf{4.05}}} (0.11) & \underline{12.35} (0.56) & \underline{34.90} (0.57) & \textcolor[HTML]{0080DB}{\underline{\textbf{97.28}}} (1.43) \\
 & $(40\times 32)$ & \underline{0.14} (0.03) & 0.63 (0.04) & \underline{4.20} (0.22) & 15.67 (3.12) & 83.10 (37.31) & 1,039 (821) \\
\bottomrule
\end{tabular}%
}%
}
\end{table}

\begin{table}[h!]
\caption{Prediction MSE (standard error) for different models and dimensions (\textbf{all entries multiplied by $10^2$}),  as in Table \ref{tab:model_l2}, but with true function $f_3(\bx) = \|\bx\|^2_2 + 0.25\sin(20\|\bx\|_2)$ and $\sigma = 0$.}
\label{tab:model_quadratic_sinusoidal}
\scriptsize
\makebox[\textwidth][c]{%
\resizebox{\textwidth}{!}{%
\begin{tabular}{llcccccc}
\toprule
\multicolumn{2}{c}{Network $\backslash$ $d$} & 1 & 2 & 5 & 10 & 20 & 50 \\
\midrule
\multirow{5}{*}{\rotatebox[origin=c]{90}{MLP}} & $(84\times 2)$ & \underline{0.32} (0.06) & \underline{2.54} (0.05) & 3.62 (0.06) & 6.45 (0.27) & 27.06 (0.82) & 359.01 (10.75) \\
 & $(84\times 4)$ & \textcolor[HTML]{f85d42}{\underline{\textbf{0.08}}} (0.01) & \textcolor[HTML]{f85d42}{\underline{\textbf{0.36}}} (0.10) & \textcolor[HTML]{f85d42}{\underline{\textbf{1.07}}} (0.08) & \textcolor[HTML]{f85d42}{\textbf{6.00}} (0.12) & \textcolor[HTML]{f85d42}{\textbf{24.01}} (0.36) & \textcolor[HTML]{f85d42}{\textbf{310.24}} (7.42) \\
 & $(84\times 8)$ & \underline{0.24} (0.10) & \underline{0.49} (0.23) & \underline{1.25} (0.20) & 6.22 (1.03) & 25.77 (2.55) & 318.23 (6.40) \\
 & $(84\times 16)$ & 1.06 (0.26) & 5.08 (2.54) & 5.93 (4.15) & 56.84 (13.39) & 119.09 (23.76) & 367.46 (31.12) \\
 & $(84\times 32)$ & 10.89 (0.11) & 21.42 (0.31) & 47.31 (0.51) & 90.59 (1.13) & 177.62 (2.07) & 442.82 (3.99) \\
\midrule
\multirow{5}{*}{\rotatebox[origin=c]{90}{ICNN}} & $(64\times 2)$ & \textcolor[HTML]{F3A11C}{\textbf{2.74}} (0.04) & \textcolor[HTML]{F3A11C}{\textbf{3.05}} (0.04) & \textcolor[HTML]{F3A11C}{\textbf{3.60}} (0.05) & \textcolor[HTML]{F3A11C}{\textbf{6.03}} (0.08) & 21.48 (0.49) & 185.38 (3.04) \\
 & $(64\times 4)$ & 4.14 (0.98) & 6.53 (2.30) & 3.77 (0.24) & 14.18 (8.54) & \textcolor[HTML]{F3A11C}{\textbf{14.76}} (0.14) & \textcolor[HTML]{F3A11C}{\textbf{155.52}} (30.56) \\
 & $(64\times 8)$ & 10.89 (0.11) & 17.97 (2.40) & 42.78 (4.14) & 90.79 (1.14) & 127.50 (23.20) & 393.62 (35.87) \\
 & $(64\times 16)$ & 10.89 (0.11) & 21.45 (0.32) & 47.39 (0.51) & 90.75 (1.14) & 178.51 (2.17) & 446.64 (3.80) \\
 & $(64\times 32)$ & 10.89 (0.11) & 21.45 (0.32) & 47.39 (0.51) & 90.79 (1.14) & 178.53 (2.17) & $>10^{10}$ ($\sim 10^{10}$) \\
\midrule
\multirow{5}{*}{\rotatebox[origin=c]{90}{GMN}} & $(60\times 2)$ & 2.86 (0.02) & 3.03 (0.03) & 3.41 (0.03) & 4.99 (0.06) & 17.77 (0.47) & 162.67 (3.60) \\
 & $(60\times 4)$ & 2.60 (0.04) & \textcolor[HTML]{b842f8}{\textbf{2.91}} (0.03) & \textcolor[HTML]{b842f8}{\textbf{3.28}} (0.05) & \textcolor[HTML]{b842f8}{\underline{\textbf{4.65}}} (0.07) & \textcolor[HTML]{b842f8}{\textbf{15.60}} (0.41) & \textcolor[HTML]{b842f8}{\textbf{135.96}} (3.25) \\
 & $(60\times 8)$ & \textcolor[HTML]{b842f8}{\textbf{2.58}} (0.02) & 2.91 (0.03) & 3.41 (0.05) & 5.89 (0.73) & 17.40 (0.95) & 182.08 (15.41) \\
 & $(60\times 16)$ & 3.13 (0.13) & 5.13 (0.47) & 11.22 (1.26) & 34.03 (0.82) & 84.45 (1.33) & 310.44 (3.16) \\
 & $(60\times 32)$ & 166.77 (44.56) & 1,260 (487) & 1,043 (311) & 1,098 (371) & 3,016 (1,283) & 3,210 (822) \\
\midrule
\multirow{5}{*}{\rotatebox[origin=c]{90}{HyCNN}} & $(40\times 2)$ & 2.85 (0.02) & 3.12 (0.04) & 3.59 (0.06) & 6.14 (0.11) & 19.16 (0.22) & 149.87 (3.17) \\
 & $(40\times 4)$ & 2.79 (0.02) & 2.95 (0.02) & 3.33 (0.04) & 4.68 (0.05) & \underline{13.57} (0.32) & \underline{109.01} (1.53) \\
 & $(40\times 8)$ & 2.72 (0.02) & 2.95 (0.04) & 3.30 (0.04) & \underline{4.11} (0.06) & \underline{9.06} (0.20) & \underline{99.14} (1.12) \\
 & $(40\times 16)$ & 2.57 (0.02) & \textcolor[HTML]{0080DB}{\textbf{2.90}} (0.03) & \textcolor[HTML]{0080DB}{\underline{\textbf{3.27}}} (0.04) & \textcolor[HTML]{0080DB}{\underline{\textbf{3.89}}} (0.05) & \textcolor[HTML]{0080DB}{\underline{\textbf{6.40}}} (0.07) & \textcolor[HTML]{0080DB}{\underline{\textbf{72.19}}} (2.30) \\
 & $(40\times 32)$ & \textcolor[HTML]{0080DB}{\textbf{2.54}} (0.02) & 2.95 (0.03) & 3.39 (0.06) & 5.79 (1.64) & 32.53 (17.59) & 751.70 (609.97) \\
\bottomrule
\end{tabular}%
}%
}
\end{table}

\begin{table}[h!]
\caption{Prediction MSE (standard error) for different models and dimensions (\textbf{all entries multiplied by $10^2$}),  as in Table \ref{tab:model_l2}, but with true function $f_4(\bx) = \|\bx\|_1$.}
\label{tab:model_l1}
\scriptsize
\makebox[\textwidth][c]{%
\resizebox{\textwidth}{!}{%
\begin{tabular}{llcccccc}
\toprule
\multicolumn{2}{c}{Network $\backslash$ $d$} & 1 & 2 & 5 & 10 & 20 & 50 \\
\midrule
\multirow{5}{*}{\rotatebox[origin=c]{90}{MLP}} & $(84\times 2)$ & \textcolor[HTML]{f85d42}{\textbf{0.25}} (0.02) & \textcolor[HTML]{f85d42}{\textbf{1.03}} (0.09) & \textcolor[HTML]{f85d42}{\textbf{12.01}} (0.56) & \textcolor[HTML]{f85d42}{\textbf{43.35}} (2.11) & \textcolor[HTML]{f85d42}{\textbf{101.20}} (1.60) & 503.41 (10.34) \\
 & $(84\times 4)$ & 0.37 (0.05) & 2.00 (0.12) & 37.81 (1.87) & 90.56 (2.45) & 130.41 (2.85) & 440.64 (8.46) \\
 & $(84\times 8)$ & 0.41 (0.05) & 1.74 (0.13) & 30.91 (2.24) & 80.75 (3.30) & 129.32 (3.41) & 433.41 (7.36) \\
 & $(84\times 16)$ & 4.45 (1.24) & 10.40 (2.37) & 19.79 (4.42) & 72.72 (6.33) & 145.20 (8.88) & \textcolor[HTML]{f85d42}{\textbf{404.26}} (9.81) \\
 & $(84\times 32)$ & 8.33 (0.05) & 16.52 (0.18) & 41.58 (0.37) & 82.59 (1.03) & 162.58 (1.65) & 413.41 (4.01) \\
\midrule
\multirow{5}{*}{\rotatebox[origin=c]{90}{ICNN}} & $(64\times 2)$ & \textcolor[HTML]{F3A11C}{\textbf{0.17}} (0.04) & \textcolor[HTML]{F3A11C}{\textbf{0.52}} (0.04) & \textcolor[HTML]{F3A11C}{\textbf{4.14}} (0.11) & \textcolor[HTML]{F3A11C}{\textbf{13.52}} (0.48) & 41.54 (0.54) & 242.05 (3.16) \\
 & $(64\times 4)$ & 1.80 (1.03) & 3.56 (1.97) & 8.27 (3.60) & 20.70 (6.78) & \textcolor[HTML]{F3A11C}{\textbf{39.70}} (1.34) & \textcolor[HTML]{F3A11C}{\textbf{213.47}} (21.38) \\
 & $(64\times 8)$ & 8.34 (0.04) & 15.02 (1.55) & 37.80 (3.56) & 82.86 (1.10) & 132.37 (14.11) & 386.25 (22.99) \\
 & $(64\times 16)$ & 8.41 (0.07) & 16.57 (0.19) & 41.67 (0.35) & 82.80 (1.10) & 163.75 (1.74) & 419.16 (3.81) \\
 & $(64\times 32)$ & 8.34 (0.04) & 16.55 (0.18) & 41.67 (0.35) & 82.82 (1.11) & 163.75 (1.78) & $>10^{10}$ ($\sim 10^{10}$) \\
\midrule
\multirow{5}{*}{\rotatebox[origin=c]{90}{GMN}} & $(60\times 2)$ & \textcolor[HTML]{b842f8}{\underline{\textbf{0.09}}} (0.02) & \textcolor[HTML]{b842f8}{\underline{\textbf{0.25}}} (0.02) & \textcolor[HTML]{b842f8}{\underline{\textbf{1.83}}} (0.08) & \textcolor[HTML]{b842f8}{\underline{\textbf{8.31}}} (0.17) & \textcolor[HTML]{b842f8}{\textbf{35.15}} (0.70) & 219.43 (2.95) \\
 & $(60\times 4)$ & 0.11 (0.03) & \underline{0.33} (0.02) & \underline{2.47} (0.11) & \underline{9.38} (0.23) & 35.17 (0.61) & \textcolor[HTML]{b842f8}{\textbf{198.67}} (2.29) \\
 & $(60\times 8)$ & \underline{0.10} (0.03) & 0.38 (0.03) & 2.92 (0.29) & 15.26 (1.85) & 40.16 (3.27) & 226.81 (11.09) \\
 & $(60\times 16)$ & 3.08 (1.59) & 5.20 (0.95) & 15.72 (0.56) & 38.74 (0.97) & 94.05 (1.95) & 321.00 (5.01) \\
 & $(60\times 32)$ & 1,784 (471) & 7,110 (2,768) & 3,107 (933) & 2,108 (720) & 4,429 (1,922) & 3,646 (950) \\
\midrule
\multirow{5}{*}{\rotatebox[origin=c]{90}{HyCNN}} & $(40\times 2)$ & \textcolor[HTML]{0080DB}{\underline{\textbf{0.10}}} (0.02) & \textcolor[HTML]{0080DB}{\underline{\textbf{0.28}}} (0.02) & \textcolor[HTML]{0080DB}{\underline{\textbf{2.44}}} (0.13) & \textcolor[HTML]{0080DB}{\underline{\textbf{10.38}}} (0.25) & 34.52 (0.31) & 205.61 (3.43) \\
 & $(40\times 4)$ & 0.11 (0.02) & 0.40 (0.04) & 2.89 (0.13) & 10.88 (0.28) & \underline{33.78} (0.55) & \underline{155.20} (2.42) \\
 & $(40\times 8)$ & 0.14 (0.03) & 0.42 (0.04) & 3.43 (0.12) & 11.13 (0.37) & \textcolor[HTML]{0080DB}{\underline{\textbf{31.29}}} (0.79) & \underline{117.15} (1.33) \\
 & $(40\times 16)$ & 0.16 (0.03) & 0.51 (0.04) & 3.56 (0.08) & 11.69 (0.46) & \underline{32.32} (0.45) & \textcolor[HTML]{0080DB}{\underline{\textbf{98.84}}} (1.14) \\
 & $(40\times 32)$ & 0.25 (0.04) & 0.65 (0.05) & 3.96 (0.27) & 14.51 (2.72) & 82.25 (38.33) & 1,042 (820) \\
\bottomrule
\end{tabular}%
}%
}
\end{table}

\begin{table}[h!]
\caption{Prediction MSE (standard error) for different models and dimensions (\textbf{all entries multiplied by $10^2$}),  as in Table \ref{tab:model_l2}, but with true function $f_5(\bx) = \max(\|\bx-\boldsymbol{\mu}\|^2, \|\bx+ \boldsymbol{\mu}\|^2)$ for a randomly drawn $\boldsymbol{\mu}$ with entries from $\calN(0,1/d)$ and  $\sigma = 1$.}
\label{tab:model_piecewise_quadratic}
\scriptsize
\makebox[\textwidth][c]{%
\resizebox{\textwidth}{!}{%
\begin{tabular}{llcccccc}
\toprule
\multicolumn{2}{c}{Network $\backslash$ $d$} & 1 & 2 & 5 & 10 & 20 & 50 \\
\midrule
\multirow{5}{*}{\rotatebox[origin=c]{90}{MLP}} & $(84\times 2)$ & \textcolor[HTML]{f85d42}{\textbf{0.72}} (0.23) & \textcolor[HTML]{f85d42}{\textbf{1.33}} (0.17) & \textcolor[HTML]{f85d42}{\textbf{9.93}} (0.62) & \textcolor[HTML]{f85d42}{\textbf{35.06}} (2.21) & \textcolor[HTML]{f85d42}{\textbf{88.85}} (2.08) & 497.18 (10.63) \\
 & $(84\times 4)$ & 0.92 (0.23) & 2.34 (0.44) & 31.04 (3.19) & 73.41 (2.45) & 107.24 (1.80) & \textcolor[HTML]{f85d42}{\textbf{437.73}} (10.81) \\
 & $(84\times 8)$ & 0.75 (0.10) & 2.38 (0.36) & 19.66 (1.98) & 62.25 (3.76) & 112.50 (2.22) & 447.43 (11.94) \\
 & $(84\times 16)$ & 7.24 (4.92) & 18.41 (8.41) & 67.21 (21.99) & 73.98 (21.01) & 194.31 (24.40) & 471.74 (21.53) \\
 & $(84\times 32)$ & 75.49 (19.84) & 98.37 (25.55) & 131.74 (13.80) & 164.71 (14.56) & 247.54 (8.99) & 509.05 (5.68) \\
\midrule
\multirow{5}{*}{\rotatebox[origin=c]{90}{ICNN}} & $(64\times 2)$ & \textcolor[HTML]{F3A11C}{\textbf{0.27}} (0.05) & \textcolor[HTML]{F3A11C}{\textbf{0.53}} (0.05) & \textcolor[HTML]{F3A11C}{\textbf{2.89}} (0.09) & \textcolor[HTML]{F3A11C}{\textbf{9.65}} (0.31) & 36.38 (0.81) & 239.99 (4.83) \\
 & $(64\times 4)$ & 22.66 (16.19) & 54.22 (30.72) & 4.52 (1.26) & 36.48 (25.71) & \textcolor[HTML]{F3A11C}{\textbf{32.51}} (1.50) & \textcolor[HTML]{F3A11C}{\textbf{206.81}} (35.14) \\
 & $(64\times 8)$ & 75.51 (19.85) & 84.53 (26.75) & 117.35 (16.28) & 165.32 (14.59) & 181.53 (29.16) & 458.51 (37.20) \\
 & $(64\times 16)$ & 75.56 (19.82) & 98.48 (25.63) & 131.81 (13.78) & 165.09 (14.57) & 249.28 (9.09) & 515.78 (5.77) \\
 & $(64\times 32)$ & 75.51 (19.85) & 98.48 (25.63) & 131.75 (13.77) & 165.13 (14.57) & 249.29 (9.12) & $>10^{10}$ ($\sim 10^{10}$) \\
\midrule
\multirow{5}{*}{\rotatebox[origin=c]{90}{GMN}} & $(60\times 2)$ & 0.18 (0.03) & 0.53 (0.04) & 2.56 (0.11) & 8.53 (0.21) & 29.84 (0.73) & 213.91 (2.79) \\
 & $(60\times 4)$ & \underline{0.15} (0.04) & \underline{0.40} (0.03) & \underline{2.36} (0.08) & \textcolor[HTML]{b842f8}{\textbf{7.89}} (0.15) & \textcolor[HTML]{b842f8}{\textbf{29.06}} (0.92) & \textcolor[HTML]{b842f8}{\textbf{190.90}} (2.55) \\
 & $(60\times 8)$ & \textcolor[HTML]{b842f8}{\underline{\textbf{0.13}}} (0.03) & \textcolor[HTML]{b842f8}{\underline{\textbf{0.40}}} (0.02) & \textcolor[HTML]{b842f8}{\underline{\textbf{2.27}}} (0.19) & 10.99 (1.44) & 31.22 (2.10) & 226.07 (13.36) \\
 & $(60\times 16)$ & 4.00 (1.38) & 13.42 (2.85) & 22.93 (3.30) & 54.97 (1.88) & 118.35 (3.06) & 355.93 (4.60) \\
 & $(60\times 32)$ & 3,000 (995) & 13,606 (6,382) & 5,186 (1,551) & 2,972 (955) & 5,751 (2,525) & 4,331 (1,102) \\
\midrule
\multirow{5}{*}{\rotatebox[origin=c]{90}{HyCNN}} & $(40\times 2)$ & 0.16 (0.03) & 0.60 (0.06) & 2.90 (0.08) & 8.70 (0.15) & 29.21 (0.59) & 202.84 (3.24) \\
 & $(40\times 4)$ & 0.20 (0.04) & 0.48 (0.03) & 2.42 (0.06) & \underline{7.83} (0.28) & \underline{25.82} (0.62) & \underline{143.14} (2.81) \\
 & $(40\times 8)$ & 0.21 (0.04) & 0.50 (0.07) & 2.39 (0.08) & \underline{7.71} (0.20) & \textcolor[HTML]{0080DB}{\underline{\textbf{22.33}}} (0.86) & \underline{100.55} (1.26) \\
 & $(40\times 16)$ & \textcolor[HTML]{0080DB}{\underline{\textbf{0.15}}} (0.03) & \textcolor[HTML]{0080DB}{\underline{\textbf{0.41}}} (0.03) & \textcolor[HTML]{0080DB}{\underline{\textbf{2.15}}} (0.08) & \textcolor[HTML]{0080DB}{\underline{\textbf{7.06}}} (0.42) & \underline{22.61} (0.27) & \textcolor[HTML]{0080DB}{\underline{\textbf{78.76}}} (1.17) \\
 & $(40\times 32)$ & 0.24 (0.05) & 0.47 (0.12) & 2.48 (0.34) & 11.02 (3.35) & 83.79 (45.14) & 1,179 (916) \\
\bottomrule
\end{tabular}%
}%
}
\end{table}

\begin{table}[h!]
\caption{Prediction MSE (standard error) for different models and dimensions (\textbf{all entries multiplied by $10^2$}),  as in Table \ref{tab:model_l2}, but with true function $f_6(\bx) = \exp(\|\bx\|_1/\sqrt{d})$ with $\sigma = 1$.}
\label{tab:model_exp}
\scriptsize
\makebox[\textwidth][c]{%
\resizebox{\textwidth}{!}{%
\begin{tabular}{llcccccc}
\toprule
\multicolumn{2}{c}{Network $\backslash$ $d$} & 1 & 2 & 5 & 10 & 20 & 50 \\
\midrule
\multirow{5}{*}{\rotatebox[origin=c]{90}{MLP}} & $(84\times 2)$ & \textcolor[HTML]{f85d42}{\textbf{0.37}} (0.06) & \textcolor[HTML]{f85d42}{\textbf{1.13}} (0.12) & \textcolor[HTML]{f85d42}{\textbf{11.99}} (0.52) & \textcolor[HTML]{f85d42}{\textbf{51.58}} (1.51) & \textcolor[HTML]{f85d42}{\textbf{253.32}} (8.46) & 10,253 (146) \\
 & $(84\times 4)$ & 0.60 (0.11) & 2.05 (0.16) & 33.95 (1.22) & 97.47 (2.86) & 277.50 (6.67) & \textcolor[HTML]{f85d42}{\textbf{9,432}} (247) \\
 & $(84\times 8)$ & 0.64 (0.13) & 2.11 (0.15) & 32.31 (2.49) & 84.22 (4.34) & 286.71 (12.79) & 9,615 (316) \\
 & $(84\times 16)$ & 0.98 (0.19) & 13.11 (5.26) & 65.21 (10.13) & 151.23 (28.01) & 574.49 (89.36) & 10,587 (318) \\
 & $(84\times 32)$ & 24.17 (0.16) & 37.47 (0.49) & 86.91 (1.10) & 218.87 (3.38) & 794.83 (10.01) & 11,077 (136) \\
\midrule
\multirow{5}{*}{\rotatebox[origin=c]{90}{ICNN}} & $(64\times 2)$ & \textcolor[HTML]{F3A11C}{\textbf{0.23}} (0.05) & \textcolor[HTML]{F3A11C}{\textbf{0.60}} (0.04) & \textcolor[HTML]{F3A11C}{\textbf{5.87}} (0.25) & \textcolor[HTML]{F3A11C}{\textbf{27.53}} (0.63) & 151.46 (2.34) & 5,412 (53) \\
 & $(64\times 4)$ & 4.99 (3.00) & 7.77 (4.53) & 6.97 (0.83) & 46.52 (19.13) & \textcolor[HTML]{F3A11C}{\textbf{127.54}} (1.17) & \textcolor[HTML]{F3A11C}{\textbf{4,615}} (657) \\
 & $(64\times 8)$ & 24.18 (0.15) & 34.15 (3.51) & 78.38 (7.75) & 219.32 (3.48) & 590.20 (93.70) & 10,412 (675) \\
 & $(64\times 16)$ & 24.24 (0.15) & 37.55 (0.50) & 87.05 (1.09) & 219.20 (3.48) & 799.24 (10.20) & 11,201 (123) \\
 & $(64\times 32)$ & 24.18 (0.15) & 37.51 (0.49) & 87.05 (1.10) & 219.29 (3.45) & 799.24 (10.42) & $>10^{11}$ ($\sim 10^{12}$) \\
\midrule
\multirow{5}{*}{\rotatebox[origin=c]{90}{GMN}} & $(60\times 2)$ & 0.17 (0.03) & \underline{0.45} (0.03) & \textcolor[HTML]{b842f8}{\underline{\textbf{2.43}}} (0.11) & \textcolor[HTML]{b842f8}{\underline{\textbf{11.68}}} (0.29) & 148.44 (3.57) & 5,216 (85) \\
 & $(60\times 4)$ & \underline{0.13} (0.04) & \textcolor[HTML]{b842f8}{\underline{\textbf{0.41}}} (0.02) & \underline{2.84} (0.13) & \underline{13.33} (0.36) & \textcolor[HTML]{b842f8}{\underline{\textbf{126.36}}} (2.46) & \textcolor[HTML]{b842f8}{\textbf{4,460}} (82) \\
 & $(60\times 8)$ & \textcolor[HTML]{b842f8}{\underline{\textbf{0.10}}} (0.03) & \underline{0.43} (0.01) & \underline{3.51} (0.29) & \underline{20.98} (2.81) & 135.15 (6.46) & 5,639 (320) \\
 & $(60\times 16)$ & 3.33 (1.26) & 8.43 (1.58) & 23.35 (0.90) & 76.80 (1.94) & 356.50 (6.27) & 8,171 (93) \\
 & $(60\times 32)$ & 1,991 (530) & 8,301 (3,238) & 4,122 (1,231) & 3,754 (1,270) & 15,390 (6,573) & 80,717 (20,600) \\
\midrule
\multirow{5}{*}{\rotatebox[origin=c]{90}{HyCNN}} & $(40\times 2)$ & 0.15 (0.02) & 0.61 (0.07) & 4.54 (0.34) & 24.43 (1.11) & 169.18 (3.02) & 4,883 (67) \\
 & $(40\times 4)$ & 0.16 (0.02) & \textcolor[HTML]{0080DB}{\textbf{0.49}} (0.04) & \textcolor[HTML]{0080DB}{\textbf{4.01}} (0.18) & \textcolor[HTML]{0080DB}{\textbf{22.38}} (0.60) & 127.19 (1.41) & \underline{3,634} (65) \\
 & $(40\times 8)$ & 0.18 (0.04) & 0.50 (0.05) & 4.82 (0.26) & 24.09 (0.41) & \underline{102.97} (1.30) & \underline{3,214} (64) \\
 & $(40\times 16)$ & \textcolor[HTML]{0080DB}{\underline{\textbf{0.14}}} (0.03) & 0.56 (0.03) & 5.17 (0.19) & 23.39 (0.82) & \textcolor[HTML]{0080DB}{\underline{\textbf{92.33}}} (0.74) & \textcolor[HTML]{0080DB}{\underline{\textbf{2,687}}} (74) \\
 & $(40\times 32)$ & 0.20 (0.04) & 0.65 (0.05) & 5.70 (0.24) & 33.67 (8.99) & 390.89 (215.62) & 21,663 (17,801) \\
\bottomrule
\end{tabular}%
}%
}
\end{table}

\clearpage
\subsubsection{Interpolation across different dimensions for sample size $n = 5{,}000$}\label{app:interpolation}

We additionally consider the noiseless interpolation setting ($\sigma = 0$) for the same collection of functions (a, b, d-f) across different dimensions $d \in \{1,2,5,10,20,50\}$ for a fixed sample size of $n = 5{,}000$. The noiseless results are reported in \Cref{tab:model_noiseless_l2} ($f_1$), \Cref{tab:model_noiseless_l4} ($f_2$), \Cref{tab:model_noiseless_l1} ($f_4$), \Cref{tab:model_noiseless_piecewise_quadratic} ($f_5$), and \Cref{tab:model_noiseless_exp} ($f_6$). For $f_3$ the prediction MSE is already reported in the noiseless setting in \Cref{tab:model_quadratic_sinusoidal}. The observations are consistent with the noisy setting, and we see that HyCNNs typically outperform all other architectures with the advantage being most pronounced in higher dimensions. Notably, in the noiseless setting, MLPs perform significantly better than ICNNs, which aligns with the theoretical considerations in \Cref{sec:approximation}, where we show that ReLU ICNNs suffer from limited expressivity due to a merely linear growth of linear pieces with depth.

\begin{table}[h!]
\caption{Prediction MSE and standard error (SE) by network and dimension (\textbf{all entries multiplied by $10^2$}) for learning the function $f_1(\bx) = \|\bx\|_2^2$ in the \emph{noiseless} setting ($\sigma = 0$). All other settings are as in \Cref{tab:model_l2}.}
\label{tab:model_noiseless_l2}
\scriptsize
\makebox[\textwidth][c]{%
\resizebox{\textwidth}{!}{%
\begin{tabular}{llcccccc}
\toprule
\multicolumn{2}{c}{Network $\backslash$ $d$} & 1 & 2 & 5 & 10 & 20 & 50 \\
\midrule
\multirow{5}{*}{\rotatebox[origin=c]{90}{MLP}} & $(84\times 2)$ & $4.7\cdot 10^{-3}$ ($2.7\cdot 10^{-3}$) & 0.02 (0.01) & \textcolor[HTML]{f85d42}{\textbf{0.21}} (0.02) & 2.21 (0.04) & 21.57 (0.48) & 340.53 (8.42) \\
 & $(84\times 4)$ & \textcolor[HTML]{f85d42}{\underline{$\mathbf{5.7\cdot 10^{-5}}$}} ($3.0\cdot 10^{-5}$) & \textcolor[HTML]{f85d42}{$\mathbf{3.3\cdot 10^{-3}}$} ($7.1\cdot 10^{-4}$) & 0.29 (0.09) & 2.49 (0.22) & \textcolor[HTML]{f85d42}{\textbf{18.82}} (0.44) & 319.93 (6.16) \\
 & $(84\times 8)$ & 0.01 (0.01) & 0.01 ($1.4\cdot 10^{-3}$) & 0.27 (0.06) & \textcolor[HTML]{f85d42}{\textbf{2.02}} (0.14) & 34.86 (14.92) & \textcolor[HTML]{f85d42}{\textbf{294.75}} (7.24) \\
 & $(84\times 16)$ & $3.8\cdot 10^{-3}$ ($2.7\cdot 10^{-3}$) & 10.90 (2.81) & 13.38 (6.29) & 79.75 (8.28) & 159.61 (14.89) & 409.68 (20.72) \\
 & $(84\times 32)$ & 8.87 (0.07) & 17.71 (0.23) & 44.38 (0.41) & 88.10 (1.19) & 174.46 (1.97) & 440.22 (4.05) \\
\midrule
\multirow{5}{*}{\rotatebox[origin=c]{90}{ICNN}} & $(64\times 2)$ & \textcolor[HTML]{F3A11C}{$\mathbf{1.8\cdot 10^{-3}}$} ($2.4\cdot 10^{-4}$) & \textcolor[HTML]{F3A11C}{\textbf{0.01}} ($1.2\cdot 10^{-3}$) & \textcolor[HTML]{F3A11C}{\textbf{0.27}} (0.01) & \textcolor[HTML]{F3A11C}{\textbf{2.49}} (0.05) & 17.61 (0.32) & 177.41 (2.44) \\
 & $(64\times 4)$ & 1.77 (1.11) & 3.48 (2.20) & 0.43 (0.22) & 10.67 (8.55) & \textcolor[HTML]{F3A11C}{\textbf{10.94}} (0.18) & \textcolor[HTML]{F3A11C}{\textbf{150.86}} (30.67) \\
 & $(64\times 8)$ & 8.87 (0.07) & 14.35 (2.28) & 39.92 (4.19) & 88.30 (1.20) & 124.61 (23.13) & 389.09 (36.85) \\
 & $(64\times 16)$ & 8.87 (0.06) & 17.73 (0.23) & 44.46 (0.40) & 88.27 (1.21) & 175.34 (2.06) & 443.96 (3.79) \\
 & $(64\times 32)$ & 8.87 (0.07) & 17.73 (0.23) & 44.46 (0.40) & 88.30 (1.20) & 175.37 (2.07) & $>10^{10}$ ($\sim 10^{10}$) \\
\midrule
\multirow{5}{*}{\rotatebox[origin=c]{90}{GMN}} & $(60\times 2)$ & $1.6\cdot 10^{-3}$ ($6.7\cdot 10^{-4}$) & 0.01 ($8.6\cdot 10^{-4}$) & 0.12 (0.01) & 1.44 (0.07) & 14.40 (0.33) & 157.64 (3.54) \\
 & $(60\times 4)$ & \textcolor[HTML]{b842f8}{$\mathbf{1.2\cdot 10^{-3}}$} ($4.0\cdot 10^{-4}$) & \textcolor[HTML]{b842f8}{$\mathbf{4.7\cdot 10^{-3}}$} ($3.6\cdot 10^{-4}$) & \textcolor[HTML]{b842f8}{\textbf{0.09}} (0.01) & \textcolor[HTML]{b842f8}{\textbf{1.18}} (0.07) & \textcolor[HTML]{b842f8}{\textbf{11.98}} (0.41) & \textcolor[HTML]{b842f8}{\textbf{137.52}} (2.91) \\
 & $(60\times 8)$ & $1.2\cdot 10^{-3}$ ($1.5\cdot 10^{-4}$) & 0.01 ($3.7\cdot 10^{-4}$) & 0.12 (0.01) & 2.47 (0.67) & 13.01 (0.77) & 177.07 (14.46) \\
 & $(60\times 16)$ & 0.14 (0.08) & 1.70 (0.42) & 7.16 (1.27) & 30.46 (0.92) & 80.86 (1.24) & 309.33 (4.02) \\
 & $(60\times 32)$ & 133.50 (36.23) & 1,044 (406) & 968.37 (286.42) & 1,059 (357) & 2,970 (1,264) & 3,190 (815) \\
\midrule
\multirow{5}{*}{\rotatebox[origin=c]{90}{HyCNN}} & $(40\times 2)$ & $7.7\cdot 10^{-4}$ ($1.8\cdot 10^{-4}$) & 0.01 ($2.9\cdot 10^{-3}$) & 0.27 (0.03) & 2.62 (0.09) & 16.34 (0.32) & 147.29 (2.73) \\
 & $(40\times 4)$ & $7.1\cdot 10^{-4}$ ($4.6\cdot 10^{-4}$) & \underline{$2.9\cdot 10^{-3}$} ($3.1\cdot 10^{-4}$) & \underline{0.09} ($3.8\cdot 10^{-3}$) & \underline{1.17} (0.02) & \underline{9.88} (0.22) & \underline{104.70} (1.30) \\
 & $(40\times 8)$ & \underline{$1.6\cdot 10^{-4}$} ($3.4\cdot 10^{-5}$) & \underline{$2.8\cdot 10^{-3}$} ($7.2\cdot 10^{-4}$) & \underline{0.04} ($1.7\cdot 10^{-3}$) & \underline{0.60} (0.01) & \underline{5.51} (0.11) & \underline{94.59} (1.39) \\
 & $(40\times 16)$ & \textcolor[HTML]{0080DB}{\underline{$\mathbf{1.2\cdot 10^{-4}}$}} ($2.1\cdot 10^{-5}$) & \textcolor[HTML]{0080DB}{\underline{$\mathbf{1.0\cdot 10^{-3}}$}} ($1.8\cdot 10^{-4}$) & \textcolor[HTML]{0080DB}{\underline{\textbf{0.03}}} ($4.6\cdot 10^{-3}$) & \textcolor[HTML]{0080DB}{\underline{\textbf{0.30}}} (0.01) & \textcolor[HTML]{0080DB}{\underline{\textbf{2.94}}} (0.05) & \textcolor[HTML]{0080DB}{\underline{\textbf{69.38}}} (1.79) \\
 & $(40\times 32)$ & $7.0\cdot 10^{-4}$ ($2.3\cdot 10^{-4}$) & 0.01 ($2.3\cdot 10^{-3}$) & 0.13 (0.04) & 2.25 (1.51) & 29.48 (17.97) & 772.69 (595.50) \\
\bottomrule
\end{tabular}%
}%
}
\end{table}

\newpage

\begin{table}[h!]
\caption{Prediction MSE and standard error (SE) by network and dimension (\textbf{all entries multiplied by $10^2$}) for learning the function $f_2(\bx) = \|\bx\|_4^4$ in the \emph{noiseless} setting ($\sigma = 0$), as in \Cref{tab:model_noiseless_l2}.}
\label{tab:model_noiseless_l4}
\scriptsize
\makebox[\textwidth][c]{%
\resizebox{\textwidth}{!}{%
\begin{tabular}{llcccccc}
\toprule
\multicolumn{2}{c}{Network $\backslash$ $d$} & 1 & 2 & 5 & 10 & 20 & 50 \\
\midrule
\multirow{5}{*}{\rotatebox[origin=c]{90}{MLP}} & $(84\times 2)$ & \textcolor[HTML]{f85d42}{$\mathbf{1.7\cdot 10^{-3}}$} ($7.5\cdot 10^{-4}$) & \textcolor[HTML]{f85d42}{\textbf{0.02}} (0.01) & \underline{0.42} (0.02) & \textcolor[HTML]{f85d42}{\textbf{8.50}} (0.25) & \textcolor[HTML]{f85d42}{\textbf{33.02}} (0.61) & 322.78 (8.01) \\
 & $(84\times 4)$ & 0.01 (0.01) & 0.03 (0.01) & \textcolor[HTML]{f85d42}{\underline{\textbf{0.38}}} (0.03) & 9.16 (0.39) & 34.36 (0.62) & 300.74 (6.57) \\
 & $(84\times 8)$ & 0.01 ($3.2\cdot 10^{-3}$) & 0.03 (0.01) & \underline{0.39} (0.03) & 8.72 (0.32) & 44.73 (10.12) & 299.67 (7.28) \\
 & $(84\times 16)$ & 3.94 (1.05) & 1.52 (1.37) & 12.61 (4.89) & 49.33 (8.93) & 95.79 (17.42) & \textcolor[HTML]{f85d42}{\textbf{284.02}} (18.91) \\
 & $(84\times 32)$ & 7.13 (0.08) & 14.34 (0.25) & 35.77 (0.40) & 70.64 (0.93) & 140.46 (1.97) & 351.73 (3.33) \\
\midrule
\multirow{5}{*}{\rotatebox[origin=c]{90}{ICNN}} & $(64\times 2)$ & \textcolor[HTML]{F3A11C}{$\mathbf{1.5\cdot 10^{-3}}$} ($2.4\cdot 10^{-4}$) & \textcolor[HTML]{F3A11C}{\underline{\textbf{0.01}}} ($8.4\cdot 10^{-4}$) & \textcolor[HTML]{F3A11C}{\textbf{0.56}} (0.06) & \textcolor[HTML]{F3A11C}{\textbf{7.62}} (0.11) & 25.91 (0.31) & 173.84 (3.04) \\
 & $(64\times 4)$ & 1.40 (0.89) & 2.85 (1.80) & 1.30 (0.17) & 14.20 (6.42) & \textcolor[HTML]{F3A11C}{\textbf{20.97}} (0.37) & \textcolor[HTML]{F3A11C}{\textbf{144.42}} (22.32) \\
 & $(64\times 8)$ & 7.13 (0.08) & 11.69 (1.86) & 32.37 (3.25) & 70.82 (0.95) & 104.91 (16.77) & 315.52 (25.65) \\
 & $(64\times 16)$ & 7.13 (0.08) & 14.36 (0.25) & 35.86 (0.40) & 70.80 (0.95) & 141.06 (2.02) & 354.30 (3.35) \\
 & $(64\times 32)$ & 7.13 (0.08) & 14.36 (0.25) & 35.86 (0.40) & 70.82 (0.95) & 141.10 (2.02) & $>10^{10}$ ($\sim 10^{10}$) \\
\midrule
\multirow{5}{*}{\rotatebox[origin=c]{90}{GMN}} & $(60\times 2)$ & $4.8\cdot 10^{-3}$ ($1.9\cdot 10^{-3}$) & 0.03 ($1.8\cdot 10^{-3}$) & \textcolor[HTML]{b842f8}{\textbf{0.52}} (0.06) & 7.65 (0.17) & 24.54 (0.64) & 161.02 (1.43) \\
 & $(60\times 4)$ & \textcolor[HTML]{b842f8}{$\mathbf{4.1\cdot 10^{-3}}$} ($1.0\cdot 10^{-3}$) & \textcolor[HTML]{b842f8}{\textbf{0.01}} ($1.2\cdot 10^{-3}$) & 0.54 (0.10) & \textcolor[HTML]{b842f8}{\textbf{7.32}} (0.10) & \textcolor[HTML]{b842f8}{\textbf{22.60}} (0.30) & \textcolor[HTML]{b842f8}{\textbf{137.61}} (1.67) \\
 & $(60\times 8)$ & 0.01 ($4.4\cdot 10^{-4}$) & 0.03 ($2.0\cdot 10^{-3}$) & 0.90 (0.07) & 8.48 (0.52) & 23.34 (0.89) & 172.69 (10.70) \\
 & $(60\times 16)$ & 0.23 (0.11) & 1.52 (0.36) & 7.80 (0.90) & 28.23 (0.59) & 70.53 (0.98) & 262.30 (3.39) \\
 & $(60\times 32)$ & 109.42 (29.38) & 838.73 (322.94) & 781.01 (230.40) & 852.66 (286.09) & 2,379 (1,008) & 2,561 (653) \\
\midrule
\multirow{5}{*}{\rotatebox[origin=c]{90}{HyCNN}} & $(40\times 2)$ & $2.4\cdot 10^{-3}$ ($4.9\cdot 10^{-4}$) & 0.03 (0.01) & 1.29 (0.14) & 9.43 (0.16) & 25.32 (0.47) & 146.98 (1.64) \\
 & $(40\times 4)$ & $1.8\cdot 10^{-3}$ ($8.1\cdot 10^{-4}$) & \textcolor[HTML]{0080DB}{\underline{\textbf{0.01}}} ($1.1\cdot 10^{-3}$) & \textcolor[HTML]{0080DB}{\textbf{0.67}} (0.08) & \underline{7.19} (0.10) & \underline{20.61} (0.37) & \underline{112.98} (1.78) \\
 & $(40\times 8)$ & \underline{$5.5\cdot 10^{-4}$} ($9.5\cdot 10^{-5}$) & \underline{0.01} ($2.3\cdot 10^{-3}$) & 0.85 (0.08) & \underline{6.56} (0.13) & \underline{17.20} (0.23) & \underline{94.98} (0.86) \\
 & $(40\times 16)$ & \underline{$7.7\cdot 10^{-4}$} ($1.5\cdot 10^{-4}$) & 0.01 ($2.7\cdot 10^{-3}$) & 1.57 (0.07) & \textcolor[HTML]{0080DB}{\underline{\textbf{6.31}}} (0.12) & \textcolor[HTML]{0080DB}{\underline{\textbf{15.51}}} (0.23) & \textcolor[HTML]{0080DB}{\underline{\textbf{78.81}}} (1.46) \\
 & $(40\times 32)$ & \textcolor[HTML]{0080DB}{\underline{$\mathbf{5.3\cdot 10^{-4}}$}} ($1.1\cdot 10^{-4}$) & 0.03 (0.01) & 2.03 (0.08) & 8.14 (1.55) & 36.81 (14.50) & 654.80 (508.01) \\
\bottomrule
\end{tabular}%
}%
}
\end{table}

\begin{table}[h!]
\caption{Prediction MSE and standard error (SE) by network and dimension (\textbf{all entries multiplied by $10^2$}) for learning the function $f_4(\bx) = \|\bx\|_1$ in the \emph{noiseless} setting ($\sigma = 0$), as in \Cref{tab:model_noiseless_l2}.}
\label{tab:model_noiseless_l1}
\scriptsize
\makebox[\textwidth][c]{%
\resizebox{\textwidth}{!}{%
\begin{tabular}{llcccccc}
\toprule
\multicolumn{2}{c}{Network $\backslash$ $d$} & 1 & 2 & 5 & 10 & 20 & 50 \\
\midrule
\multirow{5}{*}{\rotatebox[origin=c]{90}{MLP}} & $(84\times 2)$ & $1.2\cdot 10^{-3}$ ($7.6\cdot 10^{-4}$) & 0.02 (0.01) & \textcolor[HTML]{f85d42}{\textbf{0.12}} (0.04) & \textcolor[HTML]{f85d42}{\textbf{3.81}} (0.39) & 33.68 (0.76) & 372.63 (10.65) \\
 & $(84\times 4)$ & \textcolor[HTML]{f85d42}{$\mathbf{7.7\cdot 10^{-4}}$} ($5.5\cdot 10^{-4}$) & \textcolor[HTML]{f85d42}{$\mathbf{3.8\cdot 10^{-3}}$} ($6.9\cdot 10^{-4}$) & 0.39 (0.07) & 8.52 (0.82) & 33.40 (0.49) & 323.00 (7.31) \\
 & $(84\times 8)$ & $1.3\cdot 10^{-3}$ ($8.7\cdot 10^{-4}$) & 0.07 (0.03) & 0.24 (0.06) & 7.55 (0.63) & \textcolor[HTML]{f85d42}{\textbf{32.92}} (1.23) & \textcolor[HTML]{f85d42}{\textbf{309.03}} (11.15) \\
 & $(84\times 16)$ & 0.83 (0.78) & 4.91 (2.37) & 16.94 (6.34) & 51.20 (11.15) & 112.48 (18.65) & 370.59 (23.49) \\
 & $(84\times 32)$ & 8.31 (0.05) & 16.50 (0.18) & 41.56 (0.37) & 82.58 (1.04) & 162.55 (1.64) & 413.43 (4.02) \\
\midrule
\multirow{5}{*}{\rotatebox[origin=c]{90}{ICNN}} & $(64\times 2)$ & \textcolor[HTML]{F3A11C}{\textbf{0.01}} ($4.3\cdot 10^{-4}$) & \textcolor[HTML]{F3A11C}{\textbf{0.07}} (0.01) & \textcolor[HTML]{F3A11C}{\textbf{0.34}} (0.03) & \textcolor[HTML]{F3A11C}{\textbf{5.70}} (0.27) & 27.83 (0.27) & 192.76 (3.37) \\
 & $(64\times 4)$ & 1.66 (1.05) & 3.26 (2.03) & 0.68 (0.31) & 11.90 (7.68) & \textcolor[HTML]{F3A11C}{\textbf{21.51}} (0.26) & \textcolor[HTML]{F3A11C}{\textbf{163.53}} (26.01) \\
 & $(64\times 8)$ & 8.31 (0.05) & 14.96 (1.58) & 37.41 (3.86) & 82.74 (1.03) & 119.97 (19.54) & 373.88 (29.88) \\
 & $(64\times 16)$ & 8.31 (0.05) & 16.52 (0.18) & 41.60 (0.37) & 82.72 (1.04) & 163.37 (1.70) & 417.46 (3.76) \\
 & $(64\times 32)$ & 8.31 (0.05) & 16.52 (0.18) & 41.61 (0.37) & 82.74 (1.04) & 163.39 (1.74) & $>10^{10}$ ($\sim 10^{10}$) \\
\midrule
\multirow{5}{*}{\rotatebox[origin=c]{90}{GMN}} & $(60\times 2)$ & \textcolor[HTML]{b842f8}{\underline{$\mathbf{7.6\cdot 10^{-5}}$}} ($1.6\cdot 10^{-5}$) & \textcolor[HTML]{b842f8}{\underline{$\mathbf{9.3\cdot 10^{-4}}$}} ($1.2\cdot 10^{-4}$) & \textcolor[HTML]{b842f8}{\underline{\textbf{0.01}}} ($1.2\cdot 10^{-3}$) & \textcolor[HTML]{b842f8}{\underline{\textbf{0.09}}} (0.01) & 20.72 (0.48) & 175.21 (3.44) \\
 & $(60\times 4)$ & \underline{$1.9\cdot 10^{-4}$} ($4.5\cdot 10^{-5}$) & \underline{$2.7\cdot 10^{-3}$} ($7.0\cdot 10^{-4}$) & 0.04 (0.01) & \underline{0.19} (0.03) & \textcolor[HTML]{b842f8}{\underline{\textbf{18.49}}} (0.44) & \textcolor[HTML]{b842f8}{\textbf{150.33}} (2.36) \\
 & $(60\times 8)$ & $2.0\cdot 10^{-4}$ ($6.7\cdot 10^{-5}$) & 0.01 ($2.5\cdot 10^{-3}$) & 0.12 (0.04) & 2.33 (0.98) & 18.90 (1.00) & 192.45 (13.82) \\
 & $(60\times 16)$ & 0.15 (0.07) & 2.18 (0.44) & 8.59 (1.35) & 32.66 (0.93) & 82.21 (1.32) & 300.58 (2.62) \\
 & $(60\times 32)$ & 124.21 (33.69) & 978.13 (382.14) & 904.45 (267.26) & 988.83 (332.51) & 2,791 (1,190) & 2,999 (768) \\
\midrule
\multirow{5}{*}{\rotatebox[origin=c]{90}{HyCNN}} & $(40\times 2)$ & \textcolor[HTML]{0080DB}{\underline{$\mathbf{1.8\cdot 10^{-5}}$}} ($5.3\cdot 10^{-6}$) & $3.4\cdot 10^{-3}$ ($2.0\cdot 10^{-3}$) & \textcolor[HTML]{0080DB}{\underline{\textbf{0.03}}} ($4.6\cdot 10^{-3}$) & \textcolor[HTML]{0080DB}{\underline{\textbf{0.34}}} (0.04) & 23.87 (0.41) & 162.74 (1.86) \\
 & $(40\times 4)$ & $2.3\cdot 10^{-4}$ ($1.6\cdot 10^{-4}$) & \textcolor[HTML]{0080DB}{\underline{$\mathbf{1.2\cdot 10^{-3}}$}} ($3.2\cdot 10^{-4}$) & \underline{0.03} ($4.8\cdot 10^{-3}$) & 0.64 (0.04) & 18.69 (0.10) & \underline{121.16} (1.88) \\
 & $(40\times 8)$ & $2.4\cdot 10^{-4}$ ($9.2\cdot 10^{-5}$) & 0.01 ($3.3\cdot 10^{-3}$) & 0.04 ($4.0\cdot 10^{-3}$) & 1.19 (0.03) & \underline{15.94} (0.21) & \underline{104.17} (1.24) \\
 & $(40\times 16)$ & $4.5\cdot 10^{-3}$ ($1.1\cdot 10^{-3}$) & 0.02 (0.01) & 0.18 (0.03) & 2.25 (0.15) & \textcolor[HTML]{0080DB}{\underline{\textbf{14.50}}} (0.16) & \textcolor[HTML]{0080DB}{\underline{\textbf{85.06}}} (1.81) \\
 & $(40\times 32)$ & 0.02 ($2.2\cdot 10^{-3}$) & 0.07 (0.01) & 0.68 (0.15) & 5.30 (1.32) & 36.93 (15.52) & 775.87 (572.90) \\
\bottomrule
\end{tabular}%
}%
}
\end{table}

\begin{table}[h!]
\caption{Prediction MSE and standard error (SE) by network and dimension (\textbf{all entries multiplied by $10^2$}) for learning the function $f_5(\bx) = \max(\|\bx-\boldsymbol{\mu}\|_2^2, \|\bx+\boldsymbol{\mu}\|_2^2)$ in the \emph{noiseless} setting ($\sigma = 0$) based on $n = 5{,}000$ data points $\bX_i \sim \text{Unif}([-1,1]^d)$ and $Y_i = f_5(\bX_i)$, where $\boldsymbol{\mu} \in \RR^d$ has entries independently drawn from $\calN(0,1/d)$. All other settings are as in \Cref{tab:model_l2}.}
\label{tab:model_noiseless_piecewise_quadratic}
\scriptsize
\makebox[\textwidth][c]{%
\resizebox{\textwidth}{!}{%
\begin{tabular}{llcccccc}
\toprule
\multicolumn{2}{c}{Network $\backslash$ $d$} & 1 & 2 & 5 & 10 & 20 & 50 \\
\midrule
\multirow{5}{*}{\rotatebox[origin=c]{90}{MLP}} & $(84\times 2)$ & 0.04 (0.04) & \textcolor[HTML]{f85d42}{\underline{\textbf{0.01}}} ($3.0\cdot 10^{-3}$) & 0.35 (0.05) & \textcolor[HTML]{f85d42}{\textbf{2.89}} (0.12) & 24.34 (0.50) & 380.24 (10.93) \\
 & $(84\times 4)$ & \textcolor[HTML]{f85d42}{\underline{$\mathbf{2.5\cdot 10^{-3}}$}} ($1.4\cdot 10^{-3}$) & 0.02 (0.01) & \textcolor[HTML]{f85d42}{\textbf{0.33}} (0.03) & 2.96 (0.20) & \textcolor[HTML]{f85d42}{\textbf{24.08}} (1.73) & 330.84 (5.96) \\
 & $(84\times 8)$ & 0.03 (0.01) & 0.07 (0.03) & 0.37 (0.07) & 3.60 (0.78) & 24.21 (0.75) & \textcolor[HTML]{f85d42}{\textbf{323.04}} (11.85) \\
 & $(84\times 16)$ & 0.04 (0.01) & 0.06 (0.01) & 10.99 (9.48) & 31.18 (26.19) & 84.02 (30.31) & 361.58 (40.33) \\
 & $(84\times 32)$ & 75.45 (19.84) & 98.33 (25.56) & 131.70 (13.80) & 164.71 (14.56) & 247.55 (9.02) & 509.07 (5.70) \\
\midrule
\multirow{5}{*}{\rotatebox[origin=c]{90}{ICNN}} & $(64\times 2)$ & \textcolor[HTML]{F3A11C}{\textbf{0.04}} (0.01) & \textcolor[HTML]{F3A11C}{\textbf{0.08}} (0.02) & \textcolor[HTML]{F3A11C}{\textbf{0.54}} (0.04) & \textcolor[HTML]{F3A11C}{\textbf{3.05}} (0.11) & 19.84 (0.37) & 197.34 (3.25) \\
 & $(64\times 4)$ & 22.42 (16.23) & 46.00 (31.06) & 1.23 (0.72) & 29.86 (26.37) & \textcolor[HTML]{F3A11C}{\textbf{12.65}} (0.33) & \textcolor[HTML]{F3A11C}{\textbf{169.03}} (38.58) \\
 & $(64\times 8)$ & 75.46 (19.85) & 81.91 (27.46) & 116.40 (16.85) & 165.12 (14.57) & 171.89 (33.37) & 448.44 (42.94) \\
 & $(64\times 16)$ & 75.46 (19.83) & 98.45 (25.65) & 131.77 (13.81) & 164.98 (14.57) & 248.85 (9.13) & 514.14 (5.59) \\
 & $(64\times 32)$ & 75.46 (19.84) & 98.48 (25.67) & 131.70 (13.79) & 165.03 (14.56) & 248.87 (9.15) & $>10^{10}$ ($\sim 10^{10}$) \\
\midrule
\multirow{5}{*}{\rotatebox[origin=c]{90}{GMN}} & $(60\times 2)$ & 0.01 (0.01) & 0.04 (0.01) & 0.31 (0.04) & 2.20 (0.12) & 16.08 (0.38) & 167.98 (3.53) \\
 & $(60\times 4)$ & \textcolor[HTML]{b842f8}{\underline{$\mathbf{2.5\cdot 10^{-3}}$}} ($8.5\cdot 10^{-4}$) & \textcolor[HTML]{b842f8}{\underline{\textbf{0.01}}} ($1.3\cdot 10^{-3}$) & \textcolor[HTML]{b842f8}{\underline{\textbf{0.17}}} (0.01) & \textcolor[HTML]{b842f8}{\textbf{1.63}} (0.13) & \textcolor[HTML]{b842f8}{\textbf{13.79}} (0.43) & \textcolor[HTML]{b842f8}{\textbf{149.01}} (2.67) \\
 & $(60\times 8)$ & $3.2\cdot 10^{-3}$ ($9.3\cdot 10^{-4}$) & 0.02 ($4.5\cdot 10^{-3}$) & 0.23 (0.03) & 2.94 (0.74) & 14.33 (0.76) & 191.32 (16.92) \\
 & $(60\times 16)$ & 0.41 (0.16) & 8.20 (2.44) & 14.35 (3.27) & 45.05 (2.37) & 104.09 (2.02) & 343.18 (3.59) \\
 & $(60\times 32)$ & 1,397 (631) & 7,515 (4,700) & 2,970 (909) & 1,872 (581) & 4,111 (1,791) & 3,693 (925) \\
\midrule
\multirow{5}{*}{\rotatebox[origin=c]{90}{HyCNN}} & $(40\times 2)$ & $4.5\cdot 10^{-3}$ ($1.5\cdot 10^{-3}$) & 0.07 (0.02) & 0.61 (0.06) & 3.06 (0.11) & 17.23 (0.39) & 161.93 (2.22) \\
 & $(40\times 4)$ & $4.1\cdot 10^{-3}$ ($2.6\cdot 10^{-3}$) & 0.02 (0.01) & 0.20 (0.02) & \underline{1.43} (0.05) & \underline{11.10} (0.31) & \underline{114.28} (1.83) \\
 & $(40\times 8)$ & \textcolor[HTML]{0080DB}{\underline{$\mathbf{1.7\cdot 10^{-3}}$}} ($5.2\cdot 10^{-4}$) & 0.02 (0.01) & \textcolor[HTML]{0080DB}{\underline{\textbf{0.10}}} (0.01) & \underline{0.77} (0.03) & \underline{6.24} (0.10) & \underline{100.81} (1.23) \\
 & $(40\times 16)$ & $4.4\cdot 10^{-3}$ ($1.9\cdot 10^{-3}$) & \textcolor[HTML]{0080DB}{\underline{\textbf{0.01}}} ($4.5\cdot 10^{-3}$) & \underline{0.11} (0.02) & \textcolor[HTML]{0080DB}{\underline{\textbf{0.45}}} (0.01) & \textcolor[HTML]{0080DB}{\underline{\textbf{3.75}}} (0.13) & \textcolor[HTML]{0080DB}{\underline{\textbf{77.29}}} (2.43) \\
 & $(40\times 32)$ & 0.04 (0.01) & 0.13 (0.09) & 0.40 (0.11) & 3.25 (2.05) & 37.74 (22.52) & 803.29 (657.12) \\
\bottomrule
\end{tabular}%
}%
}
\end{table}

\begin{table}[h!]
\caption{Prediction MSE and standard error (SE) by network and dimension (\textbf{all entries multiplied by $10^2$}) for learning the function $f_6(\bx) = \exp(\|\bx\|_1/\sqrt{d})$ in the \emph{noiseless} setting ($\sigma = 0$), as in \Cref{tab:model_noiseless_l2}.}
\label{tab:model_noiseless_exp}
\scriptsize
\makebox[\textwidth][c]{%
\resizebox{\textwidth}{!}{%
\begin{tabular}{llcccccc}
\toprule
\multicolumn{2}{c}{Network $\backslash$ $d$} & 1 & 2 & 5 & 10 & 20 & 50 \\
\midrule
\multirow{5}{*}{\rotatebox[origin=c]{90}{MLP}} & $(84\times 2)$ & $3.5\cdot 10^{-3}$ ($1.5\cdot 10^{-3}$) & \textcolor[HTML]{f85d42}{\textbf{0.02}} ($5.0\cdot 10^{-3}$) & \textcolor[HTML]{f85d42}{\textbf{0.33}} (0.06) & \textcolor[HTML]{f85d42}{\textbf{15.70}} (1.00) & 186.32 (4.73) & 10,099 (263) \\
 & $(84\times 4)$ & \textcolor[HTML]{f85d42}{\underline{$\mathbf{1.4\cdot 10^{-3}}$}} ($7.5\cdot 10^{-4}$) & 0.02 (0.01) & 0.43 (0.04) & 24.14 (0.72) & 184.65 (5.77) & 9,596 (203) \\
 & $(84\times 8)$ & $4.5\cdot 10^{-3}$ ($2.8\cdot 10^{-3}$) & 0.03 (0.01) & 0.64 (0.11) & 45.19 (19.28) & \textcolor[HTML]{f85d42}{\textbf{177.00}} (7.57) & \textcolor[HTML]{f85d42}{\textbf{9,239}} (304) \\
 & $(84\times 16)$ & 0.01 ($4.0\cdot 10^{-3}$) & 7.32 (4.57) & 35.61 (12.93) & 143.26 (30.61) & 385.14 (89.22) & 10,469 (429) \\
 & $(84\times 32)$ & 24.14 (0.16) & 37.44 (0.49) & 86.88 (1.10) & 218.86 (3.39) & 794.79 (10.00) & 11,077 (136) \\
\midrule
\multirow{5}{*}{\rotatebox[origin=c]{90}{ICNN}} & $(64\times 2)$ & \textcolor[HTML]{F3A11C}{\textbf{0.01}} ($9.4\cdot 10^{-4}$) & \textcolor[HTML]{F3A11C}{\textbf{0.05}} ($4.4\cdot 10^{-3}$) & \textcolor[HTML]{F3A11C}{\textbf{0.44}} (0.03) & \textcolor[HTML]{F3A11C}{\textbf{17.41}} (0.41) & 139.56 (2.25) & 5,424 (88) \\
 & $(64\times 4)$ & 4.80 (3.03) & 7.34 (4.62) & 0.93 (0.38) & 34.75 (20.34) & \textcolor[HTML]{F3A11C}{\textbf{110.77}} (1.11) & \textcolor[HTML]{F3A11C}{\textbf{4,608}} (656) \\
 & $(64\times 8)$ & 24.14 (0.16) & 30.36 (4.79) & 78.07 (8.00) & 219.25 (3.39) & 594.97 (91.47) & 9,992 (843) \\
 & $(64\times 16)$ & 24.15 (0.16) & 37.51 (0.49) & 87.02 (1.09) & 219.21 (3.42) & 799.00 (10.23) & 11,198 (123) \\
 & $(64\times 32)$ & 24.14 (0.16) & 37.50 (0.49) & 87.02 (1.09) & 219.27 (3.38) & 799.00 (10.44) & $>10^{11}$ ($\sim 10^{12}$) \\
\midrule
\multirow{5}{*}{\rotatebox[origin=c]{90}{GMN}} & $(60\times 2)$ & $2.3\cdot 10^{-3}$ ($6.8\cdot 10^{-4}$) & \underline{0.02} ($1.6\cdot 10^{-3}$) & \textcolor[HTML]{b842f8}{\underline{\textbf{0.10}}} (0.01) & \textcolor[HTML]{b842f8}{\underline{\textbf{0.72}}} (0.08) & 129.55 (2.10) & 5,233 (101) \\
 & $(60\times 4)$ & $2.1\cdot 10^{-3}$ ($4.5\cdot 10^{-4}$) & \textcolor[HTML]{b842f8}{\underline{\textbf{0.01}}} ($1.8\cdot 10^{-3}$) & \underline{0.14} (0.02) & \underline{1.06} (0.08) & 111.92 (2.32) & \textcolor[HTML]{b842f8}{\textbf{4,405}} (77) \\
 & $(60\times 8)$ & \textcolor[HTML]{b842f8}{$\mathbf{1.6\cdot 10^{-3}}$} ($2.5\cdot 10^{-4}$) & 0.02 ($3.0\cdot 10^{-3}$) & 0.32 (0.06) & 5.59 (1.82) & \textcolor[HTML]{b842f8}{\underline{\textbf{109.45}}} (3.76) & 5,644 (350) \\
 & $(60\times 16)$ & 0.33 (0.17) & 3.57 (0.82) & 13.27 (1.99) & 68.30 (1.75) & 344.22 (4.99) & 8,138 (125) \\
 & $(60\times 32)$ & 361.52 (98.19) & 2,214 (860) & 1,917 (564) & 2,657 (896) & 13,734 (5,833) & 80,085 (20,435) \\
\midrule
\multirow{5}{*}{\rotatebox[origin=c]{90}{HyCNN}} & $(40\times 2)$ & $2.3\cdot 10^{-3}$ ($2.4\cdot 10^{-4}$) & 0.03 (0.01) & 0.20 (0.02) & 5.62 (1.42) & 158.98 (2.42) & 4,968 (83) \\
 & $(40\times 4)$ & $1.7\cdot 10^{-3}$ ($7.8\cdot 10^{-4}$) & \textcolor[HTML]{0080DB}{\underline{\textbf{0.01}}} ($1.1\cdot 10^{-3}$) & \textcolor[HTML]{0080DB}{\underline{\textbf{0.15}}} (0.02) & \textcolor[HTML]{0080DB}{\underline{\textbf{3.54}}} (0.21) & 110.24 (1.30) & \underline{3,715} (77) \\
 & $(40\times 8)$ & \textcolor[HTML]{0080DB}{\underline{$\mathbf{8.2\cdot 10^{-4}}$}} ($1.7\cdot 10^{-4}$) & 0.02 ($4.7\cdot 10^{-3}$) & 0.20 (0.02) & 6.63 (0.18) & \underline{87.98} (0.87) & \underline{3,237} (60) \\
 & $(40\times 16)$ & \underline{$1.5\cdot 10^{-3}$} ($5.5\cdot 10^{-4}$) & 0.03 (0.01) & 0.61 (0.10) & 9.55 (0.44) & \textcolor[HTML]{0080DB}{\underline{\textbf{75.23}}} (1.32) & \textcolor[HTML]{0080DB}{\underline{\textbf{2,707}}} (73) \\
 & $(40\times 32)$ & 0.01 ($1.7\cdot 10^{-3}$) & 0.10 (0.02) & 1.64 (0.26) & 20.91 (7.22) & 305.51 (159.72) & 21,612 (17,560) \\
\bottomrule
\end{tabular}%
}%
}
\end{table}

\newpage

\clearpage

{\reb
\subsubsection{Effect of the smoothing parameter on propagation and regression}\label{app:tau_depth}

We seek to understand the effect of the smoothing parameter $\tau\in\{0,0.5,1,2,5\}$ in \eqref{eq:logsumexp} on the performance of the HyCNN of width $W=40$ with input dimension $d=50$, and depths $L\in\{2,4,6,8,10,12,14,16\}$. Recall that for $\tau= 0$, the activation function $\overline \sigma_\tau$ is the maxout unit, whereas for $\tau>0$ it is the log-sum-exp function.

\parabold{Signal propagation and gradient profiles at initialization}
To assess the effect on the signal and gradient profiles at initialization, we consider $5{,}000$ pairs $(\bX_i, Y_i)$ of observations, with $\bX_i\sim\calN(0,I_{50})$ and $Y_i=\|\bX_i\|_2^2+\epsilon_i$, where $\epsilon_i\sim\calN(0,1)$ are independent noise terms. We investigate the respective norm of the neuron activations across layers as well as the gradient norms with respect to the parameters under the least squares loss. To increase stability we standardize the covariates and responses first to have zero mean and unit variance across all dimensions. The gradients are depicted with respect to these standardized inputs and outputs. Moreover, for the hidden-to-hidden gradients we consider the unconstrained parameterization $R_\ell$ with $V_\ell=\operatorname{softplus}(R_\ell)$ where the softplus function is applied elementwise. This procedure is repeated $100$ times across independent seeds for the data and the initialization weights. \Cref{fig:tau_depth_initialization} summarizes the resulting signal and gradient profiles at initialization. 

Larger values of the smoothing parameter $\tau$ lead to more pronounced growth in both the signal and hidden-to-hidden gradient norms across layers at initialization. In particular, $\tau=2$ and $\tau=5$ exhibit relatively large amplification effects for large depth, which we attribute to the fact that the log-sum-exp exhibits a certain positive offset compared to the maxout unit and the initialization is currently tailored for the small $\tau=0$ regime. 

\parabold{Signal propagation and gradient profiles after training}
To exemplify the effect of the smoothing parameter in dimension $d= 50$ on the training and fit of the HyCNN, we consider the same regression tasks as in Appendix \ref{app:exp_setup_regression} for $f_1(\bx)=\|\bx\|_2^2$ and $f_4(\bx)=\|\bx\|_1$ with standard normal noise $(\sigma = 1)$ and no noise $(\sigma = 0)$, and train HyCNNs with depths $L\in\{2,4,6,8,10,12,14,16\}$ and smoothing parameter $\tau\in\{0,0.5,1,2,5\}$. The results are shown in \Cref{fig:tau_depth_fit}. A small choice of the smoothing parameter $\tau$ leads to similar performance. In some instances, the HyCNN with $\tau = 0.5$ even outperforms the $\tau = 0$ HyCNN, however for $\tau\geq 1$ the performance consistently deteriorates with a more pronounced performance for depth $L\geq 10$.  When considering the signal propagation and gradient norm plots in \Cref{fig:tau_depth_trained} for the learned networks in regression task $f_1(\bx)=\|\bx\|_2^2$ and noise $\sigma = 1$, we find that the hidden-to-hidden gradient norm is roughly one order of magnitude larger for deeper networks under $\tau= 5$ compared to the regime where $\tau = 0$. This indicates that deep HyCNNs with large $\tau$ were not properly trained in this regime. This behavior is at odds with some of the optimal transport simulations in \Cref{app:OT_implementation} where choosing $\tau$ large appears to help the training process under smooth ground truth OT map. We hypothesize that since the OT map estimator is given by the gradient of a HyCNN, choosing $\tau$ large serves as some benign gradient regularization, whereas for the regression tasks it does not help. 

\begin{figure}[htbp]
    \centering
    \includegraphics[width=\textwidth, trim=0 0 0 40, clip]{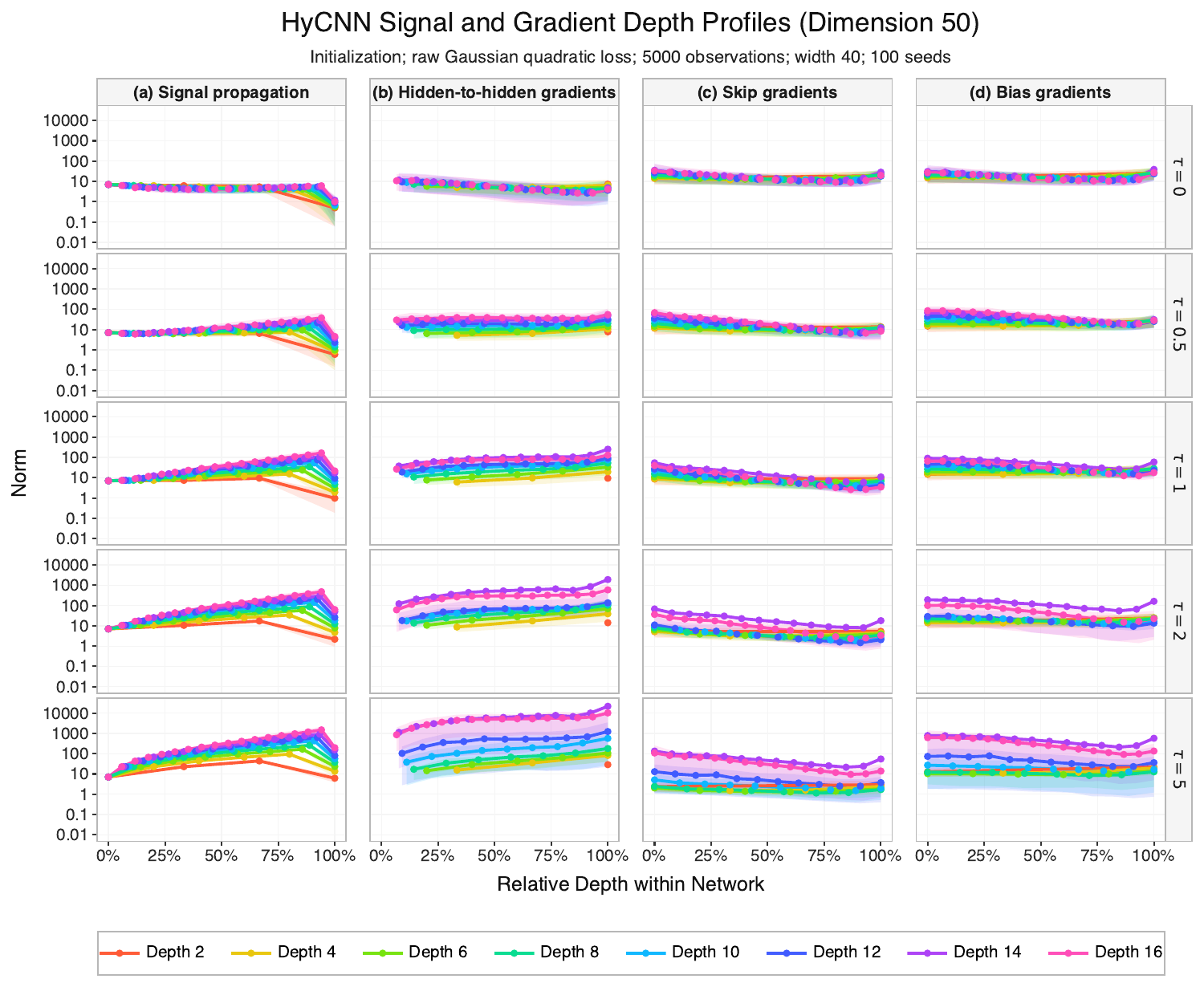}
    \caption{Signal propagation and gradient norm profiles across different layers at initialization of HyCNNs for $\tau=0,0.5,1,2,5$ (top to bottom) and width $40$ for different depths. 
    The relative depth within network refers to the respective layer index normalized by network depth, with $0\%$ referring to the input, and $100\%$ to the scalar output of the network.
        \textbf{(a)} Signals are computed from a single forward pass at initialization with standard Gaussian input $\bX \sim \calN(0,I_d)$. Each curve is the mean hidden-state $\ell_2$-norm of the neurons over 100 random seeds, and shaded bands show the $10$th$-$$90$th percentile range. $0\%$ corresponds to the input norm $\|\bX\|_2$ (which is close to $\sqrt{d} \approx 7.07$) and $100\%$ to the final hidden layer.
    \textbf{(b-d)} Gradients are computed from a single backward pass at initialization for least-squares regression on synthetic data of size $n = 5{,}000$, with $\bX_i \sim \calN(0, I_d)$ and $Y_i = \|\bX_i\|_2^2 + \epsilon_i$, where $\epsilon_i \sim \calN(0,1)$. Each  curve is the mean of the gradient norm for (b) hidden-to-hidden weight, (c) skip weights, and (d) bias weights, and shaded bands show the $10$th$-$$90$th percentile range.
    }
    \label{fig:tau_depth_initialization}
\end{figure}

\begin{figure}[htbp]
    \centering
    \includegraphics[width=\textwidth, trim =0 0 0 20, clip]{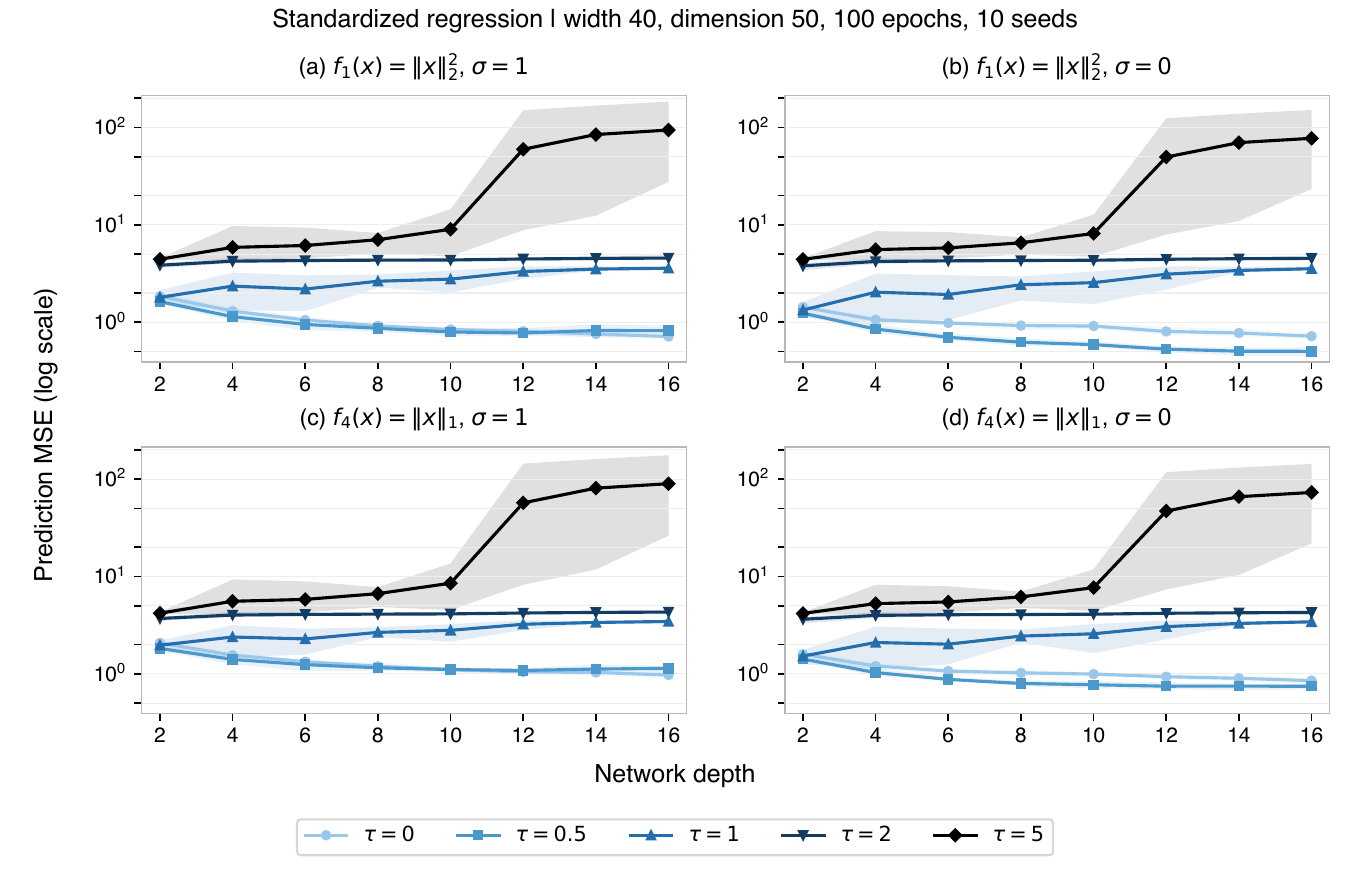}
    \caption{Prediction MSE of trained HyCNNs for width $40$ and input dimension $50$, for different depths and smoothing parameters $\tau \in \{0,0.5,1,2,5\}$. The HyCNNs are trained for $100$ epochs with Adam based on $5{,}000$ data points and batch size $1{,}000$. As part of training the data is standardized to have mean zero and variance one across all dimensions, and the networks are afterward rescaled and evaluated on unstandardized inputs and targets. Shaded bands show the $10$th$-$$90$th percentile range.}
    \label{fig:tau_depth_fit}
\end{figure}

\begin{figure}[htbp]
    \centering
    \includegraphics[page=1,width=\textwidth, trim =0 0 0 45, clip]{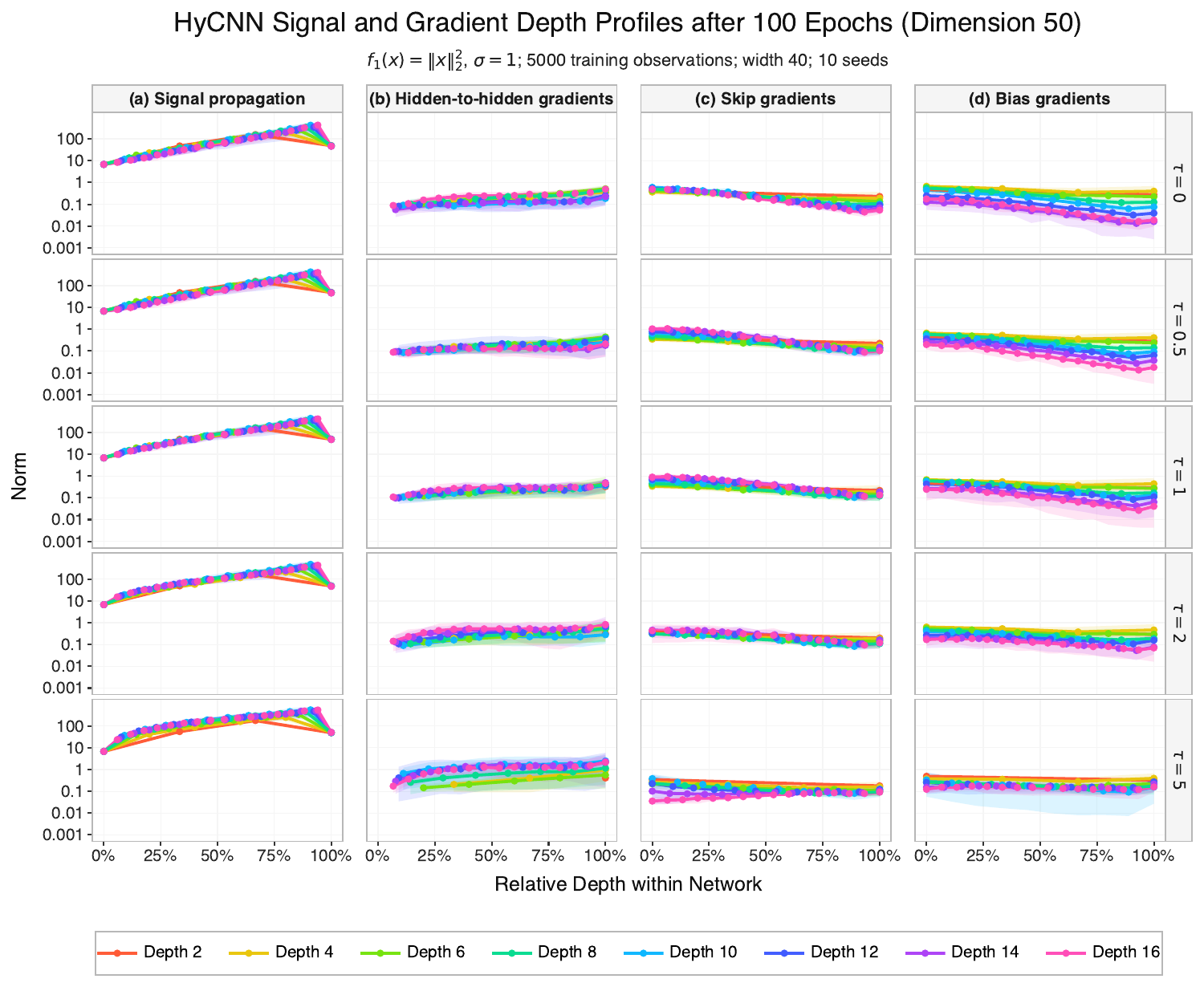}
    \caption{Signal propagation and gradient profiles after $100$ epochs from the  networks in \Cref{fig:tau_depth_fit}, for $f_1$ and $\sigma=1$. Similar to the layout of \Cref{fig:tau_depth_initialization}.}
    \label{fig:tau_depth_trained}
\end{figure}
\subsubsection{Effect of the number of activation branches}\label{app:branch_count}

One of our core theoretical insights is that choosing a two-lane maxout activation yields a substantial gain of the expressive power of HyCNNs over ICNNs with a single-branch layout. This raises the question whether the use of further branches is beneficial.

To formalize the multi-branch setting,  we define the $K$-branch HyCNN for $K \in \NN, K\geq 2$ as \begin{align*}
    \mathrm{HyCNN}_K(\bx) = \bz_{L+1}= V_{L}\bz_L+W_{L}\bx+b_{L}, \quad \bz_\ell = \max_{1\leq k\leq K}\bigl(V_{\ell-1}^k\bz_{\ell-1}+W_{\ell-1}^k\bx+\bb_{\ell-1}^k\bigr),
\end{align*}
where $\bz_0=\mathbf{0}_d$, the maximum is taken coordinatewise, and $V_\ell^k$ are entrywise non-negative matrices. 

To compare the performance of the multi-branch HyCNNs, we conduct experiments with $K = 2,3,4$ and varying depths $L \in \{2,4,6,8,10,12,14,16\}$ for regression tasks with $f_1(\bx)=\|\bx\|_2^2,$ $f_4(\bx)=\|\bx\|_1,$ and $\sigma \in \{0,1\}$. The width of the networks is chosen to lead to roughly comparable parameter counts across different depths, with widths $40$, $30$, and $24$ for $K=2,3,4$, respectively. We also consider $d=50$ and $n=5{,}000$ data points per repetition. For the initialization, we use the same initialization scheme as in Section~\ref{sec:HyCNNs}. We consider training protocols analogous to those described in Appendix~\ref{app:regression_dimensions} and the simulation results are shown in \Cref{fig:branches_depth_prediction}. 

The results indicate that increasing the number of activation branches beyond two does not lead to further improvements in prediction performance. Indeed, the two-branch HyCNN consistently achieves the lowest mean prediction MSE across all target functions, noise levels, and network depths considered in our experiments.
It is worth noting that the initialization is tailored to the two-branch HyCNN, and no specific recalibration is performed for the three- or four-branch variants. This may contribute to the observed performance differences, particularly at larger depths. However, even at smaller depths, where the initialization mismatch should be less present, the two-branch model still outperforms the higher-branch variants, suggesting that the two-branch structure itself provides a favorable balance between expressivity and stable trainability.

\begin{figure}[htbp]
    \centering
    \includegraphics[width=\textwidth]{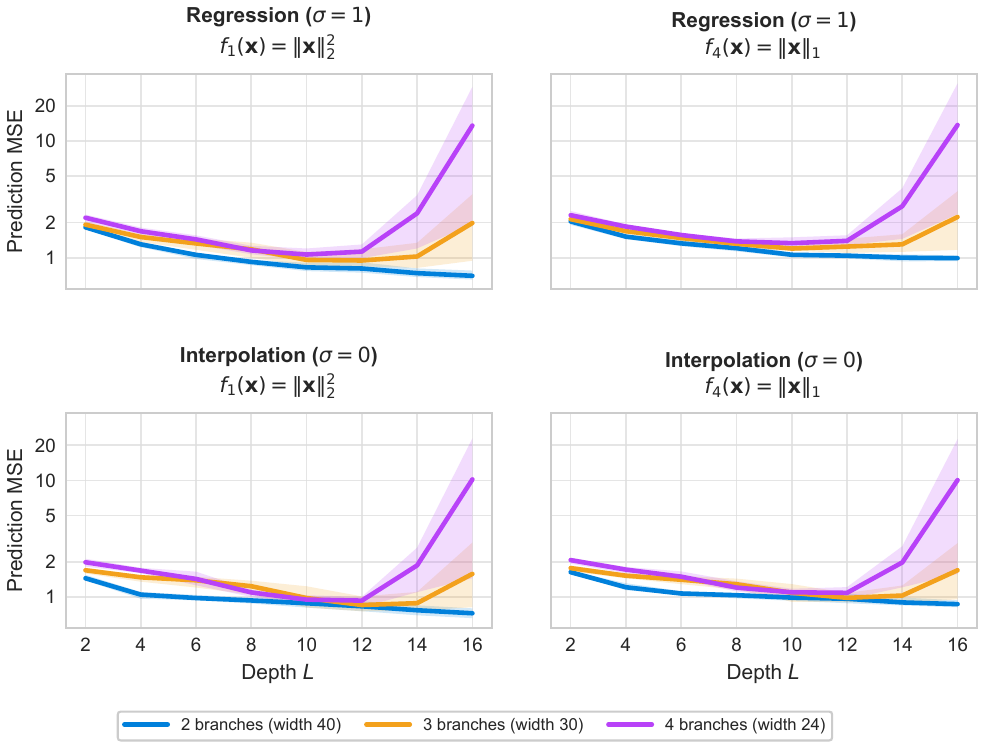}
    \caption{Performance of multi-branch HyCNNs for $K \in \{2, 3, 4\}$ activation branches for regression and interpolation tasks with $n=5{,}000$ sample points in dimension $d=50$. Curves show the prediction MSE across $10$ independent runs, and the shaded regions depict the $10$th--$90$th percentiles. }
    \label{fig:branches_depth_prediction}
\end{figure}

} 
\clearpage

\subsection{HyCNN-based OT map estimation}\label{app:OT_experiments}\label{app:OT_implementation}

In this subsection, we provide additional experimental results and implementation details complementing the optimal transport map estimation experiments described in \Cref{sec:experiments}. The training procedure for HyCNN-based OT map estimation, based on the saddle-point formulation \eqref{eq:HyCNN_OT_potential_estimation}, is summarized in \Cref{alg:HyCNN_ot}.

As in the regression experiments of \Cref{app:regression_experiments}, we enforce non-negativity of the hidden-to-hidden weight matrices in both the HyCNN potential $f_{\theta_f}$ and the HyCNN critic $g_{\theta_g}$ via the entrywise softplus reparametrization $V_\ell^{k} = \operatorname{softplus}(R_\ell^{k}) = \log(1+\exp(R_\ell^{k}))$, where $R_\ell^{k} \in \RR^{n_{\ell+1}\times n_\ell}$ is unconstrained. The same reparametrization is used for the hidden-to-output weights. Adam updates the unconstrained matrices $R_\ell^{k}$, with gradients propagated through the softplus function. Thus, the effective weights remain non-negative and both networks remain input-convex throughout training, without requiring projection onto the non-negative orthant.

\begin{algorithm}[h!]
 \caption{HyCNN optimal transport map learning}
 \label{alg:HyCNN_ot}
 \begin{algorithmic}
   \STATE {\bfseries Input:} data $(\bx_i)_{i=1}^n$ from source $P$, data $(\by_j)_{j=1}^m$ from target $Q$, batch size $M$, outer and inner loop iterations $S_{\textup{out}},S_{\textup{in}}\in \NN$, HyCNN potential $f_{\theta_f} \in \texttt{HyCNN}_\tau(\RR^d)$ and HyCNN critic $g_{\theta_g} \in \texttt{HyCNN}_\tau(\RR^d)$.
   \STATE Initialize $\theta_f$ and $\theta_g$ according to \Cref{sec:trainingHyCNNs}.
   \FOR{$t=1$ {\bfseries to} $S_{\textup{out}}$}
       \STATE Sample mini-batch $(\bX_i)_{i=1}^M$ from $(\bx_i)_{i=1}^n$.
       \FOR{$s=1$ {\bfseries to} $S_{\textup{in}}$}
           \STATE Sample mini-batch $(\bY_j)_{i=1}^M$ from $(\by_j)_{i=1}^m$.
           \STATE Update $\theta_g$ to maximize $\textstyle\frac{1}{M}\sum_{i=1}^{M} \Big[ f_{\theta_f}(\bX_i)+\langle \bY_i, \nabla g_{\theta_g}(\bY_i)\rangle - f_{\theta_f}(\nabla g_{\theta_g}(\bY_i)) \Big]$ with Adam.
       \ENDFOR
       \STATE Update $\theta_f$ to minimize $\textstyle\frac{1}{M}\sum_{i=1}^{M} \Big[f_{\theta_f}(\bX_i)+\langle \bY_i, \nabla g_{\theta_g}(\bY_i)\rangle - f_{\theta_f}(\nabla g_{\theta_g}(\bY_i)) \Big]$ using Adam.
   \ENDFOR
   \STATE {\bfseries Output:} {Estimated optimal transport map $\hat T\coloneqq \nabla \hat \phi$, where $\hat\phi = f_{\theta_f}$.}
 \end{algorithmic}
\end{algorithm}

\parabold{Network architectures}
Both the potential $f_{\theta_f}$ and the critic $g_{\theta_g}$ are parameterized by HyCNNs with the smooth log-sum-exp activation function $\overline\sigma_\tau$ from \eqref{eq:logsumexp}, using the same architecture (width and depth).

\parabold{Training}
For all experiments, we use the Adam optimizer \citep{kingma2014adam} with $\beta_1 = 0.5$ and $\beta_2 = 0.9$ and identical initial learning rate $\lambda = 10^{-2}$ for both $\theta_f$ and $\theta_g$. The learning rate follows a cosine decay schedule $\lambda_t = \lambda_{\min} + \tfrac{1}{2}(\lambda - \lambda_{\min})(1 + \cos(\pi t / S_{\textup{out}}))$ with final ratio $\lambda_{\min}/\lambda = 0.01$. Non-negativity of the effective weights in both networks is maintained by the softplus reparametrization described above, with Adam updating the corresponding unconstrained parameters.

Notably, unlike the ICNN-based neural OT method by \cite{makkuva2020optimal}, we do not allow the hidden-to-hidden weight matrices $V_\ell^j$ in the critic network to have negative entries, as we found that it degrades performance of HyCNN-based OT map estimation in high-dimensional settings. 

\parabold{Training in low dimensions} 
While discontinuities of the OT map can occur in any dimension, in low dimensions it is crucial to accurately resolve these discontinuities, as they have a pronounced effect on the quality of the estimated map. This requires the HyCNN to operate close to the maxout unit (i.e., $\tau \to 0$ in the log-sum-exp activation $\overline\sigma_\tau$). However, training with small $\tau$ is challenging because the maxout unit may select a single gate per layer, so that only the weights associated with the active gate receive gradient updates while the remaining gates are not trained. To address this, we employ a cyclic cosine annealing schedule for the smoothing parameter $\tau$, that is, we start from $\tau_0 = 1$ and then periodically oscillate $\tau$ between large and small values and decrease per cycle by a fixed decay factor. 
During phases of large $\tau$, the smooth log-sum-exp activation distributes gradients across all gates, ensuring that all weight matrices are updated. During phases of small $\tau$, the network sharpens toward the piecewise-affine $\max$ regime needed to represent discontinuous maps.

\parabold{Training in moderate and high dimensions}
In higher dimensions ($d \geq 10$), recovering the OT map is already more involved, and empirically larger values of $\tau$ yielded better performance. We therefore use a fixed smoothing parameter $\tau$ without the cyclic cosine schedule in these settings. 

\subsubsection{OT map estimation in dimension $d = 2$}\label{app:OT_dim2}

For the two-dimensional OT visualization experiments, we consider the following examples: \begin{enumerate}
    \item[$(i)$] five-component checkerboard (source) to four-component checkerboard (target), 
    \item[$(ii)$] standard normal (source) to half-moon (target),
    \item[$(iii)$] shifted half-moon (source) to shifted rotated half-moon (target), and 
    \item[$(iv)$] standard Gaussian (source) to a mixture of eight Gaussians arranged in a ring (target). 
\end{enumerate}
In Figure \ref{fig:HyCNN_OT_shapes_2d} we depict the training samples from the source distribution and target distribution as well as the pushforward of the source training samples under the learned HyCNN OT map $\nabla f$. In Figure \ref{fig:HyCNN_OT_shapes_2d_reverse} we show the pushforward of the target training samples under the learned gradient of the critic $g$, which corresponds to the reverse OT map. In Figure \ref{fig:HyCNN_OT_shapes_2d_potential}, we visualize the learned potentials $f$ and $g$ for the four examples using a contour plot. Throughout these examples, we observe that HyCNNs are capable of learning the correct target distribution (and thus the correct OT map) even in the presence of discontinuities. 

For the architecture we use HyCNNs with width $40$ and depth $4$ for both the potential $f$ and the critic $g$. We draw $n = m = M = 2{,}000$ data points from the source and target distributions. The training consists of $S_{\textup{out}} = 2500$ outer iterations for learning $f$, with $S_{\textup{in}} =10$ inner gradient steps per outer iteration for updating the critic $g$. The learning rate follows a cosine decay schedule with initial learning rate $\lambda = 10^{-2}$ and final ratio $\lambda_{\min}/\lambda = 0.01$. The smoothing parameter $\tau$ for the log-sum-exp activation is also decayed. We start with $\tau = 1$, and within a cycle of $100$ outer iterations, we decay the smoothing parameter $\tau$ by a factor of $0.8$ per cycle, with a cosine decaying schedule and ensure that by iteration $70\%$ of the cycle, $\tau$ has decayed to $0.1$ times its initial value in the cycle. Overall, this ensures that $\tau$ reaches a small value of $0.8^{25} \cdot 0.1 = 0.00038$ by the end of training, which allows the network to sharpen toward the piecewise-affine $\max$ regime needed to represent discontinuous maps, while also ensuring that all gates receive gradient updates during training.

\begin{figure}[t!]
    \begin{center}
    \centerline{\includegraphics[width=\textwidth, trim=0 0 0 0, clip]{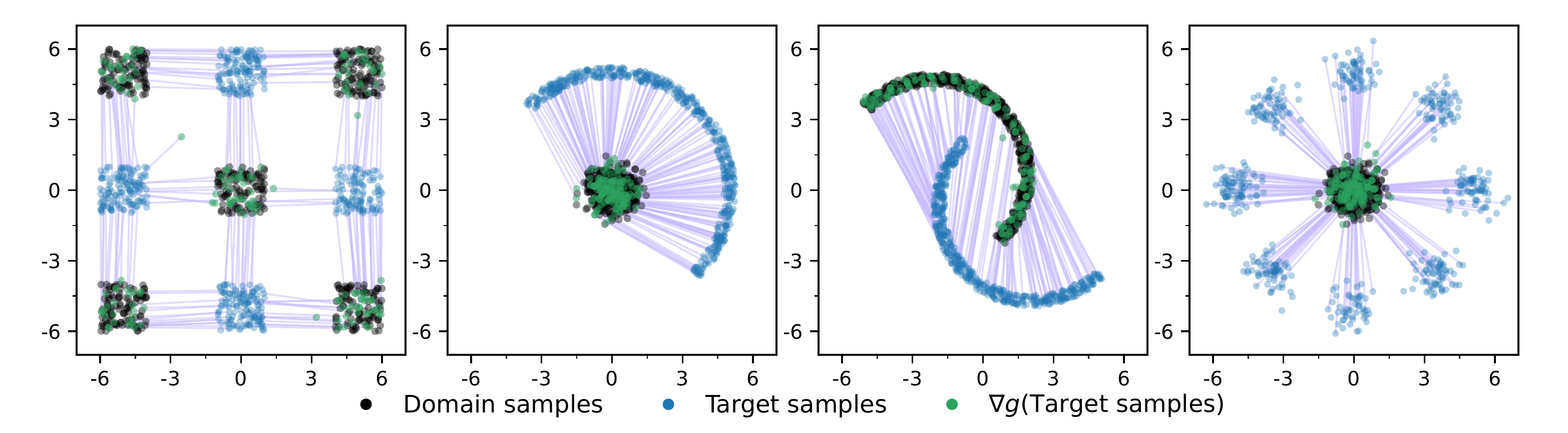}}
        \caption{Reverse OT map estimation in dimension $d=2$. Each panel shows the pushforward of the target distribution under the learned gradient of the learned HyCNN $g$ (width $40$, depth $4$) from the inner optimization problem in \Cref{alg:HyCNN_ot}, trained with $n = m = M = 2{,}000$ data points, $S_{\textup{out}} = 2500$ outer iterations, and $S_{\textup{in}} = 10$ inner steps, using Adam with cosine-decayed learning rate ($\lambda = 10^{-2}$, final ratio $0.01$) and smoothing parameter $\tau_0 = 1$ (cosine decay). From left to right: (i) five-component checkerboard $\to$ four-component checkerboard, (ii) standard normal $\to$ half-moon, (iii) shifted half-moon $\to$ shifted rotated half-moon, (iv) standard Gaussian $\to$ mixture of eight Gaussians.}
         \label{fig:HyCNN_OT_shapes_2d_reverse}
         \end{center}
\end{figure}

\begin{figure}[t!]
    \begin{center}
    \centerline{\includegraphics[width=\textwidth, trim=0 0 0 0, clip]{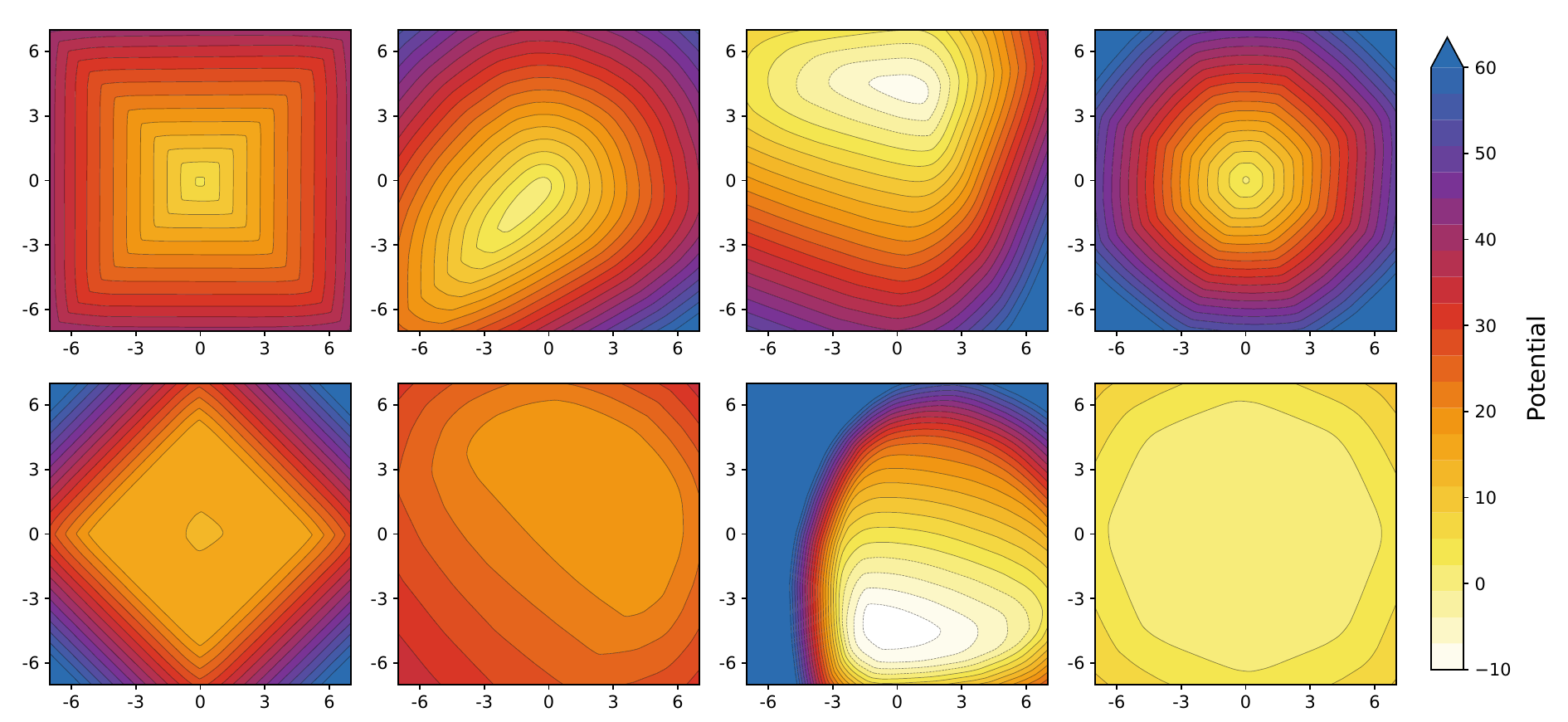}}
        \caption{Contour plots of the learned OT potential $\hat\phi$ in dimension $d=2$. Each panel shows the contour lines of the HyCNN potential (width $40$, depth $4$) learned from $n = m = M = 2{,}000$ data points for the same collection of source--target pairs as in \Cref{fig:HyCNN_OT_shapes_2d_reverse}.} %
         \label{fig:HyCNN_OT_shapes_2d_potential}
         \end{center}
\end{figure}

\subsubsection{OT map estimation across different dimensions for sample sizes $n = m = 5{,}000$}\label{app:OT_dimensions}

In the following, we present detailed performance tables for the OT map estimation experiments described in \Cref{sec:experiments}. In total, we consider the following OT potentials with corresponding OT maps, 
\begin{align*}
    &\textstyle \text{(a) } \textstyle\phi_1(\bx) = \frac{\|\bx\|^2_2}{2}, && T_1(\bx) = \bx, \\ &\text{(b) } \textstyle\phi_2(\bx) = \frac{1}{2}\sum_{i = 1}^{d} \big(1+ \frac{\sin(i)}{2}\big)x_i^2, 
    && T_2(\bx) = \textstyle\big((1+ \frac{\sin(i)}{2})x_i\big)_{i \in [d]},\\
    &\textstyle \text{(c) } \textstyle\phi_3(\bx) = \frac{\|\bx\|^2_2}{2}+ \|\bx\|_1,&& T_3(\bx) = \big(x_i + \textup{sign}(x_i)\big)_{i \in [d]},\\
    &\text{(d) } \textstyle\phi_4(\bx) = \|\bx\|^4_4,&& T_4(\bx) = \big(4x_i^3\big)_{i\in [d]}.
\end{align*}

we report the prediction MSE of the estimated OT map $\hat T = \nabla \hat \phi$ against the ground-truth OT map $\overline T = \nabla \phi_\ell$ for dimension $d = 50$. The source distribution is $P = \calN(0, I_d)$ for (a--c) and $P = \text{Unif}[-1,1]^d$ for (d), and the target distribution is $Q = (\nabla \phi_\ell)_\# P$. We use $n = m = 5{,}000$ independent samples from $P$ and $Q$, respectively. The MSE is defined as the average squared Euclidean distance between the estimated and ground-truth OT map evaluated on an independent test set, i.e., $$\text{MSE}_{\text{pred}} = \frac{1}{N_{\text{test}}} \sum_{i=1}^{N_{\text{test}}} \|\hat T(\bX_i) - \overline T(\bX_i)\|_2^2.$$

The test set consists of $N_{\text{test}} = 1{,}000$ independent samples from $P$ and $Q$. Results are averaged over $10$ independent runs with different random seeds for sample generation and network initialization. The training uses batch size $M = 256$, $S_{\textup{out}} = 1{,}000$ outer iterations, $S_{\textup{in}} = 5$ inner steps, learning rate $\lambda = 10^{-2}$ with cosine decay with final ratio $0.01$. The Adam optimizer uses $\beta_1 = 0.5$, $\beta_2 = 0.9$ for both the potential and critic networks.

We compare the HyCNN-based estimator (width $44$, depths $L\in\{2,4,6\}$, log-sum-exp smoothness $\tau\in\{0.1,1,10\}$) with the entropic OT map estimator ($\epsilon\in\{0.1,1,10\}$), the Monge gap estimator (MLP width $104$, depths $L\in\{2,4,6\}$), and multiple ICNN variants (width $64$, depths $L\in\{2,4,6\}$). Every configuration is evaluated over ten repetitions and the widths of the neural networks are chosen to ensure comparable numbers of parameters, see \Cref{tab:ot_total_parameter_counts}. In the following, we describe how these other estimators are implemented.

\parabold{Entropic OT map estimator} As a first nonparametric baseline we consider the entropic OT (EOT) map estimator \citep{seguy2017large,pooladian2021entropic}. Given independent samples $(\bX_i)_{i=1}^{n} \sim P$ and $(\bY_j)_{j=1}^{m} \sim Q$, let $\hat P_n = \frac{1}{n}\sum_i \delta_{\bX_i}$ and $\hat Q_m = \frac{1}{m}\sum_j \delta_{\bY_j}$ denote the empirical measures. The estimator involves solving the entropic OT problem with squared-Euclidean cost $c(\bx,\by) = \tfrac{1}{2}\|\bx-\by\|_2^2$ and regularization parameter $\varepsilon>0$, 
\begin{align}\label{eq:eot_primal}
\min_{\pi \in \Pi(\hat P_n, \hat Q_m)}\; \textstyle\sum_{i,j} \pi_{ij}\, c(\bX_i, \bY_j) + \varepsilon\, \mathrm{KL}\big(\pi\;\big\|\; \tfrac{1}{n}\mathbf{1}\otimes \tfrac{1}{m}\mathbf{1}\big).
\end{align}
For our purposes, we will consider $\varepsilon \in \{0.1, 1, 10\}$. The above optimization problem can be solved using the Sinkhorn algorithm \citep{cuturi2013sinkhorn,peyre2019computational}. In particular, as part of the optimization procedure, the dual entropic OT problem is solved, 
\begin{align*}
    \max_{f \in \mathbb{R}^n, g \in \mathbb{R}^m} \sum_{i=1}^{n} f_i + \sum_{j=1}^{m} g_j - \varepsilon \sum_{i,j} \exp\!\left(\frac{f_i + g_j - c(\bX_i, \bY_j)}{\varepsilon}\right), 
\end{align*}
where all dual optimizers $(f^\star, g^\star)$ are characterized by 
\begin{align*}
    f_i^\star &= -\varepsilon \log \sum_{j=1}^{m} \exp\!\left(\frac{g_j^\star - c(\bX_i, \bY_j)}{\varepsilon}\right), \qquad i \in [n],\\
    g_j^\star &= -\varepsilon \log \sum_{i=1}^{n} \exp\!\left(\frac{f_i^\star - c(\bX_i, \bY_j)}{\varepsilon}\right), \qquad j \in [m].
\end{align*}
From the dual potential $g^\star$ associated with $\hat Q_m$, the entropic OT map is then defined according to \cite{pooladian2021entropic} as
\begin{align*}
\hat T_{\varepsilon}(\bx) = \sum_{j=1}^{m} w_j^{\varepsilon}(\bx)\, \bY_j, \qquad w_j^{\varepsilon}(\bx) = \frac{\exp\!\big(\big(g^\star_j - c(\bx, \bY_j)\big)/\varepsilon\big)}{\sum_{j'=1}^{m} \exp\!\big(\big(g^\star_{j'} - c(\bx, \bY_{j'})\big)/\varepsilon\big)},
\end{align*}
which coincides with the barycentric map (i.e., the conditional expectation of $\bY$ given $\bX=\bx$) induced by the entropic OT plan from \eqref{eq:eot_primal} on the training samples $\{\bX_i\}_{i=1}^n$.
In our numerical simulations, we rely on the \texttt{geomloss} package \citep{feydy2019interpolating} using the tensorized backend, squared-Euclidean cost, blur parameter $\sqrt{\varepsilon}$, and no debiasing to compute $g^*$ which is then used to construct the entropic OT map $\hat T_\varepsilon$ as described above.

\parabold{Monge gap MLP estimator}
As a further baseline that does not rely on input-convex parameterizations, we consider the Monge gap MLP estimator of \citet{uscidda2023monge}. Here the OT map is parameterized directly as a plain MLP $T_\theta\colon \RR^d \to \RR^d$ with ELU activations (following the original implementation) and hidden widths $(104,\dots,104)$ of depth $L\in\{2,4,6\}$ (the same candidate depths as the ICNN baselines). For source and target mini-batches $\bX_1,\dots,\bX_M$ and $\bY_1,\dots,\bY_M$ of size $M$, let $\hat P_M = \tfrac{1}{M}\sum_i \delta_{\bX_i}$, $\hat Q_M = \tfrac{1}{M}\sum_i \delta_{\bY_i}$ and $(T_\theta)_\# \hat P_M = \tfrac{1}{M}\sum_i \delta_{T_\theta(\bX_i)}$ denote the associated empirical measures. Upon denoting the entropic OT cost for $c(\bx,\by) = \frac{1}{2}\|\bx-\by\|^2$ by 
\begin{align*}
\mathrm{OT}_\varepsilon(\mu,\nu) \;\coloneqq\; \min_{\pi \in \Pi(\mu,\nu)} \int c\, \mathrm{d}\pi + \varepsilon\, \mathrm{KL}(\pi \,\|\, \mu \otimes \nu),
\end{align*}
the training objective combines a (debiased) Sinkhorn divergence  $$\calS_\varepsilon(\mu,\nu) = \mathrm{OT}_\varepsilon(\mu,\nu) - \tfrac{1}{2}\mathrm{OT}_\varepsilon(\mu,\mu) - \tfrac{1}{2}\mathrm{OT}_\varepsilon(\nu,\nu)$$ under squared-Euclidean costs with the entropic Monge gap regularizer of \citet{uscidda2023monge}, namely
\begin{align}\label{eq:monge_gap_objective}
\min_\theta \; \underbrace{\calS_\varepsilon\big((T_\theta)_\# \hat P_M,\, \hat Q_M\big)}_{\text{fitting loss}} \;+\; \lambda\, \underbrace{\Big[\, \textstyle\frac{1}{M}\sum_{i=1}^M c\big(\bX_i, T_\theta(\bX_i)\big) \;-\; \mathrm{OT}_\varepsilon\!\big(\hat P_M, (T_\theta)_\# \hat P_M\big)\Big]}_{\text{Monge gap }\mathcal{M}_\varepsilon(T_\theta;\, \hat P_M)}.
\end{align} The (unregularized) Monge gap $\mathcal{M}_0(T_\theta;\hat P_M)\geq 0$ vanishes whenever $T_\theta$ transports $\hat P_M$ to $(T_\theta)_\# \hat P_M$ optimally, i.e., whenever $T_\theta$ is itself nearly the gradient of a convex function, and choosing $\epsilon$ small ensures that the Monge gap is computationally efficiently computable while being close to $\mathcal{M}_{0}$. Both terms in \eqref{eq:monge_gap_objective} are evaluated via differentiable Sinkhorn iterations provided by the \texttt{ott-jax} library of \citet{cuturi2022optimal}. We set the entropic regularization to $\varepsilon = 0.05$, the Monge gap weight to $\lambda = 0.1$, and use mean cost normalization (i.e., normalizing the cost function along the training samples to ensure that the average cost equals one) following the default configuration in \citet{uscidda2023monge}. Training minimizes \eqref{eq:monge_gap_objective} by Adam with fixed learning rate $10^{-2}$ as the objective involves no saddle point problem, batch size $M = 256$, and $S_{\textup{out}} = 1{,}000$ outer iterations; all remaining hyperparameters are kept at their \texttt{ott-jax} defaults.

\parabold{ICNN-based OT map estimators}
In addition, we compare the HyCNN-based OT map estimator to four ICNN variants, which we collectively refer to as \emph{ICNN-based methods}. All variants share the general ICNN template of \eqref{eq:ICNN}, i.e., hidden activations
\[
\bz_{\ell+1} = \sigma_\ell\big(V_\ell \bz_\ell + W_\ell \bx + \bb_\ell\big), \qquad \ell = 0,\dots,L-1,
\]
with scalar output $\operatorname{ICNN}(\bx) = V_L \bz_L + W_L \bx + \bb_L$, and differ only in their first-layer parameterization and choice of activation $\sigma_\ell$.
\begin{enumerate}
    \item[$(i)$] \textit{ICNN (ReLU).} The standard ICNN with first layer $\bz_1 = \sigma(W_0 \bx + \bb_0)$ and ReLU activation function $\sigma(a) = \max(a,0)$ in every layer.
    \item[$(ii)$] \textit{ICNN (Leaky ReLU).} Identical to $(i)$ but with the leaky-ReLU activation $\sigma(x) = \max(x,\alpha x)$, where $\alpha = 0.2$ is the negative slope.
    \item[$(iii)$] \textit{ICNNq (ReLU).} An ICNN whose first hidden layer is augmented by a learned quadratic skip connection:
    \begin{equation}
        \bz_1 = \sigma\big(W_0 \bx + (W_{\textup{q}} \bx)^{\odot 2} + \bb_0\big), \label{relu_quad}
    \end{equation}
    where $W_{\textup{q}} \in \RR^{d_1 \times d}$ is a learned weight matrix, $(\,\cdot\,)^{\odot 2}$ denotes elementwise squaring, and $\sigma$ is the ReLU activation; all subsequent layers are defined as in the standard ICNN. The quadratic term injects a non-negative, input-convex contribution into the first layer and improves the ability of ICNNs to represent mildly curved potentials.
    \item[$(iv)$] \textit{ICNNq (Softplus).} Identical to $(iii)$ but with the activation replaced by the smooth Softplus $\sigma_\tau(a) = \tau \log(1 + e^{a/\tau})$ with smoothness $\tau > 0$, which recovers ReLU as $\tau \to 0$. We consider parameter choices $\tau \in \{0.1, 1, 10\}$ and report the results for all three values.
\end{enumerate}

The potential $f_{\theta_f}$ and the critic $g_{\theta_g}$ of each ICNN baseline are both instantiated as ICNNs of the same variant $(i)$--$(iv)$, width, and depth, but they are treated differently regarding the non-negativity of their hidden-to-hidden weight matrices $V_\ell$. We enforce  the potential $f_{\theta_f}$ to be input-convex by utilizing for $V_\ell \geq 0$  the same softplus reparametrization used for the HyCNN: each $V_\ell$ is represented as the entrywise softplus $V_\ell = \operatorname{softplus}(R_\ell) = \log(1 + \exp(R_\ell))$ of an unconstrained parameter tensor $R_\ell$, so that $V_\ell \in (0,\infty)^{n_{\ell+1}\times n_\ell}$ holds automatically on every forward pass and Adam can update $R_\ell$ in $\RR^{n_{\ell+1}\times n_\ell}$ without any projection. For the critic $g_{\theta_g}$, by contrast, strict input-convexity is not needed: $g_{\theta_g}$ enters the saddle problem only through the gradient $\nabla g_{\theta_g}(\bY)$ at target samples, and a mild soft convexity suffices to make the inner Legendre-type problem well-posed. We therefore drop the softplus reparametrization for $g_{\theta_g}$ and let each hidden-to-hidden matrix of $g_{\theta_g}$ be an ordinary unconstrained real parameter, with non-negativity merely encouraged softly through the convexity penalty $\lambda_{\textup{cvx}}\|(V_\ell)_{-}\|_F^2$ that appears in the training objective \eqref{eq:ICNN_objective} below.

At initialization, all $V_\ell$ are drawn so that their first and second moments match the values derived by \citet{hoedt2023principled} for ICNNs, namely
\begin{align*}
\mu_W = \sqrt{\tfrac{6\pi}{n_\ell\, D_\ell}}, \qquad \sigma_W^2 = \tfrac{1}{n_\ell}, \qquad D_\ell = 6(\pi-1) + (n_\ell-1)\big(3\sqrt{3} + 2\pi - 6\big),
\end{align*}
where $n_\ell$ is the fan-in of layer $\ell$. For the critic $g_{\theta_g}$, whose hidden-to-hidden matrices are unconstrained, we draw $(V_\ell)_{ij} \sim \calN(\mu_W,\sigma_W^2)$ directly. For the potential $f_{\theta_f}$, whose effective weight must be non-negative by construction, we instead draw $V_\ell$ entrywise from the log-normal distribution with the same prescribed moments, $(V_\ell)_{ij} \sim \operatorname{LogNormal}(\mu_{},\sigma_{}^2)$ with
\begin{align*}
 \mu_{} = \log\!\Big(\tfrac{\mu_W^2}{\sqrt{\mu_W^2 + \sigma_W^2}}\Big),\quad \sigma_{}^2 = \log\!\Big(1 + \tfrac{\sigma_W^2}{\mu_W^2}\Big), 
\end{align*}
and then store the corresponding pre-image $R_\ell = \operatorname{softplus}^{-1}(V_\ell) = \log(\exp(V_\ell) - 1)$. By construction the effective weight at initialization is always positive and satisfies $\EE[(V_\ell)_{ij}] = \mu_W$ and $\operatorname{Var}[(V_\ell)_{ij}] = \sigma_W^2$. We choose the log-normal distribution (rather than a truncated Gaussian) because it is inherently non-negative and allows us to match the prescribed mean and variance of \citet{hoedt2023principled} simultaneously.

The remaining parameters are identical for $f_{\theta_f}$ and $g_{\theta_g}$: the input-to-hidden matrices $W_\ell$ are unconstrained and initialized entrywise as $(W_\ell)_{ij}\sim\calN(0, 1/n_\ell)$ (LeCun-normal), the hidden-layer biases are set to the constant $\mu_b = -\sqrt{3n_\ell/D_\ell}$ and the output bias to zero, and for the two quadratic-skip variants $(iii)$--$(iv)$ the additional first-layer matrix $W_{\textup{q}} \in \RR^{n_1\times d}$ is drawn entrywise as $\calN(0, 1/d)$. All random matrices are drawn independently across layers.

All ICNN-based methods are trained by solving an empirical saddle-point problem similar to our HyCNN-based estimator, with the HyCNN potential and critic replaced by the respective ICNN variant and an additional regularization term. We refer to the resulting objective as the ICNN objective; for mini-batches $(\bX_i)_{i\in [M]}$ and $(\bY_j)_{j\in [M]}$ of size $M$, it takes the form
\begin{align}\label{eq:ICNN_objective}
  \min_{\theta_f}\;\max_{\theta_g}\; \bigg\{ \textstyle\frac{1}{M}\sum_{i=1}^{M} \!\Big[ f_{\theta_f}(\bX_i) + \langle \bY_i, \nabla g_{\theta_g}(\bY_i) \rangle - f_{\theta_f}(\nabla g_{\theta_g}(\bY_i))\Big] \;-\; \lambda_{\textup{cvx}}\, \textstyle\sum_{V \in \theta_g} \| (V)_{-} \|_F^2 \bigg\},
\end{align}
where $(V)_{-} \coloneqq \max(0, -V)$ denotes the negative part of $V$ taken elementwise, $\|\cdot\|_F$ is the Frobenius norm, and the sum ranges over all hidden-to-hidden weight matrices $V$ of the critic network $g_{\theta_g}$. The parameter $\lambda_{\textup{cvx}} > 0$ controls the strength of the convexity regularization and softly penalizes negative entries in the hidden-to-hidden weights of the critic but does not enforce strict non-negativity. We use the same value $\lambda_{\textup{cvx}} = 1$ for all ICNN-based methods. The potential $f_{\theta_f}$ and the critic $g_{\theta_g}$ are parameterized by the same architecture (width, depth, activation).

For training we follow the same alternating optimization scheme as for the HyCNN-based estimator, which is inspired by \citet[Algorithm~1]{makkuva2020optimal}. We run $S_{\textup{out}} = 1{,}000$ outer iterations with $S_{\textup{in}} = 5$ inner steps, batch size $M = 256$, and the Adam optimizer \citep{kingma2014adam} with $\beta_1 = 0.5$, $\beta_2 = 0.9$, initial learning rate $\lambda = 10^{-2}$, and cosine decay to final ratio $0.01$, applied identically to both $\theta_f$ and $\theta_g$. The OT map estimate is then read off as $\hat T = \nabla f_{\theta_f}$.

\parabold{Prediction MSE tables for OT map estimation}\label{sec:OT_testmse_tables}
We first present the prediction MSE of the estimated OT map $\hat T = \nabla \hat \phi$ against the ground-truth OT map $\overline T = \nabla \phi_\ell$ for each of the four OT potentials (a--d) and dimensions $d \in \{10, 20, 50\}$. The MSE is evaluated on an independent test set of $N_{\text{test}} = 1{,}000$ samples; mean and standard error (SE) are computed over $10$ independent runs. We compare a comprehensive set of estimators: ICNN with ReLU and Leaky ReLU activations, ICNNq with a quadratic skip connection and either ReLU or Softplus activations (the latter for varying smoothing parameters $\tau \in \{0.1, 1, 10\}$), HyCNN with the log-sum-exp activation for $\tau \in \{0.1, 1, 10\}$, the Monge gap estimator, and entropic OT (EOT) with entropy regularization parameter $\epsilon \in \{0.1, 1, 10\}$. We only report the results for the HyCNN architecture (depth $L$ and smoothness $\tau$) that achieves the smallest Sinkhorn divergence on an independent validation set as described in the next section, and we report the results for all $\tau$ variants of ICNNs to provide a comprehensive picture of the performance of these baselines across different smoothness levels.
The Monge gap MLP estimator uses an MLP with width $104$, while the ICNN architectures use width $64$ and the HyCNNs use width $44$, each with depths $L \in \{2, 4, 6\}$. 
 
The results in Tables \ref{tab:ot_testmse_collapsed_psi1}, \ref{tab:ot_testmse_collapsed_psi4}, \ref{tab:ot_testmse_collapsed_psi2}, and \ref{tab:ot_testmse_collapsed_psi3} reveal several consistent patterns. On the smooth quadratic targets (a) and (b), HyCNN is the best estimator in every dimension and the three HyCNN architectures form the three smallest MSE entries throughout, dominating all baselines, including the strongest ICNN-based competitor, ICNNq (Softplus), by a factor of roughly $3$--$5$ in MSE. Here the large smoothness setting $\tau = 10$ is uniformly preferred by the Sinkhorn validation criterion. On the non-smooth target (c), HyCNN remains best at $d \in \{20, 50\}$ and is only marginally surpassed by the tuned ICNNq (Softplus) configuration at $d = 10$. On the strongly curved target (d), by contrast, ICNNq (Softplus) with a small to moderate temperature $\tau \in \{0.1, 1\}$ clearly beats HyCNN at $d = 10$ and $d = 20$, whereas HyCNN regains the lead at $d = 50$. For HyCNN on (c) and (d) the optimal smoothness is target- and dimension-dependent, with small $\tau \in \{0.1, 1\}$ preferred at $d = 10$ and $\tau = 10$ most accurate and stable at $d = 50$. Among the ICNN baselines, the quadratic-skip variants (ICNNq) clearly outperform the plain ICNN variants across all settings, and the Softplus activation is consistently more accurate and more stable than ReLU once $\tau$ is tuned: small $\tau \in \{0.1, 1\}$ is preferred at $d \in \{10, 20\}$, while large $\tau = 10$ is required at $d = 50$ to avoid the depth-induced instabilities exhibited by ICNNq (ReLU) at $L \in \{4, 6\}$ on targets (c) and (d). Even ICNNq (Softplus) with small $\tau = 0.1$ degrades sharply at $L = 6$, $d = 50$ on the non-smooth and strongly curved targets, confirming that low smoothness amplifies high-dimensional instabilities. ICNN (L-ReLU) is the least robust method overall and exhibits catastrophic failures at $L = 2$, $d = 50$ across all four targets. The best entropic OT estimator is uniformly worst on (a)--(d), and the Monge gap is outperformed by HyCNN by an order of magnitude in nearly all settings. HyCNN remains stable across $L \in \{2, 4, 6\}$ once $\tau$ is chosen appropriately, with its relative advantage most pronounced on the smooth quadratic targets and at high dimensions.

\parabold{Sinkhorn divergence based network selection}\label{sec:OT_sinkhorn_tables}

In addition to the prediction MSE tables above, we present Sinkhorn divergence tables (with regularization $\varepsilon = 0.1$) evaluated on an independent validation set of $n_{\text{valid}} = 1{,}000$ sample pairs from $P$ and $Q = (\nabla \phi_\ell)_\# P$, computed separately from the training and test data. This validation metric assesses the distributional fit of the learned OT map and serves as a guide for selecting the HyCNN architecture (depth $L$ and smoothing parameter $\tau$) for a given problem. Concretely, for a HyCNN estimator $\hat T$ based on training data, we compute for empirical measures $\hat P_{n_{\text{valid}}}, \hat Q_{n_{\text{valid}}}$ the Sinkhorn divergence \citep{feydy2019interpolating, peyre2019computational} based on the squared Euclidean distance with regularization parameter $\epsilon = 0.1$, 
\begin{align*}
    \mathcal{S}_{\epsilon}(\hat T_{\#} \hat P_{n_{\text{valid}}}, \hat Q_{n_{\text{valid}}}).
\end{align*}
For each simulation run, we draw a single pair of validation data sets from domain and target and compute the above Sinkhorn divergence. Given the $10$ independent runs for training, we thus compute $10$ independent Sinkhorn divergence values which we use to compute mean and standard error (SE). The resulting validation metrics for the HyCNN-based OT estimators are depicted in Tables \ref{tab:ot_sinkhorn_full_psi1},
\ref{tab:ot_sinkhorn_full_psi4},
\ref{tab:ot_sinkhorn_full_psi2}, and \ref{tab:ot_sinkhorn_full_psi3},   right below the respective prediction MSE table with the same ground truth. 

The Sinkhorn divergence validation metrics demonstrate that a well-performing HyCNN estimator can be consistently identified through our validation procedure. We conducted analogous experiments with alternative choices of the regularization parameter $\epsilon$ and found the validation technique to yield consistent results across these settings. A noteworthy observation is that identifying an estimator with even a moderately smaller Sinkhorn divergence can translate into a significantly improved fit. We find this phenomenon striking and attribute it to high-dimensional effects as it is inherently difficult to certify a substantially better distributional fit on the basis of independent realizations alone. %

\begin{table}[h!]
\caption{Total neural-network parameter counts for the OT estimators at the selected widths, including weights and biases. For the MLP semidual estimator, ICNN variants, and HyCNN, totals include both scalar-valued networks, the potential $f$ and critic $g$. The Monge gap MLP estimator uses a single $d$-dimensional map network. Network size is $(m\times L)$, with $L$ hidden layers; the semidual MLP has width $72$, the Monge gap MLP has width $104$, the ICNN variants have width $64$, and HyCNN has width $44$ at every depth and dimension. Counts are identical across the four OT tasks and all displayed activation and smoothness settings. Entropic OT is nonparametric and has no neural-network parameters; its data-dependent Sinkhorn dual values are excluded.}
\label{tab:ot_total_parameter_counts}
\centering
\small
\begin{tabular}{llrrr}
\toprule
Method & $(m\times L)$ & $d=10$ & $d=20$ & $d=50$ \\
\midrule
\multirow{3}{*}{Monge gap (MLP)} & $(104\times 2)$ & 13,114 & 15,204 & 21,474 \\
 & $(104\times 4)$ & 34,954 & 37,044 & 43,314 \\
 & $(104\times 6)$ & 56,794 & 58,884 & 65,154 \\
\midrule
\multirow{3}{*}{MLP (semidual)} & $(72\times 2)$ & 12,242 & 13,682 & 18,002 \\
 & $(72\times 4)$ & 33,266 & 34,706 & 39,026 \\
 & $(72\times 6)$ & 54,290 & 55,730 & 60,050 \\
\midrule
\multirow{3}{*}{ICNN (semidual)} & $(64\times 2)$ & 11,158 & 13,738 & 21,478 \\
 & $(64\times 4)$ & 30,358 & 35,498 & 50,918 \\
 & $(64\times 6)$ & 49,558 & 57,258 & 80,358 \\
\midrule
\multirow{3}{*}{ICNNq (semidual)} & $(64\times 2)$ & 12,438 & 16,298 & 27,878 \\
 & $(64\times 4)$ & 31,638 & 38,058 & 57,318 \\
 & $(64\times 6)$ & 50,838 & 59,818 & 86,758 \\
\midrule
\multirow{3}{*}{HyCNN (semidual)} & $(44\times 2)$ & 11,726 & 15,266 & 25,886 \\
 & $(44\times 4)$ & 31,086 & 38,146 & 59,326 \\
 & $(44\times 6)$ & 50,446 & 61,026 & 92,766 \\
\bottomrule
\end{tabular}
\end{table}
\begin{table}[h!]
\caption{Prediction MSE (mean and SE over $10$ runs) of the estimated OT map for the OT map (a) $\textstyle T_1(\bx) = \bx$ based on $n = m = 5{,}000$ independent training sample points from the source and target distribution. Per HyCNN architecture ($44 \times L$), the $\tau$ variant with the smallest Sinkhorn divergence validation metric ($\epsilon = 0.1$) for $1{,}000$ validation samples at that dimension is used; see Table~\ref{tab:ot_sinkhorn_full_psi1}. Per dimension, the smallest mean MSE per method type is bold and colored, with the \underline{three smallest} values across all methods underlined. Monge gap uses MLP width $104$, ICNN variants use width $64$, and both HyCNN networks use width $44$.}
\label{tab:ot_testmse_collapsed_psi1}
\scriptsize
\makebox[\textwidth][c]{%
\begin{tabular}{llccc}
\toprule
\multicolumn{2}{c}{Method $\backslash$ $d$} & 10 & 20 & 50 \\
\midrule
\multirow{3}{*}{Entropic OT} & $\epsilon = 0.1$ & 2.678 (0.017) & 12.098 (0.039) & 51.783 (0.075) \\
 & $\epsilon = 1$ & \textcolor[HTML]{006400}{\textbf{0.999}} (0.008) & \textcolor[HTML]{006400}{\textbf{6.537}} (0.034) & 41.019 (0.067) \\
 & $\epsilon = 10$ & 6.538 (0.022) & 13.129 (0.031) & \textcolor[HTML]{006400}{\textbf{32.786}} (0.058) \\
\midrule
\multirow{3}{*}{Monge gap} & $(104 \times 2)$ & \textcolor[HTML]{7B2CBF}{\textbf{1.167}} (0.025) & \textcolor[HTML]{7B2CBF}{\textbf{4.586}} (0.049) & 23.385 (0.131) \\
 & $(104 \times 4)$ & 1.275 (0.039) & 5.416 (0.084) & 23.007 (0.105) \\
 & $(104 \times 6)$ & 1.262 (0.059) & 5.793 (0.101) & \textcolor[HTML]{7B2CBF}{\textbf{21.038}} (0.141) \\
\midrule
\multirow{3}{*}{ICNN (ReLU)} & $(64 \times 2)$ & 0.852 (0.021) & 3.221 (0.061) & 20.792 (0.567) \\
 & $(64 \times 4)$ & 0.498 (0.010) & 2.222 (0.038) & 12.205 (0.352) \\
 & $(64 \times 6)$ & \textcolor[HTML]{B87700}{\textbf{0.391}} (0.007) & \textcolor[HTML]{B87700}{\textbf{1.965}} (0.065) & \textcolor[HTML]{B87700}{\textbf{9.809}} (0.186) \\
\midrule
\multirow{3}{*}{ICNN (L-ReLU)} & $(64 \times 2)$ & 1.213 (0.055) & 4.360 (0.135) & 614.316 (216.829) \\
 & $(64 \times 4)$ & 0.734 (0.029) & 2.599 (0.089) & 20.861 (1.983) \\
 & $(64 \times 6)$ & \textcolor[HTML]{A46500}{\textbf{0.679}} (0.130) & \textcolor[HTML]{A46500}{\textbf{2.002}} (0.072) & \textcolor[HTML]{A46500}{\textbf{10.183}} (0.474) \\
\midrule
\multirow{3}{*}{ICNNq (ReLU)} & $(64 \times 2)$ & \textcolor[HTML]{E31E00}{\textbf{0.222}} (0.009) & \textcolor[HTML]{E31E00}{\textbf{0.801}} (0.067) & \textcolor[HTML]{E31E00}{\textbf{4.101}} (0.299) \\
 & $(64 \times 4)$ & 0.286 (0.011) & 1.123 (0.062) & 9.581 (0.479) \\
 & $(64 \times 6)$ & 3.788 (3.404) & 1.401 (0.049) & 127.802 (79.678) \\
\midrule
\multirow{9}{*}{ICNNq (Softplus)} & $(64 \times 2); \tau = 0.1$ & 0.121 (0.008) & 0.567 (0.031) & 3.350 (0.143) \\
 & $(64 \times 2); \tau = 1$ & \textcolor[HTML]{8B1301}{\textbf{0.073}} (0.002) & \textcolor[HTML]{8B1301}{\textbf{0.365}} (0.020) & 3.032 (0.256) \\
 & $(64 \times 2); \tau = 10$ & 0.195 (0.007) & 0.486 (0.010) & \textcolor[HTML]{8B1301}{\textbf{2.046}} (0.028) \\
 & $(64 \times 4); \tau = 0.1$ & 0.192 (0.004) & 0.858 (0.036) & 7.563 (0.283) \\
 & $(64 \times 4); \tau = 1$ & 0.103 (0.003) & 0.407 (0.017) & 6.950 (0.244) \\
 & $(64 \times 4); \tau = 10$ & 0.188 (0.008) & 0.490 (0.007) & 2.367 (0.020) \\
 & $(64 \times 6); \tau = 0.1$ & 0.233 (0.010) & 0.971 (0.010) & 10.011 (1.030) \\
 & $(64 \times 6); \tau = 1$ & 0.097 (0.003) & 0.570 (0.050) & 6.617 (1.388) \\
 & $(64 \times 6); \tau = 10$ & 0.186 (0.007) & 0.504 (0.005) & 2.482 (0.021) \\
\midrule
\multirow{3}{*}{HyCNN} & $(44 \times 2)$ & \underline{0.028} (0.001) & \underline{0.109} (0.003) & \underline{0.783} (0.019) \\
 & $(44 \times 4)$ & \underline{\textcolor[HTML]{0080DB}{\textbf{0.023}}} (0.002) & \underline{\textcolor[HTML]{0080DB}{\textbf{0.074}}} (0.003) & \underline{0.405} (0.007) \\
 & $(44 \times 6)$ & \underline{0.063} (0.002) & \underline{0.079} (0.003) & \underline{\textcolor[HTML]{0080DB}{\textbf{0.381}}} (0.009) \\
\bottomrule
\end{tabular}%
}
\end{table}

\begin{table}[h!]
\caption{Sinkhorn divergence (mean and SE over $10$ runs) of the estimated OT map for the OT map (a) $\textstyle T_1(\bx) = \bx$. Per HyCNN architecture ($44 \times L$) and per dimension, the $\tau$ variant with the smallest mean Sinkhorn divergence is bold and colored, with the \underline{three smallest} values across all HyCNN rows underlined.}
\label{tab:ot_sinkhorn_full_psi1}
\scriptsize
\makebox[\textwidth][c]{%
\begin{tabular}{llccc}
\toprule
\multicolumn{2}{c}{Method $\backslash$ $d$} & 10 & 20 & 50 \\
\midrule
\multirow{9}{*}{HyCNN} & $(44 \times 2); \tau = 0.1$ & 4.324 (0.027) & 16.055 (0.055) & 61.958 (0.198) \\
 & $(44 \times 2); \tau = 1$ & 4.116 (0.022) & 15.671 (0.050) & 60.970 (0.088) \\
 & $(44 \times 2); \tau = 10$ & \underline{\textcolor[HTML]{0080DB}{\textbf{4.037}}} (0.021) & \underline{\textcolor[HTML]{0080DB}{\textbf{14.907}}} (0.033) & \underline{\textcolor[HTML]{0080DB}{\textbf{57.797}}} (0.098) \\
 & $(44 \times 4); \tau = 0.1$ & 4.201 (0.024) & 15.640 (0.035) & 62.698 (0.099) \\
 & $(44 \times 4); \tau = 1$ & 4.063 (0.020) & 15.338 (0.041) & 60.437 (0.081) \\
 & $(44 \times 4); \tau = 10$ & \textcolor[HTML]{0080DB}{\textbf{4.047}} (0.021) & \underline{\textcolor[HTML]{0080DB}{\textbf{14.900}}} (0.033) & \underline{\textcolor[HTML]{0080DB}{\textbf{57.566}}} (0.108) \\
 & $(44 \times 6); \tau = 0.1$ & 4.228 (0.028) & 15.598 (0.042) & 61.734 (0.143) \\
 & $(44 \times 6); \tau = 1$ & \underline{\textcolor[HTML]{0080DB}{\textbf{4.042}}} (0.021) & 15.056 (0.031) & 59.548 (0.148) \\
 & $(44 \times 6); \tau = 10$ & \underline{4.045} (0.020) & \underline{\textcolor[HTML]{0080DB}{\textbf{14.896}}} (0.031) & \underline{\textcolor[HTML]{0080DB}{\textbf{57.533}}} (0.088) \\
\bottomrule
\end{tabular}%
}
\end{table}

\begin{table}[h!]
\caption{Prediction MSE (mean and SE over $10$ runs) of the estimated OT map for the OT map (b) $\textstyle T_2(\bx) = \bigl((1 + \frac{\sin(i)}{2})\, x_i\bigr)_{i\in [d]}$ based on $n = m = 5{,}000$ independent training sample points from the source and target distribution. Per HyCNN architecture ($44 \times L$), the $\tau$ variant with the smallest Sinkhorn divergence validation metric ($\epsilon = 0.1$) for $1{,}000$ validation samples at that dimension is used; see Table~\ref{tab:ot_sinkhorn_full_psi4}. Per dimension, the smallest mean MSE per method type is bold and colored, with the \underline{three smallest} values across all methods underlined.}
\label{tab:ot_testmse_collapsed_psi4}
\scriptsize
\makebox[\textwidth][c]{%
\begin{tabular}{llccc}
\toprule
\multicolumn{2}{c}{Method $\backslash$ $d$} & 10 & 20 & 50 \\
\midrule
\multirow{3}{*}{Entropic OT} & $\epsilon = 0.1$ & 2.927 (0.018) & 12.404 (0.033) & 52.669 (0.107) \\
 & $\epsilon = 1$ & \textcolor[HTML]{006400}{\textbf{1.078}} (0.011) & \textcolor[HTML]{006400}{\textbf{6.589}} (0.034) & 41.518 (0.101) \\
 & $\epsilon = 10$ & 7.435 (0.026) & 14.022 (0.040) & \textcolor[HTML]{006400}{\textbf{33.704}} (0.084) \\
\midrule
\multirow{3}{*}{Monge gap} & $(104 \times 2)$ & \textcolor[HTML]{7B2CBF}{\textbf{1.461}} (0.029) & 5.239 (0.062) & 22.050 (0.144) \\
 & $(104 \times 4)$ & 1.631 (0.048) & 5.419 (0.099) & 19.884 (0.196) \\
 & $(104 \times 6)$ & 1.707 (0.030) & \textcolor[HTML]{7B2CBF}{\textbf{5.175}} (0.069) & \textcolor[HTML]{7B2CBF}{\textbf{17.915}} (0.120) \\
\midrule
\multirow{3}{*}{ICNN (ReLU)} & $(64 \times 2)$ & 1.010 (0.010) & 3.822 (0.058) & 19.617 (0.202) \\
 & $(64 \times 4)$ & 0.597 (0.007) & 2.341 (0.064) & 12.451 (0.180) \\
 & $(64 \times 6)$ & \textcolor[HTML]{B87700}{\textbf{0.426}} (0.007) & \textcolor[HTML]{B87700}{\textbf{1.939}} (0.066) & \textcolor[HTML]{B87700}{\textbf{9.786}} (0.248) \\
\midrule
\multirow{3}{*}{ICNN (L-ReLU)} & $(64 \times 2)$ & 1.389 (0.071) & 4.702 (0.157) & 60.328 (17.675) \\
 & $(64 \times 4)$ & 0.779 (0.031) & 2.735 (0.153) & 16.110 (0.854) \\
 & $(64 \times 6)$ & \textcolor[HTML]{A46500}{\textbf{0.541}} (0.021) & \textcolor[HTML]{A46500}{\textbf{2.161}} (0.068) & \textcolor[HTML]{A46500}{\textbf{9.996}} (0.275) \\
\midrule
\multirow{3}{*}{ICNNq (ReLU)} & $(64 \times 2)$ & \textcolor[HTML]{E31E00}{\textbf{0.249}} (0.020) & \textcolor[HTML]{E31E00}{\textbf{0.694}} (0.046) & \textcolor[HTML]{E31E00}{\textbf{3.633}} (0.241) \\
 & $(64 \times 4)$ & 0.318 (0.029) & 1.059 (0.037) & 8.347 (0.260) \\
 & $(64 \times 6)$ & 0.368 (0.018) & 1.455 (0.060) & 26.901 (9.868) \\
\midrule
\multirow{9}{*}{ICNNq (Softplus)} & $(64 \times 2); \tau = 0.1$ & 0.139 (0.008) & 0.550 (0.023) & 3.259 (0.174) \\
 & $(64 \times 2); \tau = 1$ & \textcolor[HTML]{8B1301}{\textbf{0.078}} (0.003) & \textcolor[HTML]{8B1301}{\textbf{0.324}} (0.012) & 3.447 (0.257) \\
 & $(64 \times 2); \tau = 10$ & 0.242 (0.013) & 0.570 (0.006) & \textcolor[HTML]{8B1301}{\textbf{2.019}} (0.056) \\
 & $(64 \times 4); \tau = 0.1$ & 0.215 (0.008) & 0.833 (0.031) & 7.303 (0.289) \\
 & $(64 \times 4); \tau = 1$ & 0.108 (0.003) & 0.445 (0.037) & 6.358 (0.318) \\
 & $(64 \times 4); \tau = 10$ & 0.242 (0.007) & 0.553 (0.011) & 2.419 (0.034) \\
 & $(64 \times 6); \tau = 0.1$ & 0.251 (0.012) & 1.022 (0.019) & 20.532 (12.168) \\
 & $(64 \times 6); \tau = 1$ & 0.111 (0.003) & 0.603 (0.028) & 4.063 (0.280) \\
 & $(64 \times 6); \tau = 10$ & 0.243 (0.009) & 0.584 (0.015) & 2.472 (0.024) \\
\midrule
\multirow{3}{*}{HyCNN} & $(44 \times 2)$ & \underline{0.030} (0.001) & \underline{0.109} (0.002) & \underline{0.710} (0.017) \\
 & $(44 \times 4)$ & \underline{\textcolor[HTML]{0080DB}{\textbf{0.024}}} (0.002) & \underline{\textcolor[HTML]{0080DB}{\textbf{0.078}}} (0.004) & \underline{\textcolor[HTML]{0080DB}{\textbf{0.413}}} (0.005) \\
 & $(44 \times 6)$ & \underline{0.032} (0.004) & \underline{0.100} (0.004) & \underline{0.478} (0.007) \\
\bottomrule
\end{tabular}%
}
\end{table}

\begin{table}[h!]
\caption{
Sinkhorn divergence (mean and SE over $10$ runs) of the estimated OT map for the OT map (b) $\textstyle T_2(\bx) = \bigl((1 + \frac{\sin(i)}{2})\, x_i\bigr)_{i\in [d]}$. Per HyCNN architecture ($44 \times L$) and per dimension, the $\tau$ variant with the smallest mean Sinkhorn divergence is bold and colored, with the \underline{three smallest} values across all HyCNN rows underlined.}
\label{tab:ot_sinkhorn_full_psi4}
\scriptsize
\makebox[\textwidth][c]{%
\begin{tabular}{llccc}
\toprule
\multicolumn{2}{c}{Method $\backslash$ $d$} & 10 & 20 & 50 \\
\midrule
\multirow{9}{*}{HyCNN} & $(44 \times 2); \tau = 0.1$ & 4.571 (0.025) & 16.288 (0.070) & 61.705 (0.245) \\
 & $(44 \times 2); \tau = 1$ & 4.333 (0.025) & 15.608 (0.042) & 61.781 (0.092) \\
 & $(44 \times 2); \tau = 10$ & \underline{\textcolor[HTML]{0080DB}{\textbf{4.257}}} (0.024) & \underline{\textcolor[HTML]{0080DB}{\textbf{14.911}}} (0.035) & \underline{\textcolor[HTML]{0080DB}{\textbf{58.321}}} (0.112) \\
 & $(44 \times 4); \tau = 0.1$ & 4.455 (0.029) & 15.683 (0.050) & 62.971 (0.217) \\
 & $(44 \times 4); \tau = 1$ & 4.293 (0.027) & 15.302 (0.037) & 60.082 (0.235) \\
 & $(44 \times 4); \tau = 10$ & \underline{\textcolor[HTML]{0080DB}{\textbf{4.276}}} (0.023) & \underline{\textcolor[HTML]{0080DB}{\textbf{14.911}}} (0.037) & \underline{\textcolor[HTML]{0080DB}{\textbf{58.102}}} (0.113) \\
 & $(44 \times 6); \tau = 0.1$ & 4.463 (0.030) & 15.604 (0.038) & 62.175 (0.105) \\
 & $(44 \times 6); \tau = 1$ & 4.277 (0.029) & 15.047 (0.030) & 59.842 (0.260) \\
 & $(44 \times 6); \tau = 10$ & \underline{\textcolor[HTML]{0080DB}{\textbf{4.269}}} (0.027) & \underline{\textcolor[HTML]{0080DB}{\textbf{14.919}}} (0.043) & \underline{\textcolor[HTML]{0080DB}{\textbf{58.444}}} (0.123) \\
\bottomrule
\end{tabular}%
}
\end{table}

\begin{table}[h!]
\caption{Prediction MSE (mean and SE over $10$ runs) of the estimated OT map for the OT map (c) $\textstyle T_3(\bx) = (x_i + \textup{sign}(x_i))_{i\in [d]}$ based on $n = m = 5{,}000$ independent training sample points from the source and target distribution. Per HyCNN architecture ($44 \times L$), the $\tau$ variant with the smallest Sinkhorn divergence validation metric ($\epsilon = 0.1$) for $1{,}000$ validation samples at that dimension is used; see Table~\ref{tab:ot_sinkhorn_full_psi2}. Per dimension, the smallest mean MSE per method type is bold and colored, with the \underline{three smallest} values across all methods underlined.}
\label{tab:ot_testmse_collapsed_psi2}
\scriptsize
\makebox[\textwidth][c]{%
\begin{tabular}{llccc}
\toprule
\multicolumn{2}{c}{Method $\backslash$ $d$} & 10 & 20 & 50 \\
\midrule
\multirow{3}{*}{Entropic OT} & $\epsilon = 0.1$ & 8.155 (0.033) & 42.243 (0.072) & 189.406 (0.327) \\
 & $\epsilon = 1$ & \textcolor[HTML]{006400}{\textbf{3.834}} (0.024) & \textcolor[HTML]{006400}{\textbf{30.997}} (0.111) & 169.342 (0.385) \\
 & $\epsilon = 10$ & 16.868 (0.025) & 33.891 (0.029) & \textcolor[HTML]{006400}{\textbf{89.193}} (0.108) \\
\midrule
\multirow{3}{*}{Monge gap} & $(104 \times 2)$ & \textcolor[HTML]{7B2CBF}{\textbf{5.666}} (0.068) & \textcolor[HTML]{7B2CBF}{\textbf{19.595}} (0.046) & 87.219 (0.238) \\
 & $(104 \times 4)$ & 8.067 (0.225) & 19.795 (0.087) & 76.500 (0.365) \\
 & $(104 \times 6)$ & 8.612 (0.133) & 21.680 (0.222) & \textcolor[HTML]{7B2CBF}{\textbf{63.571}} (0.475) \\
\midrule
\multirow{3}{*}{ICNN (ReLU)} & $(64 \times 2)$ & 4.670 (0.263) & 20.552 (0.149) & 89.585 (1.698) \\
 & $(64 \times 4)$ & \textcolor[HTML]{B87700}{\textbf{2.165}} (0.133) & 13.822 (0.240) & 66.244 (2.288) \\
 & $(64 \times 6)$ & 2.461 (0.209) & \textcolor[HTML]{B87700}{\textbf{9.641}} (0.673) & \textcolor[HTML]{B87700}{\textbf{55.948}} (0.771) \\
\midrule
\multirow{3}{*}{ICNN (L-ReLU)} & $(64 \times 2)$ & 134.056 (115.902) & 2.29e+03 (1.57e+03) & 3.82e+03 (1.13e+03) \\
 & $(64 \times 4)$ & 2.784 (0.196) & 13.704 (0.131) & 85.385 (2.898) \\
 & $(64 \times 6)$ & \textcolor[HTML]{A46500}{\textbf{2.271}} (0.194) & \textcolor[HTML]{A46500}{\textbf{11.384}} (0.232) & \textcolor[HTML]{A46500}{\textbf{59.944}} (2.004) \\
\midrule
\multirow{3}{*}{ICNNq (ReLU)} & $(64 \times 2)$ & 1.522 (0.067) & 8.499 (0.057) & \textcolor[HTML]{E31E00}{\textbf{26.985}} (0.571) \\
 & $(64 \times 4)$ & \textcolor[HTML]{E31E00}{\textbf{1.220}} (0.026) & 7.211 (0.350) & 48.320 (1.869) \\
 & $(64 \times 6)$ & 1.558 (0.093) & \textcolor[HTML]{E31E00}{\textbf{5.828}} (0.387) & 61.714 (0.967) \\
\midrule
\multirow{9}{*}{ICNNq (Softplus)} & $(64 \times 2); \tau = 0.1$ & 1.605 (0.165) & 8.410 (0.055) & 28.272 (1.198) \\
 & $(64 \times 2); \tau = 1$ & 3.519 (0.035) & 8.172 (0.025) & \textcolor[HTML]{8B1301}{\textbf{24.779}} (0.495) \\
 & $(64 \times 2); \tau = 10$ & 4.769 (0.036) & 9.492 (0.040) & 25.550 (0.229) \\
 & $(64 \times 4); \tau = 0.1$ & \underline{\textcolor[HTML]{8B1301}{\textbf{0.942}}} (0.021) & 6.896 (0.494) & 44.097 (2.595) \\
 & $(64 \times 4); \tau = 1$ & 3.231 (0.071) & 8.967 (0.105) & 30.605 (0.596) \\
 & $(64 \times 4); \tau = 10$ & 4.620 (0.047) & 9.634 (0.068) & 27.468 (0.149) \\
 & $(64 \times 6); \tau = 0.1$ & 1.069 (0.018) & \underline{\textcolor[HTML]{8B1301}{\textbf{5.025}}} (0.480) & 228.345 (118.208) \\
 & $(64 \times 6); \tau = 1$ & 2.878 (0.081) & 9.241 (0.102) & 28.474 (0.415) \\
 & $(64 \times 6); \tau = 10$ & 4.565 (0.047) & 9.568 (0.056) & 27.948 (0.100) \\
\midrule
\multirow{3}{*}{HyCNN} & $(44 \times 2)$ & 1.917 (0.055) & 7.731 (0.023) & \underline{21.750} (0.120) \\
 & $(44 \times 4)$ & \underline{1.052} (0.010) & \underline{\textcolor[HTML]{0080DB}{\textbf{4.986}}} (0.176) & \underline{19.872} (0.040) \\
 & $(44 \times 6)$ & \underline{\textcolor[HTML]{0080DB}{\textbf{1.027}}} (0.009) & \underline{5.275} (0.189) & \underline{\textcolor[HTML]{0080DB}{\textbf{19.763}}} (0.047) \\
\bottomrule
\end{tabular}%
}
\end{table}

\begin{table}[h!]
\caption{Sinkhorn divergence (mean and SE over $10$ runs) of the estimated OT map for the OT map (c) $\textstyle T_3(\bx) = (x_i + \textup{sign}(x_i))_{i\in [d]}$. Per HyCNN architecture ($44 \times L$) and per dimension, the $\tau$ variant with the smallest mean Sinkhorn divergence is bold and colored, with the \underline{three smallest} values across all HyCNN rows underlined.}
\label{tab:ot_sinkhorn_full_psi2}
\scriptsize
\makebox[\textwidth][c]{%
\begin{tabular}{llccc}
\toprule
\multicolumn{2}{c}{Method $\backslash$ $d$} & 10 & 20 & 50 \\
\midrule
\multirow{9}{*}{HyCNN} & $(44 \times 2); \tau = 0.1$ & \underline{\textcolor[HTML]{0080DB}{\textbf{10.741}}} (0.092) & 53.760 (0.268) & 222.568 (0.837) \\
 & $(44 \times 2); \tau = 1$ & 11.836 (0.122) & 54.513 (0.182) & 209.047 (0.371) \\
 & $(44 \times 2); \tau = 10$ & 14.598 (0.076) & \textcolor[HTML]{0080DB}{\textbf{53.144}} (0.078) & \underline{\textcolor[HTML]{0080DB}{\textbf{206.854}}} (0.282) \\
 & $(44 \times 4); \tau = 0.1$ & \underline{\textcolor[HTML]{0080DB}{\textbf{10.558}}} (0.076) & 50.522 (0.124) & 220.174 (0.273) \\
 & $(44 \times 4); \tau = 1$ & 11.177 (0.078) & \underline{\textcolor[HTML]{0080DB}{\textbf{49.013}}} (0.261) & 211.551 (0.211) \\
 & $(44 \times 4); \tau = 10$ & 14.652 (0.085) & 53.057 (0.073) & \underline{\textcolor[HTML]{0080DB}{\textbf{205.994}}} (0.357) \\
 & $(44 \times 6); \tau = 0.1$ & \underline{\textcolor[HTML]{0080DB}{\textbf{10.582}}} (0.074) & \underline{\textcolor[HTML]{0080DB}{\textbf{50.479}}} (0.145) & 215.187 (0.456) \\
 & $(44 \times 6); \tau = 1$ & 11.457 (0.075) & \underline{50.502} (0.179) & 208.538 (0.310) \\
 & $(44 \times 6); \tau = 10$ & 14.651 (0.077) & 53.021 (0.060) & \underline{\textcolor[HTML]{0080DB}{\textbf{205.254}}} (0.326) \\
\bottomrule
\end{tabular}%
}
\end{table}

\begin{table}[h!]
\caption{Prediction MSE (mean and SE over $10$ runs) of the estimated OT map for the OT map (d) $\textstyle T_4(\bx) = (4x_i^3)_{i \in [d]}$ based on $n = m = 5{,}000$ independent training sample points from the source and target distribution. Per HyCNN architecture ($44 \times L$), the $\tau$ variant with the smallest Sinkhorn divergence validation metric ($\epsilon = 0.1$) for $1{,}000$ validation samples at that dimension is used; see Table~\ref{tab:ot_sinkhorn_full_psi3}. Per dimension, the smallest mean MSE per method type is bold and colored, with the \underline{three smallest} values across all methods underlined.}
\label{tab:ot_testmse_collapsed_psi3}
\scriptsize
\makebox[\textwidth][c]{%
\begin{tabular}{llccc}
\toprule
\multicolumn{2}{c}{Method $\backslash$ $d$} & 10 & 20 & 50 \\
\midrule
\multirow{3}{*}{Entropic OT} & $\epsilon = 0.1$ & 8.990 (0.046) & 32.534 (0.144) & 128.256 (0.323) \\
 & $\epsilon = 1$ & \textcolor[HTML]{006400}{\textbf{3.417}} (0.029) & \textcolor[HTML]{006400}{\textbf{16.729}} (0.092) & 98.421 (0.422) \\
 & $\epsilon = 10$ & 16.439 (0.073) & 32.839 (0.138) & \textcolor[HTML]{006400}{\textbf{82.155}} (0.217) \\
\midrule
\multirow{3}{*}{Monge gap} & $(104 \times 2)$ & 10.503 (0.275) & \textcolor[HTML]{7B2CBF}{\textbf{18.845}} (0.241) & 56.263 (0.193) \\
 & $(104 \times 4)$ & \textcolor[HTML]{7B2CBF}{\textbf{10.297}} (0.227) & 19.113 (0.190) & 52.628 (0.244) \\
 & $(104 \times 6)$ & 12.489 (0.314) & 20.964 (0.265) & \textcolor[HTML]{7B2CBF}{\textbf{51.901}} (0.165) \\
\midrule
\multirow{3}{*}{ICNN (ReLU)} & $(64 \times 2)$ & 5.675 (0.143) & 14.824 (0.200) & 68.621 (2.424) \\
 & $(64 \times 4)$ & 4.821 (0.081) & 13.891 (0.320) & 46.814 (0.934) \\
 & $(64 \times 6)$ & \textcolor[HTML]{B87700}{\textbf{3.909}} (0.080) & \textcolor[HTML]{B87700}{\textbf{11.780}} (0.297) & \textcolor[HTML]{B87700}{\textbf{41.207}} (0.559) \\
\midrule
\multirow{3}{*}{ICNN (L-ReLU)} & $(64 \times 2)$ & 6.677 (0.120) & 21.465 (0.546) & 4.5e+03 (2.3e+03) \\
 & $(64 \times 4)$ & 5.140 (0.131) & 13.069 (0.265) & 65.749 (1.899) \\
 & $(64 \times 6)$ & \textcolor[HTML]{A46500}{\textbf{4.823}} (0.163) & \textcolor[HTML]{A46500}{\textbf{12.307}} (0.373) & \textcolor[HTML]{A46500}{\textbf{44.342}} (0.820) \\
\midrule
\multirow{3}{*}{ICNNq (ReLU)} & $(64 \times 2)$ & 1.754 (0.042) & 10.111 (0.121) & \textcolor[HTML]{E31E00}{\textbf{25.761}} (0.543) \\
 & $(64 \times 4)$ & \textcolor[HTML]{E31E00}{\textbf{1.435}} (0.049) & \textcolor[HTML]{E31E00}{\textbf{8.165}} (0.616) & 437.974 (50.407) \\
 & $(64 \times 6)$ & 1.746 (0.155) & 22.471 (3.388) & 1.82e+04 (2.4e+03) \\
\midrule
\multirow{9}{*}{ICNNq (Softplus)} & $(64 \times 2); \tau = 0.1$ & 0.878 (0.038) & 8.627 (0.049) & 26.297 (0.826) \\
 & $(64 \times 2); \tau = 1$ & \underline{0.546} (0.064) & 7.295 (0.422) & 26.434 (0.609) \\
 & $(64 \times 2); \tau = 10$ & 3.544 (0.021) & 7.969 (0.035) & 22.885 (0.060) \\
 & $(64 \times 4); \tau = 0.1$ & 0.722 (0.012) & 5.543 (0.209) & 53.773 (3.493) \\
 & $(64 \times 4); \tau = 1$ & \underline{\textcolor[HTML]{8B1301}{\textbf{0.360}}} (0.011) & \underline{3.532} (0.440) & 36.023 (0.716) \\
 & $(64 \times 4); \tau = 10$ & 3.348 (0.044) & 7.731 (0.036) & 22.116 (0.037) \\
 & $(64 \times 6); \tau = 0.1$ & 0.812 (0.030) & \underline{5.330} (0.183) & 2.96e+03 (1.12e+03) \\
 & $(64 \times 6); \tau = 1$ & \underline{0.412} (0.025) & \underline{\textcolor[HTML]{8B1301}{\textbf{1.913}}} (0.187) & 30.990 (0.838) \\
 & $(64 \times 6); \tau = 10$ & 3.437 (0.037) & 7.611 (0.027) & \textcolor[HTML]{8B1301}{\textbf{21.854}} (0.067) \\
\midrule
\multirow{3}{*}{HyCNN} & $(44 \times 2)$ & 2.757 (0.177) & 7.841 (0.016) & \underline{20.554} (0.066) \\
 & $(44 \times 4)$ & 1.528 (0.171) & \textcolor[HTML]{0080DB}{\textbf{7.783}} (0.020) & \underline{\textcolor[HTML]{0080DB}{\textbf{19.830}}} (0.028) \\
 & $(44 \times 6)$ & \textcolor[HTML]{0080DB}{\textbf{0.813}} (0.029) & 7.865 (0.028) & \underline{20.132} (0.044) \\
\bottomrule
\end{tabular}%
}
\end{table}

\begin{table}[h!]
\caption{Sinkhorn divergence (mean and SE over $10$ runs) of the estimated OT map for the OT map (d) $\textstyle T_4(\bx) = (4x_i^3)_{i \in [d]}$. Per HyCNN architecture ($44 \times L$) and per dimension, the $\tau$ variant with the smallest mean Sinkhorn divergence is bold and colored, with the \underline{three smallest} values across all HyCNN rows underlined.}
\label{tab:ot_sinkhorn_full_psi3}
\scriptsize
\makebox[\textwidth][c]{%
\begin{tabular}{llccc}
\toprule
\multicolumn{2}{c}{Method $\backslash$ $d$} & 10 & 20 & 50 \\
\midrule
\multirow{9}{*}{HyCNN} & $(44 \times 2); \tau = 0.1$ & \underline{\textcolor[HTML]{0080DB}{\textbf{9.940}}} (0.121) & 37.842 (0.130) & 141.020 (0.355) \\
 & $(44 \times 2); \tau = 1$ & 10.729 (0.029) & 36.647 (0.087) & 133.277 (0.744) \\
 & $(44 \times 2); \tau = 10$ & 10.455 (0.033) & \underline{\textcolor[HTML]{0080DB}{\textbf{35.179}}} (0.076) & \underline{\textcolor[HTML]{0080DB}{\textbf{131.933}}} (0.176) \\
 & $(44 \times 4); \tau = 0.1$ & \underline{\textcolor[HTML]{0080DB}{\textbf{9.574}}} (0.147) & 37.305 (0.091) & 144.998 (0.910) \\
 & $(44 \times 4); \tau = 1$ & 10.544 (0.030) & 35.759 (0.081) & 135.727 (0.193) \\
 & $(44 \times 4); \tau = 10$ & 10.438 (0.035) & \underline{\textcolor[HTML]{0080DB}{\textbf{35.055}}} (0.083) & \underline{\textcolor[HTML]{0080DB}{\textbf{131.051}}} (0.202) \\
 & $(44 \times 6); \tau = 0.1$ & \underline{\textcolor[HTML]{0080DB}{\textbf{9.196}}} (0.037) & 37.403 (0.086) & 143.907 (0.621) \\
 & $(44 \times 6); \tau = 1$ & 10.521 (0.034) & 35.508 (0.082) & 132.300 (0.195) \\
 & $(44 \times 6); \tau = 10$ & 10.440 (0.035) & \underline{\textcolor[HTML]{0080DB}{\textbf{35.100}}} (0.084) & \underline{\textcolor[HTML]{0080DB}{\textbf{131.360}}} (0.199) \\
\bottomrule
\end{tabular}%
}
\end{table}

\clearpage
{\reb
\subsubsection{Ablations on convexity, architecture, initialization, and smoothness}\label{app:OT_ablations}

In this subsection, we examine the effect of enforcing input-convexity, the choice of the architecture for the neural OT map, the initialization strategy, and the degree of smoothness in the activation functions on the performance of the neural OT map. To this end, we focus on the performance of different semi-dual estimators for the high-dimensional setting, $d=50$, and the four ground truth OT potentials $\phi_1,\ldots,\phi_4$ used in the earlier experiments. 

\parabold{Experimental setup}
Concretely, we compare unconstrained MLP potentials (ReLU or Softplus, width $72$), one-lane ICNNs (Softplus, width $64$) with and without quadratic first-layer terms, and HyCNNs (log-sum-exp, width $44$). The MLP here parameterizes a scalar potential whose gradient estimates the transport map (unlike the Monge gap MLP estimator above). We consider depths $L\in\{2,4,6\}$ and, for smooth activations, $\tau\in\{0.1,1,10\}$, with $\operatorname{Softplus}_\tau(a)=\tau\log(1+e^{a/\tau})$. For HyCNNs, we compare our initialization from \Cref{sec:trainingHyCNNs} with a near-ICNN initialization, i.e., where for each HyCNN block the hidden-to-hidden weights in the second-lane  are initialized to $10^{-5}$ and the input-affine weights and biases are set to zero. This retains two trainable lanes while starting close to a one-lane network. The details to the training procedure, including learning rates, batch sizes, and optimization steps, follow those from the previously conducted high-dimensional experiments. The results of the respective numerical experiments are displayed in \Cref{tab:ot_d50_all_networks}. All entries are test prediction MSEs, with coordinate errors summed, reported as mean (SE) over ten runs.

\parabold{Main results}
 Overall, the HyCNN architecture with our initialization consistently achieves the best performance across the considered settings. Our ablation studies reveal that convexity helps stabilize the performance and smoothness aids in improving the overall accuracy, and that proper initialization is crucial for leveraging the HyCNN architecture most effectively. Further, according to the Sinkhorn-validated selection criterion, the best HyCNN configuration is consistently identified on an independent validation set (\Cref{tab:ot_d50_sinkhorn_selected}), and leads to the best overall performance across all four considered ground truth potentials. 

\parabold{On convexity and smoothness}
The results reveal that without imposing convexity or smoothness, numerical instabilities occur. This may be a consequence of the MLP overfitting too much and its gradient being too unstable. Overall, imposing smoothness appears to consistently help performance. In contrast, imposing convexity rather stabilizes the performance across different settings and leads to performance gains for deeper architectures. 

\parabold{On the architecture}
Standard ICNNs seem to be consistently outperformed by some smooth MLP, but do not experience any catastrophic loss of performance such as non-smooth MLPs. Further, deeper ICNNs perform best $(L=6$) with moderate smoothness $\tau = 1$. We attribute this to the fact that standard ICNNs need to represent the respective quadratic nature of the ground truth potentials for which larger depth is beneficial and because the initialization by \cite{hoedt2023principled} is not optimized for large $\tau=10$. However, the considered ICNNs are consistently outperformed by deeper HyCNNs with HyCNN initialization and under near-ICNN initialization  with large-smoothness $\tau=10$. 

\parabold{On the initialization} Comparing the two HyCNN-based estimators with near-ICNN and our proposed initialization, we see that choosing the latter leads to consistently better performance. Notably, except for depth $L = 6$ and $\tau = 10$, HyCNNs with near-ICNN initialization appear to be consistently outperformed by the corresponding ICNN architecture with same depth and smoothness parameter. This suggests that the architecture of HyCNNs is not the sole reason for their performance but also its initialization. 

}

\begingroup
\fontsize{8}{11}\selectfont
\setlength{\tabcolsep}{2pt}
\setlength{\LTcapwidth}{\textwidth}
\renewcommand{\arraystretch}{1.04}
\begin{longtable}{llcccc}
\caption{Ablation studies for neural-potential OT estimators at $d=50$, with $n=m=5{,}000$. Columns correspond to the identity, anisotropic quadratic, quadratic plus $\ell_1$, and quartic potentials, denoted by $\phi_1$--$\phi_4$. Entries display the mean (standard error) of the prediction MSE over $10$ runs on $1{,}000$ independent test points. Bold colored entries mark the smallest mean within each method and potential, and underlined values mark the three smallest means per column across each ground truth OT potential $\phi_i$ for $i = 1, \dots, 4$.}
\label{tab:ot_d50_all_networks}\\
\toprule
Method & Width $\times$ depth; activation & $\phi_1$ & $\phi_2$ & $\phi_3$ & $\phi_4$ \\
\midrule
\endfirsthead
\multicolumn{6}{l}{\tablename\ \thetable\ (continued)} \\
\toprule
Method & Width $\times$ depth; activation & $\phi_1$ & $\phi_2$ & $\phi_3$ & $\phi_4$ \\
\midrule
\endhead
\midrule
\multicolumn{6}{r}{Continued on next page} \\
\endfoot
\bottomrule
\endlastfoot
\multirow{3}{*}{MLP (ReLU)} & $(72 \times 2);\ \mathrm{ReLU}$ & \ensuremath{\textcolor[HTML]{4B5563}{\mathbf{74.57}}} (\ensuremath{3.04}) & \ensuremath{\textcolor[HTML]{4B5563}{\mathbf{72.98}}} (\ensuremath{2.83}) & \ensuremath{\textcolor[HTML]{4B5563}{\mathbf{179.67}}} (\ensuremath{6.71}) & \ensuremath{\textcolor[HTML]{4B5563}{\mathbf{169.48}}} (\ensuremath{7.50}) \\*
 & $(72 \times 4);\ \mathrm{ReLU}$ & \shortstack{$7.9\times 10^{4}$\\ $(4.0\times 10^{4})$} & \shortstack{$2.7\times 10^{4}$\\ $(3620.98)$} & \shortstack{$3.6\times 10^{5}$\\ $(2.4\times 10^{5})$} & \shortstack{$1.4\times 10^{7}$\\ $(2.5\times 10^{6})$} \\*
 & $(72 \times 6);\ \mathrm{ReLU}$ & \shortstack{$1.3\times 10^{10}$\\ $(9.5\times 10^{9})$} & \shortstack{$1.1\times 10^{12}$\\ $(1.1\times 10^{12})$} & \shortstack{$2.0\times 10^{16}$\\ $(1.4\times 10^{16})$} & \shortstack{$3.8\times 10^{18}$\\ $(1.7\times 10^{18})$} \\
\midrule
\multirow{9}{*}{MLP (Softplus)} & $(72 \times 2);\ \tau=0.1$ & \ensuremath{52.50} (\ensuremath{1.92}) & \ensuremath{49.73} (\ensuremath{1.90}) & \ensuremath{149.94} (\ensuremath{5.59}) & \ensuremath{116.86} (\ensuremath{3.24}) \\*
 & $(72 \times 2);\ \tau=1$ & \ensuremath{6.47} (\ensuremath{0.30}) & \ensuremath{6.65} (\ensuremath{0.17}) & \ensuremath{33.06} (\ensuremath{0.33}) & \ensuremath{\textcolor[HTML]{D33C79}{\mathbf{22.34}}} (\ensuremath{0.03}) \\*
 & $(72 \times 2);\ \tau=10$ & \ensuremath{\underline{\textcolor[HTML]{D33C79}{\mathbf{0.70}}}} (\ensuremath{0.02}) & \ensuremath{\textcolor[HTML]{D33C79}{\mathbf{1.29}}} (\ensuremath{0.07}) & \ensuremath{\underline{\textcolor[HTML]{D33C79}{\mathbf{20.16}}}} (\ensuremath{0.06}) & \ensuremath{25.42} (\ensuremath{0.47}) \\*
 & $(72 \times 4);\ \tau=0.1$ & \shortstack{$3.3\times 10^{4}$\\ $(2.0\times 10^{4})$} & \shortstack{$1.9\times 10^{4}$\\ $(6028.90)$} & \shortstack{$1.1\times 10^{5}$\\ $(4.5\times 10^{4})$} & \shortstack{$1.0\times 10^{6}$\\ $(2.6\times 10^{5})$} \\*
 & $(72 \times 4);\ \tau=1$ & \ensuremath{10.28} (\ensuremath{1.01}) & \ensuremath{11.56} (\ensuremath{1.87}) & \ensuremath{43.07} (\ensuremath{0.96}) & \ensuremath{24.19} (\ensuremath{0.18}) \\*
 & $(72 \times 4);\ \tau=10$ & \ensuremath{0.93} (\ensuremath{0.14}) & \shortstack{$5.4\times 10^{8}$\\ $(5.4\times 10^{8})$} & \ensuremath{21.76} (\ensuremath{0.08}) & \ensuremath{25.26} (\ensuremath{0.39}) \\*
 & $(72 \times 6);\ \tau=0.1$ & \shortstack{$3.1\times 10^{9}$\\ $(3.1\times 10^{9})$} & \shortstack{$4.6\times 10^{6}$\\ $(3.1\times 10^{6})$} & \shortstack{$1.9\times 10^{11}$\\ $(1.5\times 10^{11})$} & \shortstack{$7.0\times 10^{12}$\\ $(3.8\times 10^{12})$} \\*
 & $(72 \times 6);\ \tau=1$ & \ensuremath{21.94} (\ensuremath{5.83}) & \ensuremath{19.14} (\ensuremath{4.31}) & \ensuremath{41.96} (\ensuremath{2.04}) & \ensuremath{24.43} (\ensuremath{0.46}) \\*
 & $(72 \times 6);\ \tau=10$ & \ensuremath{1.03} (\ensuremath{0.12}) & \ensuremath{1.67} (\ensuremath{0.10}) & \ensuremath{22.72} (\ensuremath{0.32}) & \ensuremath{26.64} (\ensuremath{0.47}) \\
\midrule
\multirow{9}{*}{\shortstack[l]{ICNN (Softplus)}} & $(64 \times 2);\ \tau=0.1$ & \ensuremath{18.56} (\ensuremath{0.45}) & \ensuremath{18.92} (\ensuremath{0.41}) & \ensuremath{102.28} (\ensuremath{2.20}) & \ensuremath{63.68} (\ensuremath{1.26}) \\*
 & $(64 \times 2);\ \tau=1$ & \ensuremath{6.73} (\ensuremath{0.07}) & \ensuremath{6.29} (\ensuremath{0.07}) & \ensuremath{42.26} (\ensuremath{0.87}) & \ensuremath{27.12} (\ensuremath{0.21}) \\*
 & $(64 \times 2);\ \tau=10$ & \ensuremath{54.62} (\ensuremath{8.41}) & \ensuremath{36.67} (\ensuremath{3.19}) & \ensuremath{116.14} (\ensuremath{24.68}) & \ensuremath{153.74} (\ensuremath{26.45}) \\*
 & $(64 \times 4);\ \tau=0.1$ & \ensuremath{10.60} (\ensuremath{0.25}) & \ensuremath{10.31} (\ensuremath{0.23}) & \ensuremath{59.62} (\ensuremath{0.97}) & \ensuremath{44.71} (\ensuremath{0.61}) \\*
 & $(64 \times 4);\ \tau=1$ & \ensuremath{3.24} (\ensuremath{0.09}) & \ensuremath{3.12} (\ensuremath{0.08}) & \ensuremath{30.95} (\ensuremath{0.38}) & \ensuremath{23.90} (\ensuremath{0.13}) \\*
 & $(64 \times 4);\ \tau=10$ & \ensuremath{11.94} (\ensuremath{1.17}) & \ensuremath{5.64} (\ensuremath{0.36}) & \ensuremath{36.76} (\ensuremath{0.99}) & \ensuremath{67.77} (\ensuremath{4.23}) \\*
 & $(64 \times 6);\ \tau=0.1$ & \ensuremath{7.77} (\ensuremath{0.17}) & \ensuremath{7.48} (\ensuremath{0.11}) & \ensuremath{47.15} (\ensuremath{0.71}) & \ensuremath{38.24} (\ensuremath{0.82}) \\*
 & $(64 \times 6);\ \tau=1$ & \ensuremath{\textcolor[HTML]{7538A3}{\mathbf{2.11}}} (\ensuremath{0.04}) & \ensuremath{\textcolor[HTML]{7538A3}{\mathbf{2.06}}} (\ensuremath{0.03}) & \ensuremath{\textcolor[HTML]{7538A3}{\mathbf{26.09}}} (\ensuremath{0.13}) & \ensuremath{\textcolor[HTML]{7538A3}{\mathbf{21.97}}} (\ensuremath{0.10}) \\*
 & $(64 \times 6);\ \tau=10$ & \ensuremath{13.86} (\ensuremath{1.82}) & \ensuremath{4.57} (\ensuremath{0.26}) & \ensuremath{32.82} (\ensuremath{1.26}) & \ensuremath{48.72} (\ensuremath{2.13}) \\
\midrule
\multirow{9}{*}{\shortstack[l]{HyCNN\\(near-ICNN init)}} & $(44 \times 2);\ \tau=0.1$ & \ensuremath{19.98} (\ensuremath{0.37}) & \ensuremath{20.54} (\ensuremath{0.40}) & \ensuremath{104.17} (\ensuremath{2.64}) & \ensuremath{66.95} (\ensuremath{1.35}) \\*
 & $(44 \times 2);\ \tau=1$ & \ensuremath{10.55} (\ensuremath{0.21}) & \ensuremath{10.07} (\ensuremath{0.10}) & \ensuremath{61.62} (\ensuremath{0.75}) & \ensuremath{36.45} (\ensuremath{0.60}) \\*
 & $(44 \times 2);\ \tau=10$ & \ensuremath{5.02} (\ensuremath{0.71}) & \ensuremath{18.49} (\ensuremath{3.53}) & \ensuremath{51.19} (\ensuremath{7.17}) & \ensuremath{68.15} (\ensuremath{10.99}) \\*
 & $(44 \times 4);\ \tau=0.1$ & \ensuremath{12.27} (\ensuremath{0.18}) & \ensuremath{12.25} (\ensuremath{0.23}) & \ensuremath{64.82} (\ensuremath{0.66}) & \ensuremath{50.35} (\ensuremath{1.00}) \\*
 & $(44 \times 4);\ \tau=1$ & \ensuremath{4.86} (\ensuremath{0.10}) & \ensuremath{4.83} (\ensuremath{0.08}) & \ensuremath{36.25} (\ensuremath{0.38}) & \ensuremath{25.95} (\ensuremath{0.17}) \\*
 & $(44 \times 4);\ \tau=10$ & \ensuremath{0.83} (\ensuremath{0.04}) & \ensuremath{0.98} (\ensuremath{0.07}) & \ensuremath{20.69} (\ensuremath{0.08}) & \ensuremath{27.54} (\ensuremath{0.55}) \\*
 & $(44 \times 6);\ \tau=0.1$ & \ensuremath{9.27} (\ensuremath{0.20}) & \ensuremath{9.26} (\ensuremath{0.19}) & \ensuremath{53.56} (\ensuremath{1.62}) & \ensuremath{44.98} (\ensuremath{3.18}) \\*
 & $(44 \times 6);\ \tau=1$ & \ensuremath{3.60} (\ensuremath{0.08}) & \ensuremath{3.56} (\ensuremath{0.07}) & \ensuremath{31.09} (\ensuremath{0.31}) & \ensuremath{24.09} (\ensuremath{0.25}) \\*
 & $(44 \times 6);\ \tau=10$ & \ensuremath{\textcolor[HTML]{944F20}{\mathbf{0.76}}} (\ensuremath{0.03}) & \ensuremath{\textcolor[HTML]{944F20}{\mathbf{0.85}}} (\ensuremath{0.04}) & \ensuremath{\textcolor[HTML]{944F20}{\mathbf{20.53}}} (\ensuremath{0.04}) & \ensuremath{\textcolor[HTML]{944F20}{\mathbf{23.87}}} (\ensuremath{0.27}) \\
\midrule
\multirow{9}{*}{\shortstack[l]{HyCNN\\(HyCNN init)}} & $(44 \times 2);\ \tau=0.1$ & \ensuremath{18.31} (\ensuremath{0.24}) & \ensuremath{18.30} (\ensuremath{0.22}) & \ensuremath{84.77} (\ensuremath{0.53}) & \ensuremath{58.38} (\ensuremath{0.44}) \\*
 & $(44 \times 2);\ \tau=1$ & \ensuremath{8.60} (\ensuremath{0.06}) & \ensuremath{8.85} (\ensuremath{0.06}) & \ensuremath{50.00} (\ensuremath{0.30}) & \ensuremath{31.06} (\ensuremath{0.16}) \\*
 & $(44 \times 2);\ \tau=10$ & \ensuremath{0.78} (\ensuremath{0.01}) & \ensuremath{\underline{0.71}} (\ensuremath{0.01}) & \ensuremath{21.93} (\ensuremath{0.08}) & \ensuremath{\underline{20.52}} (\ensuremath{0.05}) \\*
 & $(44 \times 4);\ \tau=0.1$ & \ensuremath{10.62} (\ensuremath{0.07}) & \ensuremath{10.77} (\ensuremath{0.09}) & \ensuremath{56.84} (\ensuremath{0.40}) & \ensuremath{43.54} (\ensuremath{0.35}) \\*
 & $(44 \times 4);\ \tau=1$ & \ensuremath{3.96} (\ensuremath{0.02}) & \ensuremath{4.07} (\ensuremath{0.05}) & \ensuremath{31.68} (\ensuremath{0.13}) & \ensuremath{24.38} (\ensuremath{0.06}) \\*
 & $(44 \times 4);\ \tau=10$ & \ensuremath{\underline{0.40}} (\ensuremath{0.006}) & \ensuremath{\underline{\textcolor[HTML]{0080A0}{\mathbf{0.42}}}} (\ensuremath{0.01}) & \ensuremath{\underline{19.89}} (\ensuremath{0.04}) & \ensuremath{\underline{\textcolor[HTML]{0080A0}{\mathbf{19.87}}}} (\ensuremath{0.04}) \\*
 & $(44 \times 6);\ \tau=0.1$ & \ensuremath{7.17} (\ensuremath{0.09}) & \ensuremath{7.31} (\ensuremath{0.09}) & \ensuremath{44.48} (\ensuremath{0.47}) & \ensuremath{36.12} (\ensuremath{0.25}) \\*
 & $(44 \times 6);\ \tau=1$ & \ensuremath{2.94} (\ensuremath{0.07}) & \ensuremath{2.98} (\ensuremath{0.07}) & \ensuremath{27.75} (\ensuremath{0.23}) & \ensuremath{22.62} (\ensuremath{0.14}) \\*
 & $(44 \times 6);\ \tau=10$ & \ensuremath{\underline{\textcolor[HTML]{0080A0}{\mathbf{0.38}}}} (\ensuremath{0.005}) & \ensuremath{\underline{0.51}} (\ensuremath{0.01}) & \ensuremath{\underline{\textcolor[HTML]{0080A0}{\mathbf{19.80}}}} (\ensuremath{0.02}) & \ensuremath{\underline{20.11}} (\ensuremath{0.05}) \\
\end{longtable}
\endgroup

\begin{table}[t!]
\caption{Prediction MSE (mean and standard error across 10 runs) based on Sinkhorn-validation based selection criterion among depth and smoothness for each method and potential from \Cref{tab:ot_d50_all_networks}. Each cell shows the mean and standard error of the selected method.
Bold text denotes the smallest selected mean in each column.}
\label{tab:ot_d50_sinkhorn_selected}
\centering
\begingroup
\fontsize{8}{11}\selectfont
\setlength{\tabcolsep}{4pt}
\renewcommand{\arraystretch}{1.1}
\begin{tabular}{llcccc}
\toprule
\multicolumn{2}{l}{Method} & $\phi_1$ & $\phi_2$ & $\phi_3$ & $\phi_4$ \\
\midrule
\multicolumn{2}{l}{MLP (ReLU)} & \shortstack{\ensuremath{\textcolor[HTML]{000000}{74.57}} (\ensuremath{3.04})\\{\scriptsize $L=2;\ \mathrm{ReLU}$}} & \shortstack{\ensuremath{\textcolor[HTML]{000000}{72.98}} (\ensuremath{2.83})\\{\scriptsize $L=2;\ \mathrm{ReLU}$}} & \shortstack{\ensuremath{\textcolor[HTML]{000000}{179.67}} (\ensuremath{6.71})\\{\scriptsize $L=2;\ \mathrm{ReLU}$}} & \shortstack{\ensuremath{\textcolor[HTML]{000000}{169.48}} (\ensuremath{7.50})\\{\scriptsize $L=2;\ \mathrm{ReLU}$}} \\*\midrule
\multicolumn{2}{l}{MLP (Softplus)} & \shortstack{\ensuremath{\textcolor[HTML]{000000}{0.70}} (\ensuremath{0.02})\\{\scriptsize $L=2;\ \tau=10$}} & \shortstack{\ensuremath{\textcolor[HTML]{000000}{1.29}} (\ensuremath{0.07})\\{\scriptsize $L=2;\ \tau=10$}} & \shortstack{\ensuremath{\textcolor[HTML]{000000}{20.16}} (\ensuremath{0.06})\\{\scriptsize $L=2;\ \tau=10$}} & \shortstack{\ensuremath{\textcolor[HTML]{000000}{22.34}} (\ensuremath{0.03})\\{\scriptsize $L=2;\ \tau=1$}} \\*\midrule
\multicolumn{2}{l}{\shortstack[l]{ICNN (Softplus)\\(no quadratic)}} & \shortstack{\ensuremath{\textcolor[HTML]{000000}{2.11}} (\ensuremath{0.04})\\{\scriptsize $L=6;\ \tau=1$}} & \shortstack{\ensuremath{\textcolor[HTML]{000000}{2.06}} (\ensuremath{0.03})\\{\scriptsize $L=6;\ \tau=1$}} & \shortstack{\ensuremath{\textcolor[HTML]{000000}{42.26}} (\ensuremath{0.87})\\{\scriptsize $L=2;\ \tau=1$}} & \shortstack{\ensuremath{\textcolor[HTML]{000000}{27.12}} (\ensuremath{0.21})\\{\scriptsize $L=2;\ \tau=1$}} \\*\midrule
\multicolumn{2}{l}{\shortstack[l]{HyCNN\\(near-ICNN init)}} & \shortstack{\ensuremath{\textcolor[HTML]{000000}{0.83}} (\ensuremath{0.04})\\{\scriptsize $L=4;\ \tau=10$}} & \shortstack{\ensuremath{\textcolor[HTML]{000000}{0.98}} (\ensuremath{0.07})\\{\scriptsize $L=4;\ \tau=10$}} & \shortstack{\ensuremath{\textcolor[HTML]{000000}{20.69}} (\ensuremath{0.08})\\{\scriptsize $L=4;\ \tau=10$}} & \shortstack{\ensuremath{\textcolor[HTML]{000000}{36.45}} (\ensuremath{0.60})\\{\scriptsize $L=2;\ \tau=1$}} \\*\midrule
\multicolumn{2}{l}{\shortstack[l]{HyCNN\\(HyCNN init)}} & \shortstack{\ensuremath{\textcolor[HTML]{000000}{\mathbf{0.38}}} (\ensuremath{0.005})\\{\scriptsize $L=6;\ \tau=10$}} & \shortstack{\ensuremath{\textcolor[HTML]{000000}{\mathbf{0.42}}} (\ensuremath{0.01})\\{\scriptsize $L=4;\ \tau=10$}} & \shortstack{\ensuremath{\textcolor[HTML]{000000}{\mathbf{19.80}}} (\ensuremath{0.02})\\{\scriptsize $L=6;\ \tau=10$}} & \shortstack{\ensuremath{\textcolor[HTML]{000000}{\mathbf{19.87}}} (\ensuremath{0.04})\\{\scriptsize $L=4;\ \tau=10$}} \\
\bottomrule
\end{tabular}
\endgroup
\end{table}
\section{Single Cell Data Analysis}\label{app:single_cell}

A central task in modern functional genomics and drug discovery is to predict how a population of cells responds to an external perturbation such as a chemical compound, a genetic intervention, or a change in culture conditions. Because most profiling technologies (single-cell RNA sequencing, mass cytometry, multiplexed imaging) destroy the cell upon measurement, one cannot observe the same cell before and after treatment. As a consequence, the unperturbed \emph{control} population and the \emph{perturbed} population are only accessible through two independent collections of single-cell measurements, and the learning problem is intrinsically one of matching \emph{distributions} rather than individual pairs of data points. This viewpoint was advocated by \citet{schiebinger2019optimal}, who used optimal transport to infer developmental trajectories from unpaired snapshots of single-cell gene expression during reprogramming, and was subsequently taken up in a series of works on perturbation response modeling, most notably \citet{bunne2023learning}. These works established optimal transport as a natural and principled language for formulating single-cell response prediction.

\parabold{The 4i dataset} The dataset we consider is the multiplexed imaging dataset of \citet{bunne2023learning}, obtained from the \emph{iterative indirect immunofluorescence imaging} (4i) technology. It is based on the sequential application of indirect immunofluorescence staining cycles to the same specimen, which allows a very large number of protein markers to be probed on the same set of cells. The result is a high-dimensional, single-cell-resolved readout of the proteomic state of each cell under control and under a wide range of drug treatments, made available by \citet{bunne2023learning} in a preprocessed form that we use without further modification. For each drug treatment we obtain two point clouds in a common feature space, one from the untreated control population and one from the cells that were exposed to the drug, with sample sizes between roughly $1{,}300$ and $3{,}000$ cells per treatment (see \Cref{tab:4i-bestval-SD-n50}).

\parabold{Mathematical task} Denote by $P$ the distribution of the control cell population and by $Q_D$ the distribution of the cell population after exposure to drug $D$. We are given measurements $\bX_1, \dots, \bX_n$ from the control cell population and after treatment $D$, denoted by $\bY_1, \dots, \bY_m$, which we model as independent and identically distributed realizations from $P$ and $Q_D$, respectively.  We wish to construct an estimator $\hat T \colon \RR^p \to \RR^p$ such that the pushforward $\hat T_{\#} P$ is as close as possible to $Q_D$. The estimator is evaluated on a held-out test partition consisting of samples from both $P$ and $Q_D$, and the quality of the fit is measured by the Sinkhorn divergence, a geometrically meaningful two-sample discrepancy measure between probability distributions \citep{feydy2019interpolating,peyre2019computational}, between $\hat T_{\#} \hat P_{\mathrm{test}}$ and $\hat Q_{D,\mathrm{test}}$. %

\parabold{Why optimal transport} Among distribution matching methodologies, optimal transport is particularly appealing in this setting for two reasons. Conceptually, the static OT map minimizes the expected squared displacement and thus embodies the biological belief that cells evolve and proliferate along minimal-effort trajectories in feature space, rather than being reshuffled arbitrarily. This heuristic is motivated by the presence of an underlying potential, which drives the development and change of cells as they undergo the effects of a treatment. Methodologically, OT has been repeatedly reported to perform well as a predictive framework for single-cell perturbation response, starting from the developmental-trajectory analysis of \citet{schiebinger2019optimal} and including the neural OT approach of \citet{bunne2023learning}, among others. We therefore adopt neural OT map estimation as the basis of our analysis and compare our HyCNN architecture against the most competitive convex neural OT baseline identified in our synthetic experiments.

\parabold{Concrete modeling choices} For each drug $D$, we train two families of neural OT estimators to map the control distribution onto the perturbed distribution: a HyCNN of {\reb width $44$ and depth $L = 4$}, and an ICNN of width $64$ and depth $L = 4$ equipped with softplus activations and a quadratic term in the first layer (ICNN(sp+q)). The ICNN(sp+q) variant consistently emerged as one of the strongest convex neural baselines in our synthetic OT experiments, see \Cref{app:OT_experiments}, which motivates its use as the reference competitor here. The width--depth configuration $64 \times 4$ was chosen because depth four has become a de facto standard in the neural OT literature \citep{makkuva2020optimal,bunne2023learning,uscidda2023monge}, {\reb and because the corresponding width of HyCNNs yields models whose parameter size is comparable to the ICNN.} For the HyCNN, we employ a smooth log-sum-exp activation with a smoothing parameter $\tau = 1$, a choice that is capable of adapting to smooth structure while remaining flexible. The ICNN baseline uses softplus activations with $\tau = 1$ for the same reason. 

\parabold{Learning algorithm} For training, we first split the control and drug cell data into $70\%-10\%-20\%$ for training, validation, and testing, respectively. We estimate the OT map using the variational semidual formulation as in \eqref{eq:HyCNN_OT_potential_estimation} for HyCNN and \eqref{eq:ICNN_objective} for ICNN, in which the outer and inner potentials are both parameterized by a HyCNN or ICNN. For the HyCNN and ICNN, we impose essentially the same learning algorithm as in \Cref{app:OT_implementation}. To assess the effect of the training budget, we consider two complementary schedules that are applied identically to both architectures. The \emph{short} schedule trains the networks for $S_{\textup{out}} = 1{,}000$ outer iterations with $S_{\textup{in}} = 5$ inner iterations per outer step, using the Adam optimizer with $\beta_1 = 0.5$, $\beta_2 = 0.9$, initial learning rate $\lambda = 10^{-2}$, and a cosine schedule with final learning-rate ratio $10^{-2}$. The \emph{long} schedule trains the networks for $S_{\textup{out}} = 50{,}000$ outer iterations with $S_{\textup{in}} = 10$ inner iterations per outer step, using the Adam optimizer with the same momentum parameters and a constant learning rate $\lambda = 5\times 10^{-4}$. The long schedule is roughly $100$ times more expensive than the short one (per-treatment training times of approximately $40$s, $60$s, $2$h, and $3$h for the short-trained ICNN, short-trained HyCNN, long-trained ICNN, and long-trained HyCNN, respectively, on our hardware); for this reason the long schedule is run for a single random seed per treatment, whereas the short schedule is repeated for $50$ seeds. All remaining hyperparameters are kept identical across the two architectures to keep the comparison as controlled as possible.

\parabold{Evaluation and model selection} To monitor and select models during training, we evaluate each candidate on a dedicated validation partition drawn from the training set. The validation criterion is the Sinkhorn divergence $\mathcal{S}{\epsilon}$ with squared Euclidean cost and regularization $\epsilon = 0.1$ \citep{feydy2019interpolating,peyre2019computational} between the predicted pushforward $\hat T_{\#} \hat P_{\mathrm{val}}$ and the ground-truth validation target $\hat Q_{D,\mathrm{val}}$. %
Validation losses are recorded every $10$ outer iterations under the short schedule, yielding $100$ candidate checkpoints per run, and every $100$ outer iterations under the long schedule, yielding $500$ candidate checkpoints per run. For each run, we identify the $K = 10$ checkpoints with the smallest validation Sinkhorn divergence and report the mean test Sinkhorn divergence over the $20\%$ held-out test partition. Averaging over the ten best-validation checkpoints, rather than selecting the single best one, ensures that the reported value reflects an entire neighborhood of the validation optimum and confirms that the comparison is stable along the training trajectory rather than an artifact of a single lucky checkpoint.
This \emph{best-validation averaging} rule is applied identically to both architectures. Under the short schedule, we repeat each experiment over $50$ random seeds governing data splits and network initialization, and report cross-seed means together with standard errors. Under the long schedule, the same rule applies, but due to the substantially higher per-run computational cost, we conduct only a single run per treatment, so no standard error is reported.

\parabold{Results and discussion} The resulting test-set Sinkhorn divergences are collected in \Cref{tab:4i-bestval-SD-n50}, which reports both training schedules for each of the $35$ drug treatments. Each table entry is the test-set Sinkhorn divergence averaged over the $K = 10$ checkpoints with the smallest validation loss along the corresponding training run, and this best-validation averaging is applied identically to both architectures and both schedules. The left two columns show, for each treatment, the cross-seed mean of this checkpoint-averaged quantity (with standard error in parentheses) over $50$ random seeds under the short schedule. The right two columns show the same checkpoint-averaged quantity for the single random-seed run under the long schedule, for which no standard error is reported. The results of the first seed of the short schedule, together with the long-schedule run, are visualized as grouped bar charts in \Cref{fig:HyCNN_training_OT}\textbf{(b)}; see \Cref{fig:HyCNN_training_OT_large} for an enlarged version. In the figures, treatments are ordered by short-trained HyCNN test performance.

Three patterns emerge consistently from \Cref{tab:4i-bestval-SD-n50}. First, HyCNN attains a smaller test-set Sinkhorn divergence than ICNN consistently across treatments and schedules. Under the short schedule, HyCNN strictly improves over ICNN on all $35$ treatments, with (cross-seed) mean improvements that are typically within the range $25$-$40\%$ and outside one standard error. Under the long schedule, HyCNN remains the strictly better architecture on {\reb $31$ of the $35$ treatments, with the four exceptions (Dasatinib, Imatinib, Pomalidomide-carfilzomib-dexamethasone, and Regorafenib)} showing only marginal differences in favor of ICNN. Since each entry already averages over the ten best validation checkpoints, these observations reflect a consistent improvement between the two architectures along their respective training trajectories. Second, the long schedule lowers the test-set Sinkhorn divergence relative to the short schedule across nearly every treatment for both architectures, confirming that the additional optimization budget is genuinely informative rather than an artifact of the learning-rate schedule. Third, the short-trained HyCNN performs competitively with the long ICNN despite using roughly $100$ times less compute. It matches or improves upon the long ICNN on $14$ of the $35$ treatments and is often outperformed only by small margins on the remaining ones. Altogether, these findings indicate that, within the neural OT framework, replacing the ICNN potential by a HyCNN consistently improves the distributional fit, and that this improvement is not an artifact of an under-trained ICNN baseline. This is a useful insight for practitioners of neural OT in single-cell applications, as the resulting gain comes at only a slight additional per-iteration cost due to the two-gate architecture of HyCNN blocks.

\begin{figure}[h!]
\centering
\includegraphics[width=\textwidth]{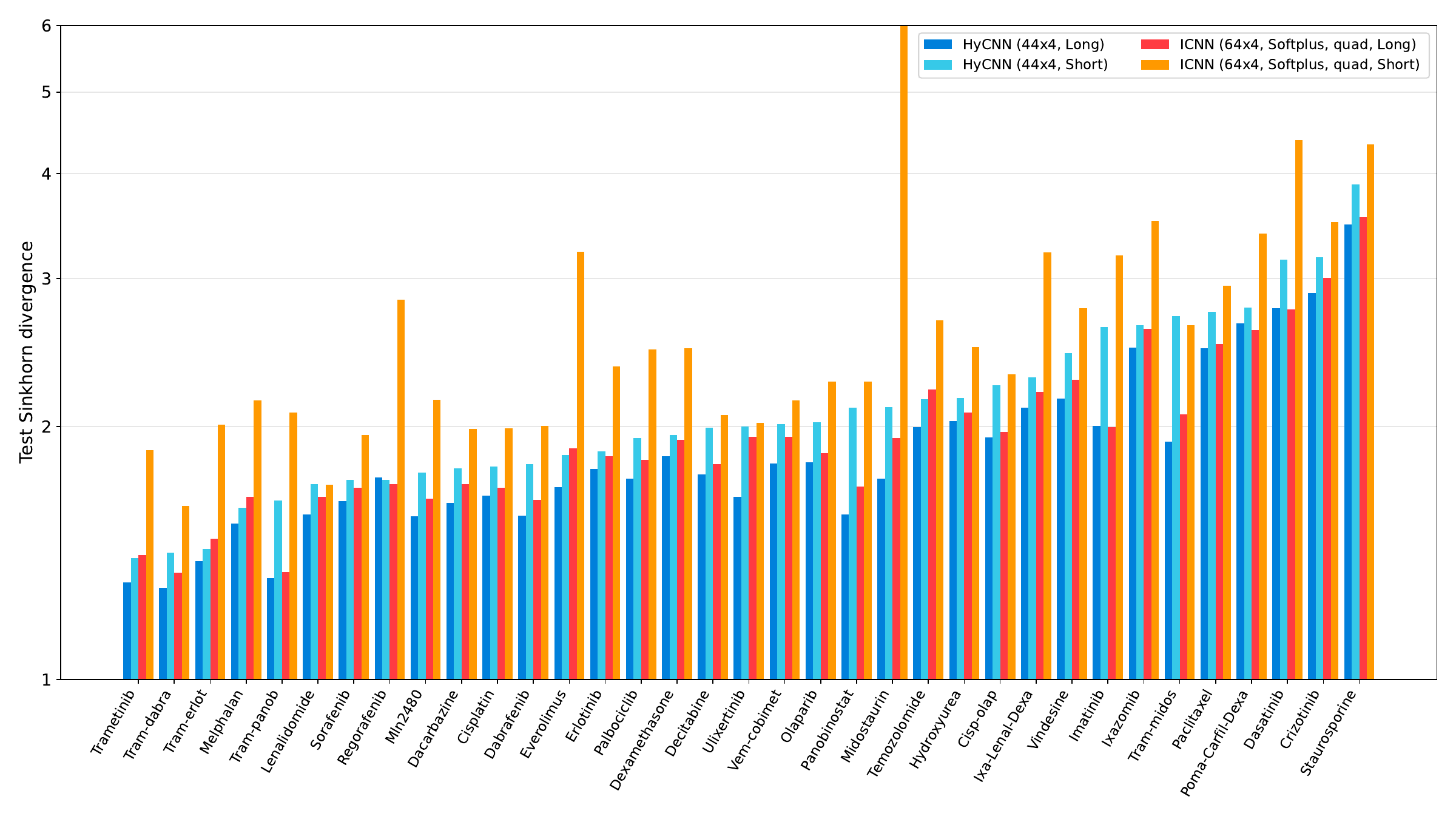}
\caption{Enlarged view of Figure \ref{fig:HyCNN_training_OT}\textbf{(b)} with drug labels on the $x$-axis. Sinkhorn divergence between predicted and ground-truth drug-perturbed distributions on a $20\%$ held-out test partition of the 4i dataset for each drug, based on ICNN ($64\times4$, Softplus, $\tau = 1$, quad) and HyCNN ($44\times 4$, log-sum-exp, $\tau = 1$), trained for $S_{\textup{out}} = 1{,}000$ outer iterations with $S_{\textup{in}} = 5$ inner iterations (short) and $S_{\textup{out}} = 50{,}000$ outer iterations with $S_{\textup{in}} = 10$ inner iterations (long). For each architecture and schedule, the displayed value is the test-set Sinkhorn divergence averaged over the $K = 10$ checkpoints with smallest validation loss along the corresponding training run. Per treatment, the distributional fit is computed on a fixed train/validation/test split, for which reason the values for the short runs slightly deviate from the cross-seed means reported in \Cref{tab:4i-bestval-SD-n50}. We use the following abbreviations: Cisp-olap for Cisplatin-olaparib, Ixa-Lenal-Dexa for Ixazomib-lenalidomide-dexamethasone, Poma-Carfil-Dexa for Pomalidomide-carfilzomib-dexamethasone, Tram-dabra for Trametinib-dabrafenib, Tram-erlot for Trametinib-erlotinib, Tram-midos for Trametinib-midostaurin, Tram-panob for Trametinib-panobinostat, and Vem-cobimet for Vemurafenib-cobimetinib.}
\label{fig:HyCNN_training_OT_large}
\end{figure}

\begin{table}[t!]
\centering
\caption{Sample sizes and test-set prediction errors across drug treatments for ICNN ($64\times4$) with softplus ($\tau = 1$) and a quadratic term in the first layer, and HyCNN ($44\times 4$) with log-sum-exp ($\tau = 1$). The two architectures are compared under two training budgets: the short run ($S_{\textup{out}} = 1{,}000$ outer iterations, $S_{\textup{in}} = 5$ inner steps, initial learning rate $\lambda = 0.01$, final learning-rate ratio $0.01$ under a cosine scheduler), and the long run ($S_{\textup{out}} = 50{,}000$ outer iterations, $S_{\textup{in}} = 10$ inner steps, constant learning rate $\lambda = 5\times 10^{-4}$). For each drug, we report the test-set Sinkhorn divergence ($\epsilon = 0.1$) on a $20\%$ test partition, averaged over the $K = 10$ checkpoints with smallest loss on a $10\%$ validation partition; validation losses are recorded every $10$ outer iterations under the short schedule and every $100$ outer iterations under the long schedule. The left two columns report the cross-seed mean of this checkpoint-averaged quantity (with standard error in parentheses) across $50$ random seeds; the right two columns report the same checkpoint-averaged quantity for the single long-schedule run per treatment. Bold indicates the smaller mean within each $(S_{\textup{out}}, S_{\textup{in}})$ pair.     Treatments are abbreviated as in Figure \ref{fig:HyCNN_training_OT_large}.}
\label{tab:4i-bestval-SD-n50}
\begin{tabular}{|c|c|l l|c c|}
\hline
\multirow{3}{*}{Treatment} & \multirow{3}{*}{Sample size} & \multicolumn{4}{|c|}{\textcolor{white}{$e^{e^{e^e}}$}Sinkhorn Divergence~($\epsilon = 0.1$)\textcolor{white}{$e^{e^{e^e}}$}} \\[0.05cm]
 & & \multicolumn{1}{|c}{ICNN(sp+q)} & \multicolumn{1}{c|}{HyCNN}& \multicolumn{1}{c}{ICNN(sp+q)} & \multicolumn{1}{c|}{HyCNN} \\[0.05cm]
 &&\multicolumn{2}{|c|}{$(S_{\textup{out}} = 1{,}000,\ S_{\textup{in}} = 5)$}  & \multicolumn{2}{c|}{$(S_{\textup{out}} = 50{,}000,\ S_{\textup{in}} = 10)$} \\[0.05cm]
\hline
\textcolor{white}{$e^{e^{e^e}}$}Cisplatin\textcolor{white}{$e^{e^{e^e}}$} & 2693 & 2.86~(0.63) & \textbf{1.75~(0.011)} & 1.69 & \textbf{1.66} \\
Cisp-olap & 2245 & 2.65~(0.13) & \textbf{2.04~(0.019)} & 1.97 & \textbf{1.94} \\
Crizotinib & 1978 & 5.35~(0.97) & \textbf{3.19~(0.041)} & 3.00 & \textbf{2.88} \\
Dabrafenib & 2675 & 2.57~(0.41) & \textbf{1.64~(0.012)} & 1.64 & \textbf{1.57} \\
Dacarbazine & 2485 & 2.49~(0.25) & \textbf{1.67~(0.013)} & 1.71 & \textbf{1.62} \\
Dasatinib & 2250 & 5.13~(0.55) & \textbf{3.19~(0.038)} & \textbf{2.75} & 2.76 \\
Decitabine & 2648 & 2.54~(0.16) & \textbf{1.92~(0.014)} & 1.80 & \textbf{1.75} \\
Dexamethasone & 2796 & 2.88~(0.29) & \textbf{1.95~(0.018)} & 1.93 & \textbf{1.84} \\
Erlotinib & 2406 & 3.32~(0.47) & \textbf{1.91~(0.014)} & 1.84 & \textbf{1.78} \\
Everolimus & 2592 & 3.23~(0.6) & \textbf{1.85~(0.011)} & 1.88 & \textbf{1.69} \\
Hydroxyurea & 2715 & 2.95~(0.35) & \textbf{1.98~(0.012)} & 2.08 & \textbf{2.03} \\
Imatinib & 2631 & 3.24~(0.3) & \textbf{2.32~(0.023)} & \textbf{2.00} & \textbf{2.00} \\
Ixazomib & 2514 & 3.72~(0.18) & \textbf{2.79~(0.019)} & 2.62 & \textbf{2.48} \\
Ixa-Lenal-Dexa & 2522 & 3.40~(0.32) & \textbf{2.16~(0.013)} & 2.20 & \textbf{2.10} \\
Lenalidomide & 2524 & 2.36~(0.26) & \textbf{1.56~(0.011)} & 1.65 & \textbf{1.57} \\
Melphalan & 2352 & 3.75~(1.3) & \textbf{1.71~(0.012)} & 1.65 & \textbf{1.53} \\
Midostaurin & 2425 & 3.24~(0.26) & \textbf{2.07~(0.018)} & 1.94 & \textbf{1.73} \\
Mln2480 & 2367 & 2.47~(0.3) & \textbf{1.55~(0.014)} & 1.64 & \textbf{1.56} \\
Olaparib & 2540 & 3.50~(0.54) & \textbf{2.01~(0.015)} & 1.86 & \textbf{1.81} \\
Paclitaxel & 2636 & 4.72~(1.2) & \textbf{2.76~(0.019)} & 2.51 & \textbf{2.48} \\
Palbociclib & 2322 & 3.01~(0.4) & \textbf{1.89~(0.015)} & 1.83 & \textbf{1.73} \\
Panobinostat & 2492 & 2.95~(0.34) & \textbf{1.90~(0.018)} & 1.70 & \textbf{1.57} \\
Poma-Carfil-Dexa & 2545 & 3.82~(0.23) & \textbf{2.70~(0.021)} & \textbf{2.61} & 2.65 \\
Regorafenib & 2524 & 4.91~(1.7) & \textbf{1.75~(0.013)} & \textbf{1.71} & 1.74 \\
Sorafenib & 2763 & 3.02~(0.53) & \textbf{1.66~(0.013)} & 1.69 & \textbf{1.63} \\
Staurosporine & 1310 & 13.94~(8.8) & \textbf{4.08~(0.047)} & 3.55 & \textbf{3.48} \\
Temozolomide & 2513 & 7.90~(5.5) & \textbf{1.92~(0.02)} & 2.21 & \textbf{2.00} \\
Trametinib & 2637 & 3.42~(1.1) & \textbf{1.37~(0.0092)} & 1.41 & \textbf{1.31} \\
Tram-dabra & 2958 & 2.23~(0.3) & \textbf{1.31~(0.011)} & 1.34 & \textbf{1.29} \\
Tram-erlot & 2452 & 2.07~(0.14) & \textbf{1.39~(0.01)} & 1.47 & \textbf{1.38} \\
Tram-midos & 2211 & 5.60~(1.3) & \textbf{2.59~(0.026)} & 2.07 & \textbf{1.92} \\
Tram-panob & 2377 & 2.52~(0.37) & \textbf{1.56~(0.016)} & 1.34 & \textbf{1.32} \\
Ulixertinib & 2858 & 2.27~(0.14) & \textbf{1.78~(0.014)} & 1.94 & \textbf{1.65} \\
Vem-cobimet & 2695 & 2.70~(0.17) & \textbf{1.93~(0.018)} & 1.94 & \textbf{1.81} \\
Vindesine & 2102 & 3.49~(0.41) & \textbf{2.65~(0.026)} & 2.27 & \textbf{2.16} \\
\hline\end{tabular}
\end{table}
\clearpage

\parabold{CellOT and future directions}
For context, the CellOT pipeline of \citet{bunne2023learning} trains the ICNN-based neural OT model for substantially more iterations than the $50{,}000$ outer steps used in our long schedule; for instance, \citet{bunne2023learning} use $S_{\textup{out}} = 250{,}000$ iterations with $S_{\textup{in}} = 10$ inner loops at a constant learning rate of $10^{-4}$. Our long schedule already addresses much of the concern that the ICNN baseline could simply be under-trained: increasing the ICNN budget from $1{,}000$ to $50{,}000$ outer iterations narrows its absolute error noticeably, yet HyCNN remains the better architecture on {\reb $31$} of the $35$ treatments under the long schedule, and the short-trained HyCNN remains broadly competitive with the long-trained ICNN despite using $100$ times less compute. Since both architectures continue to benefit from additional training while preserving this ordering, we expect a full $250{,}000$-iteration evaluation against the CellOT baseline to reinforce rather than overturn our qualitative picture reported here. A~comprehensive comparison with a larger budget is, therefore, a natural next step.

\end{document}